\documentclass[mnsc,nonblindrev]{informs3}

\OneAndAHalfSpacedXI 

\def\theARTICLETOP{}
\def\theLRHFirstLine{}
\def\theLRHSecondLine{}
\def\theRRHFirstLine{}
\def\theRRHSecondLine{}

\usepackage{natbib}
\bibpunct[, ]{(}{)}{,}{a}{}{,}%
\def\bibfont{\small}%

\usepackage{graphicx}
\usepackage{subfigure, epsfig}
\usepackage{natbib}

\bibpunct[, ]{(}{)}{,}{a}{}{,}%
\def\bibfont{\small}%

 \bibpunct[, ]{(}{)}{,}{a}{}{,}%
\def\bibfont{\small}%

\usepackage[utf8]{inputenc} 
\usepackage[T1]{fontenc}    
\usepackage{xr-hyper}
\usepackage[hidelinks,breaklinks]{hyperref}       

\usepackage{url}            
\usepackage{booktabs}       
\usepackage{amsfonts}       
\usepackage{nicefrac}       
\usepackage{microtype}      
\usepackage{xcolor}         
\usepackage{verbatim}       
\usepackage{subfigure}

\usepackage{amssymb}
\usepackage{float}
\usepackage{mathrsfs}
\usepackage{amsmath}
\usepackage{amsfonts}
\usepackage{enumitem}
\usepackage{tikz}
\usetikzlibrary{shapes.geometric, arrows}

\usepackage{accents}
\usetikzlibrary{arrows.meta}
\usepackage{color}
\usepackage{cleveref}
\crefname{section}{§}{§§}
\Crefname{section}{§}{§§}
\newcommand{\citeg}[1]{\citep[see, e.g.,][]{#1}}
\usepackage[ruled,linesnumbered]{algorithm2e}
\usepackage{multirow}
\usepackage{array} 

\usepackage{threeparttable}
\usepackage{rotating}

\usepackage{diagbox}
\usepackage{lmodern}
\usepackage{anyfontsize}

\usepackage{pifont}  

\def\bx{\boldsymbol x}
\def\by{\boldsymbol y}

\def\bv{\boldsymbol \nu}

\def\bp{\boldsymbol p}

\def\beps{\boldsymbol \epsilon}
\def\blam{\boldsymbol \lambda}

\def\blambda{\boldsymbol \lambda}
\def\bdelta{\boldsymbol \delta}
\def\eps{\epsilon}

\def\cN{\mathcal{N}}
\def\cM{\mathcal{M}}
\def\cS{\mathcal{S}}

\def\cI{\mathcal{I}}

\def\prum{\mathcal{P}_{\textsc{rum}}}
\def\pmdm{\mathcal{P}_{\textsc{mdm}}}

\def\preg{\mathcal{P}_{\textsc{reg}}}
\def\pmnl{\mathcal{P}_{\textsc{mnl}}}
\def\pnl{\mathcal{P}_{\textsc{nl}}}

\def\R{\mathbb{R}}

\TheoremsNumberedThrough     
\ECRepeatTheorems

\EquationsNumberedThrough    

\MANUSCRIPTNO{MS-RMA-2023-02287.R2}

\begin{document}



\RUNTITLE{A Nonparametric Approach with Marginals for Choice}

\TITLE{A Nonparametric Approach with Marginals for Modeling Consumer Choice}

\ARTICLEAUTHORS{%
\AUTHOR{Yanqiu Ruan}
\AFF{Singapore University of Technology and Design, Singapore, \EMAIL{yanqiu\_ruan@mymail.sutd.edu.sg}}
\AUTHOR{Xiaobo Li}
\AFF{National University of Singapore, Singapore, \EMAIL{iselix@nus.edu.sg}}
\AUTHOR{Karthyek Murthy}
\AFF{Singapore University of Technology and Design, Singapore, \EMAIL{karthyek\_murthy@sutd.edu.sg}}
\AUTHOR{Karthik Natarajan}
\AFF{Singapore University of Technology and Design, Singapore, \EMAIL{karthik\_natarajan@sutd.edu.sg}}

} 

\ABSTRACT{Given data on the choices made by consumers for different offer sets, a key challenge is to develop parsimonious models that describe and predict consumer choice behavior while being amenable to prescriptive tasks such as pricing and assortment optimization. The marginal distribution model (MDM) is one such model, which requires only the specification of marginal distributions of the random utilities. 
This paper aims to establish necessary and sufficient conditions for given choice data to be consistent with the MDM hypothesis, inspired by the usefulness of 
similar characterizations 
for the random utility model (RUM). This endeavor leads to an exact characterization of the set of choice probabilities that the MDM can represent. Verifying the consistency of choice data with this characterization is equivalent to solving a polynomial-sized linear program. Since the analogous verification task for RUM is computationally intractable and neither of these models subsumes the other, MDM is helpful in striking a balance between tractability and representational power. The characterization is then used with robust optimization for making data-driven sales and revenue predictions for new unseen assortments. When the choice data lacks consistency with the MDM hypothesis, finding the best-fitting MDM choice probabilities reduces to solving a mixed integer convex program. 
\textcolor{black}{Numerical results using real world data and synthetic data demonstrate that MDM exhibits competitive representational power and prediction performance compared to RUM and parametric models while being significantly faster in computation than RUM. }
}


\HISTORY{}

\maketitle

%


\section{Introduction}

Discrete choice models have been used extensively in economics \citep{allenby1995using}, marketing \citep{mcfadden1986choice}, healthcare \citep{de2018impact}, transportation \citep{akiva}, and operations management \citep{talluri2004revenue}. Such models describe the observable distribution of demand from the behavior of one or more consumers who choose their most preferred alternative from a discrete collection of alternatives. 

As choice models specify the conditional probability distribution over any offer set, they are inherently high dimensional. Given data on the choices made by consumers over a limited collection of offer sets (also referred as assortments), the specification of a choice model hypothesis is essential in linking data from the observed offer sets to predictions for new offer sets for which no data is available. The classical multinomial logit choice model (MNL) derived by \cite{luce1959individual} and \cite{plackettMNL} is among the simplest and most widely used choice model. It stipulates that the ratio of choice probabilities for any two alternatives $i$ and $j$ does not depend on any alternatives other than $i$ and $j.$ 
A popular model at the expressive end of the spectrum is the random utility model (RUM)  which hypothesizes that the utilities of the alternatives are random variables and the consumers are utility maximizers. In settings with finite alternatives, MNL is subsumed by RUM and a generic RUM is describable by a distribution over the rankings (or preference lists) of the alternatives \citep{mas1995microeconomic}. Such a description of RUM over $n$ alternatives requires about  $n!$ parameters, and even the task of verifying whether given choice data is consistent with the RUM hypothesis is computationally intractable  (see \citealt{jagabathula2019limit}, \citealt{pmlr-v162-almanza22a}).



There has been a recent surge of interest in developing choice models with good representational power using machine learning techniques. Examples of such models include those proposed by \cite{wang,sifringer2020enhancing} and \cite{aouad2022representing}, who utilize neural networks to fit expressive utilities within the context of MNL and RUM hypotheses. Additionally, the decision forest choice model \citep{chen2022decision,chen2019use} has been shown to be capable of approximating any choice data with increasing forest depth. The expressiveness of these models however comes at the cost of requiring significant amounts of data and computation to learn. Furthermore, it has been observed that the reliability of economic information obtained from deep neural network based models is compromised when the data size is small (see \citealt{wang}). Therefore, a natural question is to examine the representational power of other choice models and to identify choice model hypotheses that offer a balance between richer representational power and tractability while allowing for robust procedures for estimation and prediction from limited data.

\subsection{The choice model and the research questions}
\label{sec:intro-model-questions}
An alternative to RUM in offering a substantial generalization to MNL is the marginal distribution model (MDM) proposed by  \citet{natarajan2009persistency}. It subsumes MNL  (see \citealt{mishra2014theoretical}) and the well-known additive perturbed utility (APU) model treated in \citet{fudenberg2015stochastic}. Specifying an MDM choice model requires only the specification of the marginal distributions of the random utilities of the alternatives. Then the MDM choice probabilities are computed with the extremal distribution maximizing the expected consumer utility over all joint distributions with the given collection of marginals. A precise description of  MDM is provided in Section \ref{sec:literature review}. A key advantage of this model is that it allows choice probabilities to be readily computed from tractable convex optimization formulations.  Besides tractability, MDM has been shown to exhibit good empirical performance in various applications using real-world datasets (see \citealt{natarajan2009persistency,mishra2014theoretical,ahipacsaouglu2019distributionally,sun2020unified,zhenzhen,liu2022pricing}). 
More recently, price optimization has been shown to be computationally tractable with MDM (see \citealt{zhenzhen}) and a half approximation guarantee has been developed for profit-nested heuristic in assortment optimization (see \citealt{sun2020unified}). The formulation of MDM has also become useful in deriving prophet inequalities for Bayesian online selection problems (see \citealt{feldman2021online}) and solving smoothed optimal transport formulations (see \citealt{tacskesen2022semi}). 

Although a general specification of MDM does not impose restrictions on the marginal distributions of the random utilities, the estimation of MDM from data in practice typically requires first committing to an appropriate parametric family for the marginal distributions of the utilities (an exception in  \citealt{zhenzhen} which uses piece-wise linear marginal cumulative distribution functions). 
 Upon fixing suitable parametric families for the marginal distributions, the respective parameters are estimated from data using a procedure like maximum-likelihood and the corresponding choice probability predictions are made using convex optimization (see, e.g., \citealt{mishra2014theoretical}). 
Fixing the ``right'' parametric families can however be a tricky exercise and is prone to suffering from underfitting and overfitting issues. Just as how parametric restrictions to RUM are deemed to be restrictive (see, e.g.,  \citealt{farias2013nonparametric}),  a workflow requiring prior commitment to fixed parametric families of distributions does not allow one to leverage the full modeling power offered by MDM to extract as much structural information as possible from the data.

While one might seek to conduct data-driven estimation and prediction under the MDM hypothesis, the challenge remains that it is currently unclear how to do so without imposing parametric assumptions on the marginal distributions of utilities. {\color{black} Specifically, given data from different offer sets, is there a tractable method to verify whether the data can be represented by MDM? If the data cannot be represented by MDM, how close is it to being MDM-representable? Moreover, can we leverage this representability for prediction and prescription tasks?}



More formally, suppose that $\mathcal{N} = \{1,\ldots,n\}$ is the universe of products, and we have choice data for a collection $\mathcal{S}$ of subsets of $\mathcal{N}.$ For each subset $S \in \mathcal{S},$ let $p_{i,S} \in [0,1]$ denote the fraction of customers who purchased product $i$ when the assortment $S$ was offered. {\color{black} Our objective is to determine whether, and to what extent, the choice data, denoted as $\bp_\cS = (p_{i,S} : i \in S, S \in \cS)$, deviates from an MDM instance.} Using a tractable characterization, we aim to utilize it to develop data-driven procedures that can leverage the full modeling power of the MDM hypothesis to make revenue and sales predictions, without restricting one to make parametric distributional assumptions.

\subsection{Contributions}
\label{sec:intro-contributions}
An effort towards addressing these goals leads us to the following contributions in this paper. 
\subsubsection*{An exact characterization of the choice probabilities represented by MDM and its tractability.}
\label{sec:intro-MDMchar}
We show that the choice data $\bp_\cS = (p_{i,S}: i \in S, S \in \cS)$ given for a collection of assortments $\mathcal{S}$ is representable by MDM if and only if there exists a utility function $U: \mathcal{S} \rightarrow \mathbb{R}$ representing the preferences expressed across assortments in the choice data; in particular, the utility $U$ should exhibit a strict preference for an assortment $T$ over another assortment $S$ containing a common product $i$ if $p_{i,S} < p_{i,T},$ and exhibit indifference between $S$ and $T$ if $p_{i,S} = p_{i,T} \neq 0$ (see  Theorem \ref{thm:feascon-mdm} for a precise statement). The existence of a utility function implies a rational preference relation (or a ranking) over assortments, and it allows us to make the following deductions regarding the tractability and representational power of MDM: 
\begin{itemize}[leftmargin=1.25em]
\item 
The characterization in Theorem \ref{thm:feascon-mdm}, in terms of the existence of a ranking over assortments, lends itself to be verified with a linear program whose size is polynomial in the number of products and assortments. This is in contrast to RUM which requires the existence of a distribution over the $n!$ rankings possible for $n$ products. Thus, unlike RUM, verifying the consistency of given choice data with the MDM hypothesis can be accomplished in polynomial time.

\item 
The collection of MDM representable choice probabilities possesses a non-zero measure when considered relative to the collection of all possible choice probabilities, unlike models with a fixed number of parameters, such as MNL or nested logit, have zero measure (see Theorem \ref{thm:mdm_positive_measure}). Additionally, the characterization in Theorem 1 reveals that MDM and RUM do not subsume each other in terms of the choice probabilities they can represent (see Theorem \ref{thm:mdm-rum-relation}). 
\end{itemize}

\subsubsection*{A nonparametric data-driven approach to prediction and estimation.} 
\label{sec:intro-app}
As the sales data available in practice is often inadequate to entirely specify the probabilistic behavior of the utilities, we utilize the nonparametric MDM characterization in Theorem 1 together with robust optimization as the basis for making sales and revenue predictions for any new assortment with no prior sales data.  Specifically, we develop a data-driven approach which builds upon the exact MDM characterization to produce worst-case estimates of sales and revenues computed over all MDM instances that are consistent with the given choice data. 
This robust approach mitigates the risk of misspecification and enables end-to-end learning for MDM. The characterization can also be used to develop optimistic best-case estimates, which, along with the worst-case estimates, yield prediction intervals for sales and revenues over all MDM instances consistent with the given choice data. The procedure yields narrower intervals for sales and revenue predictions when data for more assortments become available, as is desirable for any data-driven method.

When the choice data is not fully consistent with the MDM hypothesis, we develop a ``limit of MDM'' formulation to quantify the degree of inconsistency. Inspired by the  ``limits of rationality'' measure proposed in \cite{jagabathula2019limit}, we define the limit of MDM as the smallest loss that can be obtained by fitting MDM to the given choice data. A model which attains the minimum loss can be interpreted as offering the best fit, within the MDM family, to any given choice data. This best fitting MDM can be subsequently used to produce robust revenue and sales predictions as described above. Utilizing the exact MDM characterization in Theorem \ref{thm:feascon-mdm}, we reduce the computation of the limit of MDM to the rank aggregation problem \citep{dwork2001rank} and show that it is NP-hard. 
We develop a mixed integer convex program that is applicable generally for computing the limit (see Proposition \ref{prop:micp}). We also propose algorithms whose running time are polynomial in both the number of alternatives and the size of the assortment collection if the assortment collection possesses suitable structure (see Corollary \ref{cor:limit_structure}).

\subsubsection*{Numerical insights.}
\label{sec:intro-expt}
{\color{black} Numerical experiments conducted on real data from JD.com \citep{shen2020jd} and synthetic data reveal the following insights on the performance of the proposed nonparametric MDM model, in relation to that of RUM and common parametric models such as MNL, Markov chain choice model (MCCM), and Latent Class-MNL (LC-MNL): First, the numerical results reveal that MDM offers almost as good a fit as the expressive RUM class 
while being computationally much faster. In addition, MDM offers a significantly better fit than popular parametric models such as MNL, LC-MNL, and MCCM, with MCCM offering competitive performance in some instances. Next, when measuring the predictive abilities of the models by means of (i) ranking of unseen assortments in terms of the respective revenues and (ii) the revenue of the optimal assortment identified by each model from a given collection of unseen assortments, we find that MDM outperforms MNL and robust nonparametric approach to RUM, and it performs comparably well as MCCM and LC-MNL. Together, these observations underscore MDM's utility in offering a good balance between representation power and computational effort, while allowing for robust procedures for estimation and prediction which perform well in the absence of data about unseen assortments.}   

{\color{black} Figure \ref{fig:flowchat} summarizes how the above applications of our exact characterization of the MDM model (Theorem \ref{thm:feascon-mdm}) could combine  to constitute a practically useful workflow in which optimization serves as the basis for computing choice probabilities and revenues for previously unobserved assortments.}
{\color{black}
\begin{figure}[h!] 
\tikzstyle{startstop} = [rectangle, rounded corners, 
minimum width=3cm, 
minimum height=1cm,
text centered, 
draw=black, 
fill=white!30]

\tikzstyle{io} = [trapezium, 
trapezium stretches=true, 
trapezium left angle=70, 
trapezium right angle=110, 
minimum width=4.5cm, 
minimum height=1cm, 
text width=4.5cm, 
draw=black, fill=white!30]

\tikzstyle{process} = [rectangle, 
minimum width=5cm, 
minimum height=1cm, 
text width=5cm, 
draw=black, 
fill=white!30]

\tikzstyle{process2} = [rectangle, 
minimum width=4cm, 
minimum height=1cm, 
text width=4cm, 
draw=black, 
fill=white!30]

\tikzstyle{decision} = [diamond, 
minimum width=3cm, 
minimum height=1cm, 
text width=3cm, 
text centered, 
draw=black, 
fill=white!30]

\tikzstyle{arrow} = [thick,->,>=stealth]

\tikzset{global scale/.style={
    scale=#1,
    every node/.append style={scale=#1}
  }
}

\begin{center}
\scalebox{0.95}{\begin{tikzpicture}[node distance=2cm]
\node (dec1)[decision]{Is $\bp_\cS$  represented by MDM? };
\node (pro1)[process, right of = dec1, xshift = 3.5cm, yshift = 1.2cm ] {Set MDM choice probabilities $\bx_\cS^* = \bp_\cS$}; 
\node (pro2) [process, right of = dec1, xshift = 3.5cm, yshift = -1.2 cm ]{Find the best fitting MDM choice probabilities $\bx_\cS^*$ to $\bp_\cS$};
\node (pro3) [process, right of = dec1, xshift = 9.5 cm ]{Nonparametric revenue prediction for new assortment $A$ with $\bx_\cS^*$};
\draw (2.1,0) -- (2.5,0);
\draw (2.5,0) -- (2.5,-1);
\draw (2.5,0) -- (2.5,1);
\draw [arrow] (2.5,-1) |- node[anchor=east] {NO} (pro2);
\draw [arrow] (2.5,1) |- node[anchor=east] {YES} (pro1);
\draw (8.15,1) -- (8.5,1);
\draw (8.15,-1) -- (8.5,-1);
\draw (8.5,1) -- (8.5,0);
\draw (8.5,-1) -- (8.5,0);
\draw [arrow] (8.5,0) -- (pro3);
\node[text width=4cm] at (0, - 2.8) {Polynomial-sized LP};
\node[text width=4cm] at (0.8, - 3.2) {(Theorem  \ref{thm:feascon-mdm}) };
\node[text width=4cm] at (6.2, - 2.2)  {Compact MICP};
\node[text width=4cm] at (6.3, - 2.6) {(Proposition  \ref{prop:micp}) };
\node [text width=4cm] at (12.2,  -1)  {Compact MILP};
\node[text width=4cm] at (12.3, - 1.4) {(Proposition \ref{prop:wc-rev-milp})};
\end{tikzpicture}}
\end{center}
\caption{A summary of main contributions of the paper together with a workflow}
\label{fig:flowchat}
\end{figure}}The rest of this paper is organized as follows. We begin with a brief review of related literature and provide a precise description of MDM in Section \ref{sec:literature review}. We derive the exact characterization of MDM representable choice probabilities in Section \ref{sec:mdm-characterization} and discuss its implications for tractability and representational power. In Sections \ref{sec:prediction-mdm} and \ref{sec:lom}, we develop data-driven methods for prediction and estimation as applications of the characterization. 
We conclude with a discussion after presenting the results of numerical experiments in Section \ref{sec:experiment} and \ref{sec:exp_limitation}. 
Proofs, illustrative examples, results of additional experiments, and more information on the experiments are in the electronic companion (EC).

\section{Related literature and a description of MDM}
\label{sec:literature review}
We begin with a concise overview of studies that aim to characterize and relate choice probabilities obtainable under prominent choice model hypotheses. Additionally, we discuss related results on their tractability and methods for estimation when equipped with sales data.

\subsection{On the 
characterizations available for choice models}
\label{sec:lit-rev-char}
RUM is perhaps the most popular class of models in choice modeling. 
RUM assumes that the utility of each alternative $i$ in the collection of products $\mathcal{N} = \{1,\ldots,n\}$ takes the form $\Tilde{u}_i = \nu_i+\Tilde{\eps}_i$, where $\bv=(\nu_1,...,\nu_n)$ and $\Tilde{\beps}=(\Tilde{\eps}_1,...,\Tilde{\eps}_n)$ denote the deterministic and stochastic parts of the utilities, respectively. Assuming a joint distribution $\theta$ on the random part $\Tilde{\beps}$, the probability of choosing product $i$ in an assortment $S \subseteq \cN$ is given by $p_{i,S} = \mathbb{P}_{\Tilde{\beps} \sim \theta} (i = \arg\max_{j\in S} \{\nu_j + \tilde{\epsilon}_j\}),
$ where the probability of ties is assumed to be $0$. Here note that we do not explicitly model the outside option if available; instead, we treat the outside option as one of the products in $\cN$. 
Modeling the joint probability distribution of the random utilities with specific parametric distribution families leads to parametric subclasses of RUM such as  MNL \citeg{McFa73},  generalized extreme value model \citeg{mcfadden1977modelling},   nested logit model \citeg{mcfadden1980econometric}, multinomial Probit model \citeg{thurstone1927law, daganzo2014multinomial}, mixed logit model \citeg{mcfadden2000mixed}, and the exponomial choice model \citeg{alptekinouglu2016exponomial}.  
Nonparametric choice models, such as the rank list model \citeg{block1959random,farias2013nonparametric}, have also proven to be useful in practice. 
The Markov chain choice model proposed by \cite{blanchet2016markov} is a special case of the rank-list model \citep{berbeglia2016discrete}. 

Beginning with \cite{Marschak1960} and \cite{block1959random}, considerable effort has been devoted in econometrics towards understanding the restrictions imposed on choice data by the RUM hypothesis. The class of RUM and the class of rank list models are shown to be equivalent in \cite{block1959random}. \cite{falmagne1978representation} has shown that a RUM can represent a system of choice probabilities over all possible assortments if and only if the Block-Marschak conditions are met, see also \cite{barbera1986falmagne}.  \cite{mcfadden1990stochastic} has shown that under certain conditions, the axiom of revealed stochastic preference provides necessary and sufficient conditions for the choice probabilities under different assortments that can be recreated by a RUM. \citet{mcfadden2000mixed} has demonstrated that any RUM can be approximated closely by a mixed logit model. \cite{mcfadden2006revealed} adds more conditions that relate to the findings in \cite{falmagne1978representation} and \cite{mcfadden1990stochastic}. However, verifying these conditions is   computationally intractable when there are a large number of products. 
\cite{jagabathula2019limit} has been the first to relate the hardness of the stochastically  rationalizable property stipulated by RUM to the notion of choice depth.
They provide examples of structured assortment collections for which verification of consistency with RUM is computationally tractable.


Using a representative agent model (RAM) constitutes another popular optimization based approach to model choice. In RAM, a single agent makes a choice on behalf of the entire population. To make her choice, the agent takes into account the expected utility while preferring some degree of diversification. More precisely, given an assortment $S$, the representative agent solves 
\begin{align} \label{e1}
    \max \left\{\bv^T\bx - C(\bx) \ \left| \ \bx \in \Delta_{n-1}, \ x_i=0 \ \forall i \notin S \right.\right\},
\end{align}
where $\Delta_{n-1} = \{ \bx \in \mathbb{R}^n_+  | \sum_{i \in \cN}x_i = 1 \}$ is the unit simplex and 
$C(\bx):\Delta_{n-1}\mapsto\mathbb{R}$ is a convex perturbation function that rewards diversification. 
The optimal $x_i$ value provides the fraction of the population that chooses alternative $i$ in assortment $S$. \cite{hofbauer} has shown that all RUM can be expressed using a representative agent model under appropriate conditions on the perturbation functions $C(\bx)$ (see also \citealt{feng65technical}). 

The APU model in  \citet{fudenberg2015stochastic} can be obtained as a special case of RAM in \eqref{e1} by taking the additive and separable perturbation $C(\bx)=\sum_i c(x_i),$ where $c(x):[0,1]\mapsto\mathbb{R}$ is a strictly convex function. \cite{fudenberg2015stochastic}  demonstrates acyclicity and ordinal IIA property, which is a relaxation of Luce's IIA condition, as two alternative conditions that characterize the richness of APU representable choice probabilities. Equipped with these results, \cite{fudenberg2015stochastic} argues for APU as a considerably simpler and expressive model alternative which helps go beyond RUM while requiring only the specification of the univariate convex perturbation function $c(\cdot).$



\subsection{The marginal distribution model and related literature}
\label{sec:lit-rev-mdm}
The marginal distribution model (MDM)  is a semiparametric choice model that yields choice probabilities from limited information on the joint distribution of the random utilities. As in RUM, the starting point of MDM is that the utility of each alternative $i$ in the collection of products $\mathcal{N} = \{1,\ldots,n\}$ takes the form $\Tilde{u}_i = \nu_i+\Tilde{\eps}_i$, where $\bv=(\nu_1,...,\nu_n)$ and $\Tilde{\beps}=(\Tilde{\eps}_1,...,\Tilde{\eps}_n)$ denote the deterministic and stochastic parts of the utilities. MDM only requires  the specification of the marginal distributions $F_1,\ldots,F_n$ of the random variables $\tilde{\epsilon}_1,\ldots,\tilde{\epsilon}_n,$ and does not impose any independence assumption among $\tilde{\epsilon}_1,\ldots,\tilde{\epsilon}_n$. To describe the model, let $\Theta$  denote the collection of joint distributions for $\Tilde{\beps}$ with the given marginal distributions $F_1,\ldots,F_n.$ For any  assortment $S \subseteq \mathcal{N},$ the MDM considers maximization of expected consumer utility over all distributions in $\Theta$:  
\begin{align}
    \sup_{\theta \in \Theta} \ \mathbb{E}_{\Tilde{\beps} \sim \theta} \left[ \max_{i\in S} \{\nu_i + \tilde{\epsilon}_i\} \right].\label{mdm}
\end{align}
The probability of choosing a product $i$ from the assortment $S,$   given by $p_{i,S} = \mathbb{P}_{\Tilde{\beps} \sim \theta^*} (i = \arg\max_{j\in S} \{\nu_j + \tilde{\epsilon}_j\}),$ is evaluated with the distribution $\theta^*$ which attains the maximum in (\ref{mdm}). A key advantage of this model is that the choice probabilities are readily computable via convex optimization, as described in Lemma \ref{optcon} below.
\begin{assumption}
Each random term $\tilde{\epsilon}_i,$ where $i\in \cN,$ is an absolutely continuous random variable with a strictly increasing marginal distribution $F_i(\cdot)$ on its support and $\mathbb{E}|\tilde{\epsilon}_i| < \infty$.\label{asp:general}
\end{assumption}
\begin{lemma}
\label{optcon}
\citep{natarajan2009persistency,mishra2014theoretical,chen2022distributionally} Under Assumption \ref{asp:general}, the choice probabilities for a distribution which attains the maximum in (\ref{mdm}) is unique and is given by the optimal solution of the following strictly concave maximization problem over the simplex:
\begin{align}
\max \left\{ \sum_{i \in S}\nu_ix_i + \sum_{i \in S} \int_{1-x_i}^1 F_i^{-1}(t)\,dt \ \  \left| \ \ \sum_{i \in S}x_i = 1, x_i \geq 0 \ \forall i \in S \right.\right\}, \label{cvmdm}
\end{align}
with the convention that $F_i^{-1}(0)= \lim_{t\downarrow 0} F_i^{-1}(t)$ and $F_i^{-1}(1)= \lim_{t\uparrow 1} F_i^{-1}(t)$.
\end{lemma}

The reformulation of MDM in (\ref{cvmdm}) shows that it is a special case of the representative agent model (\ref{e1}), in which the perturbation function is strictly convex and separable of the form $C(\bx) = -\sum_{i} \int_{1-x_i}^1 F_i^{-1}(t)\,dt$. If the marginal distributions are identical, MDM reduces to the additive perturbed utility (APU) model described in Section \ref{sec:lit-rev-char}. Thus by subsuming the APU model, MDM  provides a probabilistic utility interpretation for APU. Given any assortment $S \subseteq \cN,$ a vector $(x_i^\ast: i \in S) \in \mathbb{R}^{\vert S \vert}$ maximizes \eqref{cvmdm} and  yields the MDM choice probabilities $p_{i,S} = x^\ast_{i}$ if and only if it satisfies the following optimality conditions for \eqref{cvmdm}: 
\begin{align}
\label{optimality}
    \nu_i+F_i^{-1}(1-x_i^*)-\lambda+\lambda_i = 0,& \quad  \forall i \in {S}, \nonumber\\
    \lambda_i x_i^*=0, &\quad \forall i \in {S},\\
    \sum_{i \in S} x_i^* = 1,&\nonumber\\
    x_i^* \geq 0, \lambda_i \geq 0,& \quad \forall i \in {S}, \nonumber
\end{align}
where $\lambda$ and $\lambda_i$ are the Lagrange multipliers associated with the constraints defining the simplex. An additional assumption that $F_i^{-1}(1) = +\infty$ is often made \citeg{mishra2014theoretical} to guarantee strictly positive choice probabilities. However, in real datasets, one or more alternatives offered in an assortment might  never be chosen by consumers. In this paper, we allow for this possibility by permitting the right end point of the support, $F_i^{-1}(1),$ to be possibly finite or infinite.



\subsection{Related estimation approaches}
Maximum likelihood estimation is the widely used method for parameter estimation when working with parametric subclasses of RUM and MDM. 
One may refer \cite{manski1981alternative,manski1977estimation, cosslett1981maximum} and \cite{mishra2014theoretical} for illustrations on how one may use maximum likelihood to estimate parametric models including the MNL model, the multinomial Probit model, and the generalized extreme value model from choice data, after making specific distributional assumptions under the RUM and MDM paradigms.    \cite{farias2013nonparametric} offers a major departure by developing a nonparametric approach that makes revenue predictions under the RUM hypothesis without imposing restrictive parametric distributional assumptions. In particular, \cite{farias2013nonparametric} showcases the efficacy of using the worst-case expected revenue over the collection of RUM models that are consistent with the available sales data. \cite{sturt2024value} presents an account of when these robust RUM revenue predictions can serve as a suitable basis for assortment optimization. Given an opportunity to collect data by performing pricing experiments, \cite{zhenzhen} and \cite{liu2022pricing} use a nonparametric approach to identify the MDM objective in \eqref{cvmdm}, upto a constant shift, by performing sufficient pricing experiments. Our approach for revenue predictions, which is described in Section \ref{sec:intro-app}, follows the same philosophy as \cite{farias2013nonparametric}, but with the novelty of using the MDM characterization to make predictions that are consistent with the MDM hypothesis. {\color{black} Our work also relates to the nonparametric estimation literature in economics \citeg{berry2014identification,compiani2022market}. This body of research examines the identifiability and asymptotic properties of nonparametric estimation methods, such as the method of sieves. Through asymptotic analysis, they provide confidence intervals for predictions. Our approach in this paper differs by remaining within the MDM framework, which allows us to preserve its parsimony while maintaining the nonparametric flexibility of marginal distributions. \textcolor{black}{Unlike the fully nonparametric methods, remaining within the MDM framework might impose restrictions on the correlation structure of the utilities. We empirically illustrate this phenomenon and its implications in Section \ref{sec:exp_correlation}.} Furthermore, our prediction interval is based on the best- and worst-case scenarios within the MDM class, instead of being based on asymptotic analysis.} 

\section{An exact characterization for MDM and its implications}
\label{sec:mdm-characterization}
In this section, we first develop necessary and sufficient conditions for the choice probabilities given by a collection of assortments that are representable by MDM. We follow this up with a discussion of its implications for tractability and representational power.

\subsection{A tractable characterization for MDM}
We begin by recalling that $\cN = \{1,\ldots,n\}$ denotes the universe of the products (or alternatives) and $\cS$ denotes the collection of subsets of $\mathcal{N}$ for which choice data is available.  Each $S \in \cS$ is an assortment of the products presented to the consumers. For each assortment $S \in \cS$, let $p_{i,S}$ be the fraction of population who choose product $i \in S.$ For any assortment collection $\cS,$
let $\cI_\cS$ denote the collection of all product-assortment pairs $(i,S)$ with $i\in S$  and $S \in \cS.$
Then the observed choice probability collection $\bp_\cS = (p_{i,S}: i \in S, S \in \cS)$ is a non-negative vector in $\mathbb{R}^{\vert \cI_\cS \vert}$ and satisfies $\sum_{i \in S} p_{i,S} = 1$ for every $S \in \cS.$ We are interested in identifying necessary and sufficient conditions on the observable sales data  $\bp_\cS$ which make it consistent with the MDM hypothesis. A natural follow-up question is: Can these conditions be verified in polynomial time? The following theorem provides an affirmative answer to these questions.
\begin{theorem}[A tractable characterization for MDM]
\label{thm:feascon-mdm}
    Under Assumption \ref{asp:general}, a choice probability collection $\bp_\cS$ is representable by an MDM  if and only if there exists a function $\lambda: \mathcal{S} \rightarrow \mathbb{R}$
    such that for any two assortments $S,T \in \cS$ containing a common product $i \in \cN,$ 
\begin{align}
\begin{aligned}
     &\lambda(S)  > \lambda(T) \quad \text{if} \quad p_{i,S} < p_{i,T},\\
     & \lambda(S)  = \lambda(T)  \quad \text{if}\quad  p_{i,S} = p_{i,T} \neq 0.   \label{eq:mdm-feascon}
\end{aligned}
\end{align}
As a result, checking whether the given choice data $\bp_\cS$ satisfies the MDM hypothesis can be accomplished by solving a  linear program 
with $\mathcal{O}(\vert \cS \vert)$ continuous variables and $\mathcal{O}(n \vert \cS \vert) $ constraints.  
\end{theorem}

A proof for Theorem \ref{thm:feascon-mdm} is given immediately following this discussion. Letting $U(S) = -\lambda(S)$ for $S \in \mathcal{S},$ one may understand the characterization in Theorem \ref{thm:feascon-mdm} as follows: The choice data $\bp_\cS = (p_{i,S}: i \in S, S \in \cS)$ given for a collection of assortments $\mathcal{S}$ is representable by MDM if and only if there exists a utility function $U: \mathcal{S} \rightarrow \mathbb{R}$ satisfying the following two conditions: (i) By assigning $U(S) < U(T),$ the utility $U$ should exhibit a strict preference for an assortment $T$ over another assortment $S$ containing a common product $i$ whenever $p_{i,S} < p_{i,T};$ and (ii) by assigning $U(S) = U(T),$ it should exhibit indifference between $S$ and $T$ whenever $p_{i,S} = p_{i,T} \neq 0.$ 

Since utility functions represent rational preference relationships (or rankings), one may equivalently understand the conditions in Theorem \ref{thm:feascon-mdm} as stipulating the existence of a preference relation over the assortment collection $\cS$ which is consistent with the partial preferences observed in the choice data. 
This characterization for MDM, in terms of the existence of a consistent ranking over assortments, is in contrast to RUM which requires the existence of a probability distribution over the $n!$ rankings possible for $n$ products.  This is the key reason, why unlike RUM, verifying the consistency of given choice data with the MDM hypothesis can be done in polynomial time. 

\begin{proof}
\textit{Necessity of \eqref{eq:mdm-feascon}}: Suppose $\bp_\cS$ is MDM-representable. Then there exist marginal distributions $\{F_i: i \in \cN\}$ and deterministic utilities $\{\nu_i: i \in \cN\}$ such that for any assortment $S \in \cS,$  the given choice probability vector $(p_{i,S}: i \in S)$ and the respective Lagrange multipliers $\lambda_S,\{\lambda_{i,S}: i \in S\}$ are obtainable by solving the optimality conditions $\eqref{optimality}.$   That is, there exist $\{\lambda_S, \lambda_{i,S}: i \in S, S \in \cS\}$ for some fixed choice of $\{F_i: i \in \cN\}$ and  $\{\nu_i: i \in \cN\}$ such that  
\begin{align}
     \nu_i + F^{-1}_i(1-p_{i,S}) - \lambda_S + \lambda_{i,S} &\ =\ 0 \quad \forall (i,S)\in \cI_\cS,\label{eq:opcon1} \\
  \lambda_{i,S} \, p_{i,S} &\ = \ 0 \quad \forall (i,S)\in \cI_\cS.\label{eq:opcon2}
\end{align}
For each product $i \in \mathcal{N}$ and any two assortments  $S,T\in \cS$ containing $i$ as a common product,
$$\lambda_S - \nu_i =  \lambda_{i,S}+F^{-1}_i(1-p_{i,S}) \quad \mbox{ and } \quad  \lambda_T - \nu_i =  \lambda_{i,T}+F^{-1}_i(1-p_{i,T}).$$
If $p_{i,S} < p_{i,T}$,  then $\lambda_{i,S} \geq 0$ and $\lambda_{i,T} = 0$ because of the complementary slackness condition \eqref{eq:opcon2}. Since $F^{-1}_i(1-p)$ is a strictly decreasing function over $p \in [0,1]$, by \eqref{eq:opcon1}, we obtain:
$$\lambda_S - \nu_i \   \geq \    F^{-1}_i(1-p_{i,S})\  > \   F^{-1}_i(1-p_{i,T})  \ =\  \lambda_T - \nu_i.$$
Adding $\nu_i$ on both sides, we obtain that the Lagrange multipliers should satisfy $\lambda_S > \lambda_T.$ If on the other hand $p_{i,S} = p_{i,T} \neq 0,$ we have $\lambda_{i,S} = \lambda_{i,T} =  0$ from the optimality conditions. Then $\lambda_S - \nu_i   =   F^{-1}_i(1-p_{i,S}) =  F^{-1}_i(1-p_{i,T})  = \lambda_T - \nu_i$.
Again, adding $\nu_i$ on both sides, we obtain that the Lagrange multipliers should satisfy $\lambda_S = \lambda_T.$ Thus, setting $\lambda(S) = \lambda_S$ for all $S \in \cS,$ we see that there exists a function  $\lambda: \mathcal{S} \rightarrow \mathbb{R}$ satisfying \eqref{eq:mdm-feascon}.

\textit{Sufficiency of \eqref{eq:mdm-feascon}}:
Given $\bp_\cS$ and $\lambda: \cS \rightarrow \mathbb{R}$ such that \eqref{eq:mdm-feascon} holds  for all $(i,S),(i,T)\in \cI_\cS$, we next exhibit a construction of marginal distributions $(F_i: i \in \mathcal{N})$ and utilities $(\nu_i: i \in \mathcal{N})$ for MDM. This construction will be such that it yields the given $(p_{i,S}: i \in S)$ as the corresponding choice probabilities from the optimality conditions in \eqref{optimality}, for any assortment $S \in \cS.$ 

For any product $i \in \cN,$ let $\mathcal{S}_i = \{S \in \mathcal{S}: i \in S\}$ denote the subcollection of assortments $S \in \mathcal{S}$ which contain the product $i$ and let $m_i = \vert \mathcal{S}_i \vert.$ Let $l_i$ denote the number of assortments containing product $i$ for which $p_{i,S} > 0.$ Here $l_i=m_i$ when the choice probabilities $\{p_{i,S}: S \in \cS_i\}$ are all non-zero. Equipped with this notation, we construct the marginal distribution $F_i(\cdot)$ for any product $i \in \cN$ as follows:
\begin{enumerate}[label=(\alph*),leftmargin=*]
\item Consider any ordering $(S_1, S_2,\ldots,S_{l_i}, S_{l_i + 1}, \ldots, S_{m_i})$  over the assortments in $\mathcal{S}_i$ for which $\lambda(S_1) \leq \lambda(S_2) \leq \ldots \leq \lambda(S_{l_i}) < \lambda(S_{l_i + 1}) \leq \lambda(S_{l_i + 2}) \leq   \ldots \leq \lambda(S_{m_i}).$ With $l_i$ defined as the number of assortments in $\cS_i$ for which $p_{i,S} > 0,$ note that it is necessary to have $\lambda(S_{l_i}) < \lambda(S_{l_i + 1})$ whenever $l_i < m_i.$ This follows from the observations that $\lambda(\cdot)$ satisfies \eqref{eq:mdm-feascon} and $p_{i,S_{l_i}} > 0 = p_{i,S_{l_i + 1}}.$  Further, due to the conditions in \eqref{eq:mdm-feascon}, the choice probabilities $(p_{i,S}: S \in \cS_i)$ must necessarily satisfy the ordering 
$p_{i,S_1} \geq p_{i,S_2} \geq \ldots  \geq p_{i,S_{l_i}} > 0$ and $p_{i,S_{l_i+1}} = p_{i,S_{l_i+2}} = \ldots, p_{i,S_{m_i}} = 0.$ 
\item Construct the cumulative distribution function $F_i(\cdot)$ by first setting $F_i(\lambda(S_k)) = 1 - {p}_{i,S_k}$ for $k=1,\cdots, l_i.$ With this assignment, we complete the construction of the distribution $F_i$ in between these points by connecting them with line segments as follows: For any two consecutive assortments $S_k$ and $S_{k+1}$ in the ordering satisfying $\lambda(S_k) < \lambda(S_{k+1}),$ connect the respective points $(\lambda(S_k), 1-p_{i,S_k})$ and $(\lambda(S_{k+1}), 1-p_{i,S_{k+1}})$ with a line segment (see Figure \ref{fig:mdm-construction}). For $k \leq l_i,$ note that if the consecutive assortments $S_k$ and $S_{k+1}$ are such that $\lambda(S_k) = \lambda(S_{k+1}),$ then the corresponding points $(\lambda(S_k), 1-p_{i,S_k})$ and $(\lambda(S_{k+1}), 1-p_{i,S_{k+1}})$ coincide and there is no need to connect them. Further note that $p_{i,S_{k}} > p_{i,S_{k+1}}$ when  $\lambda(S_k) < \lambda(S_{k+1}),$ because of \eqref{eq:mdm-feascon}, and hence the cumulative distribution function $F_i$ is strictly increasing in the interval
$[\lambda(S_1), \lambda(S_{l_i})].$


\item Lastly we construct the tails of the distribution $F_i$ as follows: For the right tail,  connect the points $(\lambda(S_{l_i}), 1-p_{i,S_{l_i}})$ and $(\lambda(S_{l_i+1}),\, 1)$ with a line segment if $l_i < m_i.$ We then have $F_i(x) = 1$ for every $x \geq \lambda(S_{l_i +1})$ and therefore $F_i^{-1}(1) = \lambda({S_{l_i + 1}}).$
If $l_i = m_i,$ connect the points $(\lambda(S_{l_i}), 1-p_{i,S_{l_i}})$ and $(\lambda(S_{l_i}) + \delta, \, 1)$ by choosing any arbitrary $\delta > 0$ (see Figure \ref{fig:mdm-construction}). In this case, we will have $F_i(x) = 1$ for every $x \geq \lambda(S_{l_i}) + \delta.$
For the left tail, if ${p}_{i,S_1} = 1$, then we have $F_i(x) = 0$ for every $x \leq \lambda(S_1).$ Both the cumulative distribution functions drawn in Figure \ref{fig:mdm-construction} illustrate this case. On the other hand,  if ${p}_{i,S_1} < 1$, we use a line segment to connect $(\lambda(S_1),1-{p}_{i,S_1})$ and $(\lambda(S_1)-\delta, \,0)$ by choosing an arbitrary $\delta > 0$. In this case,  $F_i(x) = 0$ for every $x \leq \lambda(S_1) -\delta.$
\end{enumerate}

\begin{figure}[!htbp]
\centering
\begin{tikzpicture}[scale=0.8]
 \draw[-stealth] (0,0) -- (7,0)node[below]{{$x$}} ;
 \draw[-stealth] (0,0)node[left]{{$1-p_{i,S_{1}} = 0$}} -- (0,3.5) node[left]{{$1$}} --(0,4) node[left]{{$F_i(x)$}};
  \draw[] (0.05,0)node[below]{$\lambda(S_{1})$} -- (0.5,0.4) -- (1.3,0.7) --  (2.7,1.5)-- (3.5,2.3)-- (4.3,2.7)--(6,3.5);
  \draw[dashed] (0,0.7)node[left]{{$1-{p}_{i,S_{k}}$}} -- (1.3,0.7) --  (1.3,0)node[below]{$\lambda({S_{k}})$};
  \draw[dashed] (0,1.5)node[left]{{$1 - {p}_{i,S_{k+1}}$}} -- (2.7,1.5) --  (2.7,0)node[below]{$\lambda(S_{k+1})$};
  -- (4,2);
 \draw[dashed] (0,2.7)node[left]{{$1-{p}_{i,S_{l_i}}$}} -- (4.3,2.7) --  (4.3,0)node[below]{$\lambda({S_{l_i}})$};
  \draw[dashed] (0,3.5) -- (6,3.5) --  (6,0)node[below]{$\lambda({S_{l_i +1}})$};
  \fill[black] (0.05,0) circle (1pt);
  \fill[black] (0.5,0.4) circle (1pt);
  \fill[black] (1.3,0.7) circle (1pt);
    \fill[black] (2.7,1.5) circle (1pt);
      \fill[black] (3.5,2.3) circle (1pt);
        \fill[black] (4.3,2.7) circle (1pt);
          \fill[black] (6,3.5) circle (1pt);
\end{tikzpicture}
\begin{tikzpicture}[scale=0.8]
 \draw[-stealth] (0,0) -- (7.2,0)node[below]{{$x$}} ;
 \draw[-stealth] (0,0)node[left]{{$1-p_{i,S_{1}} = 0$}} -- (0,3.5) node[left]{{$1$}} --(0,4) node[left]{{$F_i(x)$}};
  \draw[] (0.05,0)node[below]{$\lambda(S_{1})$} -- (0.5,0.4) -- (1.3,0.7) --  (2.7,1.5)-- (3.5,2.3)-- (4.3,2.7)--(6,3.5);
  \draw[dashed] (0,0.7)node[left]{{$1-{p}_{i,S_{k}}$}} -- (1.3,0.7) --  (1.3,0)node[below]{$\lambda({S_{k}})$};
  \draw[dashed] (0,1.5)node[left]{{$1 - {p}_{i,S_{k+1}}$}} -- (2.7,1.5) --  (2.7,0)node[below]{$\lambda(S_{k+1})$};
  -- (4,2);
 \draw[dashed] (0,2.7)node[left]{{$1-{p}_{i,S_{l_i}}$}} -- (4.3,2.7) --  (4.3,0)node[below]{$\lambda({S_{l_i}})$};
  \draw[dashed] (0,3.5) -- (6,3.5) --  (6,0)node[below]{$\lambda({S_{l_i}})+\delta$};
  \fill[black] (0.05,0) circle (1pt);
  \fill[black] (0.5,0.4) circle (1pt);
  \fill[black] (1.3,0.7) circle (1pt);
    \fill[black] (2.7,1.5) circle (1pt);
      \fill[black] (3.5,2.3) circle (1pt);
        \fill[black] (4.3,2.7) circle (1pt);
          \fill[black] (6,3.5) circle (1pt);
\end{tikzpicture}
\begin{center}
    (a) \hspace{220pt} (b)
\end{center}
\caption{An illustration of the construction of the marginal distribution $F_i$ when: (a) there is an assortment $S$ for which $p_{i,S} = 0$ (the case where $l_i < m_i$) and (b) $p_{i,S} > 0$ for all assortments with product $i$ (the case where $l_i = m_i$).} 
\label{fig:mdm-construction}
\end{figure}
\vspace{-0.1cm}
The above construction gives marginal distribution functions $(F_i: i \in \cN)$ which are absolutely continuous and strictly increasing within its support. We next show that the constructed marginal distributions yield the given choice probabilities $(p_{i,S}: i \in S),$ for any assortment $S \in \cS,$ when they are used in the  optimality conditions \eqref{optimality} together with the assignment $\nu_i=0,$ for $i \in \cN.$  
 In other words, given $\bp_\cS$,  we next verify that 
\begin{align*}
     F^{-1}_i(1-p_{i,S}) - \lambda(S) + \lambda_{i,S} =0, \quad 
     \lambda_{i,S} \  p_{i,S} =0, \quad \text{ and } \quad 
     \lambda_{i,S} \geq 0, \quad  \forall (i,S)\in \cI_\cS.
\end{align*}
For any $(i,S)\in \cI_\cS$ with $p_{i,S}>0$, we have from the construction of $F_i$ that $F_i(\lambda(S)) = 1 - {p}_{i,S}$. Then for such $p_{i,S}$, we see that the optimality condition $F^{-1}_i(1-p_{i,S}) - \lambda(S) + \lambda_{i,S} =0$ readily holds since the optimality conditions also stipulate that $\lambda_{i,S}=0$ when $p_{i,S} > 0.$ 

\noindent 
For any $(i,S) \in \cI_\cS$ such that $p_{i,S} = 0,$ we have from Steps (a) and (c) of the above construction that $\lambda(S) \geq \lambda(S_{l_i + 1}) = F_i^{-1}(1) = F_i^{-1}(1-p_{i,S}).$ Then if we take $\lambda_{i,S} = \lambda(S) - \lambda(S_{l_i +1}),$ we again readily have  $F^{-1}_i(1-p_{i,S}) - \lambda(S) + \lambda_{i,S} =0.$  This completes the verification that for any choice data $\bp_\cS$ satisfying \eqref{eq:mdm-feascon}, there exists marginal distributions $\{F_i: i \in \cN\}$ and deterministic utilities $\{ \nu_i : i \in \cN\}$ which yield $\bp_\cS$ as the MDM choice probabilities.

Lastly, checking whether the conditions in \eqref{eq:mdm-feascon} are satisfied for given choice data $\bp_\cS$ is equivalent to testing if there exists an assignment for variables $(\lambda_S: S \in \mathcal{S})$ and $\epsilon > 0$ such that,
\begin{align*}
  &\lambda_S  \geq \lambda_T  + \epsilon \quad \text{if} \quad p_{i,S} < p_{i,T},\\
     & \lambda_S  = \lambda_T  \quad \text{if}\quad p_{i,S} = p_{i,T} \neq 0,
\end{align*}
for all $(i,S),(i,T)\in \cI_\cS.$
This is possible in polynomial time by solving a linear program where the above conditions are formulated as constraints and maximizing $\epsilon$ and then checking if the optimal value is strictly positive. This linear program involves $|\cS|$ variables for $(\lambda_S: S \in \mathcal{S})$ and one variable for $\epsilon,$ and at most $n|\cS|$ constraints. 
\end{proof}

\subsection{On the representational power of MDM}
\label{sec:mdm-rep-power}
For any assortment collection $\cS,$ let $\pmdm(\cS)$ denote the collection of choice probabilities for the assortments in $\cS$ which are representable by any MDM choice model. We use the notation $\lambda(S)$ and $\lambda_S$ interchangeably here onwards. Due to the characterization in Theorem \ref{thm:feascon-mdm}, we have the following succinct description for MDM:  $\pmdm(\cS) = \text{Proj}_{\bx} (\Pi_\cS) :=  \left\{ \bx: (\bx,\boldsymbol{\lambda}) \in 
\Pi_\cS \right\},$ where $\Pi_\cS$ is
defined as 
\begin{align}
    \Pi _\cS = \Big\{ (\bx,\boldsymbol{\lambda}) \in \mathbb{R}^{|\cI_\cS|} \times  \mathbb{R}^{ |\cS|}: & \, x_{i,S} \geq 0, \forall (i,S) \in \cI_\cS,\ \sum_{i \in S} x_{i,S} = 1,  \forall S \in \cS,    \label{eq:lifted-set} \\ 
    &\lambda_S > \lambda_T \text{ if } x_{i,S} < x_{i,T}, \  \lambda_S = \lambda_T \text{ if } x_{i,S} = x_{i,T} \neq 0, \forall (i,S), (i,T) \in \cI_\cS 
    \Big\}, \nonumber
\end{align}
for any assortment collection $\cS.$ One may understand the set $\Pi_\cS$ as the collection of MDM choice probabilities augmented with the disutilities $\lambda(\cdot)$ over the assortments. In Theorems \ref{thm:mdm_positive_measure} and \ref{thm:mdm-rum-relation} below, we seek to use the characterization in Theorem \ref{thm:feascon-mdm} to understand the representation power of MDM. 

\begin{theorem}\label{thm:mdm_positive_measure}
For any assortment collection $\cS,$ the collection of choice probabilities represented by MDM has a positive measure. Specifically, $\mu\big(\pmdm(\cS)\big) > 0,$ where $\mu$ is the Lebesgue measure on the product of probability simplices denoted by $\prod_{S \in \cS}\Delta_S$ where $\Delta_S = \{(x_{i,S}: i \in S): x_{i,S} \geq 0, \forall i \in S, \sum_{i \in S} x_{i,S} = 1 \}.$ \textcolor{black}{
On the other hand, the choice probabilities represented by MNL and nested logit possess zero Lebesgue measure with respect to $\mu$ due to the restrictions imposed by the Independence of Irrelevant Alternatives (IIA) property overall for MNL and within the nests for nested logit.
}
\end{theorem}

Theorem \ref{thm:mdm_positive_measure} brings out the contrast with the representation power of parametric choice models such as MNL and nested logit. 
In Lemma \ref{Regularity} below, we observe that the choice probabilities modeled by MDM are \textit{regular} in the sense that the probability
of choosing a specific product $i \in S$ cannot increase if $S$ is enlarged. \textcolor{black}{Proofs of Theorem 2 and Lemma 2 are given in Sections \ref{pf:thm_mdm_measure} and \ref{pf:Regularity}.}
\begin{lemma}
Suppose that the choice probability collection $\bp_\cS \in \pmdm(\cS).$ Then for any two assortments $S,T\in \cS,$ we have 
\begin{itemize}
    \item[a)]  $\bp_\cS$ satisfies the regularity property, that is, $p_{i,S} \geq p_{i,T}$ for all $i \in S$ and $S \subset T;$  and
    \item[b)]  $p_{i,S} \leq p_{i,T}$ if  there exists $j$ such that $p_{j,S} < p_{j,T}$ and $i,j \in S \cap T.$ 
\end{itemize}
\label{Regularity}
\end{lemma}
\noindent 
An interpretation of Lemma \ref{Regularity} b) is that if assortment $T$ is weaker than assortment $S$, in the sense that a common product $j$ has a strictly smaller market share in $S$ than in $T$, then the market share of any other common product $i$ in $S$ must also be no greater than that in $T$. As demonstrated in Theorem \ref{thm:feascon-mdm}, this ranking is reflected in the ordering of the Lagrange multipliers in the utility maximization problem of the representative customer. 

Utilizing the MDM characterization in Theorem \ref{thm:feascon-mdm},  Theorem \ref{thm:mdm-rum-relation} below shows that MDM and RUM do not subsume each other generally. It also reveals that RUM and MDM have equivalent representational power as the class of regular choice models when the assortments collection $\cS$ has a special structure, like nested or laminar collections that are frequently encountered in inventory and revenue management applications. 
To state Theorem \ref{thm:mdm-rum-relation}, we require the following definitions. 

An assortment collection $\cS = \{S_1,S_2,\ldots,S_m\}$ is said to be \textit{nested} if $S_1\subset S_2\subset \ldots \subset S_m$ for some indexing of the assortments. In other words, the smaller sets are always contained in the larger sets. An assortment collection $\cS$ is said to be \textit{laminar} if for any two distinct sets $S,T\in \cS$, either $S\subset T$, or $T\subset S$, or $S\cap T = \emptyset$. Equivalently, any two sets are either disjoint or related by containment. For any assortment collection $\cS,$ let $\prum(\cS)$ and $\preg(\cS)$ denote the collection of choice probabilities over the assortments in $\cS$ which are representable by RUM and the class of regular choice models, respectively. {\color{black} Then $\prum(\cS)$ is given by, 
\begin{align*}
\prum(\cS) := \Big\{ \bx :\exists P(\sigma) \textrm{ s.t. }  x_{i,S} = \sum_{\sigma \in \Sigma_n}  P(\sigma) \mathbb{I}[\sigma, i, S], \ \forall (i,S) \in \cI_\cS, \  \sum_{\sigma \in \Sigma_n} P(\sigma)=1,\ P(\sigma) \geq 0, \  \forall \sigma \in \Sigma_n  \Big\},
\end{align*}
where $\Sigma_n$ denote the set of all permutations of $n$ alternatives and each element $\sigma\in \Sigma_n$ denotes a ranking of $n$ alternatives with $\mathbb{I}[\sigma, i, S]$ being the indicator variable that takes a value of 1 if and only if product $i$ is the most preferred product in assortment $S$ under $\sigma$. 
Similarly, $\preg(\cS)$ is given by, 
\begin{align*}
   \preg(\cS): =  \Big\{ \bx : & \exists \  \by \textrm{ s.t. } x_{i,S} = y_{i,S}, \forall (i,S) \in \cI_\cS, \  y_{i,S} \geq 0, \ \forall i\in S \subseteq \cN, \sum_{i\in S}y_{i,S} =1, \forall S \subseteq \cN, \\
& y_{i,S} \leq y_{i,T}, \; \forall S,T \subseteq \cN, \,  \forall  i \in T \subset S  \Big\}.
\end{align*}} 
For any assortment collection $\cS,$ it is well-known that 
$\prum(\cS) \subseteq \preg(\cS)$ \citeg{berbeglia_ec} and we have from Lemma \ref{Regularity} that  $\pmdm(\cS) \subseteq \preg(\cS).$ From the first condition in \eqref{eq:mdm-feascon}, we observe that the collection $\pmdm(\cS)$ is not necessarily a closed set. We use $\textnormal{closure}\big(\pmdm(\cS)\big)$ to denote the closure of $\pmdm(\cS).$ 

\begin{theorem}[Relationship between MDM and RUM]
Suppose that $\mathcal{S}$ is a collection of assortments formed over $n$ products. Then the following hold:
\begin{itemize}
    \item[a)] When $n = 2,$ the choice probabilities represented by RUM and MDM coincide; when $n = 3,$ the collection of choice probabilities represented by MDM is subsumed by that of RUM; and when  $n \geq 4,$  there exist choice probabilities over $\cS$ that can be represented by both RUM and MDM and neither models subsume the other: specifically,  there exist assortment collections $\cS$ such that  $\pmdm(\cS) \not\subset \prum(\cS)$ and $\prum(\cS) \not\subset \pmdm(\cS).$     
    \item[b)] If the assortment collection $\cS$ is either nested or laminar, then the corresponding choice probabilities over $\cS$ represented by MDM, RUM, and the class of regular models enjoy the following equivalence regardless of $n:$ $\prum(\cS) = \textnormal{closure}\big(\pmdm(\cS)\big) = \preg(\cS).$
\end{itemize}
 \label{thm:mdm-rum-relation}
\end{theorem}
{\color{black} The key observation for Theorem \ref{thm:mdm-rum-relation} a) is that MDM does not necessarily satisfy the Block and Marschak conditions \citep{block1959random}, which are necessary for RUM, while RUM does not necessarily satisfy Lemma \ref{Regularity} b). Moreover, MNL is a special case of both MDM and RUM. To understand why Theorem \ref{thm:mdm-rum-relation} b) holds, note that $\preg(\cS) \subseteq \preg(\cS)^\prime$ with
\begin{align*}
\preg(\cS)^\prime : = \Big\{\bx : x_{i,S} \geq 0,\forall (i,S)\in \cI_\cS,\, \sum_{i\in S} x_{i,S} = 1,  \forall S \in \cS,  \, x_{i,S} \leq x_{i,T}, \; \forall S,T \in \cS, \,  \forall i \in T \subset S  \Big\},
\end{align*}
where $\preg(\cS)^\prime$ can be interpreted as the extended set of $\preg(\cS)$ in which regularity is imposed  only over  the  assortments in the collection $\cS.$ 
Then for any $\bp_\cS \in \preg(\cS)^\prime$, when $\cS$ is a nested collection, say $S_1 \subset S_2 \subset \ldots \subset S_m$, there exists a unique ranking $\lambda_{S_1} \leq \lambda_{S_2} \leq \ldots \leq \lambda_{S_m}$ over assortments such that MDM can represent $\bp_\cS$. When $\cS$ forms a laminar collection, considering any two disjoint assortments $S$ and $T$, neither the regular model nor MDM imposes constraints on the probabilities. Thus, it suffices to focus on the case where $S$ and $T$ are nested. In this case, we observe that $\boldsymbol{\lambda}$ satisfies $\lambda_S \leq \lambda_T$ if $S \subset T$, ensuring that $\bp_\cS \in \text{closure}(\pmdm(\cS))$. The equivalence of $\preg(\cS)$ and $\prum(\cS)$ when $\cS$ is a nested or laminar collection is referred to the proof in Section \ref{pf:mdm-rum-relation}, where we demonstrate the existence of a RUM graph representation, as introduced in \cite{fiorini2004short}, for such $\bp_\cS \in \preg(\cS)^\prime$. 
} \textcolor{black}{A complete proof of Theorem \ref{thm:mdm-rum-relation} is provided in Section \ref{pf:mdm-rum-relation}. }

Next, we show how MDM provides greater representational power compared to both APU and MNL. 
In Proposition \ref{prop:general_iia} below, we present the Generalized Ordinal IIA property of MDM, which relaxes the Ordinal IIA of APU from \cite{fudenberg2015stochastic}, as well as the IIA of MNL from \cite{luce1959individual}. \textcolor{black}{The proof of Proposition \ref{prop:general_iia} is provided in Section \ref{pf:general_iia}.}
\begin{proposition}[Relationship between MDM, APU and MNL] 
\label{prop:general_iia}
Suppose $\bp_\cS$ is representable by MDM and $p_{i,S} >0 ,$ $ \forall (i,S) \in  \cI_\cS$, then $\bp_\cS$ satisfies the following  Generalized Ordinal IIA property: there exist some continuous and strictly decreasing $f_i: [0,1] \to \mathbb{R}_+ \cup \{\infty\}, i=1,...,n$ with $f_i(0) = \infty$ such that
\begin{align}
    \frac{f_i(p_{i,S})}{f_j(p_{j,S})} = \frac{f_i(p_{i,T})}{f_j(p_{j,T})}
\end{align}
for each pair of assortments $S,T\in \cS$ and alternatives $i,j\in S\cap T$. Especially, when the choice probabilities are positive and the collection $\cS$ contains all assortments with sizes 2 and 3, $\bp_\cS$ is represented by MDM if and only if $\bp_\cS$ satisfies the Generalized Ordinal IIA property. This property of MDM implies that MDM includes APU and MNL as special cases:
\begin{itemize} 
    \item[a)] (Theorem 1 in \citealt{fudenberg2015stochastic}) When the functions $f_i, i=1,...,n $ are identical, the Generalized Ordinal IIA reduces to the Ordinal IIA of APU. 
    \item[b)] When the functions satisfy $f_i(p) = p, \ \forall p,$ the Generalized Ordinal IIA reduces to the  IIA of MNL.
\end{itemize}
\end{proposition}
The Generalized Ordinal IIA property states that the choice probabilities of products can be rescaled by product-specific functions, but the ratio of these rescaled probabilities between any two products must remain constant across any pair of assortments. 


We provide examples to illustrate the representation power of these three models. Note that $\bp_\cS$ can be represented by APU if and only if there exists $\blam \in \mathbb{R}^{|\cS|}$ and $\bv \in  \mathbb{R}^n$ such that $\lambda_S + \nu_i > \lambda_T + \nu_j $ if and only if $p_{i,S} > p_{i,T}$, for all $(i,S), (j,T) \in \cI_\cS$ \citep{fudenberg2015stochastic}. The $\bp_\cS$ in Table \ref{tab:mdm_not_apu_mnl} can be represented by MDM by simply setting $\lambda_A = 10$ and $\lambda_B = 5$. However, it cannot be represented by the APU model because $p_{1,A} < p_{2,A}$ implies $\nu_1 < \nu_2$, while $p_{1,B} > p_{2,B}$ implies $\nu_1 > \nu_2$, which leads to a contradiction. The $\bp_\cS$ in Table \ref{tab:apu_not_mnl} can be represented by APU, since we can set $\lambda_A = 8$, $\lambda_B = 10$, $\nu_1 = 1$, $\nu_2 = 2$, $\nu_3 = 5$, and $\nu_4 = 0$ such that $p_{3,A} > p_{2,B} > p_{1,B} > p_{2,A} = p_{4,B} >  p_{1,A} $ and $\lambda_A + \nu_3 > \lambda_B + \nu_2 > \lambda_B+ \nu_1 > \lambda_A + \nu_2 = \lambda_B + \nu_4 > \lambda_A + \nu_1 $. Both instances cannot be represented by MNL since IIA property is violated, i.e., $\frac{p_{1,A}}{p_{2,A}} \neq \frac{p_{1,B}}{p_{2,B}}.$

\begin{table}[htb!] 
\centering
\begin{minipage}{0.4\linewidth}
\centering
\caption{$\bp_\cS$ can be represented by MDM but not APU and MNL} \label{tab:mdm_not_apu_mnl}
\begin{tabular}{|c|c|c|}
\hline
Alternative & A = \{1,2,3\} & B= \{1,2,4\} \\ \hline
1           & 0.20           & 0.45         \\ \hline
2           & 0.25          & 0.30          \\ \hline
3           & 0.55          & -            \\ \hline
4           & -             & 0.25         \\ \hline
\end{tabular}
\end{minipage}%
\hspace{0.1\linewidth} 
\begin{minipage}{0.4 \linewidth}
\centering
\caption{$\bp_\cS$ can be represented by MDM and APU but not MNL} \label{tab:apu_not_mnl}
\begin{tabular}{|c|c|c|}
\hline
Alternative & A = \{1,2,3\} & B= \{1,2,4\} \\ \hline
1           & 0.20           & 0.30         \\ \hline
2           & 0.25          & 0.45          \\ \hline
3           & 0.55          & -            \\ \hline
4           & -             & 0.25         \\ \hline
\end{tabular}
\end{minipage}
\end{table}

\section{A nonparametric approach towards prediction for new assortments}
\label{sec:prediction-mdm}
As an application of the exact characterization derived in Theorem \ref{thm:feascon-mdm}, we first develop a nonparametric data-driven approach for making revenue or sales predictions for new assortments with no prior sales data.  
The key idea behind the proposed nonparametric approach is as follows. To predict the revenue or sales for a new assortment, we consider the collection of all MDM choice models which are consistent with the observed sales data and offer the worst-case expected revenue over this collection as an estimate for the revenue or sales. Thus, robust optimization serves as the basis in our approach for allowing data to select a suitable model based on the prediction task at hand. 

\subsection{A robust optimization formulation for sales and revenue predictions}
As in the previous sections, let $\cN = \{1,\ldots,n\}$ denote the universe of products, $\cS$ denote the collection of assortments for which historical choice data, denoted by $\bp_\cS,$ is available. Suppose that we wish to make sales or revenue predictions for a new assortment $A \notin \cS.$ Utilizing MDM hypothesis to make predictions is most sensible when the given choice data exhibits the MDM demand characteristics identified in Theorem \ref{thm:feascon-mdm}. Therefore, we begin this section with the assumption that the choice data $\bp_\cS$ is MDM-representable. For $i \in \cN,$ let $r_i \in (0,\infty)$  denote the revenue obtained by selling one unit of the product $i.$ The collection of all MDM choice probability vectors $\bx_A = (x_{i,A}: i \in A)$ for the new assortment $A$ which are consistent  with the observed choice data $\bp_\cS$ is given by,
\begin{align*}
   \mathcal{U}_A 
   := \left\{ \bx_A: (\bp_\cS, \bx_A) \in \pmdm(\cS^\prime) \right\}, 
\end{align*}
where $\cS^\prime = \cS \cup \{A\}.$ Here, as before, $\pmdm(\cS^\prime)$ denotes the collection of all MDM representable choice probabilities over the assortment collection $\cS^\prime.$ The worst-case expected revenue over the consistent collection $\mathcal{U}_A$ is then defined as 
\begin{align}
    \underline{r}(A) := \inf_{\bx_{A} \in \, \mathcal{U}_A}\sum_{i \in A} r_i x_{i,A}. 
    \label{eq:rob-revenues}
\end{align}
Due to the exact characterization in Theorem \ref{thm:feascon-mdm}, we obtain the following reformulation for $\underline{r}(A),$ \textcolor{black}{with the proof provided in Section \ref{pf:wc-rev-reform}.}
\begin{proposition}
    \label{prop:wc-rev-reform}
    Suppose that the given choice data collection $\bp_\cS \in \pmdm(\cS)$ and $A \subseteq \cN$ is an assortment not in $\cS.$ Then the worst-case expected revenue $\underline{r}(A)$ equals, 
\begin{subequations}
\begin{align}
    & \min_{\bx_{A},\blam \in \mathbb{R}^{|\cS^\prime|}}  \sum_{i \in A}   r_i x_{i,A} \nonumber\\
    & \text{s.t.}\quad 
    x_{i,A}  \  \leq \   p_{i,S} \quad \text{if} \quad  \lambda_A \geq \lambda_S, \qquad  \forall i\in A, \ (i,S) \in \cI_\cS, \label{model:mdm_revenue_prediction-a}\\     
   & \quad\quad\;\; x_{i,A}  \  \geq \   p_{i,S} \quad \text{if} \quad  \lambda_A \leq \lambda_S, \qquad  \forall i\in A, \ (i,S) \in \cI_\cS, \label{model:mdm_revenue_prediction-b}\\      
    & \quad\quad\;\; \sum_{i\in A}x_{i,A}=1,\nonumber\\
    & \quad\quad\;\; x_{i,A} \geq 0,\quad\quad\quad\quad\quad\quad\quad\quad\quad\quad \forall i\in A, \nonumber\\
    & \quad\quad\;\; \lambda_S  \ > \  \lambda_T  \quad \text{if} \quad p_{i,S} < p_{i,T}, \qquad\ \,   \forall\, (i,S),(i,T) \in \cI_\cS, \nonumber\\
    & \quad\quad\;\; \lambda_S \  = \ \lambda_T    \quad \text{if} \quad p_{i,S}  =  p_{i,T} \neq 0, \quad  \forall \,(i,S),(i,T) \in \cI_\cS \nonumber.
    \end{align}
\end{subequations}
\end{proposition}
One may also obtain worst-case sales predictions for a product $i,$ when offered within assortment $A$, by letting $r_i = 1$ and $r_j = 0$ for $j \neq i$ in the objective in the reformulation in Proposition \ref{prop:wc-rev-reform}. Similarly, replacing the minimization in this  reformulation with a maximization yields a best-case (optimistic) revenue estimate $\bar{r}(A).$ 

Observe when choice data is available for a richer assortment collection, it leads to a less-conservative estimate for $\underline{r}(A)$ and a narrower interval $[\underline{r}(A),\bar{r}(A)]$ as plausible values for revenue estimates which are consistent with the given data and the MDM hypothesis.  This is because the number of constraints in the constraint collections \eqref{model:mdm_revenue_prediction-a}-\eqref{model:mdm_revenue_prediction-b} is larger when the assortment collection $\cS$ for which choice data is available is made richer. In other words, $\cI_{\cS_1} \subseteq \cI_{\cS_2}$ when $\cS_1 \subseteq \cS_2$ and therefore the resulting $\pmdm(\cS_2 \cup \{A\})$ is nested within $\pmdm(\cS_1 \cup \{A\}).$



\subsection{A mixed integer linear formulation for the worst case expected revenue}
As a generally applicable approach for evaluating the worst-case revenue $\underline{r}(A),$ one may model the ``if'' conditions in \eqref{model:mdm_revenue_prediction-a} - \eqref{model:mdm_revenue_prediction-b} via additional binary variables $(\delta_{A,S}, \delta_{S,A}: S \in \cS)$ to obtain the mixed-integer linear reformulation with $\mathcal{O}(n|\cS|)$  binary variables, $\mathcal{O}(n+|\cS|)$ continuous variables, and $\mathcal{O}(n|\cS|)$ constraints as follows. \textcolor{black}{The proof of Proposition \ref{prop:wc-rev-milp} below is provided in Section \ref{pf:wc-rev-milp}.}
\begin{proposition}
\label{prop:wc-rev-milp}
Suppose that the assumptions in Proposition \ref{prop:wc-rev-reform} are satisfied. Then 
for any $0 < \epsilon < 1 /(2\vert \cS \vert),$
the worst-case expected revenue $\underline{r}(A)$ equals the value of the following mixed integer linear program:   
\begin{subequations}
\begin{align}
\min_{\bx_A,\blambda,\bdelta\!\!\,}\quad&  \quad  \sum_{i \in A}  r_i x_{i,A} \nonumber\\
 \text{s.t.}\quad \ -&\delta_{A,S} \ \leq\    \lambda_A  -  \lambda_S \ \leq \  1 -  (1+\epsilon)\delta_{A,S}, \qquad\quad\  \ \forall \,  i \in A, \, (i,S) \in \cI_\cS, \label{mdm:milp-a}\\
 -&\delta_{S,A}  \ \leq\    \lambda_S  -  \lambda_A \ \leq \  1 -  (1+\epsilon)\delta_{S,A}, \qquad\quad\  \ \forall \,  i \in A, \, (i,S) \in \cI_\cS, \label{mdm:milp-b}\\
&  \delta_{A,S} - 1 \ \leq \  x_{i,A} - p_{i,S} \ \leq \ 1 - \delta_{S,A}, \quad\qquad\quad\,   \forall \,  i \in A, \, (i,S) \in \cI_\cS,\label{mdm:milp-c}\\
 -&(\delta_{A,S}+  \delta_{S,A}) \ \leq \ x_{i,A} - p_{i,S} \ \leq \ \delta_{A,S} + \delta_{S,A},\ \quad     \forall \,  i \in A, \, (i,S) \in \cI_\cS,\label{mdm:milp-d}\\ 
& \lambda_S - \lambda_T \ \geq \ \epsilon, \qquad\qquad\qquad\! \forall \, (i,S), (i,T) \in \cI_\cS \ \textnormal{ s.t. }\  p_{i,S} < p_{i,T} \nonumber,\\
&\lambda_S - \lambda_T \ = \ 0, \qquad\qquad\quad\ \ \, \forall \, (i,S), (i,T) \in \cI_\cS \ \textnormal{ s.t. }\  p_{i,S} = p_{i,T} \neq 0 \nonumber,\\
& \sum_{i \in A} x_{i,A} \ = \ 1,  \nonumber \\
& 0 \leq \lambda_A \leq 1,\quad\  x_{i,A} \ \geq\  0,  \ \forall \,  i \in A \nonumber,\quad\ 0 \, \leq\,  \lambda_S \leq\,  1, \quad  \delta_{A,S}, \delta_{S,A} \, \in \, \{0,1\}, \  \forall \, S \in \cS. \nonumber 
\end{align}
\end{subequations}
Likewise, the optimistic expected revenue $\bar{r}(A) := \sup_{\bx_A \in \mathcal{U}_A} \sum_{i\in A} r_i x_{i,A} $ equals the optimal value obtained by maximizing over the constraints in the above mixed integer linear program.  
\end{proposition}

\subsection{Polynomial-time algorithms for prediction with structured collections}  \label{sec:algorithms-wc-rev}
Besides the generally applicable mixed integer linear program in Proposition \ref{prop:wc-rev-milp}, we develop an alternative solution approach leveraging special structures, such as nested or laminar structures, in the assortment collection in order to evaluate the worst-case revenues $\underline{r}(A)$ in polynomial time. Corollary \ref{cor:prediction_structure}, Proposition  \ref{prop:prediction_common_product} and Corollary  \ref{cor:prediction_nest} below show that computing the worst-case expected revenue can be efficient when the assortment collection $\cS$ is structured, \textcolor{black}{and the proofs of these results can be found in Sections \ref{pf:prediction_structure}, \ref{pf:prediction_common_product} and \ref{pf:prediction_nest} respectively.} 

\begin{corollary}\label{cor:prediction_structure}
When $\cS^\prime$ is either nested or laminar, evaluating $\underline{r}(A)$ in \eqref{eq:rob-revenues} is equivalent to solving the following linear program with $O(n)$ continuous variables and $O(n\vert \cS \vert)$ constraints:
\begin{align}
\begin{aligned}
    \min_{\bx_A} \quad& \sum_{i\in A} r_i x_{i,A}\\
   \text{s.t.} \quad &
     x_{i,A} \leq p_{i,S},  \text{ s.t. } S \subset A,\quad \forall i\in A, (i,S)\in \cI_\cS,\\
    & x_{i,A} \geq p_{i,S},  \text{ s.t. } A \subset S,\quad \forall i\in A, (i,S)\in \cI_\cS, \\
    & \sum_{i\in A} x_{i,A} = 1,\quad x_{i,A} \geq 0,\quad\!\! \forall i\in A.\label{model:prediction_structure}
\end{aligned}
\end{align}
\end{corollary}
The result in Corollary \ref{cor:prediction_structure} follows from the conclusion in Theorem \ref{thm:mdm-rum-relation} that $\textnormal{closure}\big(\pmdm(\cS^\prime)\big) = \preg(\cS^\prime)  = \preg(\cS^\prime)^\prime$ when $\cS^\prime$ is either nested or laminar. 

{\color{black}Next, we investigate a special structure of $\cS$ that allows $\underline{r}(A)$ to be computed in polynomial time. This structure corresponds to the case where a common product is present in all assortments. Without loss of generality, assume $\cS = \{S_1, S_2, \dots, S_m\}$. For simplicity, define $S_0$ and $S_{m+1}$ as imaginary sets and assign $p_{i,S_0} = 1$ and $p_{i,S_{m+1}} = 0$ for any $i \in A$.
\begin{proposition}\label{prop:prediction_common_product}
Suppose that the assumptions in Proposition \ref{prop:wc-rev-reform} are satisfied and there exists a product $i^{*}$ being included in every assortment of the collection $\cS$. Without loss of generality,  assume that the assortments $S_1,\ldots,S_m$ are such that $p_{i^{*},S_1} \geq p_{i^{*},S_2} \geq \cdots \geq p_{i^{*},S_m}.$ Then for any given $A,$ 
\begin{align}
\begin{aligned} \underline{r}(A) = \min_{k=0,1,\cdots,|\cS|}\; &\quad\mathcal{\boldsymbol{R}}_k \\
\text{where}\quad \mathcal{\boldsymbol{R}}_k &= \min_{\bx_A} \sum_{i\in A} r_i x_{i,A}  \\
   & \quad\;\;  \text{s.t.} \quad 
     x_{i,A} \leq p_{i,S_{k}}, \quad\! \ \ \forall i \in A \cap S_{k}, \\
    &\quad\quad\quad\;\;\; x_{i,A} \geq p_{i,S_j}, \quad \ \ \forall  j \geq k+1, S_j \in \cS, i\in S_j\cap A ,\\
    & \quad\quad\quad\;\;\; \sum_{i\in A} x_{i,A} = 1,\quad\;\;\; x_{i,A} \geq 0\quad \forall i\in A.  
\end{aligned}\label{model:prediction_common_product}
\end{align}
\end{proposition}
Observe that \eqref{model:prediction_common_product} involves $|\cS| + 1$ linear programs, each with $\mathcal{O}(n)$ continuous variables and $\mathcal{O}(n|\cS|)$ constraints. Therefore, both \eqref{model:prediction_structure} and \eqref{model:prediction_common_product} are polynomial solvable. In practice, a common approach to model the outside option if available is to treat it as a common product included in every assortment. Proposition \ref{prop:prediction_common_product} guarantees that computing $\underline{r}(A)$ becomes polynomially solvable in the case where the outside option is always available in this case.   

Building on Proposition \ref{prop:prediction_common_product}, it is clear that when the assortment collection $\cS$ is nested, for any given unseen assortment $A,$ $\underline{r}(A)$ can be computed efficiently. Without loss of generality, when $\cS$ is nested, let 
$S_1\subset S_2\subset \cdots\subset S_m.$
}
\begin{corollary}
\label{cor:prediction_nest}
Suppose that the assortment collection $\cS$ is nested. Then for any given $A,$  
\begin{align}
\begin{aligned} \underline{r}(A) = \min_{k=0,1,\cdots,|\cS|}\; &\quad\mathcal{\boldsymbol{R}}_k \\
\text{where}\quad \mathcal{\boldsymbol{R}}_k &= \min_{\bx_A} \sum_{i\in A} r_i x_{i,A}  \\
   & \quad\;\;  \text{s.t.} \quad 
     x_{i,A} \leq p_{i,S_{k}}, \quad\! \ \ \forall i \in A \cap S_{k}, \\
    &\quad\quad\quad\;\;\; x_{i,A} \geq p_{i,S}, \quad\quad\; \forall i \in A, (i,S)\in \cI_\cS, S_{k+1} \subseteq S,\\
    & \quad\quad\quad\;\;\; \sum_{i\in A} x_{i,A} = 1,\quad\;\;\; x_{i,A} \geq 0\quad \forall i\in A.  
    \label{model:prediction_nested}
\end{aligned}
\end{align}
\end{corollary}

\section{Limit of MDM and the estimation of best-fitting MDM probabilities}
\label{sec:lom}
Customer preferences captured by choice data however clearly need not always satisfy a specific choice model hypothesis perfectly. Considering choice data instances that are not MDM-representable, we next seek to quantify the limit or the cost of approximating given choice data with MDM and a procedure for identifying  MDM-representable choice probabilities offering the best fit. 

\subsection{A limit of MDM formulation}\label{sec:limit_formulation}
Given choice data $\bp_\cS$ and any $\bx_\cS \in \pmdm(\cS),$ suppose that a loss function $\bx_\cS \mapsto  \text{loss}(\bp_\cS,\bx_\cS)$ measures the degree of inconsistency in approximating choice data $\bp_\cS$ with an MDM-consistent choice probability assignment $\bx_\cS.$ We take the loss function to be non-negative and convex, and  satisfying the property that 
$\text{loss}(\bp_\cS,\bx_\cS )=0$ if and only if $\bx_\cS = \bp_\cS$. Suppose that $(n_S: S \in \cS)$ is a vector of non-negative weights over assortments in $\cS.$ Then a norm-based loss such as $\sum_{S \in \cS} n_S \Vert \bp_S - \bx_S \Vert$ or a Kullback-Liebler divergence based loss such as $-\sum_{S \in \cS} n_S \sum_{i \in S} p_{i,S}\log(x_{i,S}/p_{i,S})$ serve as prominent examples among the losses which satisfy these assumptions. For $S \in \cS,$ the weight $n_S$ may be taken, for example, to be the frequency with which the offer set $S$ has been shown to customers in the choice dataset.

We define the limit of the MDM, denoted by $\mathcal{L}(\bp_\cS),$ as the smallest value of $\text{loss}(\bp_\cS,\bx_\cS )$ attainable by fitting the observed data $\bp_{\cS}$ with an MDM choice model:
\begin{align}
    \mathcal{L}(\bp_\cS) = \inf \left\{ \text{loss}(\bp_\cS,\bx_\cS)\,:\, \bx_{\cS} \in \pmdm(\cS) \right\}.
    \label{eq:limit-of-mdm}
\end{align}
 As is evident from the definition above,  evaluating the limit  $\mathcal{L}(\bp_\cS)$ can be viewed as identifying a choice probability assignment $\bx_\cS^\ast$ which is consistent with the MDM hypothesis and is about as close any MDM can be to the observed choice data $\bp_\cS.$ Thus any $\bx_\cS^\ast$ attaining the minimum in \eqref{eq:limit-of-mdm} can be seen as offering the best fit, within the MDM family, to the observed choice data. In particular, suppose we take  $\text{loss}(\bp_\cS,\bx_\cS ) = - \sum_{S \in \cS} n_S \sum_{i \in S} p_{i,S}\log(x_{i,S}/p_{i,S})$ and the weight $n_S,$ for $S \in \cS,$ to be equal to the number of observations available for an assortment $S$ in the choice dataset. Then, as highlighted in Example 2.1 of  \cite{jagabathula2019limit}, $\bx_{\cS}^\ast$ is a minimizer in the limit formulation \eqref{eq:limit-of-mdm} if and only if it maximizes the likelihood. Thus, in this case, a solution to the limit \eqref{eq:limit-of-mdm} can be viewed as being obtained from  \textit{maximum likelihood estimation} in the MDM family without any parametric restrictions.
 

Besides this use in estimation, one may also use the limit $\mathcal{L}(\bp_\cS)$ as a diagnostic tool for determining how well MDM is suitable for fitting choice data and comparing it with how effective any parametric subclass is in accomplishing the same. To see this use at a conceptual level, suppose that $\boldsymbol{\bar{x}}_\cS$ denotes the choice probabilities obtained by fitting a parametric subclass of MDM, such as MNL (or) marginal exponential model \citep{mishra2014theoretical}. Then, as put forward by \cite{jagabathula2019limit}, one may view the overall loss captured by $\textnormal{loss}(\bp_{\cS},\boldsymbol{\bar{x}}_\cS)$ as below: 
\begin{align*}
    \textnormal{loss}(\bp_{\cS},\boldsymbol{\bar{x}}_\cS) &= \mathcal{L}(\bp_{\cS}) \ + \ \big\{\textnormal{loss}(\bp_{\cS},\boldsymbol{\bar{x}}_\cS) - \mathcal{L}(\bp_{\cS})\big\},
\end{align*}
where the second component  $\textnormal{loss}(\bp_{\cS},\boldsymbol{\bar{x}}_\cS) - \mathcal{L}(\bp_{\cS})$ is the incremental cost that comes with employing a parametric model within the MDM family in order to approximate the choice data. 
If data suggests that this incremental parametric cost is higher relative to the limit $\mathcal{L}(\bp_\cS),$ then one should consider a richer parametric model (or) use the general nonparametric MDM over the chosen parametric class. If, on the other hand,  the loss $\mathcal{L}(\bp_\cS)$ due to MDM itself is large, then MDM should possibly not be considered as a suitable model for the given choice data. 

Recall the characterization $\pmdm(\cS)$ as the projection $\{ \bx: (\bx,\lambda) \in \Pi_\cS\},$ where $\Pi_\cS$ is defined in \eqref{eq:lifted-set}. Due to this characterization, we have the following equivalent formulation for the limit $\mathcal{L}(\bp_\cS)$ \textcolor{black}{and provide the corresponding proof in Section \ref{pf:limit-mdm}.}

\begin{proposition}\label{prop:limit-mdm}
Under Assumption \ref{asp:general},
 the limit $\mathcal{L}(\bp_\cS)$ equals 
\begin{align}
\begin{aligned}
     &\min_{\bx_{\cS},\blambda }\quad  \sum_{S\in \cS}  \textnormal{loss}(\bp_{S},\bx_{S})\\
     & \text{s.t.}\quad  x_{i,S} \,\geq\, x_{i,T}  \text{ if } \lambda_S \leq \lambda_T  , \ \    \forall\, (i,S),(i,T) \in \cI_{\cS},\\
   & \quad\quad \sum_{i \in S}x_{i,S} = 1,\  \qquad\qquad\quad \forall\, S \in \cS,\\
   &\quad\quad \ x_{i,S} \geq 0, \qquad\qquad\qquad\ \   \forall\, (i,S) \in \cI_{\cS}.  \label{model:lom}
\end{aligned}
\end{align}
\end{proposition}

The set of MDM-representable choice probabilities $\pmdm(\cS)$ is non-convex (see, Example \ref{eg:mdm_nonconvexity}). The following theorem is based on reducing a specific instance of the formulation \eqref{model:lom} to the Kemeny optimal rank aggregation problem. \textcolor{black}{A complete proof of Theorem \ref{thm:NP} below is given in Section \ref{pf:np}.} 
\begin{theorem}
    Problem \eqref{model:lom} is NP-hard. 
    \label{thm:NP}
\end{theorem}

\subsection{A mixed integer convex reformulation for the limit of MDM} \label{sec:limit_reformulation}
Proposition \ref{prop:micp} below provides a generally applicable mixed-integer convex reformulation for \eqref{model:lom}. \textcolor{black}{A proof of Proposition \ref{prop:micp} is provided in Section \ref{pf:micp}.}
\begin{proposition} 
\label{prop:micp} 
Suppose that Assumption \ref{asp:general} is satisfied. Then for any $0 < \epsilon < 1 /(2\vert \cS \vert),$ the limit  $\mathcal{L}(\bp_\cS)$ equals the value of the following mixed integer convex program: 
\begin{align}
\begin{aligned}
& \min_{\bx,\blambda,\bdelta}\quad  \sum_{S\in \cS}  \textnormal{loss}(\bp_{S},\bx_{S})\\
& \text{s.t.}\quad   -\delta_{S,T}  \leq   \lambda_{S} - \lambda_{T} \leq   1 -  (1+\epsilon)\delta_{S,T}, \quad\quad\quad\;\;\; \forall \, (i,S),\,(i,T) \in \cI_\cS, \\
& \quad\quad\quad \delta_{S,T} - 1  \leq  x_{i,S}-x_{i,T}  \leq 1 - \delta_{T,S}, \quad\quad\quad\quad\    \forall \, (i,S),\,(i,T) \in \cI_\cS, \\
&\quad\quad-(\delta_{S,T} + \delta_{T,S})  \leq x_{i,S}-x_{i,T}  \leq  \delta_{S,T} + \delta_{T,S}, \quad\,  \forall \, (i,S),\,(i,T) \in \cI_\cS, \\
&\quad\quad\quad\sum_{i\in S} x_{i,S} = 1, \ \forall \, S\in \cS, \quad\quad\;\;\;  x_{i,S} \geq 0,\  \forall\, (i,S)\in \cI_\cS,\\
&\quad\quad\quad 0 \leq \lambda_{S} \leq 1,\ \forall\, S\in \cS, \qquad\quad  \delta_{S,T} \in \{0,1\},\   \forall\, S,T \in \cS.\label{model:mdm_micp}   
\end{aligned}     
\end{align} 
\end{proposition}
Suppose that  $(\bx_\cS^\ast, \blambda^\ast)$ attains the minimum in \eqref{model:lom} or equivalently in  \eqref{model:mdm_micp}. Then for any new assortment $A \notin S,$ one may use the constraints in Proposition \ref{prop:wc-rev-milp} to obtain the robust revenue estimate $\underline{r}(A)$ consistent  with the fitted choice probabilities $\bx_\cS^\ast$ as below:
\begin{align}
&\!\!\!\!\min_{\bx,\lambda_A,\bdelta\!\!\,} \quad  \sum_{i \in A}  r_i x_{i,A} \nonumber\\
& \text{s.t.}\quad  -\delta_{A,S} \ \leq\    \lambda_A  -  \lambda_S^\ast \ \leq \  1 -  (1+\epsilon)\delta_{A,S}, \qquad\qquad \forall \,  i \in A, \, (i,S) \in \cI_\cS, \nonumber\\
&\quad\quad\;\; -\delta_{S,A}  \ \leq\    \lambda_S^\ast  -  \lambda_A \ \leq \  1 -  (1+\epsilon)\delta_{S,A}, \qquad\qquad  \forall \,  i \in A, \, (i,S) \in \cI_\cS, \nonumber\\
&\quad\quad\;\; \quad  \delta_{A,S} - 1 \ \leq \  x_{i,A} - x_{i,S}^\ast \ \leq \ 1 - \delta_{S,A}, \quad\quad\quad\quad\;\;   \forall \,  i \in A, \, (i,S) \in \cI_\cS,\nonumber\\
&\quad\quad-(\delta_{A,S}+  \delta_{S,A}) \ \leq \ x_{i,A} - x_{i,S}^\ast \ \leq \ \delta_{A,S} + \delta_{S,A},\ \quad     \forall \,  i \in A, \, (i,S) \in \cI_\cS,\nonumber\\ 
&\quad\quad\quad\ \sum_{i \in A} x_{i,A} \ = \ 1, \quad 0 \leq \lambda_A \leq 1, \nonumber  \\
&\quad\quad\quad\ x_{i,A} \ \geq\  0,  \ \forall \,  i \in A,  \quad  \delta_{A,S}, \delta_{S,A} \, \in \, \{0,1\}, \  \forall \, S \in \cS. \nonumber 
\end{align}

Since  $\pmdm(\cS)$ is not a closed set and the constraints in \eqref{model:lom} allow $x_{i,S}^\ast = x_{i,T}^\ast$ even when the counterpart $\lambda_S^\ast \neq \lambda_T^\ast,$ the solution $\bx_\cS^\ast$ can only be guaranteed to be arbitrarily close to the MDM-representable collection $\pmdm(\cS).$ Therefore if one wishes to obtain a $\delta$-optimal MDM-representable choice probability assignment, for some $\delta > 0$, they may do so as follows: Equipped with the optimal value $\mathcal{L}(\bp_\cS) = \text{loss}(\bp_\cS,\bx_\cS^\ast)$ and an  optimal $\blambda^\ast,$ a $\delta$-optimal MDM-representable choice probability assignment $\bx_\cS$ can be obtained  by solving the following convex program: 
\begin{align} 
&\!\!\!\!\!\!\!\!\!\!\!\!\!\!\!\!\!\!\!\max_{\bx_\cS,\epsilon}  \quad  \epsilon \label{mdm:perturb} \\   
 \text{s.t.}\quad &
\textnormal{loss}(\bp_\cS, \bx_\cS) \leq \textcolor{black}{\mathcal{L}(\bp_\cS)(1 + \delta)},\nonumber\\
& x_{i,S}  \geq x_{i,T} + \epsilon \text{ if } \lambda_S^\ast < \lambda_T^\ast, \quad\qquad  \forall (i,S),(i,T) \in \cI_\cS,\nonumber\\
& x_{i,S}   = x_{i,T}  \text{ if } \lambda_S^\ast = \lambda_T^\ast,\qquad\qquad \ \, \forall (i,S),(i,T) \in \cI_\cS,\nonumber\\
&\sum_{i \in S}x_{i,S} = 1,\ \forall S \in \cS,
\quad  x_{i,S} \geq 0,  \ \forall (i,S) \in \cI_\cS.   \nonumber 
\end{align}
{\color{black}Given the constraints in \eqref{mdm:perturb} and the characterization in Theorem \ref{thm:feascon-mdm}, any choice probability collection $\bx_\cS$ obtained by solving \eqref{mdm:perturb} is MDM-representable. Moreover, it cannot be improved to achieve a fit better than within a magnitude of $\delta$ for any arbitrary $\delta > 0$, due to the constraint $\text{loss}(\bp_{\cS}, \bx_{\cS}) \leq \mathcal{L}(\bp_\cS)(1 + \delta)$. We provide Example \ref{eg:mdm_feasible_prob} in the appendix to show the procedure to obtain a $\delta$-optimal MDM-representable choice probability assignment $\bx_\cS$ when $\bp_\cS$ is not MDM-representable.}
Note that when $\text{loss}(\cdot,\cdot)$ is defined in terms of the $L_1$-norm, the formulation \eqref{mdm:perturb} is a linear program with $\mathcal{O}(n|\cS|)$ continuous variables and $\mathcal{O}(n|\cS|^2)$ constraints and \eqref{model:mdm_micp} is a mixed-integer linear program  with $\mathcal{O}(|\cS|^2)$ binary variables, $\mathcal{O}(n|\cS|)$ continuous variables and $\mathcal{O}(n|\cS|^2)$ constraints.

\subsection{Polynomial time algorithms for special cases}\label{sec:algorithms-limit}
Besides the mixed integer convex program in Proposition \ref{prop:micp}, we develop an alternative solution approach that seeks to evaluate the limit of MDM by searching over admissible rankings over assortments in $\cS.$ This algorithm is capable of evaluating the limit in polynomial time either if the assortment collection $\cS$ possesses a nested or laminar structure, or, if $\vert \cS \vert$ is fixed. In particular, Corollary \ref{cor:limit_structure} below shows that evaluating $\mathcal{L}(\bp_\cS)$ can be efficient by utilizing the Theorem \ref{thm:mdm-rum-relation} conclusion that  $\textnormal{closure}\big(\pmdm(\cS)\big) = \preg(\cS) = \preg(\cS)^\prime$ under nested or laminar  $\cS.$ In this case, the constraints of $\mathcal{L}(\bp_\cS)$ in \eqref{model:lom} can be replaced with the regularity conditions for choice probabilities over $\cS.$ 

\begin{corollary}\label{cor:limit_structure}
When $\cS$ is nested or laminar, evaluating the limit $\mathcal{L}(\bp_\cS)$ in Proposition \ref{prop:limit-mdm} reduces to the following convex program with $\mathcal{O}(n|\cS|)$ continuous variables and $\mathcal{O}(n|\cS|^2)$ linear constraints: 
\begin{align}
\begin{aligned}
&\min_{\bx_{\cS}}\quad  \sum_{S\in \cS}  \textnormal{loss}(\bp_{S},\bx_{S})\\
& \text{s.t.}\quad  x_{i,S} \,\geq\, x_{i,T}  \text{ if } S \subset T  , \ \   \forall\, (i,S),(i,T) \in \cI_\cS,\\
&\quad\quad \sum_{i \in S}x_{i,S} = 1,\  \forall\, S \in \cS,\quad x_{i,S} \geq 0, \  \forall\, (i,S) \in \cI_\cS.\label{model:lom_structure}
\end{aligned}
\end{align}
\end{corollary}
{\color{black} Beyond the special cases of nested or laminar structures, it is also possible to evaluate $\mathcal{L}(\bp_\cS)$ in Proposition \ref{prop:limit-mdm} efficiently when the number of assortments $|\cS|$ is fixed. 
In this scenario, even without specific structural properties, the problem remains polynomial in the number of alternatives $n$. 
An algorithm for this fixed assortment size case is provided in Section \ref{alg:limit}.
}


\section{Numerical Experiments}
\label{sec:experiment}
\subsection{Numerical Experiments with Real-world Data}
\label{sec:real_data_exp}
{\color{black} 
In this section, we use the dataset from JD.com (see \citealt{shen2020jd}) to evaluate (i) the explanatory ability of MDM as captured by the limit formulation \eqref{model:mdm_micp}, and (ii) the efficacy of predictions obtained from the workflow in Figure \ref{fig:flowchat}.
We compare the performance of MDM with that of RUM and commonly used parametric models such as MNL, Markov Chain Choice Model (MCCM), and LC-MNL.} 

\subsubsection*{Data processing.} 
The dataset includes millions of transaction records and over 3000 products. Each transaction record specifies the set of products viewed by a customer (by clicking these products on the website) and the purchase  made, if any, by  the customer. 
We assume that the set of products viewed by the customer is the offered assortment. We select the top 8 purchased products and combine the remaining products and the non-purchase option as the outside option for customers. We remove  transaction records in which multiple units of a product  is purchased in one record. As a result, the processed dataset contains 1784 customers and 8097 times of purchases in total.
After preprocessing the data, we group the transaction records by product-assortment pairs and count the frequency of each pair. Dividing this frequency by the number of times the corresponding assortment is offered results in the empirical choice probabilities, denoted by $\bp_\cS$.

{\color{black}
\subsubsection*{Comparison of explanatory abilities of the models.}
It turns out that none of the models mentioned above, including  RUM and MDM,  exactly represent the choice data $\bp_\cS$ from the JD.com dataset, even when restricted to the eleven assortments shown at least 100 times. This is conceivably due to the observation noise resulting from limited data (sampling error). However, as we shall describe next, the loss in approximating the choice data with MDM is small, when compared with common parametric models such as MNL, MCCM, and LC-MNL (with 10 classes). 
We defer the discussion on the representational power of the models to Section \ref{sec:discuss_representability}, where the sufficiently large sample size ensures that the impact of sampling error is negligible.

Specifically, we measure the average absolute deviation loss for each model by restricting our dataset to assortments offered at least $n_S$ times, where $n_S$ is varied between $20$ and $60$. Given choice data $\bp_\cS$, and the estimated choice probabilities by the choice model $\bx_\cS$, the average absolute deviation loss can be written as $(\sum_{S\in \cS}n_S \sum_{i\in S} p_{i,S} |x_{i,S} - p_{i,S}|) / \sum_{S\in \cS} n_S.$
For MDM, we compute the loss using the MDM limit formulation \eqref{model:mdm_micp}, while for RUM, we apply the limit formulation from \cite{jagabathula2019limit}. For MCCM, we let the choice probabilities be consistent with the definition of MCCM \citep{blanchet2016markov} and minimize the average absolute deviation loss (see, Section \ref{sec:limit_exp}). For MNL or LC-MNL, we first obtain the maximum likelihood estimators (MLE) of the model and then use the MLE to report the average absolute deviation loss between the estimated choice probabilities under the MLE and the choice data. 
The results in Table \ref{tab:jd-estimation} show that the nonparametric models such as MDM and RUM suffer much lesser loss  than the parametric models 
in approximating the choice data, and hence possess greater explanatory ability. 
Specifically, MDM exhibits significantly lower loss compared to the parametric models: approximately 
76\% lesser than the best-fitting MNL, 71\% lesser than the best-fitting MCCM, and 39\% lesser than the best-fitting LC-MNL for the largest collection size. 
Beyond the explanatory power, MDM also demonstrates significantly shorter computing times compared to RUM, MCCM, and LC-MNL,  highlighting its computational efficiency in practical applications. 

When comparing the average absolute deviation loss across different models, it is important to note that the losses for MNL and LC-MNL may not represent the minimum achievable values since directly minimizing the average absolute deviation loss is computationally challenging and not typical in practice. To further validate the estimation performance of MDM, we also compare the average Kullback-Leibler (KL) loss across MDM, MNL, and LC-MNL. The average KL loss is defined as: $-(\sum_{S \in \cS} n_S \sum_{i \in S} p_{i,S}\log(x_{i,S}/p_{i,S}) ) / \sum_{S\in \cS} n_S $. Due to the computational challenges associated with calculating the minimum average KL loss for MDM, we report an average KL loss that is suboptimal (i.e., larger than the minimum loss) for MDM; see Section \ref{sec:limit_exp_mdm} for details. 
Despite this limitation, Table \ref{tab:jd-kl} shows that both MDM and LC-MNL suffer considerably lower losses compared to MNL in all tested scenarios. Furthermore, MDM remains computationally efficient with the estimation method considered, highlighting its robustness and practicality in estimating choice probabilities.

\begin{table}[htb!]
\centering
\caption{Average absolute deviation loss ($10^{-3}$) and run time (sec) comparison with the JD.com dataset. In the first column $n_\cS \geq 20,$ for example, means the dataset is restricted to assortments that have been shown at least 20 times. This restriction results in a dataset with  $\vert \cS \vert = 29$  assortments, which is indicated in the second column.}\label{tab:jd-estimation}
\begin{tabular}{cc|cc|cc|cc|cc|cc}
\hline
\multirow{2}{*}{\begin{tabular}[c]{@{}c@{}}$n_\cS$ \\ $\geq$\end{tabular}} & \multirow{2}{*}{$|\cS|$} & \multicolumn{2}{c|}{MDM} & \multicolumn{2}{c|}{RUM} & \multicolumn{2}{c|}{MNL} & \multicolumn{2}{c|}{MCCM} & \multicolumn{2}{c}{LC-MNL} \\ \cline{3-12} 
                                                                                 &                      & loss       & time        & loss       & time        & loss        & time       & loss        & time        & loss        & time        \\ \hline
20                                                                               & 29                   & 8.04       & 16.73       & 7.41       & 126.97      & 33.8        & 0.23       & 28.52       & 30.13       & 13.23       & 75.67       \\
30                                                                               & 24                   & 7.38       & 6.04        & 6.85       & 107.64      & 31.83       & 0.24       & 26.24       & 17.22       & 11.01       & 80.47       \\
40                                                                               & 19                   & 6.77       & 1.64        & 6.43       & 88.86       & 29.1        & 0.26       & 22.56       & 7.53        & 11.67       & 94.03       \\
50                                                                               & 15                   & 5.15       & 0.59        & 4.75       & 74.89       & 23.92       & 0.25       & 11.68       & 2.12        & 9.24        & 67.07       \\
60                                                                               & 13                   & 4.42       & 0.23        & 4.19       & 64.53       & 20.14       & 0.24       & 9.62        & 1.74        & 7.9         & 71.97       \\ \hline
\end{tabular}
\end{table}

\vspace{-20pt}

\begin{table}[htb!]
\centering
\caption{Average KL loss ($10^{-3}$) comparison among MDM, MNL, and LC-MNL with JD.com dataset} \label{tab:jd-kl}
\begin{tabular}{cc|cc|c|c}
\hline
\multirow{2}{*}{\begin{tabular}[c]{@{}c@{}}$n_\cS$ \\ $\geq$ \end{tabular}} & \multirow{2}{*}{$|\cS|$} & \multicolumn{2}{c|}{MDM} & \multicolumn{1}{c|}{MNL} & \multicolumn{1}{c}{LC-MNL} \\ \cline{3-6} 
                                                                  &                    & loss  & time   & loss  & loss   \\ \hline
20                                                                & 29                 & 2.37  & 17.19  & 9.85  & 1.76   \\
30                                                                & 24                 & 2.21  & 6.61   & 8.84  & 1.61   \\
40                                                                & 19                 & 2.14  & 2.08   & 8.11  & 1.50   \\
50                                                                & 15                 & 1.43  & 0.94   & 6.03  & 1.16   \\
60                                                                & 13                 & 1.23  & 0.66   & 5.08  & 0.93   \\ \hline
\end{tabular}
\end{table}

\subsubsection*{Comparison of predictive and prescriptive abilities.}

We next evaluate the predictive and prescriptive abilities of the above models by comparing their accuracies in  (i) predicting a ranking over unseen test assortments based on their expected revenues, and (ii) identifying the optimal assortment in a given collection of unseen test assortments. For comparing the quality of rankings, we use the well-known Kendall Tau distance (see Definition \ref{dfn:k-dis} for details) to evaluate the closeness of the predicted ranking of each model with the ground truth hidden from training. The Kendall Tau distance reported in Table \ref{tab:jd-kendall} can be interpreted as the average number of pairwise assortment swaps needed to transform the predicted ranking to the ground truth ranking. \textcolor{black}{Let $\mathcal{A}$ denote the collection of testing unseen assortments, and $\delta$ and $\tau$ are ordered lists of all elements in  $\mathcal{A}$. The reported Kendall Tau distance can be computed as $\sum_{\{S,T\}\in \mathcal{A}, S \succ T  } [\mathbb{I} ( \delta(S) < \delta(T), \tau(S) > \tau(T) ) + \mathbb{I} ( \delta(S) > \delta(T), \tau(S) < \tau(T) )]$, where $\mathbb{I} (\cdot)$ is the indicator function.}
Here the average is taken across 50 train-test instances generated as follows: Considering assortments that are shown at least  $n_S$ times (with $n_S$ varying from 20 to 60), we generate a test set instance by randomly selecting approximately 20\% of the assortments under consideration and using the remaining assortments for training.  For measuring the effectiveness of the prescribed optimal assortment, we follow the same train-test split to report the average relative error between the highest ground truth revenue (among the considered test assortments), \textcolor{black}{denoted as $R^*_{\text{true}}(\mathcal{A})$} and the ground truth revenue of the assortment identified as optimal by each model, \textcolor{black}{\textcolor{black}{denoted as $R_{true}(A^*_{\text{model}})$}. The relative error reported can be computed as $(R^*_{\text{true}}(\mathcal{A}) - R_{true}(A^*_{\text{model}}) )/ R^*_{\text{true}}(\mathcal{A}) .$ A similar criteria has been used in \cite{farias2013nonparametric}.} These relative errors are reported in Table \ref{tab:jd-regret}. 

For MDM, since the outside option is present in every assortment, we compute the robust revenue $\underline{r}(A)$ for each tested assortment $A \in \mathcal{A} $ using \eqref{model:prediction_common_product} when the training data can be represented by MDM. If the training data is not fully representable by MDM, we solve the limit of MDM (Problem \eqref{model:mdm_micp}) and then use the best-fitting choice probabilities to proceed as before. In turn, these steps precisely replicate the workflow in Figure \ref{fig:flowchat}.
For RUM, we apply the robust revenue prediction approach from \cite{farias2013nonparametric}. Following the same sequence of steps advocated in  \cite{farias2013nonparametric}, we take the data to be lying in the RUM representable region with 99\% confidence, which is in turn the lowest value ensuring feasibility for all tested instances.  
For MNL, MCCM, and LC-MNL, we use MLE from the training data to estimate the choice probabilities of the test assortment and then compute its revenue accordingly for each of the models. 
The average out-of-sample performances reported in Tables \ref{tab:jd-kendall} and \ref{tab:jd-regret} show that MDM, MCCM, and LC-MNL outperform RUM and MNL, yielding lower average Kendall Tau distances and lower relative error in predicting the best assortments across almost all scenarios. Additionally, Table \ref{tab:jd-performance_ranking} reports the average ranking of each model based on their Kendall Tau distances and relative errors in predicting the best assortments over the 50 tested instances. Overall, Tables \ref{tab:jd-kendall} - \ref{tab:jd-performance_ranking} show that MDM, MCCM, and LC-MNL perform comparably well, followed by RUM, with MNL offering the worst performance among the models considered.

\begin{table}[htb!]
\centering
\caption{Kendall Tau distance comparison with the JD.com dataset. Here $\vert \cS \vert$ and $\vert \mathcal{A} \vert$ denote the number of assortments in train and test sets, respectively. Standard errors are reported in parentheses.}\label{tab:jd-kendall}
\begin{tabular}{ccc|c|c|c|c|c}
\hline
$n_\cS$ $\geq$  & $ |\cS| $ & $|\mathcal{A}|$ & \multicolumn{1}{c|}{MDM} & \multicolumn{1}{c|}{RUM}  & \multicolumn{1}{c|}{MNL} & \multicolumn{1}{c|}{MCCM} & LC-MNL \\ \hline
20 & 24 & 5 & \textbf{2.48} (0.19) & 3.16 (0.24) & 3.48 (0.27) & 2.56 (0.20) & 2.68 (0.21) \\
30 & 20 & 4 & 2.12 (0.22) & 2.44 (0.22) & 2.3 (0.20) & 1.82 (0.16) & \textbf{1.6} (0.15) \\
40 & 16 & 3 & 1.12 (0.12) & 1.34 (0.15) & 1.36 (0.14) & 1.14 (0.13) & \textbf{0.84} (0.10) \\
50 & 12 & 3 & 0.82 (0.12) & 1.0 (0.13) & 0.94 (0.13) & \textbf{0.68} (0.10) & 0.76 (0.12) \\
60 & 11 & 2 & 0.36 (0.07) & \textbf{0.34} (0.07) & 0.4 (0.07) & 0.38 (0.07) & 0.4 (0.07) \\
\hline
\end{tabular}
\vspace{0.15cm}
\begin{flushleft}\footnotesize Notes. Standard errors are reported in parentheses. Bold values indicate the best-performing model in each row.  \end{flushleft}
\end{table}


\begin{table}[htb!]
\centering
\caption{Comparison of relative error (\%) in revenue of the top-ranked assortment with the JD.com dataset}
\label{tab:jd-regret}
\begin{tabular}{ccc|c|c|c|c|c}
\hline
$n_\cS$ $\geq$ & $ |\cS| $ & $|\mathcal{A}|$ & \multicolumn{1}{c|}{MDM} & \multicolumn{1}{c|}{RUM}  & \multicolumn{1}{c|}{MNL} & \multicolumn{1}{c|}{MCCM} & LC-MNL  \\ \hline
20 & 24 & 5 & 5.76 (1.80) & 11.16 (2.75) & 11.0 (3.04) & \textbf{5.34} (1.83) & 6.34 (1.85)  \\
30 & 20 & 4 & 11.55 (2.79) & 16.32 (3.58) & 17.42 (3.58) & 10.36 (2.58) & \textbf{8.12} (2.29)  \\
40 & 16 & 3 & 10.22 (2.67) & 17.21 (3.40) & 21.36 (3.58) & 12.54 (2.58) & \textbf{8.35} (2.29)  \\
50 & 12 & 3 & 8.63 (2.68) & 15.50 (3.31) & 15.72 (3.58) & \textbf{4.96} (2.95) & 8.32 (2.64)  \\
60 & 11 & 2 & 10.36 (2.50) & \textbf{8.90} (2.26) & 13.42 (3.00) & 9.65 (2.36) & 11.82 (2.74)  \\ \hline
\end{tabular}
\vspace{0.15cm}
\begin{flushleft}\footnotesize Notes. Standard errors are reported in parentheses. Bold values indicate the best-performing model in each row.  \end{flushleft}
\end{table}

\begin{table}[htb!]
\centering
\caption{Average rank of the models' performance on JD.com dataset (smaller rank means better performance)}\label{tab:jd-performance_ranking}
\scalebox{0.9}{\begin{tabular}{ccc|ccccc|ccccc}
\hline
\multirow{2}{*}{\begin{tabular}[c]{@{}c@{}}$n_\cS$\\ $\geq$ \end{tabular}} & \multirow{2}{*}{$|\cS|$} & \multirow{2}{*}{$|\mathcal{A}|$} & \multicolumn{5}{c|}{\begin{tabular}[c]{@{}c@{}}average ranking of models based on \\  Kendall Tau distance\end{tabular}} & \multicolumn{5}{c}{\begin{tabular}[c]{@{}c@{}} average ranking of models based on relative \\  error in revenue  of the top-ranked assortment\end{tabular}} \\ \cline{4-13} 
                                                                  &                    &                    & MDM                & RUM                & MNL                & MCCM               & LC-MNL               & MDM                     & RUM                     & MNL                     & MCCM                   & LC-MNL                   \\ \hline
20                                                                & 24                 & 5                  & \textbf{2.65}               & 3.14               & 3.70               & 2.72               & 2.79                 & \textbf{2.87}                    & 3.15                    & 3.10                    & 2.89                   & 2.99                     \\
30                                                                & 20                 & 4                  & 3.00               & 3.48               & 3.36               & \textbf{2.81}               & 2.35                 & 3.00                    & 3.12                    & 3.20                    & 2.91                   & \textbf{2.77}                     \\
40                                                                & 16                 & 3                  & 2.93               & 3.28               & 3.34               & 3.00               & \textbf{2.45}                 & 2.76                    & 3.24                    & 3.44                    & 2.91                   & \textbf{2.65}                     \\
50                                                                & 12                 & 3                  & 3.02               & 3.25               & 3.16               & \textbf{2.71}               & 2.86                 & 2.92                    & 3.14                    & 3.27                    & \textbf{2.69}                   & 2.98                     \\
60                                                                & 11                 & 2                  & 2.96               & \textbf{2.91}               & 3.06               & 3.01               & 3.06                 & 2.96                    & \textbf{2.91}                    & 3.06                    & 3.01                   & 3.06                     \\ \hline
\end{tabular}}
\vspace{0.15cm}
\begin{flushleft}\footnotesize Notes. Bold values indicate the best-performing model in each row.  \end{flushleft}
\end{table}

\subsection{Numerical Experiments with Synthetic Data}
\label{sec:synthetic_exp}
To gain further insights into the comparative performance of the models, we consider generating data synthetically from two parametric subclasses of RUM, namely, the 
heteroskedastic extreme value (HEV) model introduced by \cite{bhat1995heteroscedastic} and the Probit choice model (see \cite{train2009discrete} and references therein). Note that these models do not belong to the MDM class.  
We use the same product-assortment information structure from the JD.com dataset as a basis to randomly generate 50 underlying HEV and Probit instances. In addition, for each instance, we sample customer purchase data precisely the same number of times the assortment has been offered to customers in the JD.com dataset. With 50 HEV-based and 50 Probit-based synthetic purchase datasets obtained in this manner, we explore the explanatory, predictive, and prescriptive abilities of all the models under comparison as described before in Section \ref{sec:real_data_exp}. 
Further details of data generation are provided in Section \ref{sec:synthetic_data_generation}. \textcolor{black}{We also refer readers to Section \ref{sec:addtional_exp_robustness} for additional experiments that validate the robustness of our findings with random synthetic data instances under varied assortment structures.}

\subsubsection*{Comparison of explanatory, predictive, and prescriptive abilities under HEV. } 
\label{sec:exp_hev_w_sampling}

We report the average absolute deviation loss and standard error over 50 randomly generated HEV-based instances under each testing scenario for the above models in Table \ref{tab:hev-estimation}. The results show that the models perform similarly to their performance on the JD.com dataset. Notably, the nonparametric MDM and RUM models achieve significantly lower average absolute deviation loss when approximating HEV-based instances, indicating superior explanatory power compared to the parametric models. Specifically, MDM exhibits a loss reduction of approximately 85\% compared to the best-fitting MNL, 10\% compared to MCCM, and 40\% compared to LC-MNL for the largest collection size considered. In addition to its strong explanatory ability, MDM also demonstrates significantly faster computation times compared to RUM, MCCM, and LC-MNL, with this computational advantage being even more pronounced than in the JD.com dataset. To further validate MDM's estimation performance, we also compare the average KL loss across MDM, MNL, and LC-MNL. 
Although we report sub-optimal KL losses for MDM due to computational challenges,
the results in Table \ref{tab:hev-kl} demonstrate that both MDM and LC-MNL significantly outperform MNL across all tested scenarios, further highlighting MDM's robustness in accurately estimating choice probabilities.

Next, we report the quality of predicted rankings captured by Kendall Tau distances in Table \ref{tab:hev-kendall} and the average relative error in the revenue of the top-ranked assortment in Table \ref{tab:hev-regret}, and summarize the average ranking of the performance of the models in Table \ref{tab:hev-performance_ranking}. For executing the robust approach to RUM developed in \cite{farias2013nonparametric} in these tables, we have taken the instances to be lying in the RUM representable region with 99.5\% confidence. Robust revenue prediction using RUM has not always been feasible with this choice despite 99.5\% being a relatively high confidence level to use: As an example, we find  8\% of instances in the dataset obtained by restricting to $n_\cS \geq 20$ turn out to be infeasible under RUM. Therefore the numbers reported for RUM in Tables \ref{tab:hev-kendall} - \ref{tab:hev-regret} are obtained by restricting only to feasible instances. 
Overall,  from Tables \ref{tab:hev-kendall} - \ref{tab:hev-performance_ranking}, we infer that MCCM performs the best, followed by MDM and LC-MNL exhibiting comparable performance, with RUM and MNL ranking near the last among the models considered. 

\begin{table}[htb]
\centering
\caption{Average absolute deviation loss ($10^{-3}$) and run time (sec) comparison with HEV-based instances}
\label{tab:hev-estimation}
\scalebox{0.87}{
\begin{tabular}{cccccccccccc}
\hline
\multirow{2}{*}{\begin{tabular}[c]{@{}c@{}}$n_\cS$\\ $\geq$ \end{tabular}} & \multicolumn{1}{c|}{\multirow{2}{*}{$|\cS|$}} & \multicolumn{2}{c|}{MDM}                               & \multicolumn{2}{c|}{RUM}                               & \multicolumn{2}{c|}{MNL}                               & \multicolumn{2}{c|}{MCCM}                              & \multicolumn{2}{c}{LC-MNL}                   \\ \cline{3-12} 
                                                                  & \multicolumn{1}{c|}{}                     & loss          & \multicolumn{1}{c|}{time} & loss          & \multicolumn{1}{c|}{time} & loss          & \multicolumn{1}{c|}{time} & loss          & \multicolumn{1}{c|}{time} & loss          & time         \\ \hline
20                                                                & \multicolumn{1}{c|}{29}                   & 5.32 (0.16)          & \multicolumn{1}{c|}{25.66}      & 4.77 (0.17)          & \multicolumn{1}{c|}{5395.05}    & 36.69 (0.59)         & \multicolumn{1}{c|}{0.26}       & 5.88 (0.17)          & \multicolumn{1}{c|}{968.45}     & 8.85 (0.30)          & 2061.64              \\
30                                                                & \multicolumn{1}{c|}{24}                   & 4.26 (0.16)          & \multicolumn{1}{c|}{3.63}       & 3.73 (0.16)          & \multicolumn{1}{c|}{5301.42}    & 33.94 (0.55)         & \multicolumn{1}{c|}{0.24}       & 4.66 (0.16)          & \multicolumn{1}{c|}{1065.45}     & 7.31 (0.28)          & 2024.75              \\
40                                                                & \multicolumn{1}{c|}{19}                   & 3.19 (0.16)          & \multicolumn{1}{c|}{0.80}       & 2.76 (0.15)          & \multicolumn{1}{c|}{3703.93}    & 30.50 (0.50)         & \multicolumn{1}{c|}{0.25}       & 3.28 (0.15)          & \multicolumn{1}{c|}{579.55}     & 5.73 (0.30)          & 2145.74              \\
50                                                                & \multicolumn{1}{c|}{15}                   & 2.21 (0.14)          & \multicolumn{1}{c|}{0.19}       & 1.89 (0.13)          & \multicolumn{1}{c|}{3395.35}    & 26.45 (0.50)         & \multicolumn{1}{c|}{0.29}       & 2.18 (0.13)          & \multicolumn{1}{c|}{542.14}     & 4.60 (0.50)          & 2032.44              \\
60                                                                & \multicolumn{1}{c|}{13}                   & 1.83 (0.13)          & \multicolumn{1}{c|}{0.09}       & 1.54 (0.12)          & \multicolumn{1}{c|}{3191.30}    & 23.51 (0.48)         & \multicolumn{1}{c|}{0.32}       & 1.65 (0.13)          & \multicolumn{1}{c|}{416.36}     & 3.72 (0.45)          & 1663.41              \\ \hline
\end{tabular}}
\vspace{0.15cm}
\begin{flushleft}\footnotesize Notes. Standard errors are reported in parentheses. \end{flushleft}
\end{table}

\begin{table}[htb!]
\centering
\caption{Comparison on average KL loss ($10^{-3}$) among MDM, MNL, and LC-MNL with HEV-based instances}
\label{tab:hev-kl}
\begin{tabular}{c c|cc|c|c}
\hline
\multirow{2}{*}{\begin{tabular}[c]{@{}c@{}}$n_\cS$\\ $\geq$ \end{tabular}} & \multicolumn{1}{c|}{\multirow{2}{*}{$|\cS|$}} & \multicolumn{2}{c}{MDM} & \multicolumn{1}{|c}{MNL} & \multicolumn{1}{|c}{LC-MNL} \\
\cline{3-6} 
 &  & loss & time (sec) & loss & loss \\
\hline
20 & 29 & 2.83 (0.26) & 31.15 & 19.43 (0.48) & 1.33 (0.07) \\
30 & 24 & 2.34 (0.37) & 4.58 & 17.30 (0.49) & 0.83 (0.05) \\
40 & 19 & 1.85 (0.39) & 1.53 & 14.87 (0.45) & 0.56 (0.05) \\
50 & 15 & 0.86 (0.23) & 0.72 & 12.32 (0.43) & 0.47 (0.17) \\
60 & 13 & 0.88 (0.25) & 0.57 & 10.72 (0.42) & 0.35 (0.12) \\
\hline
\end{tabular}
\vspace{0.15cm}
\begin{flushleft}\footnotesize Notes. Standard errors are reported in parentheses. \end{flushleft}
\end{table}


\begin{table}[htb!]
\centering
\caption{Kendall Tau distance comparison with HEV-based instances}
\label{tab:hev-kendall}
\begin{tabular}{ccc|c|c|c|c|c}
\hline
$n_\cS$ $\geq$
& $|\cS|$ & $|\mathcal{A}|$ & MDM         & RUM         & MNL         & MCCM                             & LC-MNL       \\ \hline
20 & 24  & 5   & 2.54 (0.26) & 3.18 (0.31) & 3.44 (0.31) & 2.44 (0.23) & \textbf{2.12} (0.23) \\
30 & 20  & 4   & \textbf{1.4} (0.19)  & 1.84 (0.20) & 2.16 (0.23) & 1.74 (0.20) & 1.56 (0.20) \\
40 & 16  & 3   & 0.78 (0.12) & 0.86 (0.12) & 0.82 (0.13) & \textbf{0.5} (0.08)  & 0.58 (0.10) \\
50 & 12  & 3   & 0.78 (0.13) & 0.92 (0.13) & 0.9 (0.13)  & \textbf{0.5} (0.09)  & 0.84 (0.12) \\
60 & 11  & 2   & 0.2 (0.06)  & 0.28 (0.06) & 0.26 (0.06) & \textbf{0.18} (0.05) & 0.22 (0.06) \\ \hline
\end{tabular}
\vspace{0.15cm}
\begin{flushleft}\footnotesize Notes. Standard errors are reported in parentheses. Bold values indicate the best-performing model in each row.  \end{flushleft}
\end{table}


\begin{table}[htb!]
\centering
\caption{Comparison of relative error (\%) in revenue of the top-ranked assortment with HEV-based instances}
\label{tab:hev-regret}
\begin{tabular}{ccc|c|c|c|c|c}
\hline
$n_\cS$ $\geq$ & $|\cS|$ & $|\mathcal{A}|$ & MDM         & RUM         & MNL         & MCCM                             & LC-MNL       \\ \hline 
20 & 24  & 5   & 11.19 (1.99) & 9.53 (2.25)  & 15.01 (2.60) & 9.68 (1.99)  & \textbf{9.16} (1.76)  \\
30 & 20  & 4   & 11.86 (2.64) & 15.54 (3.03) & 21.07 (3.26) & 13.93 (2.99) & \textbf{11.55} (2.60) \\
40 & 16  & 3   & 8.01 (2.25)  & 11.27 (2.65) & 12.73 (2.88) & 6.35 (2.08)  & \textbf{5.02} (1.47)  \\
50 & 12  & 3   & 10.32 (2.69) & 12.36 (2.97) & 14.65 (3.25) & \textbf{8.07} (2.30)  & 11.99 (2.71) \\
60 & 11  & 2   & 7.72 (2.54)  & 10.07 (2.82) & 8.74 (2.63)  & \textbf{5.06} (1.91)  & 6.67 (2.27)  \\ \hline
\end{tabular}
\vspace{0.15cm}
\begin{flushleft}\footnotesize Notes. Standard errors are reported in parentheses. Bold values indicate the best-performing model in each row.  \end{flushleft}
\end{table}


\begin{table}[htb!]
\centering
\caption{Average rankings of prescriptive abilities among the compared models with HEV-based instances }\label{tab:hev-performance_ranking}
\scalebox{0.95}{\begin{tabular}{ccc|ccccc|ccccc}
\hline
\multirow{2}{*}{\begin{tabular}[c]{@{}c@{}}$n_\cS$\\ $\geq$ \end{tabular}} & \multirow{2}{*}{$|\cS|$} & \multirow{2}{*}{$|\mathcal{A}|$} & \multicolumn{5}{c|}{\begin{tabular}[c]{@{}c@{}}average ranking of models with\\  Kendall Tau distance\end{tabular}} & \multicolumn{5}{c}{\begin{tabular}[c]{@{}c@{}} average ranking of models with relative error\\   in revenue of the top-ranked assortment\end{tabular}} \\ \cline{4-13} 
                                                                  &                    &                    & MDM                & RUM                & MNL                & MCCM               & LC-MNL               & MDM                     & RUM                     & MNL                     & MCCM                   & LC-MNL                   \\ \hline
20                                                                & 24                 & 5                  & 2.88               & 3.40               & 3.62               & \textbf{2.72}               & 2.38                 & 3.16                    & \textbf{2.82}                    & 3.21                    & 2.95                   & 2.87                     \\
30                                                                & 20                 & 4                  & \textbf{2.46}               & 3.12               & 3.50               & 3.11               & 2.81                 & \textbf{2.80}                    & 2.95                    & 3.44                    & 2.91                   & 2.90                     \\
40                                                                & 16                 & 3                  & 3.09               & 3.30               & 3.19               & \textbf{2.66}               & 2.76                 & 2.98                    & 3.12                    & 3.20                    & \textbf{2.78}                   & 2.92                     \\
50                                                                & 12                 & 3                  & 2.97               & 3.16               & 3.20               & \textbf{2.54}               & 3.13                 & 2.92                    & 3.01                    & 3.15                    & \textbf{2.81}                   & 3.11                     \\
60                                                                & 11                 & 2                  & 2.93               & 3.13               & 3.08               & \textbf{2.88}               & 2.98                 & 2.93                    & 3.13                    & 3.08                    & \textbf{2.88}                   & 2.98                     \\ \hline
\end{tabular}}
\vspace{0.15cm}
\begin{flushleft}\footnotesize Notes. Bold values indicate the best-performing model in each row.  \end{flushleft}
\end{table}

\subsubsection*{Comparison of explanatory, predictive, and prescriptive abilities under Probit. } \label{sec:exp_probit_w_sampling}
We report the average absolute deviation loss and the standard error of the compared models over 50 randomly generated Probit-based instances with negatively correlated utilities in Table \ref{tab:Probit-estimation}. The results demonstrate that MDM, RUM, and MCCM incur significantly lower costs in approximating Probit instances compared to MNL and LC-MNL. Specifically, MDM exhibits significantly lower loss compared to the parametric models: approximately  85\% lesser than the best-fitting MNL and 36\% lesser than the best-fitting LC-MNL for the largest collection size considered. Similar to the computation performance above, MDM consistently demonstrates significantly shorter computing time compared to RUM, MCCM, and LC-MNL with Probit-based instances. When comparing the average KL loss across MDM, MNL, and LC-MNL, the results in Table \ref{tab:Probit-kl} demonstrate that both MDM and LC-MNL significantly outperform MNL across all tested scenarios.

We also report the predictive and prescriptive abilities of the compared models with Probit-based instances with (i) the average Kendall Tau distance comparison over assortments in Table \ref{tab:Probit-kendall},  and (ii) the average relative error in revenue of the top-ranked assortment in Table \ref{tab:Probit-regret}. Here as well, we  consider the instances lying in RUM within  99.5\% confidence region around the observed data. The  percentage of the  instances tested infeasible under RUM is  6\% for the dataset formed by restricting to  assortments for which $n_\cS \geq 20$.  The numbers reported for RUM in Tables \ref{tab:Probit-kendall} - \ref{tab:Probit-regret} are obtained by restricting only to feasible instances. Overall, we infer the following when considering both the comparison criteria in Tables \ref{tab:Probit-kendall} - \ref{tab:Probit-regret} and the average ranking of models' performance summarized in Table \ref{tab:Probit-performance_ranking}: MCCM performs the best overall, followed by LC-MNL, with MDM in the third place, followed by RUM, and MNL performing the worst among the considered models.

\begin{table}[htb!]
\centering
\caption{Average absolute deviation loss ($10^{-3}$) and run time (sec) comparison with Probit-based instances}
\label{tab:Probit-estimation}
\scalebox{0.85}{\begin{tabular}{cc|cccccccccc}
\hline
\multirow{2}{*}{\begin{tabular}[c]{@{}c@{}}$n_\cS$\\ $\geq$ \end{tabular}} & \multirow{2}{*}{$|\cS|$} & \multicolumn{2}{c|}{MDM}                               & \multicolumn{2}{c|}{RUM}                               & \multicolumn{2}{c|}{MNL}                               & \multicolumn{2}{c|}{MCCM}                              & \multicolumn{2}{c}{LC-MNL}                             \\ \cline{3-12} 
                                                                  &                       & loss          & \multicolumn{1}{c|}{time} & loss          & \multicolumn{1}{c|}{time} & loss          & \multicolumn{1}{c|}{time} & loss          & \multicolumn{1}{c|}{time} & loss          & time         \\ \hline
20                                                                & 29                    & 4.58 (0.13)          & \multicolumn{1}{c|}{97.73}      & 3.83 (0.13)          & \multicolumn{1}{c|}{5355.87}    & 30.85 (0.50)         & \multicolumn{1}{c|}{0.26}       & 4.80 (0.13)          & \multicolumn{1}{c|}{599.08}     & 7.16 (0.33)          & 2114.44              \\
30                                                                & 24                    & 3.62 (0.13)          & \multicolumn{1}{c|}{5.36}       & 2.95 (0.12)          & \multicolumn{1}{c|}{4331.74}    & 28.44 (0.51)         & \multicolumn{1}{c|}{0.25}       & 3.68 (0.12)          & \multicolumn{1}{c|}{761.11}     & 5.42 (0.19)          & 2241.20              \\
40                                                                & 19                    & 2.51 (0.12)          & \multicolumn{1}{c|}{0.84}       & 1.98 (0.12)          & \multicolumn{1}{c|}{3644.84}    & 25.24 (0.49)         & \multicolumn{1}{c|}{0.24}       & 2.44 (0.13)          & \multicolumn{1}{c|}{514.44}     & 4.31 (0.53)          & 2646.29              \\
50                                                                & 15                    & 1.65 (0.12)          & \multicolumn{1}{c|}{0.20}       & 1.19 (0.11)          & \multicolumn{1}{c|}{3108.81}    & 21.64 (0.48)         & \multicolumn{1}{c|}{0.24}       & 1.34 (0.12)          & \multicolumn{1}{c|}{339.75}     & 3.32 (0.76)          & 2481.70              \\
60                                                                & 13                    & 1.20 (0.12)          & \multicolumn{1}{c|}{0.07}       & 0.86 (0.11)          & \multicolumn{1}{c|}{2863.83}    & 19.14 (0.44)         & \multicolumn{1}{c|}{0.24}       & 0.91 (0.12)          & \multicolumn{1}{c|}{466.14}     & 2.29 (0.50)          & 1892.71              \\ \hline
\end{tabular}}
\vspace{0.15cm}
\begin{flushleft}\footnotesize Notes. Standard errors are reported in parentheses. \end{flushleft}
\end{table}

\begin{table}[htb!]
\centering
\caption{Comparison on average KL loss ($10^{-3}$) among MDM, MNL, and LC-MNL with Probit-based instances}
\label{tab:Probit-kl}
\begin{tabular}{cc|cc|c|c}
\hline
\multirow{2}{*}{\begin{tabular}[c]{@{}c@{}}$n_\cS$\\ $\geq$ \end{tabular}} & \multicolumn{1}{c|}{\multirow{2}{*}{$|\cS|$}} & \multicolumn{2}{c}{MDM} & \multicolumn{1}{|c}{MNL} & \multicolumn{1}{|c}{LC-MNL} \\
\cline{3-6} 
 &  & loss & time & loss & loss \\
\hline
20 & 29 & 2.44 (0.19) & 114.77 & 9.61 (0.27) & 1.20 (0.10) \\
30 & 24 & 1.91 (0.26) & 6.35 & 8.23 (0.25) & 0.66 (0.04) \\
40 & 19 & 1.20 (0.21) & 1.21 & 6.90 (0.25) & 0.51 (0.13) \\
50 & 15 & 0.55 (0.10) & 0.53 & 5.55 (0.22) & 0.45 (0.19) \\
60 & 13 & 0.43 (0.10) & 0.38 & 4.81 (0.21) & 0.26 (0.11) \\
\hline
\end{tabular}
\vspace{0.15cm}
\begin{flushleft}\footnotesize Notes. Standard errors are reported in parentheses. \end{flushleft}
\end{table}


\begin{table}[htb!]
\centering
\caption{Kendall Tau distance comparison with Probit-based instances}
\label{tab:Probit-kendall}
\begin{tabular}{ccc|c|c|c|c|c}
\hline
$n_\cS$ $\geq$ & $|\cS|$ & $|\mathcal{A}|$ & MDM         & RUM         & MNL         & MCCM                             & LC-MNL       \\ \hline
20 & 24  & 5   & 1.96 (0.21) & 3.07 (0.26) & 2.56 (0.20) & 1.9 (0.21) & \textbf{1.72} (0.22) \\
30 & 20  & 4   & 1.0 (0.14)  & 1.52 (0.21) & 1.22 (0.15) & \textbf{0.84} (0.13) & 1.02 (0.12) \\
40 & 16  & 3   & 0.5 (0.11) & 0.8 (0.12) & 0.64 (0.10) & \textbf{0.46} (0.08)  & \textbf{0.46} (0.08) \\
50 & 12  & 3   & 0.62 (0.12) & 0.86 (0.14) & 0.68 (0.12)  & \textbf{0.32} (0.09)  & 0.32 (0.12) \\
60 & 11  & 2   & 0.16 (0.05)  & 0.18 (0.05) & 0.2 (0.06) & \textbf{0.12} (0.05) & \textbf{0.12} (0.05) \\ \hline
\end{tabular}
\vspace{0.15cm}
\begin{flushleft}\footnotesize Notes. Standard errors are reported in parentheses. Bold values indicate the best-performing model in each row.  \end{flushleft}
\end{table}


\begin{table}[htb]
\centering
\caption{Comparison of relative error (\%) in revenue of the top-ranked assortment with Probit-based instances}
\label{tab:Probit-regret}
\begin{tabular}{ccc|c|c|c|c|c}
\hline
$n_\cS$ $\geq$ & $|\cS|$ & $|\mathcal{A}|$ & MDM         & RUM         & MNL         & MCCM                             & LC-MNL       \\ \hline 
20 & 24 & 5 & 9.21 (2.06) & 9.99 (2.01)  & 7.97 (1.71) & 7.31 (1.65) & \textbf{5.47} (1.46) \\
30 & 20 & 4 & \textbf{4.01} (1.44) & 10.56 (2.58) & 8.18 (2.06) & 4.65 (1.41) & 5.47 (1.63) \\
40 & 16 & 3 & 7.95 (2.19) & 9.67 (2.42)  & 9.41 (2.33) & 6.19 (1.88) & \textbf{6.17} (1.79) \\
50 & 12 & 3 & 7.97 (2.29) & 9.85 (2.53)  & 9.45 (2.48) & \textbf{3.53} (1.76) & 7.31 (2.26) \\
60 & 11 & 2 & 5.07 (1.86) & 7.02 (2.25)  & 7.11 (2.25) & \textbf{3.93} (1.82) & 4.86 (1.87) \\ \hline
\end{tabular}
\vspace{0.15cm}
\begin{flushleft}\footnotesize Notes. Standard errors are reported in parentheses. Bold values indicate the best-performing model in each row.  \end{flushleft}
\end{table}


\begin{table}[htb!]
\centering
\caption{Average rankings of prescriptive abilities among the compared models with Probit-based instances }\label{tab:Probit-performance_ranking}
\scalebox{0.95}{\begin{tabular}{ccc|ccccc|ccccc}
\hline
\multirow{2}{*}{\begin{tabular}[c]{@{}c@{}}$n_\cS$\\ $\geq$ \end{tabular}} & \multirow{2}{*}{$|\cS|$} & \multirow{2}{*}{$|\mathcal{A}|$} & \multicolumn{5}{c|}{\begin{tabular}[c]{@{}c@{}}average ranking of models with\\  Kendall Tau distance\end{tabular}} & \multicolumn{5}{c}{\begin{tabular}[c]{@{}c@{}} average ranking of models with relative error \\  in revenue of the top-ranked assortment\end{tabular}} \\ \cline{4-13} 
                                                                  &                    &                    & MDM                & RUM                & MNL                & MCCM               & LC-MNL               & MDM                     & RUM                     & MNL                     & MCCM                   & LC-MNL                   \\ \hline
20                                                                & 24                 & 5                  & 2.72               & 3.80               & 3.44               & 2.66               & \textbf{2.38}                 & 3.04                    & 3.12                    & 3.07                    & 3.02                   & \textbf{2.74}                     \\
30                                                                & 20                 & 4                  & 2.91               & 3.48               & 3.19               & \textbf{2.51}               & 2.91                 & \textbf{2.82}                    & 3.17                    & 3.12                    & 2.92                   & 2.97                     \\
40                                                                & 16                 & 3                  & 2.97               & 3.43               & 3.05               & \textbf{2.76}               & 2.79                 & 2.91                    & 3.14                    & 3.12                    & 2.94                   & \textbf{2.89}                     \\
50                                                                & 12                 & 3                  & 3.01               & 3.43               & 3.12               & \textbf{2.49}               & 2.95                 & 3.09                    & 3.16                    & 3.13                    & \textbf{2.63}                   & 2.99                     \\
60                                                                & 11                 & 2                  & 3.00               & 3.05               & 3.10               & \textbf{2.90}               & 2.95                 & 3.00                    & 3.05                    & 3.10                    & \textbf{2.90}                   & 2.95                     \\ \hline
\end{tabular}}
\vspace{0.15cm}
\begin{flushleft}\footnotesize Notes. Bold values indicate the best-performing model in each row.  \end{flushleft}
\end{table}

\subsection{The explanatory abilities of the compared models} \label{sec:discuss_representability}

In this section, we further examine the explanatory power of the compared models including MDM, RUM, MNL, MCCM, LC-MNL (with 2 classes), and LC-MNL (with 10 classes) under randomly generated instances of HEV and Probit models. \textcolor{black}{As the underlying choice probabilities of HEV and Probit models do not admit a closed-form expression, we generate 10,000 customer utility realizations from the respective models. Utilizing these realizations, we  compute the underlying choice probabilities  as the sample means of the realizations of individual purchases.} 
Detailed information on the data generation process is provided in Section \ref{sec:data_gen_true_hev_probit}. The evaluation is based on the average KL loss for each model, fitted to 50 randomly generated instances of HEV and Probit. As in the previous experiments,  we consider assortments offered at least $n_S$ times, where $n_S$ varies between $20$ and $60$ across different testing scenarios. Since the procedures for checking representability of MCCM and LC-MNL have not been fully explored in the literature, we evaluate the representational power of the models under comparison by reporting the proportion of instances where the average KL loss is below $10^{-6}$. Additionally, due to the nonconvex feasible region in the MLE formulation of MCCM, some rare, randomly generated instances encounter convergence issues when solved using a continuous solver. Therefore, we report the average KL loss only for the instances that successfully converged to an acceptable level. Tables \ref{tab:rep_true_hev_probit} - \ref{tab:kl_true_probit} show that MDM consistently demonstrates strong explanatory power across both HEV and Probit instances. It has a significantly higher fraction of instances with average KL loss below $10^{-6}$ and a substantially lower average KL loss compared to simpler models like MNL and LC-MNL (with 2 classes). As expected, LC-MNL (with 10 classes) achieves the best performance overall, since it can approximate RUM when the number of classes is sufficiently large. However, MDM holds its ground effectively, consistently showing a lower average KL loss than MCCM and performing comparably well as LC-MNL (with 10 classes) in several cases.

\begin{table}[htb!]
\centering
\caption{Fractions of HEV and Probit instances with average KL loss below $10^{-6}$ of the compared models}
\label{tab:rep_true_hev_probit}
\scalebox{0.89}{\begin{tabular}{cc|ccccc|ccccc}
\hline
\multirow{3}{*}{\begin{tabular}[c]{@{}c@{}}$n_\cS$\\ $\geq$ \end{tabular}} & \multirow{3}{*}{$|\cS|$} & \multicolumn{5}{c|}{HEV}                                                                                                                                                                                                                                                                                                   & \multicolumn{5}{c}{Probit}                                                                                                                                                                                                                                                                                                \\ \cline{3-12} 
                     &                                  & \multicolumn{1}{c|}{\multirow{2}{*}{MDM}} & \multicolumn{1}{c|}{\multirow{2}{*}{MNL}} & \multicolumn{1}{c|}{\multirow{2}{*}{MCCM}} & \multicolumn{1}{c|}{\multirow{2}{*}{\begin{tabular}[c]{@{}c@{}}LC-MNL\\  (2 classes)\end{tabular}}} & \multirow{2}{*}{\begin{tabular}[c]{@{}c@{}}LC-MNL \\ (10 classes)\end{tabular}} & \multicolumn{1}{c|}{\multirow{2}{*}{MDM}} & \multicolumn{1}{c|}{\multirow{2}{*}{MNL}} & \multicolumn{1}{c|}{\multirow{2}{*}{MCCM}} & \multicolumn{1}{c|}{\multirow{2}{*}{\begin{tabular}[c]{@{}c@{}}LC-MNL\\  (2 classes)\end{tabular}}} & \multirow{2}{*}{\begin{tabular}[c]{@{}c@{}}LC-MNL\\  (10 classes)\end{tabular}} \\
                     &                                  & \multicolumn{1}{c|}{}                     & \multicolumn{1}{c|}{}                     & \multicolumn{1}{c|}{}                      & \multicolumn{1}{c|}{}                                                                               &                                                                                 & \multicolumn{1}{c|}{}                     & \multicolumn{1}{c|}{}                     & \multicolumn{1}{c|}{}                      & \multicolumn{1}{c|}{}                                                                               &                                                                                 \\ \hline
20                   & 29                               & \multicolumn{1}{c|}{0.26}                 & \multicolumn{1}{c|}{0}                    & \multicolumn{1}{c|}{0.1}                   & \multicolumn{1}{c|}{0}                                                                              & 0.72                                                                            & \multicolumn{1}{c|}{0.34}                 & \multicolumn{1}{c|}{0}                    & \multicolumn{1}{c|}{0.62}                  & \multicolumn{1}{c|}{0}                                                                              & 0.42                                                                            \\ \hline
30                   & 24                               & \multicolumn{1}{c|}{0.44}                 & \multicolumn{1}{c|}{0}                    & \multicolumn{1}{c|}{0.25}                  & \multicolumn{1}{c|}{0}                                                                              & 0.75                                                                            & \multicolumn{1}{c|}{0.57}                 & \multicolumn{1}{c|}{0}                    & \multicolumn{1}{c|}{0.87}                  & \multicolumn{1}{c|}{0}                                                                              & 0.32                                                                            \\ \hline
40                   & 19                               & \multicolumn{1}{c|}{0.56}                 & \multicolumn{1}{c|}{0}                    & \multicolumn{1}{c|}{0.44}                  & \multicolumn{1}{c|}{0.02}                                                                           & 0.98                                                                            & \multicolumn{1}{c|}{0.73}                 & \multicolumn{1}{c|}{0}                    & \multicolumn{1}{c|}{0.85}                  & \multicolumn{1}{c|}{0}                                                                              & 0.31                                                                            \\ \hline
50                   & 15                               & \multicolumn{1}{c|}{0.73}                 & \multicolumn{1}{c|}{0}                    & \multicolumn{1}{c|}{0.71}                  & \multicolumn{1}{c|}{0}                                                                              & 0.96                                                                            & \multicolumn{1}{c|}{0.87}                 & \multicolumn{1}{c|}{0}                    & \multicolumn{1}{c|}{1}                     & \multicolumn{1}{c|}{0}                                                                              & 0.36                                                                            \\ \hline
60                   & 13                               & \multicolumn{1}{c|}{0.84}                 & \multicolumn{1}{c|}{0}                    & \multicolumn{1}{c|}{0.77}                  & \multicolumn{1}{c|}{0.02}                                                                           & 0.98                                                                            & \multicolumn{1}{c|}{0.93}                 & \multicolumn{1}{c|}{0}                    & \multicolumn{1}{c|}{0.95}                  & \multicolumn{1}{c|}{0}                                                                              & 0.5                                                                             \\ \hline
\end{tabular}}
\end{table}
\vspace{-1cm}


\begin{table}[htb!]
\centering
\caption{Comparison of average KL loss ($10^{-5}$) with HEV instances}
\label{tab:kl_true_hev}
\scalebox{0.95}{\begin{tabular}{cc|c|c|c|c|c}
\hline
$n_\cS$ $\geq$ & $|\cS|$ & MDM & MNL & MCCM & LC-MNL (2 classes) & LC-MNL (10 classes) \\
\hline
20 & 29 & 5.58 (1.76) & 480.99 (67.67) & 16.78 (3.17) & 155.58 (22.42) & 0.53 (0.26) \\
30 & 24 & 3.82 (2.58) & 274.36 (66.09) & 9.85 (3.16) & 61.78 (10.51) & 0.08 (0.03) \\
40 & 19 & 1.70 (0.92) & 365.65 (83.53) & 4.88 (1.40) & 74.33 (16.47) & 0.02 (0.02) \\
50 & 15 & 0.87 (0.60) & 247.71 (49.11) & 1.97 (1.26) & 39.44 (7.92) & 0.16 (0.15) \\
60 & 13 & 0.73 (0.50) & 224.72 (59.29) & 1.55 (0.76) & 28.99 (5.24) & 0.16 (0.16) \\
\hline
\end{tabular}}
\vspace{0.15cm}
\begin{flushleft}\footnotesize Notes. Standard errors are reported in parentheses.  \end{flushleft}
\end{table}

\vspace{-25pt}


\begin{table}[htb!]
\centering
\caption{Comparison of the average KL loss ($10^{-6}$) with Probit instances} 
\label{tab:kl_true_probit}
\scalebox{0.95}{\begin{tabular}{cc|c|c|c|c|c}
\hline 
$n_\cS$ $\geq$ & $|\cS|$ & MDM & MNL & MCCM & LC-MNL (2 classes) & LC-MNL (10 classes) \\ \hline
20 & 29 & 4.25 (0.84) & 474.27 (67.89) & 7.94 (4.66) & 281.40 (38.87) & 204.68 (43.88) \\
30 & 24 & 2.14 (0.55) & 421.53 (48.96) & 14.05 (13.40) & 278.10 (27.74) & 201.07 (27.91) \\
40 & 19 & 0.95 (0.22) & 397.53 (42.59) & 2.28 (1.30) & 302.71 (38.48) & 234.58 (41.65) \\
50 & 15 & 0.76 (0.40) & 367.89 (37.39) & 0.00 (0.00) & 239.63 (28.61) & 160.68 (27.54) \\
60 & 13 & 0.19 (0.08) & 276.90 (28.83) & 1.52 (1.09) & 187.20 (30.00) & 104.96 (25.99) \\
\hline 
\end{tabular}}
\vspace{0.15cm}
\begin{flushleft}\footnotesize Note. Standard errors are reported in parentheses. \end{flushleft}
\end{table}

\section{Limitations}
\label{sec:exp_limitation}
In this section, we discuss the impact of sampling uncertainty and the possible limitations of modeling choice behavior with MDM.
\subsection{Sampling uncertainty} \label{sec:exp_sampling_uncertainty}
While the workflow in Figure \ref{fig:flowchat} provides bounds on predicted revenues for the unseen assortments, these bounds are subject to statistical sampling uncertainty depending on the frequency of purchases in the historical data. 
\textcolor{black}{In our experiments with the JD.com dataset, we find the extent of this uncertainty, quantified via standard error and relative error estimates, to be relatively small. Following the approach in  \cite{farias2013nonparametric}, we obtain the maximum standard error    $\max_{(i,S)\in \cI_{\cS}} \sqrt{\frac{p_{i,S}(1-p_{i,S})}{n_S}}$ to be $0.033$, and the maximum relative error  $\max_{(i,S)\in \cI_{\cS}}\frac{1}{p_{i,S}} \sqrt{\frac{p_{i,S}(1-p_{i,S})}{n_S}} \times 100 $ to be $4.35$\%.}
Although our prediction results are subject to sampling uncertainty, the low standard error observed in our experiment results suggests that the impact on the accuracy of our bounds is limited, and the estimates remain stable for this dataset.


To  explore how statistical sampling uncertainty may, in general, impact the lower and upper bounds for revenue predictions for an unseen assortment,  we next compare predictions obtained under two different extremes of offer times: First, the case of small offer times where $n_S = 20,$ and next, the case of large offer times where $n_S = 1000.$ For this experiment, we generate data from MDM with  marginals being nonidentical exponential distribution. 
The details of data generation are provided in Section \ref{sec:data_gen_mem}. Constructing 95\% confidence regions around empirical purchase frequencies, we obtain the worst-case and best-case revenue estimates of the unseen assortment for each instance 
(see \eqref{model:mdm_ci}).  Figure \ref{fig:sample_uncertainty} shows that predictions for the smaller sample size ($n_S = 20$) exhibit more variation around the true revenue of 1.08, with the worst-case prediction being around 0.4 and the best-case prediction being around 1.6. As expected,  predictions obtained for the larger sample size ($n_S = 1000$) are more tightly clustered around the true revenue, with a worst-case of around 0.9 and a best-case of around $1.2$. 
This highlights that while sampling uncertainty affects prediction precision, larger sample size reduces its impact. Comprehensively addressing the impact of statistical sampling uncertainty is  an important pursuit for future research.  
In the results reported for JD.com dataset, the impact of  sampling uncertainty has been less of an issue.

\subsection{Correlation structure in utilities}
\label{sec:exp_correlation}
We now discuss the limitations caused due to the correlation structure implied by the extremal distribution that is chosen by MDM in \eqref{mdm}. \textcolor{black}{Tables \ref{tab:hev_couplas_estimation} and \ref{tab:hev_couplas_prediction} present the MDM estimation and prediction results respectively,} in which data is generated from RUM where the stochastic component of the product utilities are distributed according to Gumbel distributions. For the copula governing the joint distribution among the utilities, we consider two cases:  (i) independent copula, and (ii) comonotonic copula.   Note that the former case of RUMs with independent Gumbel marginals is the HEV choice model considered in Section \ref{sec:exp_hev_w_sampling}. The details on data generation are presented in Section \ref{sec:coupla_data}. \textcolor{black}{MDM is seen to perform significantly better in the independent copula case, with (a) lower estimation errors across all test scenarios in Table \ref{tab:hev_couplas_estimation}, (b) much lower average Kendall Tau distance of the test assortments, and (c) lower relative error in predicting the best assortments, across almost all scenarios in Table \ref{tab:hev_couplas_prediction}, as compared to the setting with comonotonic copula. }This highlights that MDM's performance drops when the utilities are very positively correlated, as observed in the case of the comonotonic copula. These findings are consistent with the theoretical results in \cite{natarajan2009persistency}, where the extremal distribution of the MDM formulation tends to negatively correlate the utilities, making it less effective in estimation and prediction for scenarios with very positively correlated utilities.

\begin{figure}[htb!]
    \centering
    \begin{minipage}{0.45\textwidth}
        \centering        
        \includegraphics[scale=0.42]{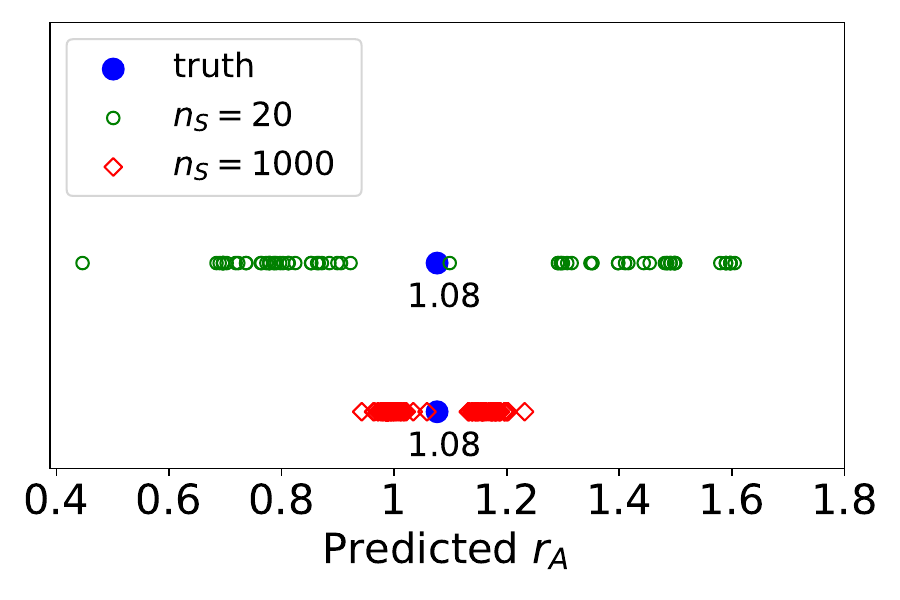}
        \caption{Illustraction of prediction intervals for an unseen assortment under sampling uncertainty }\label{fig:sample_uncertainty}
    \end{minipage}
    \hfill
    \begin{minipage}{0.54\textwidth}
    \vspace{-1cm}
\begin{table}[H]
\centering
\caption{Average absolute deviation loss $(10^{-4})$ of MDM with parametric RUMs instances under different correlations} 
\label{tab:hev_couplas_estimation} 
\begin{tabular}{cc|c|c}
\hline
\begin{tabular}[c]{@{}c@{}}$n_\cS$\\ $\geq$ \end{tabular} & $|\cS|$ & \begin{tabular}[c]{@{}c@{}}independent coupla\\ loss (10$^{-4}$)\end{tabular} & \begin{tabular}[c]{@{}c@{}}comonotonic coupla\\ loss (10$^{-4}$)\end{tabular} \\ \hline
20 & 29  & 0.714 (0.098)  & 5.49 (0.875)   \\ \hline
30 & 24  & 0.558 (0.087)  & 4.36 (0.791)   \\ \hline
40 & 19  & 0.364 (0.079)  & 2.99 (0.617)   \\ \hline
50 & 15  & 0.244 (0.061)  & 1.72 (0.564)   \\ \hline
60 & 13  & 0.198 (0.058)  & 1.23 (0.512)   \\ \hline
\end{tabular}
\vspace{0.15cm}
\begin{flushleft}\footnotesize Notes. Standard errors are reported in parentheses. \end{flushleft}
\end{table}
    \end{minipage}
\end{figure}
}
\vspace{-20pt}
\begin{table}[htb!]
\centering
\caption{Prediction performance of MDM with parametric RUMs instances under different correlations} 
\label{tab:hev_couplas_prediction} 
\scalebox{0.9}{\begin{tabular}{ccc|cc|cc}
\hline
\multirow{2}{*}{\begin{tabular}[c]{@{}c@{}}$n_\cS$\\ $\geq$ \end{tabular}} & \multirow{2}{*}{$|\cS|$} & \multirow{2}{*}{$|\mathcal{A}|$} & \multicolumn{2}{c|}{Kendall Tau distance}                     & \multicolumn{2}{c}{\begin{tabular}[c]{@{}c@{}}relative error in revenue of \\ the top ranked assortment (\%)\end{tabular}} \\ \cline{4-7} 
                                     &                      &                      & \multicolumn{1}{c|}{independent copula} & comonotonic copula & \multicolumn{1}{c|}{independent copula}        & comonotonic copula       \\ \hline
20                                   & 24                   & 5                    & \multicolumn{1}{c|}{0.68 (0.14)}       & 1.02 (0.23)         & \multicolumn{1}{c|}{1.40 (0.69)}                                       & 5.00 (2.15)                                       \\ \hline
30                                   & 20                   & 4                    & \multicolumn{1}{c|}{0.36 (0.07)}        & 0.5 (0.10)          & \multicolumn{1}{c|}{1.32 (0.60)}                                       & 4.34 (1.87)                                       \\ \hline
40                                   & 16                   & 3                    & \multicolumn{1}{c|}{0.26 (0.07)}        & 0.16 (0.05)         & \multicolumn{1}{c|}{0.77 (0.56)}                                       & 2.38 (1.15)                                       \\ \hline
50                                   & 12                   & 3                    & \multicolumn{1}{c|}{0.36 (0.088)}       & 0.46 (0.11)         & \multicolumn{1}{c|}{1.60 (0.76)}                                       & 5.36 (2.29)                                       \\ \hline
60                                   & 11                   & 2                    & \multicolumn{1}{c|}{0.2 (0.06)}         & 0.12 (0.05)         & \multicolumn{1}{c|}{8.45 (2.63)}                                       & 4.07 (2.26)                                       \\ \hline
\end{tabular}}
\vspace{0.15cm}
\begin{flushleft}\footnotesize Notes. Standard errors are reported in parentheses. \end{flushleft}
\end{table}

\section{Conclusions}
\label{sec:conclusion}
We identify the first-known sufficient and necessary conditions to verify whether a set of given choice data can be represented by MDM. Besides being verifiable in polynomial time, these representable conditions lead to (i) a mixed integer linear program to predict the revenue or choice probability of alternatives for an unseen assortment, and (ii) a mixed integer convex program that can give the closest fitting MDM choice probabilities to the given dataset. 
{\color{black} Our numerical experiments with both real and synthetic datasets show that MDM offers competitive representational power and prediction accuracy when compared to RUM and its parametric subclasses, while also being computationally more efficient. This highlights MDM’s potential as a practical alternative for modeling choice data, balancing both performance and computational effort.}

\bibliographystyle{abbrvnat}
\bibliography{ref1.bib}

\begin{thebibliography}{58}
\providecommand{\natexlab}[1]{#1}
\providecommand{\url}[1]{\texttt{#1}}
\expandafter\ifx\csname urlstyle\endcsname\relax
  \providecommand{\doi}[1]{doi: #1}\else
  \providecommand{\doi}{doi: \begingroup \urlstyle{rm}\Url}\fi

\bibitem[Ahipasaoglu et~al.(2019)Ahipasaoglu, Arikan, and Natarajan]{ahipacsaouglu2019distributionally}
S.~D. Ahipasaoglu, U.~Arikan, and K.~Natarajan.
\newblock Distributionally robust markovian traffic equilibrium.
\newblock \emph{Transportation Science}, 53\penalty0 (6):\penalty0 1546--1562, 2019.

\bibitem[Allenby and Ginter(1995)]{allenby1995using}
G.~M. Allenby and J.~L. Ginter.
\newblock Using extremes to design products and segment markets.
\newblock \emph{Journal of Marketing Research}, 32\penalty0 (4):\penalty0 392--403, 1995.

\bibitem[Almanza et~al.(2022)Almanza, Chierichetti, Kumar, Panconesi, and Tomkins]{pmlr-v162-almanza22a}
M.~Almanza, F.~Chierichetti, R.~Kumar, A.~Panconesi, and A.~Tomkins.
\newblock {RUM}s from head-to-head contests.
\newblock In \emph{Proceedings of the 39th International Conference on Machine Learning}, volume 162 of \emph{Proceedings of Machine Learning Research}, pages 452--467. PMLR, 17--23 Jul 2022.

\bibitem[Alptekino{\u{g}}lu and Semple(2016)]{alptekinouglu2016exponomial}
A.~Alptekino{\u{g}}lu and J.~H. Semple.
\newblock The exponomial choice model: A new alternative for assortment and price optimization.
\newblock \emph{Operations Research}, 64\penalty0 (1):\penalty0 79--93, 2016.

\bibitem[Aouad and D{\'e}sir(2022)]{aouad2022representing}
A.~Aouad and A.~D{\'e}sir.
\newblock Representing random utility choice models with neural networks.
\newblock \emph{arXiv preprint arXiv:2207.12877}, 2022.

\bibitem[Barber{\'a} and Pattanaik(1986)]{barbera1986falmagne}
S.~Barber{\'a} and P.~K. Pattanaik.
\newblock Falmagne and the rationalizability of stochastic choices in terms of random orderings.
\newblock \emph{Econometrica: Journal of the Econometric Society}, pages 707--715, 1986.

\bibitem[Ben-Akiva and Lerman(1985)]{akiva}
M.~Ben-Akiva and S.~R. Lerman.
\newblock \emph{Discrete Choice Analysis: Theory and Application to Travel Demand}.
\newblock The MIT Press, 1985.

\bibitem[Berbeglia(2016)]{berbeglia2016discrete}
G.~Berbeglia.
\newblock Discrete choice models based on random walks.
\newblock \emph{Operations Research Letters}, 44\penalty0 (2):\penalty0 234--237, 2016.

\bibitem[Berbeglia and Joret(2017)]{berbeglia_ec}
G.~Berbeglia and G.~Joret.
\newblock Assortment optimisation under a general discrete choice model: A tight analysis of revenue-ordered assortments.
\newblock In \emph{Proceedings of the 2017 ACM Conference on Economics and Computation}, EC '17, page 345–346. ACM, 2017.
\newblock ISBN 9781450345279.

\bibitem[Berry and Haile(2014)]{berry2014identification}
S.~T. Berry and P.~A. Haile.
\newblock Identification in differentiated products markets using market level data.
\newblock \emph{Econometrica}, 82\penalty0 (5):\penalty0 1749--1797, 2014.

\bibitem[Bhat(1995)]{bhat1995heteroscedastic}
C.~R. Bhat.
\newblock A heteroscedastic extreme value model of intercity travel mode choice.
\newblock \emph{Transportation Research Part B: Methodological}, 29\penalty0 (6):\penalty0 471--483, 1995.

\bibitem[Blanchet et~al.(2016)Blanchet, Gallego, and Goyal]{blanchet2016markov}
J.~Blanchet, G.~Gallego, and V.~Goyal.
\newblock A markov chain approximation to choice modeling.
\newblock \emph{Operations Research}, 64\penalty0 (4):\penalty0 886--905, 2016.

\bibitem[Block and Marschak(1960)]{block1959random}
H.~D. Block and J.~Marschak.
\newblock Random orderings and stochastic theories of response.
\newblock \emph{In: Economic Information, Decision, and Prediction. Theory and Decision Library}, 7-1, 1960.

\bibitem[Chen et~al.(2022)Chen, Ma, Natarajan, Simchi-Levi, and Yan]{chen2022distributionally}
L.~Chen, W.~Ma, K.~Natarajan, D.~Simchi-Levi, and Z.~Yan.
\newblock Distributionally robust linear and discrete optimization with marginals.
\newblock \emph{Operations Research}, 70\penalty0 (3):\penalty0 1822--1834, 2022.

\bibitem[Chen et~al.(2019)Chen, Gallego, and Tang]{chen2019use}
N.~Chen, G.~Gallego, and Z.~Tang.
\newblock The use of binary choice forests to model and estimate discrete choices.
\newblock \emph{arXiv preprint arXiv:1908.01109}, 2019.

\bibitem[Chen and Mi{\v{s}}i{\'c}(2022)]{chen2022decision}
Y.-C. Chen and V.~V. Mi{\v{s}}i{\'c}.
\newblock Decision forest: A nonparametric approach to modeling irrational choice.
\newblock \emph{Management Science}, 2022.

\bibitem[Compiani(2022)]{compiani2022market}
G.~Compiani.
\newblock Market counterfactuals and the specification of multiproduct demand: A nonparametric approach.
\newblock \emph{Quantitative Economics}, 13\penalty0 (2):\penalty0 545--591, 2022.

\bibitem[Cosslett(1981)]{cosslett1981maximum}
S.~R. Cosslett.
\newblock Maximum likelihood estimator for choice-based samples.
\newblock \emph{Econometrica: Journal of the Econometric Society}, pages 1289--1316, 1981.

\bibitem[Daganzo(1979)]{daganzo2014multinomial}
C.~Daganzo.
\newblock \emph{Multinomial probit: the theory and its application to demand forecasting}.
\newblock Academic Press, New York, 1979.

\bibitem[de~Bekker-Grob et~al.(2018)de~Bekker-Grob, Veldwijk, Jonker, Donkers, Huisman, Buis, Swait, Lancsar, Witteman, Bonsel, et~al.]{de2018impact}
E.~W. de~Bekker-Grob, J.~Veldwijk, M.~Jonker, B.~Donkers, J.~Huisman, S.~Buis, J.~Swait, E.~Lancsar, C.~L. Witteman, G.~Bonsel, et~al.
\newblock The impact of vaccination and patient characteristics on influenza vaccination uptake of elderly people: a discrete choice experiment.
\newblock \emph{Vaccine}, 36\penalty0 (11):\penalty0 1467--1476, 2018.

\bibitem[Dwork et~al.(2001)Dwork, Kumar, Naor, and Sivakumar]{dwork2001rank}
C.~Dwork, R.~Kumar, M.~Naor, and D.~Sivakumar.
\newblock Rank aggregation methods for the web.
\newblock In \emph{Proceedings of the 10th international conference on World Wide Web}, pages 613--622, 2001.

\bibitem[Falmagne(1978)]{falmagne1978representation}
J.-C. Falmagne.
\newblock A representation theorem for finite random scale systems.
\newblock \emph{Journal of Mathematical Psychology}, 18\penalty0 (1):\penalty0 52--72, 1978.

\bibitem[Farias et~al.(2013)Farias, Jagabathula, and Shah]{farias2013nonparametric}
V.~F. Farias, S.~Jagabathula, and D.~Shah.
\newblock A nonparametric approach to modeling choice with limited data.
\newblock \emph{Management science}, 59\penalty0 (2):\penalty0 305--322, 2013.

\bibitem[Feldman and Topaloglu(2017)]{feldman2017revenue}
J.~B. Feldman and H.~Topaloglu.
\newblock Revenue management under the markov chain choice model.
\newblock \emph{Operations Research}, 65\penalty0 (5):\penalty0 1322--1342, 2017.

\bibitem[Feldman et~al.(2021)Feldman, Svensson, and Zenklusen]{feldman2021online}
M.~Feldman, O.~Svensson, and R.~Zenklusen.
\newblock Online contention resolution schemes with applications to bayesian selection problems.
\newblock \emph{SIAM Journal on Computing}, 50\penalty0 (2):\penalty0 255--300, 2021.

\bibitem[Feng et~al.(2017)Feng, Li, and Wang]{feng65technical}
G.~Feng, X.~Li, and Z.~Wang.
\newblock Technical note—on the relation between several discrete choice models.
\newblock \emph{Operations Research}, 65\penalty0 (6):\penalty0 1429--1731, 2017.

\bibitem[Fiorini(2004)]{fiorini2004short}
S.~Fiorini.
\newblock A short proof of a theorem of falmagne.
\newblock \emph{Journal of Mathematical Psychology}, 48\penalty0 (1):\penalty0 80--82, 2004.

\bibitem[Fudenberg et~al.(2015)Fudenberg, Iijima, and Strzalecki]{fudenberg2015stochastic}
D.~Fudenberg, R.~Iijima, and T.~Strzalecki.
\newblock Stochastic choice and revealed perturbed utility.
\newblock \emph{Econometrica}, 83\penalty0 (6):\penalty0 2371--2409, 2015.

\bibitem[Hofbauer and Sandholm(2002)]{hofbauer}
J.~Hofbauer and W.~H. Sandholm.
\newblock On the global convergence of stochastic fictitious play.
\newblock \emph{Econometrica}, 70\penalty0 (6):\penalty0 2265--2294, 2002.

\bibitem[Jagabathula and Rusmevichientong(2019)]{jagabathula2019limit}
S.~Jagabathula and P.~Rusmevichientong.
\newblock The limit of rationality in choice modeling: Formulation, computation, and implications.
\newblock \emph{Management Science}, 65\penalty0 (5):\penalty0 2196--2215, 2019.

\bibitem[Liu et~al.(2022)Liu, Sun, Teo, and Liu]{liu2022pricing}
C.~Liu, H.~Sun, C.-P. Teo, and M.~Liu.
\newblock Pricing analytics of primary and ancillary products using conversion rate data.
\newblock \emph{Available at SSRN 4095769}, 2022.

\bibitem[Luce(1959)]{luce1959individual}
R.~D. Luce.
\newblock \emph{Individual choice behavior}.
\newblock John Wiley, 1959.

\bibitem[Manski and Lerman(1977)]{manski1977estimation}
C.~F. Manski and S.~R. Lerman.
\newblock The estimation of choice probabilities from choice based samples.
\newblock \emph{Econometrica: Journal of the Econometric Society}, pages 1977--1988, 1977.

\bibitem[Manski and McFadden(1981)]{manski1981alternative}
C.~F. Manski and D.~McFadden.
\newblock Alternative estimators and sample designs for discrete choice analysis.
\newblock \emph{Structural analysis of discrete data with econometric applications}, 2:\penalty0 2--50, 1981.

\bibitem[Marschak(1960)]{Marschak1960}
J.~Marschak.
\newblock \emph{Binary-Choice Constraints and Random Utility Indicators}, pages 218--239.
\newblock Springer Netherlands, Dordrecht, 1960.
\newblock ISBN 978-94-010-9276-0.
\newblock \doi{10.1007/978-94-010-9276-0_9}.
\newblock URL \url{https://doi.org/10.1007/978-94-010-9276-0_9}.

\bibitem[Mas-Colell et~al.(1995)Mas-Colell, Whinston, Green, et~al.]{mas1995microeconomic}
A.~Mas-Colell, M.~D. Whinston, J.~R. Green, et~al.
\newblock \emph{Microeconomic theory}, volume~1.
\newblock Oxford university press New York, 1995.

\bibitem[McFadden(1973)]{McFa73}
D.~McFadden.
\newblock Conditional logit analysis of qualitative choice behaviour.
\newblock In P.~Zarembka, editor, \emph{Frontiers in Econometrics}, pages 105--142. Academic Press New York, New York, NY, USA, 1973.

\bibitem[McFadden(1978)]{mcfadden1977modelling}
D.~McFadden.
\newblock Modelling the choice of residential location.
\newblock \emph{Spatial Interaction Theory and Residential Location}, pages 75--96, 1978.

\bibitem[McFadden(1980)]{mcfadden1980econometric}
D.~McFadden.
\newblock Econometric models for probabilistic choice among products.
\newblock \emph{Journal of Business}, pages S13--S29, 1980.

\bibitem[McFadden(1986)]{mcfadden1986choice}
D.~McFadden.
\newblock The choice theory approach to market research.
\newblock \emph{Marketing science}, 5\penalty0 (4):\penalty0 275--297, 1986.

\bibitem[McFadden and Richter(1990)]{mcfadden1990stochastic}
D.~McFadden and M.~K. Richter.
\newblock Stochastic rationality and revealed stochastic preference. preferences, uncertainty, and optimality, essays in honor of leo hurwicz, 1990.

\bibitem[McFadden and Train(2000)]{mcfadden2000mixed}
D.~McFadden and K.~Train.
\newblock Mixed mnl models for discrete response.
\newblock \emph{Journal of Applied Econometrics}, 15\penalty0 (5):\penalty0 447--470, 2000.

\bibitem[McFadden(2006)]{mcfadden2006revealed}
D.~L. McFadden.
\newblock Revealed stochastic preference: a synthesis.
\newblock In \emph{Rationality and Equilibrium}, pages 1--20. Springer, 2006.

\bibitem[Mishra et~al.(2014)Mishra, Natarajan, Padmanabhan, Teo, and Li]{mishra2014theoretical}
V.~K. Mishra, K.~Natarajan, D.~Padmanabhan, C.-P. Teo, and X.~Li.
\newblock On theoretical and empirical aspects of marginal distribution choice models.
\newblock \emph{Management Science}, 60\penalty0 (6):\penalty0 1511--1531, 2014.

\bibitem[Natarajan et~al.(2009)Natarajan, Song, and Teo]{natarajan2009persistency}
K.~Natarajan, M.~Song, and C.-P. Teo.
\newblock Persistency model and its applications in choice modeling.
\newblock \emph{Management Science}, 55\penalty0 (3):\penalty0 453--469, 2009.

\bibitem[Plackett(1975)]{plackettMNL}
R.~L. Plackett.
\newblock The analysis of permutations.
\newblock \emph{Journal of the Royal Statistical Society. Series C (Applied Statistics)}, 24\penalty0 (2):\penalty0 193--202, 1975.
\newblock ISSN 00359254, 14679876.
\newblock URL \url{http://www.jstor.org/stable/2346567}.

\bibitem[Schrijver(1998)]{schrijver1998theory}
A.~Schrijver.
\newblock \emph{Theory of linear and integer programming}.
\newblock John Wiley \& Sons, 1998.

\bibitem[Shen et~al.(2020)Shen, Tang, Wu, Yuan, and Zhou]{shen2020jd}
M.~Shen, C.~S. Tang, D.~Wu, R.~Yuan, and W.~Zhou.
\newblock Jd. com: Transaction-level data for the 2020 msom data driven research challenge.
\newblock \emph{Manufacturing \& Service Operations Management}, 2020.

\bibitem[Sifringer et~al.(2020)Sifringer, Lurkin, and Alahi]{sifringer2020enhancing}
B.~Sifringer, V.~Lurkin, and A.~Alahi.
\newblock Enhancing discrete choice models with representation learning.
\newblock \emph{Transportation Research Part B: Methodological}, 140:\penalty0 236--261, 2020.

\bibitem[Sturt(2024)]{sturt2024value}
B.~Sturt.
\newblock The value of robust assortment optimization under ranking-based choice models.
\newblock \emph{Management Science}, 2024.

\bibitem[Sun et~al.(2020)Sun, Ahipasaoglu, and Li]{sun2020unified}
Z.~Sun, S.~Ahipasaoglu, and X.~Li.
\newblock A unified analysis for assortment planning with marginal distributions.
\newblock \emph{Available at SSRN 3638783}, 2020.

\bibitem[Talluri and Van~Ryzin(2004)]{talluri2004revenue}
K.~Talluri and G.~Van~Ryzin.
\newblock Revenue management under a general discrete choice model of consumer behavior.
\newblock \emph{Management Science}, 50\penalty0 (1):\penalty0 15--33, 2004.

\bibitem[Ta{\c{s}}kesen et~al.(2022)Ta{\c{s}}kesen, Shafieezadeh-Abadeh, and Kuhn]{tacskesen2022semi}
B.~Ta{\c{s}}kesen, S.~Shafieezadeh-Abadeh, and D.~Kuhn.
\newblock Semi-discrete optimal transport: Hardness, regularization and numerical solution.
\newblock \emph{Mathematical Programming}, pages 1--74, 2022.

\bibitem[Thurstone(1927)]{thurstone1927law}
L.~L. Thurstone.
\newblock A law of comparative judgment.
\newblock \emph{Psychological review}, 34\penalty0 (4):\penalty0 273, 1927.

\bibitem[Train(2008)]{train2008algorithms}
K.~E. Train.
\newblock Em algorithms for nonparametric estimation of mixing distributions.
\newblock \emph{Journal of Choice Modelling}, 1\penalty0 (1):\penalty0 40--69, 2008.

\bibitem[Train(2009)]{train2009discrete}
K.~E. Train.
\newblock \emph{Discrete choice methods with simulation}.
\newblock Cambridge university press, 2009.

\bibitem[Wang et~al.(2020)Wang, Wang, and Zhao]{wang}
S.~Wang, Q.~Wang, and J.~Zhao.
\newblock Deep neural networks for choice analysis: Extracting complete economic information for interpretation.
\newblock \emph{Transportation Research Part C}, 118, 2020.

\bibitem[Yan et~al.(2022)Yan, Natarajan, Teo, and Cheng]{zhenzhen}
Z.~Yan, K.~Natarajan, C.-P. Teo, and C.~Cheng.
\newblock A representative consumer model in data-driven multi-product pricing optimization.
\newblock \emph{Management Science}, 68\penalty0 (8):\penalty0 5798--5827, 2022.

\end{thebibliography}

\ECSwitch


\ECHead{E-companion}
The E-companion is organized as follows. 
Sections \ref{sec:pf_sec2} to \ref{sec:pf_limit} present the proofs of the results, respectively, in Sections \ref{sec:mdm-rep-power}, \ref{sec:prediction-mdm}, and \ref{sec:lom}. 
Section \ref{sec:eg} collects and presents all the illustrative examples mentioned in the paper. 
Section \ref{alg:limit} provides an algorithm to evaluate the limit of MDM in \eqref{model:lom}. 
Section \ref{sec:implementation} details the implementation of the experiments. 
Section \ref{sec:synthetic_data_generation} presents the details for the synthetic data generations for experiments in Section \ref{sec:synthetic_exp} and Section \ref{sec:exp_limitation}.
\textcolor{black}{Section \ref{sec:addtional_exp_robustness} provides synthetic data experiments that validate the robustness of the numerical findings across diverse assortment structures.}  
Sections  \ref{sec:addtional_exp_results} and  \ref{sec:addtional_exp_results_real} present additional numerical results with synthetic and real data respectively.


\section*{Table of Contents for the E-companion}  
\noindent  
\begin{description}  
    \item[\textnormal{EC.1.}] \hyperref[sec:pf_sec2]{Proofs of the Results in Section \ref{sec:mdm-rep-power}} \dotfill \pageref{sec:pf_sec2}
    \begin{description}
        \item[\textnormal{EC.1.1}] \hyperref[pf:thm_mdm_measure]{Proof of Theorem \ref{thm:mdm_positive_measure}}\dotfill \pageref{pf:thm_mdm_measure}
        \item[\textnormal{EC.1.2}] \hyperref[pf:Regularity]{Proof of Lemma \ref{Regularity}}\dotfill \pageref{pf:Regularity}
        \item[\textnormal{EC.1.3}] \hyperref[pf:mdm-rum-relation]{Proof of Theorem \ref{thm:mdm-rum-relation}}\dotfill \pageref{pf:mdm-rum-relation}
        \item[\textnormal{EC.1.4}] \hyperref[pf:general_iia]{Proof of Proposition \ref{prop:general_iia}}\dotfill \pageref{pf:general_iia}
    \end{description}
    
    \item[\textnormal{EC.2.}] \hyperref[sec:pf_sec3]{Proofs of the Results in Section \ref{sec:prediction-mdm}}
    \begin{description}
        \item[\textnormal{EC.2.1}] \hyperref[pf:wc-rev-reform]{Proof of Proposition \ref{prop:wc-rev-reform}}\dotfill \pageref{pf:wc-rev-reform}
        \item[\textnormal{EC.2.2}] \hyperref[pf:wc-rev-milp]{Proof of Proposition \ref{prop:wc-rev-milp}}\dotfill \pageref{pf:wc-rev-milp}
        \item[\textnormal{EC.2.3}] \hyperref[pf:prediction_structure]{Proof of Corollary \ref{cor:prediction_structure}}\dotfill \pageref{pf:prediction_structure}
        \item[\textnormal{EC.2.4}] \hyperref[pf:prediction_common_product]{Proof of Proposition  \ref{prop:prediction_common_product}}\dotfill \pageref{pf:prediction_common_product}
        \item[\textnormal{EC.2.5}] \hyperref[pf:prediction_nest]{Proof of Corollary  \ref{cor:prediction_nest}}\dotfill \pageref{pf:prediction_nest}
    \end{description}
    
    \item[\textnormal{EC.3.}] \hyperref[sec:pf_limit]{Proofs of the Results in Section \ref{sec:lom}} \dotfill \pageref{sec:pf_limit}
    \begin{description}
        \item[\textnormal{EC.3.1}] \hyperref[pf:limit-mdm]{Proof of Proposition \ref{prop:limit-mdm}}\dotfill \pageref{pf:limit-mdm}
        \item[\textnormal{EC.3.2}] \hyperref[pf:np]{Proof of Theorem \ref{thm:NP}}\dotfill \pageref{pf:np}
        \item[\textnormal{EC.3.3}] \hyperref[pf:micp]{Proof of Proposition \ref{prop:micp}}\dotfill \pageref{pf:micp}
    \end{description}

    \item[\textnormal{EC.4}] \hyperref[sec:eg]{Illustrative Examples for Section \ref{sec:lom} } \dotfill \pageref{sec:eg}
    \begin{description}
        \item[\textnormal{EC.5.1}] \hyperref[sec:mdm_nonconvexity]{An example to show the non-convexity of MDM feasible region}\dotfill \pageref{sec:mdm_nonconvexity}
        \item[\textnormal{EC.5.2}] \hyperref[sec:mdm_feasible_prob]{An example to show using \eqref{mdm:perturb} to get MDM-representable probabilities}\dotfill \pageref{sec:mdm_feasible_prob}
    \end{description}
    
    \item[\textnormal{EC.5}] \hyperref[alg:limit]{An Algorithm for Evaluating the Limit of MDM when $\vert \cS \vert$ is Small} \dotfill \pageref{alg:limit}

    \item[\textnormal{EC.6}] \hyperref[sec:implementation]{Additional Useful Details on the Experiments in the Paper} \dotfill \pageref{sec:implementation}
    \begin{description}
        \item[\textnormal{EC.6.1}] \hyperref[sec:implementation]{Representation test experiment implementation details}\dotfill \pageref{sec:implementation}
        \begin{description}
            \item[\textnormal{EC.6.1.1}]  \hyperref[sec:representation_mdm]{Checking the representability of MDM}\dotfill \pageref{sec:representation_mdm}
            \item[\textnormal{EC.6.1.2}]  \hyperref[sec:representation_rum]{Checking the representability of RUM}\dotfill \pageref{sec:representation_rum}
            \item[\textnormal{EC.6.1.3}]  \hyperref[sec:representation_mnl]{Checking the representability of MNL}\dotfill \pageref{sec:representation_mnl}
        \end{description}
        
        \item[\textnormal{EC.6.2}] \hyperref[sec:limit_exp]{Limit experiment implementation details}\dotfill \pageref{sec:limit_exp}
        \begin{description}
            \item[\textnormal{EC.6.2.1}]  \hyperref[sec:limit_exp_mdm]{Limit of MDM}\dotfill \pageref{sec:limit_exp_mdm}
            \item[\textnormal{EC.6.2.2}]  \hyperref[sec:limit_exp_rum]{Limit of RUM}\dotfill \pageref{sec:limit_exp_rum}
            \item[\textnormal{EC.6.2.3}]  \hyperref[sec:limit_exp_mnl]{Limit of MNL}\dotfill \pageref{sec:limit_exp_mnl}
            \item[\textnormal{EC.6.2.4}]  \hyperref[sec:limit_exp_mccm]{Limit of MCCM}\dotfill \pageref{sec:limit_exp_mccm}
            \item[\textnormal{EC.6.2.5}]  \hyperref[sec:limit_exp_lcmnl]{Limit of LC-MNL}\dotfill \pageref{sec:limit_exp_lcmnl}
        \end{description}

        \item[\textnormal{EC.6.3}] \hyperref[sec:prediction_exp]{Prediction experiment implementation details}\dotfill \pageref{sec:prediction_exp}
            \begin{description}
                \item[\textnormal{EC.6.3.1}]  \hyperref[sec:prediction_exp_mdm]{Revenue prediction with MDM using estimate-then-predict}\dotfill \pageref{sec:prediction_exp_mdm}
                \item[\textnormal{EC.6.3.2}]  \hyperref[sec:prediction_exp_mdm]{Revenue prediction with MDM under data uncertainty}\dotfill \pageref{sec:prediction_exp_mdm_ci}
                \item[\textnormal{EC.6.3.3}]  \hyperref[sec:prediction_exp_rum]{Revenue prediction with RUM}\dotfill \pageref{sec:prediction_exp_rum}
                \item[\textnormal{EC.6.2.4}]  \hyperref[sec:prediction_exp_mnl_lc_mnl]{Revenue prediction with MNL, LC-MNL}\dotfill \pageref{sec:prediction_exp_mnl_lc_mnl}
                \item[\textnormal{EC.6.3.5}]  \hyperref[sec:limit_exp_mccm]{Revenue prediction with MCCM}\dotfill \pageref{sec:prediction_exp_mccm}
                \item[\textnormal{EC.6.3.6}]  \hyperref[sec:prediction_exp_ranking]{Average Ranking of Models Based on Prediction Performance}\dotfill \pageref{sec:prediction_exp_ranking}
            \end{description}
    \end{description}

    \item[\textnormal{EC.7}] \hyperref[sec:synthetic_data_generation]{Details on Synthetic Data Generation } \dotfill \pageref{sec:synthetic_data_generation}
    \begin{description}
        
        \item[\textnormal{EC.7.1}] \hyperref[sec:synthetic_data_multinational_draws]{A general procedure of generating instances for synthetic data experiments in Section \ref{sec:synthetic_exp} and Section \ref{sec:exp_sampling_uncertainty}}\dotfill \pageref{sec:synthetic_data_multinational_draws}
        \begin{description}
            \item[\textnormal{EC.7.1.1}] \hyperref[sec:data_gen_hev]{HEV instance in Section \ref{sec:synthetic_exp}}\dotfill \pageref{sec:data_gen_hev}
            \item[\textnormal{EC.7.1.2}] \hyperref[sec:data_gen_Probit]{Probit instance with negative correlations in utilities in Section \ref{sec:synthetic_exp}}\dotfill \pageref{sec:data_gen_Probit}
        \end{description}

        \item[\textnormal{EC.7.2}] \hyperref[sec:data_gen_true_hev_probit]{Data Generation for experiments in Section \ref{sec:discuss_representability}}\dotfill \pageref{sec:data_gen_true_hev_probit}

        \item[\textnormal{EC.7.3}] \hyperref[sec:data_gen_mem]{Data generation for experiments in Section \ref{sec:exp_sampling_uncertainty}}\dotfill \pageref{sec:data_gen_mem}

        \item[\textnormal{EC.7.4}] \hyperref[sec:coupla_data]{Data generation for experiments in Section \ref{sec:exp_correlation}}\dotfill \pageref{sec:coupla_data}

    \end{description}

     \item[\textnormal{EC.8}] \hyperref[sec:addtional_exp_robustness]{Robustness of numerical results in Section \ref{sec:synthetic_exp}} \dotfill \pageref{sec:addtional_exp_robustness}
    
    \item[\textnormal{EC.9}] \hyperref[sec:addtional_exp_results]{Additional numerical results with synthetic data} \dotfill \pageref{sec:addtional_exp_results}
    \begin{description}
        
        \item[\textnormal{EC.9.1}] \hyperref[sec:addtional_exp_results_rep]{The representation power and tractability of MDM compared to RUM and MNL}\dotfill \pageref{sec:addtional_exp_results_rep}

        \item[\textnormal{EC.9.2}] \hyperref[sec:addtional_exp_results_prediction]{Revenue and choice probability prediction with nonparametric MDM}\dotfill \pageref{sec:addtional_exp_results_prediction}

        \item[\textnormal{EC.9.3}] \hyperref[sec:addtional_exp_results_estimation]{Estimation performance of MDM compared to RUM and MNL}\dotfill \pageref{sec:addtional_exp_results_estimation}

        \item[\textnormal{EC.9.4}] \hyperref[sec:representability_outside]{Impact on explanatory abilities of models with common alternatives in the assortments}\dotfill \pageref{sec:representability_outside}

    \end{description}
    
    \item[\textnormal{EC.10}] \hyperref[sec:addtional_exp_results_real]{Additional numerical results with real data} \dotfill \pageref{sec:addtional_exp_results_real}

\end{description}
\newpage

\section{Proofs of the Results in Section \ref{sec:mdm-rep-power}}
\label{sec:pf_sec2}

\subsection{Proof of Theorem \ref{thm:mdm_positive_measure}}
\label{pf:thm_mdm_measure}
To prove MDM possesses positive measure in Theorem \ref{thm:mdm_positive_measure}, we first present Lemma \ref{lemma:mdm_distinct} below that deals with choice probabilities of the form obtained by the MNL model.
\begin{lemma}\label{lemma:mdm_distinct}
For any fixed $n$, let $S,T \subseteq \cN$ with $|S|,|T|\geq 2$. Let $i \in S\cap T.$ Then there exists a set of positive integers $x_1,...,x_n$ such that 
\begin{align}
\begin{aligned}
     \frac{x_i}{\sum_{k \in S} x_k } \neq \frac{x_i}{\sum_{k \in T} x_k },\label{eq:mdm_distinct}
\end{aligned}
\end{align}

as long as $S \neq T$.     
\end{lemma}
\begin{proof}
We prove Lemma \ref{lemma:mdm_distinct} holds for a set of positive integers such that $x_k = 2^k$ when $k\geq 1.$ When $S \neq T$, to show $\frac{x_i}{\sum_{k \in S} x_k } \neq \frac{x_i}{\sum_{k \in T} x_k },$ it's equivalent to show $\sum_{k \in S} x_k \neq \sum_{k \in T} x_k.$ It's obvious that $\sum_{k \in S} x_k \neq \sum_{k \in T} x_k$ when $S\subset T$ or $T \subset S$ for the selected $x_k$ values. Next, we prove that $\sum_{k \in S} x_k \neq \sum_{k \in T} x_k$ when $S\not\subset T$ or $T \not\subset S.$ Let $k_1 = \argmax_{k\in S\setminus T} x_k $ and $k_2 = \argmax_{k\in T \setminus S} x_k.$ Without loss of generality, let $k_1 > k_2.$ Then, we have $k_1 - k_2 \geq 1.$ We have 
\begin{align*}
\sum_{k \in S} x_k - \sum_{k \in T} x_k \geq x_{k_1} - \sum_{k \in T \setminus S} x_k \geq x_{k_1} - \sum_{i=1}^{k_2} x_i = 2^{k_1} - \frac{2*(1-2^{k_2})}{1-2} = 2^{k_1} - 2^{k_2 + 1}+2>0.
\end{align*}
The first inequality is due to $k_1 \in S\setminus T,$ and the second inequality is due to $T \setminus S \subseteq \{1,\cdots k_2\}.$ The first equality is due to the formula of the sum of geometric series. This completes the proof. 
\end{proof}

Recall that the probability of choosing product $i$ in assortment $S$ under an MNL model is $p_{i,S} = {\frac{e^{\nu_i}}{ \sum_{j \in S} e^{\nu_j}} }.$ Then, there exists $\nu_i = \ln x_i$ for all $i\in \cN$, such that the instance in Lemma \ref{lemma:mdm_distinct} is an MNL instance. 
\begin{proof}{\color{black} Equipped with Lemma \ref{lemma:mdm_distinct}, we prove that MDM possesses positive measure as follows:
Let $\bp_\cS$ be an instance generated as in Lemma \ref{lemma:mdm_distinct}. 
Then, such $\bp_\cS$ is an MNL instance and satisfies $p_{i,S} >0$, $\forall (i,S) \in \cI_\cS$ and $p_{i,S} \neq p_{i,T}$, $\forall (i,S),(i,T)\in \cI_\cS .$
Since MNL is a special case of MDM (\cite{mishra2014theoretical}), $\bp_\cS$ is MDM-representable. 

We next show that any instance $\bp_\cS^\prime$ that lies in the ball centered at $\bp_\cS$ with a specified radius $\epsilon >0$ is an MDM instance. Let $0<\epsilon < \min_{ (k,S),(k,T)\in \cI_\cS} |p_{k,S}-p_{k,T}|$.  We perturb $\bp_\cS$ to be $\bp_\cS^\prime$ as follows: Arbitrarily choosing $(i,S),(j,S)\in \cI_\cS$, let 
\begin{align*}
p_{i,S}^\prime= p_{i,S} + \eps, \mbox{ and }  p_{j,S}^\prime= p_{j,S} - \eps,
\end{align*}
and keep other entries of $\bp_\cS^\prime$ the same as $\bp_\cS$ by setting
\begin{align*}
p_{k,T}^\prime= p_{k,T} , \ \forall (k,T) \in \cI_\cS, \mbox{ with }  (k,T) \neq (i,S) \mbox{ and }  (k,T) \neq (j,S).
\end{align*}
Then, $\bp_\cS^\prime$  lies in the ball centered at $\bp_\cS$ with the radius $\epsilon >0$ and we have $\sum_{i\in S } p^\prime_{i,S} = 1$, for all $S\in \cS$. Next, we show $\bp_\cS^\prime$ is MDM-representable by showing that $\bp_\cS^\prime$ satisfies the MDM-representable conditions \eqref{eq:mdm-feascon} in Theorem \ref{thm:feascon-mdm}.

Since $\bp_\cS$ is MDM-representable, we have 
\begin{align*}
 \lambda_S  > \lambda_T  \mbox{ if } p_{i,S} < p_{i,T} \; \forall (i,S),(i,T)\in \cI_\cS,\mbox{ and } \lambda_T  > \lambda_S  \mbox{ if } p_{j,T} < p_{j,S} \; \forall (j,S),(j,T)\in \cI_\cS.
\end{align*}
Since $ \epsilon < \min_{ (k,S),(k,T)\in \cI_\cS}  |p_{k,S}-p_{k,T}|,$ we have 
\begin{align*}
    \lambda_S  > \lambda_T  \mbox{ if } p_{i,S} + \epsilon < p_{i,T} \; \forall (i,S),(i,T)\in \cI_\cS, \mbox{ and }  \lambda_T  > \lambda_S  \mbox{ if } p_{j,T} < p_{j,S} - \epsilon \; \forall (j,S),(j,T)\in \cI_\cS.
\end{align*}
By the construction of $\bp_\cS^\prime$, equivalently, we have 
\begin{align*}
    \lambda_S  > \lambda_T  \mbox{ if } p_{i,S}^\prime< p_{i,T}^\prime\; \forall (i,S),(i,T)\in \cI_\cS, \mbox{ and }  \lambda_T  > \lambda_S  \mbox{ if } p_{j,T}^\prime< p_{j,S}^\prime\; \forall (j,S),(j,T)\in \cI_\cS.
\end{align*}
Thus, $\bp_\cS^\prime$ is  MDM-representable.

We next prove the choice probabilities represented by the MNL model and nested logit model
possess zero Lebesgue measure one by one as follows.
\begin{enumerate}[leftmargin=*]
    \item \textit{Measure zero of MNL.}
    To show the choice probabilities represented by MNL possess zero Lebesgue measure, it suffices to show that, for any $n\geq 3,$ there exists an assortment collection $\cS$ such that $\mu(\pmnl(\cS))=0$. Consider the nested assortment collection, $\cS=\{S_1,S_2,\cdots,S_n\}$ with $S_i = \{1,2,\cdots,i\}$, $\forall i\leq n.$ We have 
    \begin{align*}
        \pmnl(\cS) = \{ (x_{i,S}: i\in S, S\in \cS) | & x_{i,S} \geq 0, \forall i \in S,\forall S\in \cS, \, \sum_{i \in S} x_{i,S} = 1, \forall S\in \cS, \\
        &\frac{x_{i,S}}{x_{j,S}} = \frac{x_{i,T}}{x_{j,T}},\, \forall i,j\in S\cap T, \forall S,T \in \cS \}.
    \end{align*}
    We define the following set $B$ by reducing the constraints in $\pmnl(\cS)$:
    \begin{align*}
        \pmnl(\cS) \subseteq B:= \{ (x_{i,S}: i\in S, S\in \cS) | & x_{i,S} \geq 0, \forall i \in S,\forall S\in \cS, \, \sum_{i \in S} x_{i,S} = 1, \forall S\in \cS, \\
        &\frac{x_{1,S_{n-1}}}{x_{2,S_{n-1}}} = \frac{x_{1,S_n}}{x_{2,S_n}} \}.
    \end{align*}
    Consider $\prod_{S \in \cS}\Delta_S$ where $\cS=\{S_1,S_2,\cdots,S_n\}$ with $S_i = \{1,2,\cdots,i\}$, $\forall i\leq n,$ as a linear system with $\frac{n(n+1)}{2}$ nonnegative variables and $n$ constraints where the constraints defined by the probability simplexes. We can see that the constraints are linearly independent because of the nested structure of the collection. It follows that $\text{dim}(\prod_{S \in \cS}\Delta_S) = \frac{n(n+1)}{2} - n = \frac{n^2-n}{2}.$  $\text{dim}(B) = \frac{n(n+1)}{2} - n = \frac{n^2-n}{2} - 2$, since $\frac{x_{1,S_{n-1}}}{x_{2,S_{n-1}}} = \frac{x_{1,S_n}}{x_{2,S_n}}$ is a surface in 2-dimensional space. Since  $\text{dim}(\pmnl(\cS)) \leq \text{dim}(B) < \text{dim}(\prod_{S \in \cS}\Delta_S),$ $\mu\big(\pmnl(\cS)\big)=0$ on $\prod_{S \in \cS}\Delta_S$.

    \item \textit{Measure zero of nested logit.} 
    Consider the same nested assortment collection as above with $n\geq 3.$
    The assumption on the nested logit model is that $n$ alternatives are partitioned into nests. Within a nest, the IIA property holds.

    \begin{enumerate}[label=(\alph*),leftmargin=*]
        \item When the nests are singletons or when all products are in the same nest, the nested logit model reduces to MNL. $\pnl(\cS) = \pmnl(\cS)$. Thus, $\mu(\pnl(\cS))= \mu(\pmnl(\cS))=0$.
        \item When the number of nests is greater than 2 and less than $n$, there are at least 2 products in the same nest. This implies that there exists $i,j\in \cN$, such that $i,j$ are in the same nest. We have $i,j \in S_i \cap S_j$ because of the nested collection structure. We define the following set C by reducing the constraints in $\pnl(\cS)$:
        \begin{align*}
        \pnl(\cS) \subseteq C:= \{ (x_{i,S}: i\in S, S\in \cS) | & x_{i,S} \geq 0, \forall i \in S,\forall S\in \cS, \, \sum_{i \in S} x_{i,S} = 1, \forall S\in \cS, \\
        &\frac{x_{i,S_i}}{x_{j,S_i}} = \frac{x_{i,S_j}}{{x_{j,S_j}}} \}.
    \end{align*}
       We see that $\text{dim}(C) = \text{dim}(B)$. In this case, $\mu(\pnl(\cS))= 0$ on $\prod_{S \in \cS}\Delta_S$.
    
    \end{enumerate}
    To sum up, for all possible nests, $\mu(\pnl(\cS))= 0.$ There are at most $2^n$ nests. The countable union of zero measure sets has measure zero. Thus, $\mu(\pnl(\cS))= 0$ on $\prod_{S \in \cS}\Delta_S$.
    
\end{enumerate}}
\end{proof}

\subsection{Proof of Lemma \ref{Regularity}}
\label{pf:Regularity}
\begin{proof}
We first prove a) in Lemma \ref{Regularity} by contradiction. From Theorem \ref{thm:feascon-mdm}, if there exists some alternative $i$ such that $p_{i,S} > p_{i,S \cap T}$, then we have $\lambda_S<\lambda_{S \cap T}$, which is equivalent to $\lambda_S < \lambda_{S \cap T} $ for all $ j \in S \cap T$. This implies $\lambda_S \leq \lambda_{S \cap T} $ which gives $p_{j,S} \geq p_{j,S \cap T}$ for all $(j,S),(j,S\cap T)\in \cI_\cS$. Since $\sum_{j\in S} p_{j,S} = 1$, we get $\sum_{j\in S \cap T} p_{j,S \cap T} < 1$ contradicting the condition $\sum_{j\in S \cap T} p_{j,S \cap T} = 1$. For b), from Theorem \ref{thm:feascon-mdm}, if $i,j\in S\cap T$, we have
    $p_{j,S} < p_{j,T} \Rightarrow \lambda_S > \lambda_T \Rightarrow \lambda_S \geq  \lambda_T \Rightarrow    p_{i,S} \leq p_{i,T}.$
\end{proof}

\subsection{Proof of Theorem \ref{thm:mdm-rum-relation}}
\label{pf:mdm-rum-relation}

\begin{proof}
\textit{We prove $a)$ of Theorem \ref{thm:mdm-rum-relation} first.} We use the following notations for the rank list model since any RUM can be described by a rank list model \citeg{block1959random}. Let $\Sigma_n$ denote the set of all permutations of $n$ alternatives. Each element $\sigma\in \Sigma_n$ denotes a ranking of $n$ alternatives. For instance, $\sigma = \{1\succ 2\succ 3\}$ means alternative 1 is more preferred than alternative 2 which is more preferred than alternative 3. The probability of each ranking is $P(\sigma)$ and $\sum_{\sigma\in \Sigma_n} P(\sigma) = 1$. We prove the result case by case.
\begin{enumerate}[leftmargin=*]
    \item $n=2$: Here $\pmdm(\cS) = \prum(\cS)$. This is straightforward since all probabilities that satisfy $0 \leq p_{1,\{1,2\}} \leq  p_{1,\{1\}} =1$, and $ 0 \leq p_{2,\{1,2\}} \leq  p_{2,\{2\}} =1$ where $p_{1,\{1,2\}} +p_{2,\{1,2\}} =1,$
    are representable by both models.

    \item $n=3$: Lemma \ref{Regularity} implies that $\pmdm(\cS) \subseteq \preg(\cS)$. When $n=3$, all the possible assortments include $\{1,2\} , \{1,3\}, \{2,3\},\{1,2,3\}$. The regular model imposes the following regularity constraints:
    \begin{align*}
    & p_{1, \{1,2\} } \geq p_{1, \{1,2,3\} },\  p_{2, \{1,2\} } \geq p_{2, \{1,2,3\} },\\
    & p_{1, \{1,3\} } \geq p_{1, \{1,2,3\} },\ p_{3, \{1,3\} } \geq p_{3, \{1,2,3\} },\\
    & p_{2, \{2,3\} } \geq p_{2, \{1,2,3\} },\ p_{3, \{2,3\} } \geq p_{3, \{1,2,3\} }.
    \end{align*}
    $\prum(\cS) =  \preg(\cS)$ for any given $\cS$ since
    \begin{align*}
        P(\{1\succ 2\succ 3\}) = p_{2,\{2,3\}} - p_{2,\{1,2,3\}}\geq 0 \text{ and } P(\{1\succ 3\succ 2\}) = p_{3,\{2,3\}} - p_{3,\{1,2,3\}}\geq 0,\\
        P(\{2\succ 1\succ 3\}) = p_{1,\{1,3\}} - p_{1,\{1,2,3\}}\geq 0 \text{ and } P(\{2\succ 3\succ 1\}) = p_{3,\{1,3\}} - p_{3,\{1,2,3\}}\geq 0,\\
        P(\{3\succ 1\succ 2\}) = p_{1,\{1,2\}} - p_{1,\{1,2,3\}}\geq 0 \text{ and } P(\{3\succ 2\succ 1\}) = p_{2,\{1,2\}} - p_{2,\{1,2,3\}}\geq 0,
    \end{align*}
    where $ \sum_{\sigma\in \Sigma_n} P(\sigma) = 3 - 2 = 1$. 
    
    We next show that $\pmdm(\cS) \subset \prum(\cS)$ for $n = 3$ by giving an example of choice probabilities with $\cS = \{\{1,2,3\}, \{1,2\}, \{1,3\}, \{2,3\}\}$ in Table \ref{tab:notmdm-n=3} that can be represented by RUM but not by MDM.
\begin{table}[H]
\centering
\caption{Choice probabilities that cannot be represented by MDM for $n=3$.} \label{tab:notmdm-n=3}
\scalebox{0.8}{\begin{tabular}{|c|c|c|c|c|}
\hline
\begin{tabular}[c]{@{}c@{}}Alternative \end{tabular} & \multicolumn{1}{l|}{A=\{1,2,3\}} & \multicolumn{1}{l|}{B=\{1,2\}} & C=\{1,3\} & D=\{2,3\} \\ \hline
1                                                            & 1/3                              & 5/9                            & 4/9       &  -         \\ \hline
2                                                            & 1/3                              & 4/9                            &    -       & 5/9       \\ \hline
3                                                            & 1/3                              &     -                           & 5/9       & 4/9       \\ \hline
\end{tabular}}
\end{table}

This collection of choice probabilities $\bp_\cS$ cannot be represented by MDM because $p_{1,B} > p_{1,C}$, $p_{2,D} > p_{2,B}$, $p_{3,C}>p_{3,D}$ implies $\lambda_B<\lambda_C$, $\lambda_D < \lambda_B$ and $\lambda_C < \lambda_D$. This gives $\lambda_D < \lambda_B<\lambda_C<\lambda_D$ which is inconsistent. So, $\bp_\cS$ in Table \ref{tab:notmdm-n=3} cannot be represented by MDM. On the other hand, it is straightforward to check that by setting the ranking probabilities for RUM as follows: $P(\{1\succ 2\succ 3\}) = 2/9$, $P(\{1\succ 3\succ 2\}) = 1/9$, $P(\{2\succ 1\succ 3\}) = 1/9$, $P(\{2\succ 3\succ 1\}) = 2/9$, $P(\{3\succ 1\succ 2\}) = 2/9$, $P(\{3\succ 2\succ 1\}) = 1/9$,
we obtain the choice probabilities in Table \ref{tab:notmdm-n=3}. This implies $\bp_\cS$ in table \ref{tab:notmdm-n=3} can be represented by RUM but not MDM.
\item $n \geq 4$: We show $\pmdm(\cS) \not\subset \prum(\cS)$ and $\prum(\cS) \not\subset \pmdm(\cS)$ by providing two examples: (1) $\bp_\cS$ can be represented by RUM but not MDM and (2) $\bp_\cS$ can be represented by MDM but not RUM when $\cS = \{ \{1,2,3,4\}, \{1,2,3\}, \{1,2,4\}, \{1,2\} \}$. The examples are provided for $n=4$. For larger $n$, we can add the alternatives in the assortments and set the choice probabilities for these added alternatives to be zero. 

\textit{To show $\pmdm(\cS) \cap \prum(\cS) \neq \emptyset$:}  Firstly, we observe that the multinomial logit choice probabilities can be obtained from both RUM and MDM. This follows from using independent and identically distributed Gumbel distributions for the joint distribution of the random parts of utilities for RUM \citeg{akiva} and identical exponential distributions for the marginals of the random parts of utilities for MDM \citeg{mishra2014theoretical}. Hence, the intersection between the two sets is nonempty for any $n$. 

\textit{To show $\prum(\cS) \not\subset \pmdm(\cS)$:} Consider the choice probabilities in Table \ref{tab:n=4case2}.
\begin{table}[H]
\centering
\caption{Choice probabilities can be represented by RUM but not by MDM for $n = 4$.}\label{tab:n=4case2}
\scalebox{0.8}{\begin{tabular}{ccccc}
\hline
\multicolumn{1}{|c|}{\begin{tabular}[c]{@{}c@{}}Alternative\end{tabular}} & \multicolumn{1}{c|}{A=\{1,2,3,4\}} & \multicolumn{1}{c|}{B=\{1,2,3\}} & \multicolumn{1}{c|}{C=\{1,2,4\}} & \multicolumn{1}{c|}{D=\{1,2\}} \\ \hline
\multicolumn{1}{|c|}{1}                                                                & \multicolumn{1}{c|}{3/20}           & \multicolumn{1}{c|}{7/20}         & \multicolumn{1}{c|}{1/4}         & \multicolumn{1}{c|}{1/2}      \\ \hline
\multicolumn{1}{|c|}{2}                                                                & \multicolumn{1}{c|}{3/20}           & \multicolumn{1}{c|}{1/4 }         & \multicolumn{1}{c|}{7/20}         & \multicolumn{1}{c|}{1/2}      \\ \hline
\multicolumn{1}{|c|}{3}                                                                & \multicolumn{1}{c|}{7/20}           & \multicolumn{1}{c|}{2/5}        & \multicolumn{1}{c|}{-}            & \multicolumn{1}{c|}{-}          \\ \hline
\multicolumn{1}{|c|}{4}                                                                & \multicolumn{1}{c|}{7/20}           & \multicolumn{1}{c|}{-}            & \multicolumn{1}{c|}{2/5}        & \multicolumn{1}{c|}{-}          \\ \hline
\end{tabular}}
\end{table}

This can be recreated by RUM using the distribution over the ranking as follows:
\vspace{-0.4cm}
\begin{table}[H]
\centering
\begin{tabular}{lll}
$P(\{1\succ 2 \succ 3 \succ4\})=1/40$ & $P(\{1\succ 2\succ 4\succ 3\})=1/40$ & $P(\{1\succ 3\succ 2 \succ4\})=1/40$\\
$P(\{1\succ 3 \succ4 \succ 2\})=1/40$ & $P(\{1 \succ4\succ 2\succ 3\})=1/40$ & $P(\{1 \succ4\succ 3\succ 2\})=1/40$\\
$P(\{2 \succ1\succ 3\succ 4\})=1/40$ & $P(\{2\succ 1\succ 4\succ 3\})=1/40$ & $P(\{2\succ 3\succ 1\succ 4\})=1/40$\\
$P(\{2\succ 3 \succ4 \succ1\})=1/40$ & $P(\{2\succ 4\succ 1 \succ3 \})=1/40$ & $P(\{2\succ 4 \succ 3\succ 1\})=1/40$\\
$P(\{3\succ 1\succ 2\succ 4\})=1/20$ & $P(\{3\succ 1\succ 4\succ 2\})=1/20$ & $P(\{3\succ 2\succ 1\succ 4\})=1/10$ \\
$P(\{3\succ 2 \succ4\succ 1\})=1/10$ & $P(\{3\succ 4\succ 1\succ 2\})=1/40$ &  $P(\{3 \succ4 \succ2 \succ 1\})=1/40$ \\
$P(\{4\succ 1\succ 3\succ 2\})=1/10$ & $P(\{4\succ 1\succ 3\succ 2\})=1/10$ & $P(\{4\succ 2\succ  1\succ 3\})=1/20$ \\
$P(\{4\succ 2\succ 3 \succ1\})=1/20$ & $P(\{4\succ 3\succ 1\succ 2\})=1/40$ & $P(\{4\succ 3\succ 2\succ 1\})=1/40$
\end{tabular}
\end{table}
\vspace{-0.4cm}
Now $p_{1,B} > p_{1,C}$ implies $\lambda_B < \lambda_C$ and $p_{2,B} > p_{2,C}$ implies $\lambda_B > \lambda_C$. Hence $\bp_\cS$ in Table \ref{tab:n=4case2} is not representable by MDM. 

\textit{To show $\pmdm(\cS) \not\subset \prum(\cS)$: } Next consider the choice probabilities in Table \ref{tab:n=4case3}.
\begin{table}[H]
\centering
\caption{Choice probabilities can be represented by MDM but not by RUM for $n = 4$.}\label{tab:n=4case3}
\scalebox{0.8}{\begin{tabular}{ccccc}
\hline
\multicolumn{1}{|c|}{\begin{tabular}[c]{@{}c@{}}Alternative\end{tabular}} & \multicolumn{1}{c|}{A=\{1,2,3,4\}} & \multicolumn{1}{c|}{B=\{1,2,3\}} & \multicolumn{1}{c|}{C=\{1,2,4\}} & \multicolumn{1}{c|}{D=\{1,2\}} \\ \hline
\multicolumn{1}{|c|}{1}                                                                & \multicolumn{1}{c|}{0.1}           & \multicolumn{1}{c|}{0.2}         & \multicolumn{1}{c|}{0.2}         & \multicolumn{1}{c|}{0.25}      \\ \hline
\multicolumn{1}{|c|}{2}                                                                & \multicolumn{1}{c|}{0.2}           & \multicolumn{1}{c|}{0.25}         & \multicolumn{1}{c|}{0.25}         & \multicolumn{1}{c|}{0.75}      \\ \hline
\multicolumn{1}{|c|}{3}                                                                & \multicolumn{1}{c|}{0.2}           & \multicolumn{1}{c|}{0.55}        & \multicolumn{1}{c|}{-}            & \multicolumn{1}{c|}{-}          \\ \hline
\multicolumn{1}{|c|}{4}                                                                & \multicolumn{1}{c|}{0.5}           & \multicolumn{1}{c|}{-}            & \multicolumn{1}{c|}{0.55}        & \multicolumn{1}{c|}{-}          \\ \hline
\end{tabular}}
\end{table}
Here $p_{1,A} < p_{1,B} = p_{1,C} < p_{1,D}$ implies $\lambda_A > \lambda_B = \lambda_C > \lambda_D$, and $p_{2,A} < p_{2,B} = p_{2,C} < p_{2,D}$ implies $\lambda_A > \lambda_B = \lambda_C > \lambda_D$, and $p_{3,A} < p_{3,B} $ implies $\lambda_A > \lambda_B$, and $p_{4,A} < p_{4,C} $ implies $\lambda_A > \lambda_C$. So we have $\lambda_A > \lambda_B = \lambda_C > \lambda_D$ which is easy to enforce and so $\bp_\cS$ can be represented by MDM. A necessary condition for $\bp_\cS$ to be representable by RUM are the Block-Marshak conditions provided in \citet{block1959random} (also see Theorem 1 in \citealt{fiorini2004short}). If the choice probabilities are representable by RUM, one of these conditions is given by $p_{1,A} + p_{1,D} \geq p_{1,B} + p_{1,C}$. Here $p_{1,A} + p_{1,D} = 0.1 + 0.25 = 0.35 < 0.4 = 0.2 + 0.2 = p_{1,B} + p_{1,C}$. So, $\bp_\cS$ is not representable by RUM.
\end{enumerate}

\textit{We prove $b)$ of Theorem \ref{thm:mdm-rum-relation} as follows.} We know that $\pmdm(\cS) = \prum(\cS) = \preg(\cS)$ when $n=2$ and $\prum(\cS) \subseteq \preg(\cS)$ and $\textnormal{closure}\big(\pmdm(\cS)\big) \subseteq \preg(\cS)$ for any $\cS$. To show $b)$, we just need to show $\preg(\cS) \subseteq \prum(\cS)$ and $\preg(\cS) \subseteq \textnormal{closure}\big(\pmdm(\cS)\big)$ when $\cS$ is nested or laminar.  It suffices to show $\preg(\cS)^\prime \subseteq \prum(\cS)$ and $\preg(\cS)^\prime \subseteq \textnormal{closure}\big(\pmdm(\cS)\big)$ when $\cS$ is nested or laminar, since $\preg(\cS) \subseteq \preg(\cS)^\prime$. 

 $\preg(\cS) \subseteq \textnormal{closure}\big(\pmdm(\cS)\big)$ with nested or laminar $\cS$:  Under a nested collection $\cS = \{S_1, S_2, \ldots, S_m\}$ with $S_1 \subset S_2 \subset \ldots \subset S_m$, we have
$$\preg(\cS)^\prime = \Big\{\bx \in \mathbb{R}^{\cI_\cS} : x_{i,S} \geq 0,\forall (i,S)\in \cI_\cS,\, \sum_{i\in S} x_{i,S} = 1,  \forall S \in \cS,  \, x_{i,S_k} \leq x_{i,S_j} \; \forall j,k\in[m], j<k,i\in S_i   \Big\},$$
where $[m]$ denotes $\{1,2,\ldots,m\}.$
Under a laminar collection, we have 
$$\preg(\cS)^\prime = \Big\{\bx \in \mathbb{R}^{\cI_\cS} : \, x_{i,S} \geq 0,\forall (i,S)\in \cI_\cS,\,\sum_{i} x_{i,S} = 1,  \forall S \in \cS,  \, x_{i,S} \leq x_{i,T} \; \forall T \subset S, (i,S),(i,T) \in \cI_\cS  \Big\}.$$

\textit{$\preg(\cS)^\prime \subseteq \textnormal{closure}\big(\pmdm(\cS)\big)$ with a nested or laminar $\cS$:} 
We show that for any $\bp_\cS \in \preg(\cS)^\prime$, we can construct $\blam_\cS$ such that $(\bp_\cS, 
\blambda_\cS) \in \Pi_\cS^{'},$ where  
\begin{align*}
    \Pi_\cS^{'} := \Big\{ (\bx,\boldsymbol{\lambda}) \in \mathbb{R}^{\cI_\cS} \times  \mathbb{R}^{ \cS}:  \, & x_{i,S} \geq 0,\forall (i,S)\in \cI_\cS,\, \sum_{i\in S} x_{i,S} = 1,  \forall S \in \cS,  \; \\
    &\lambda_S \geq \lambda_T \text{ if } x_{i,S} \leq x_{i,T}, \forall (i,S), (i,T) \in \cI_\cS \Big\}.    
\end{align*}

Suppose that $\cS = \{S_1, S_2,\cdots,S_m\}$ is a nested collection with $S_1\subset S_2\subset \ldots \subset S_m.$
Then we take any $\blambda_\cS$ satisfying $\lambda_{S_1} \leq \lambda_{S_2} \leq \cdots \leq \lambda_{S_m}.$   Since $p_{i,S_j} \geq p_{i,S_k}$ for any $j \leq k$ (due to $\bp_\cS \in \preg(\cS)^\prime$), the resulting 
 $(\bp_\cS, 
  \blambda_\cS) \in \Pi_\cS^{'}.$  As a result,  $\bp_\cS \in \textnormal{closure}\big(\pmdm(\cS)\big)$.

 If $\cS$ is laminar, we take any $\blambda_\cS$ such that $\lambda_{S} \leq \lambda_{T}$ if $S, T\in\cS$ with $S\subset T.$ Since $p_{i,S} \geq p_{i,T}$ due to the regularity $\bp_\cS \in \preg(\cS)^\prime,$ the resulting $(\bp_\cS, 
  \blambda_\cS) \in \Pi_\cS^{'}.$ Hence $\bp_\cS \in \textnormal{closure}\big(\pmdm(\cS)\big)$ and $\preg(\cS) \subseteq \preg(\cS) ^\prime \subseteq \textnormal{closure}\big(\pmdm(\cS)\big)$ with a nested or laminar $\cS.$

\textit{$\preg(\cS)^\prime \subseteq \prum(\cS)$ with a nested or laminar $\cS$:} 
We next show that, for any $\bp_\cS \in \preg(\cS)^\prime,$ there exists a probability distribution $(P(\sigma):\sigma \in \Sigma_n)$ such that $\bp_\cS \in \prum(\cS).$
\begin{enumerate}[label=(\arabic*),leftmargin=*]
\item For a nested collection $\cS$, we prove $\preg(\cS)^\prime \subseteq \prum(\cS).$ Without loss of generality, let $\cS = \{S_1,S_2,\ldots,S_m\}$  be $S_k = \{1,\ldots, k\}$ for $k=1,\cdots, m.$ Next, for any $\bp_\cS \in \preg(\cS)^\prime,$ we prove the existence of a probability distribution $(P(\sigma):\sigma \in \Sigma_m)$ such that $\bp_\cS \in \prum(\cS)$ from the point of view of polyhedral combinations. To show the existence of a probability distribution $(P(\sigma):\sigma \in \Sigma_m)$ for $\bp_\cS \in \prum(\cS)$ is equivalent to showing $\bp_\cS$ lies in the multiple choice polytope characterized as,
$$\text{convex hull of }\{ (\mathbb{I}[\sigma,i,S]:i\in S,S\in \cS)\in \{0,1\}^{\sum_{S\in\cS } |S|}:\sigma \in \Sigma_m \},$$
where $\mathbb{I}[\sigma,i,S]=1$ if and only if $i = \argmin_{j\in S} \sigma(j)$ 
(see Section 3 in \citealt{fiorini2004short} and Lemma 2.5 of \citealt{jagabathula2019limit}). 

Now, we show that $\bp_\cS$ lies in the multiple choice polytope via a graph representation of the multiple choice polytope following the steps in Section 3 in \cite{fiorini2004short}. Let $\boldsymbol{D} = (\cN_0,A)$ be a simple, acyclic directed graph, and let $m+1$ be the source node and $0$ be the sink node of $\boldsymbol{D},$ where $\cN_0 = \{1,\cdots,m\} \cup \{m+1, 0\}.$ We encode each $m+1 - 0$ directed path $\Pi$ of its arc set $A$ in $\boldsymbol{D}$ by means of the indicator characteristic vector in the set $\{ (\mathbb{I}[\sigma,i,S]:i\in S,S\in \cS)\in \{0,1\}^{\sum_{S\in\cS |S|}}:\sigma \in \Sigma_m \} \in \mathbb{R}^A$ , which we denote $r^\Pi.$ The convex hull of the vectors $r^\Pi,$ for a $m+1 - 0$ directed path $\Pi$ in $\boldsymbol{D}$, is referred to as the  $m+1 - 0$ directed path polytope of $\boldsymbol{D}.$ For a node $v$ of $\boldsymbol{D}$, let $\delta^-(v) = \{(w,v):w\in \cN_0, (w,v)\in A\}$ represent the nodes incoming to node $v$ and $\delta^+(v) = \{(v,w):w\in \cN_0, (v,w)\in A\}$ represent the nodes outgoing from node $v$. For $B\subseteq A,$ let $r(B) = \sum \{r(v,w):(w,v)\in B\}.$ Let $M$ be the matrix whose rows are indexed by modes of $\boldsymbol{D}$ such that the entry corresponding to node $v$ and arc $a$ equals to $1$ if $a$ enters $v$, and $-1$ if $a$ leaves $0$, and $0$ else. It's well known that $M$ is totally unimodular \citep{schrijver1998theory}. This implies that the polyhedron $\{r\in \mathbb{R}^A: Mr=d, r\geq 0\}$ has all its vertices integer for every integral vector $d\in \mathbb{R}^A.$ Assume that $\delta(m+1)^- = \delta(0)^+ = \emptyset.$ 
\begin{lemma}[Theorem 2 in \citealt{fiorini2004short}]\label{lemma:flow_rum}
A point $r\in \mathbb{R}^A$ belongs to the $m+1 - 0$ directed path polytope $\boldsymbol{D}$ if and only if 
\begin{align}
    & r(\delta^-(v)) - r(\delta^+(v)) = 0, \quad \forall  v\in \cN_0\setminus \{m+1,0\},\label{eq:flow_conversation}\\
    & r(\delta^-(0)) = 1,\label{eq:flow1}\\
    & r(v,w)\geq 0, \quad \forall (w,v)\in A.\label{eq:flow_postive}
\end{align}
In network flows, \eqref{eq:flow_conversation}-\eqref{eq:flow_postive} define a flow of value 1  in the network $\boldsymbol{D} = (\cN_0,A),$ with source node $m+1$ and sink node $0.$
\end{lemma}
By Lemma \ref{lemma:flow_rum}, to show $\bp_\cS$ lies in the multiple choice polytope, we need to show $(r(w,v):(w,v)\in A)$ based on $\bp_\cS$ satisfying $\eqref{eq:flow_conversation}-\eqref{eq:flow_postive}$ in Lemma \ref{lemma:flow_rum}. We demonstrate $(r(w,v):(w,v)\in A)$ under $\bp_\cS$ as follows:
\begin{align*}
     r(m+1,j) & = \sum_{\sigma \in \Sigma_{m}:\argmin_{ i \in S_m } \sigma(i ) = j } P(\sigma) =  p_{j,S_m},\quad \;\forall\, j = 1,\cdots,m,\\
    r(i,j) & = \sum_{\sigma \in \Sigma_{m}:\argmin_{ k \in S_{i-1} } \sigma(k) = j } P(\sigma) - \sum_{\sigma \in \Sigma_{m}:\argmin_{ k \in S_i } \sigma(k) = j } P(\sigma) \\
    & =  p_{j,S_{i-1}} - p_{j,S_i}, \quad \forall\, i = 2,\cdots,m, \, 1\leq j \leq i-1,\\
     r(1,0) &=1.
\end{align*}
We provide Figure \ref{fig:flow_construction} to illustrate $(r(w,v):(w,v)\in A)$ in the graph $\boldsymbol{D}$. 
\begin{figure}[H]
    \centering
    \includegraphics[scale=0.23]{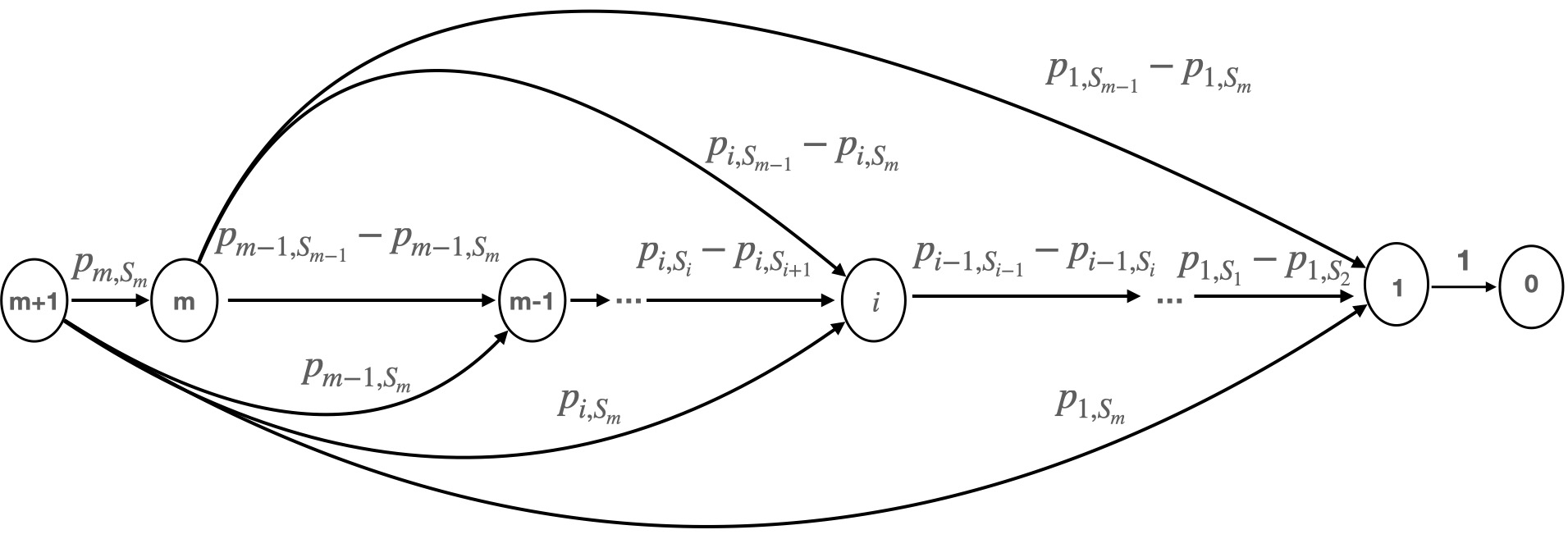}
    \caption{An illustration of $\boldsymbol{D}$ given $\bp_\cS \in \preg(\cS)^\prime$ with a nested collection $\cS$}
\label{fig:flow_construction}
\end{figure}
Next, we verify such $(r(w,v):(w,v)\in A)$ satisfies \eqref{eq:flow_conversation}-\eqref{eq:flow_postive}. For \eqref{eq:flow_conversation}, define $p_{j,S_{m+1}} = 0$ for all $j.$ For a node $j\in \{2,\cdots,m\},$ 
    \begin{align*}
        r(\delta^-(j)) &= \sum_{k=j+1}^{m+1} r(k,j) \\
        &= \sum_{k=j}^m p_{j,S_k} - p_{j,S_{k+1}} = p_{j,S_j} - p_{j,S_{j+1}} +  p_{j,S_{j+1}} - p_{j,S_{j+2}} + \cdots + p_{j,S_m} - p_{j,S_{m+1}}\\
        &=  p_{j,S_j} - p_{j,S_{m+1}} 
        = p_{j,S_j}.\\
        r(\delta^+(j)) &= \sum_{k=1}^{j-1} p_{k,S_j-1} - p_{k,S_j}\\
        & = p_{1,S_{j-1}} - p_{1,S_j} + p_{2,S_{j-1}} - p_{2,S_j} +\cdots + p_{j-1,S_{j-1}} - p_{j-1,S_j}\\
        & = \sum_{k=1}^{j-1} p_{k,S_{j-1}} - \sum_{k=1}^{j-1} p_{k,S_j} = 1 - \sum_{k=1}^{j-1} p_{k,S_j} = p_{j,S_j}.
    \end{align*}
    For the node $1,$
    \begin{align*}
        r(\delta^-(1)) &= \sum_{k=2}^{m+1} r(k,1) \\
        &= p_{1,S_1} - p_{1,S_2} + p_{1,S_2} - p_{1,S_3} + \cdots + p_{1,S_m} - p_{1,S_{m+1}} \\
        & = p_{1,S_1} - p_{1,S_{m+1}} = 1-0=1 = r(\delta^+(1)).
    \end{align*}
Therefore, $r(\delta^-(j)) = r(\delta^+(j)) $ for $j\in \{1,2,\cdots,m\}.$ \eqref{eq:flow_conversation} is satisfied by $(r(w,v):(w,v)\in A).$ For \eqref{eq:flow1}, $r(\delta^-(0)) = r(1,0) =  1.$ For \eqref{eq:flow_postive}, $ r(m+1,j) = p_{j,S_m}\geq 0,\; \forall\, j = 1,\cdots,m$ because of the nonnegativity of choice probabilities.  We have $ r(i,j) = p_{j,S_{i-1}} - p_{j,S_i}\geq 0,\; \forall\, i = 2,\cdots,m, \, 1\leq j \leq i-1$ since $p_S$ satisfies regularity,then, $\bp_\cS \in \preg(\cS)^\prime.$ Further  $r(1,0) =1 >0.$ 

The assignment $(r(w,v):(w,v)\in A)$ satisfies \eqref{eq:flow_conversation}-\eqref{eq:flow_postive} in Lemma \ref{lemma:flow_rum}. Therefore, 
for any $\bp_\cS \in \preg(\cS)^\prime,$ there exists a probability distribution $(P(\sigma):\sigma \in \Sigma_m)$ such that $\bp_\cS \in \prum(\cS).$ This implies $\preg(\cS) \subseteq \preg(\cS)^\prime \subseteq \prum(\cS)$ under the nested collection $\cS.$

\item We prove $\preg(\cS)^\prime \subseteq \prum(\cS)$ under the laminar collection. By the definition of the laminar collection, we know that for $S, T \in \cS$, either $S\subset T$, or $T\subset S$, or $S\cap T = \emptyset$. Then, it suffices to construct a distribution $P(\cdot)$ for $\hat{\cS} \subset \cS$ such that $\hat{\cS}$ is a nested collection. Following the proof in (1), we have $\preg(\cS) \subseteq \preg(\cS)^\prime \subseteq \prum(\cS)$ under a laminar collection.
\end{enumerate}
\end{proof}

{\color{black}\subsection{Proof of Proposition \ref{prop:general_iia}}
\label{pf:general_iia}
\begin{proof}
\textit{To show the "if" direction:} Since $\bp_\cS$ is represented by MDM and  $p_{i,S} >0 $ for all $(i,S) \in  \cI_\cS$, for each pair of assortment $S,T\in \cS$ and alternatives $i\in S\cap T$, there exists marginals $F_i,F_j$, Lagrange multipliers $\lambda_S$ and $\lambda_T$, $\lambda_{i,S}$ and $\lambda_{i,T}$ and  the followings are satisfied:
\begin{align*}
    & \nu_i+F_i^{-1}(1-p_{i,S})-\lambda_S +\lambda_{i,S} = 0, \   \lambda_{i,S} \ p_{i,S}=0,\mbox{~ and ~} \lambda_{i,S}\geq 0, \ \forall i\in S , \mbox{~ and ~}\\
    & \nu_i+F_i^{-1}(1-p_{i,T})-\lambda_S +\lambda_{i,T} = 0, \   \lambda_{i,T} \ p_{i,T}=0,\mbox{~ and ~} \lambda_{i,T}\geq 0, \ \forall i\in T .
\end{align*}
Since $p_{i,S}>0$ and $p_{i,T}>0$, we have $\lambda_{i,S} = \lambda_{i,T} = 0.$ We then have 
\begin{align*}
    & p_{i,S} = 1-F_i(\lambda_S - \nu_i) \mbox{~and~}  p_{j,S} = 1-F_j(\lambda_S - \nu_j), \mbox{~ and, ~}\\
    & p_{i,T} = 1-F_i(\lambda_T - \nu_i) \mbox{~and~}  p_{j,T} = 1-F_j(\lambda_T - \nu_j).  
\end{align*}
This leads to the following:
\begin{align*}
    \nu_j - \nu_i = F_i^{-1}(1-p_{i,S}) - F_j^{-1}(1-p_{j,S}) = F_i^{-1}(1-p_{i,T}) - F_j^{-1}(1-p_{j,T}).
\end{align*}
Then,
\begin{align*}
    \text{exp}(v_j-v_i) = \frac{\text{exp}(F_i^{-1}(1-p_{i,S}))}{\text{exp}(F_j^{-1}(1-p_{j,S}))} = \frac{\text{exp}(F_i^{-1}(1-p_{i,T}))}{\text{exp}(F_j^{-1}(1-p_{j,T}))}.
\end{align*}
Setting the strictly decreasing functions $f_i: [0,1)\to \mathbb{R}_+$ by $f_i(0)=\infty$ and $f_i(p) = \text{exp}[F_i^{-1}(1-p)]$, for all $p >0$,$\forall i\in \cN,$ we obtain:
\begin{align*}
    \text{exp}(v_j-v_i) = \frac{f_i(p_{i,S})}{f_j(p_{j,S})} = \frac{f_i(p_{i,T})}{f_j(p_{j,T})}
\end{align*}

\textit{To show the "only if" direction:}. Given $\bp_\cS$ and $\{f_i:i\in \cN\}$ satisfies Generalized Ordinal IIA, we construct the deterministic utilities $\bv = \{\nu_i: i\in \cN\}$ and the marginal distributions $\{F_i: i\in \cN\}$ for MDM such that this construction yields the given $(p_{i,S}:i\in S)$ as the corresponding choice probabilities from the KKT optimality conditions in \eqref{optimality}, for any assortment $S\in \cS$.

Let $F_i (x) = \log^{-1} (f_i(1-x))$, equivalently $F_i^{-1}(1-x) = \log(f_i(x)).$ For any fixed product $i$, set $\nu_i := 0$. For any other products such that $j\neq i$, set $\nu_j := F_i^{-1}(1-p_{i,\{i,j\}}) - F_j^{-1}(1-p_{j,\{i,j\}}).$

Take arbitrary assortment $S \in S$, and pick any two distinct products $j,k \in S.$ There are two exclusive cases.

\textit{Case 1:} $i\in \{j,k\}$. Without loss of generality, set $j=i.$ Then,
\begin{align*}
    \nu_k - \nu_i 
    &= F_i^{-1}(1-p_{i,\{i,k\}}) - F_k^{-1}(1-p_{k,\{i,k\}})\\
    &= \log(\frac{f_i(p_{i,\{i,k\}})}{f_k(p_{k,\{i,k\}})  } )\\
    & = \log(\frac{f_i(p_{i,S})}{f_k(p_{k,S})  } ) \tag{\text{Generalized Ordinal IIA}}\\
    & = F_i^{-1}(1-p_{i,S}) - F_k^{-1}(1-p_{k,S}).
\end{align*}

We next verify
$$\nu_i+F_i^{-1}(1-p_{i,S}^*)-\lambda_S +\lambda_{i,S} = 0, \   \lambda_{i,S} p_{i,S}=0,\mbox{~ and ~} \lambda_{i,S}\geq 0, \ \forall i\in S .$$

Since $p_{i,S}>0$ for all $ i\in S,$  we have $ \lambda_{i,S} =0$ in the KKT conditions. Then, there exists $\lambda_S$ such that  $\nu_k + F_k^{-1}(1-p_{k,S}) -\lambda_S = 0  $ and $\nu_i + F_i^{-1}(1-p_{i,S}) -\lambda_S = 0  $ since we have proved that $\nu_k + F_k^{-1}(1-p_{k,S}) = \nu_i + F_i^{-1}(1-p_{i,S}).$

\textit{Case 2:} $i\notin \{j,k\}$. Then,
\begin{align*}
    \nu_k - \nu_j 
    &= F_i^{-1}(1-p_{i,\{i,k\}}) - F_k^{-1}(1-p_{k,\{i,k\}}) + F_j^{-1}(1-p_{j,\{i,j\}}) -  F_i^{-1}(1-p_{i,\{i,j\}})  \\
    &= \log(\frac{f_i(p_{i,\{i,k\}})}{f_k(p_{k,\{i,k\}})} \frac{f_j(p_{j,\{i,j\}})}{f_i(p_{i,\{i,j\}})} )\\
    &= \log(\frac{f_i(p_{i,\{i,j,k\}})}{f_k(p_{k,\{i,j,k\}})} \frac{f_j(p_{j,\{i,j,k\}})}{f_i(p_{i,\{i,j,k\}})} ) \tag{\text{Generalized Ordinal IIA}} \\
    & = \log(\frac{f_i(p_{i,\{i,j,k\}})}{f_k(p_{k,\{i,j,k\}}) } ) \\
    & = \log(\frac{f_j(p_{j,S})}{f_k(p_{k,S}) } ) \tag{\text{Generalized Ordinal IIA}} \\
    & = F_j^{-1}(1-p_{j,S}) - F_k^{-1}(1-p_{k,S}).
\end{align*}

We next verify
$$\nu_i+F_i^{-1}(1-p_{i,S})-\lambda_S +\lambda_{i,S} = 0, \   \lambda_{i,S} p_{i,S}=0,\mbox{~ and ~} \lambda_{i,S}\geq 0, \ \forall i\in S .$$

Since $p_{i,S}>0$ for all $ i\in S,$  we have $ \lambda_{i,S} =0$ in the KKT conditions. Then, there exists $\lambda_S$ such that  $\nu_k + F_k^{-1}(1-p_{k,S}) -\lambda_S = 0  $ and $\nu_j + F_j^{-1}(1-p_{j,S}) -\lambda_S = 0  $ since we have proved that $\nu_k + F_k^{-1}(1-p_{k,S}) = \nu_j + F_j^{-1}(1-p_{j,S}).$

Therefore, the equalities in the above two cases imply the KKT optimality conditions at $S$ are satisfied. Since this holds for any assortment $S\in \cS$, this proves that $\bp_\cS$ is represented by MDM.

\end{proof}

}

\section{Proofs of the Results in Section \ref{sec:prediction-mdm}}
\label{sec:pf_sec3}
\subsection{Proof of Proposition \ref{prop:wc-rev-reform}}
\label{pf:wc-rev-reform}
\begin{proof}
  Due to Theorem \ref{thm:feascon-mdm}, we have  $\pmdm(\cS^\prime) = \text{Proj}_{\bx}(\Pi_{\cS^\prime}),$ where the lifted set $\Pi _{\cS^\prime}$ equals 
\begin{align*}
  \Big\{ (\bx,\boldsymbol{\lambda}) \in \mathbb{R}^{\mathcal{I}_{\cS^\prime}} \times  \mathbb{R}^{ \cS^\prime} :&  \, x_{i,S} \geq 0, \ \sum_{i} x_{i,S} = 1,  \forall S \in \cS^\prime,   \nonumber \\ 
    &\lambda_S > \lambda_T \text{ if } x_{i,S} < x_{i,T}, \  \lambda_S = \lambda_T \text{ if } x_{i,S} = x_{i,T} \neq 0, \ \forall (i,S), (i,T) \in \cI_{\cS^\prime} 
    \Big\}, 
\end{align*}
following the definition in  \eqref{eq:lifted-set}. Since   $\mathcal{U}_A  := \{ \bx_A: (\bx_\cS, \bx_A, \boldsymbol{\lambda}) \in \Pi_{\cS^\prime}, \ \bx_\cS = \bp_\cS\},$ the non-numbered constraints in the formulation in Proposition \ref{prop:wc-rev-reform} are obtained by replacing $\bx_\cS = \bp_\cS$ in the above description of the lifted set $\Pi_{\cS^\prime}.$ For deducing the  remaining constraints  \eqref{model:mdm_revenue_prediction-a} 
- \eqref{model:mdm_revenue_prediction-b}, 
we proceed as follows: Consider any $(i,S) \in \cI_\cS.$  From the description of $\Pi_{\cS^\prime},$ 
observe that an assignment for $x_{i,A}, \lambda_A,  \lambda_S$ in any $(\bx,\blambda) \in \Pi_\cS^\prime$  satisfying $\bx_S = \bp_S$ necessarily satisfies one of the following four cases: 

In \textit{Case 1,} we have $\lambda_A > \lambda_S$ and  $x_{i,A} < p_{i,S}$: If $\lambda_A,\lambda_S$ is such that  $\lambda_A > \lambda_S,$ this informs the restriction $\{x_{i,A}: x_{i,A} < p_{i,S}\}$ on the values $x_{i,A}$ can take. The closure of this  restricted collection $\{x_{i,A}: x_{i,A} < p_{i,S}\}$ equals $\{x_{i,A}: x_{i,A} \leq p_{i,S}\}.$ 

In \textit{Case 2,} we have $\lambda_A < \lambda_S$ and  $x_{i,A} > p_{i,S}$: 
 If $\lambda_A,\lambda_S$ is such that  $\lambda_A < \lambda_S,$ the closure of the corresponding restriction $\{x_{i,A}: x_{i,A} > p_{i,S}\}$ equals $\{x_{i,A}: x_{i,A} \geq p_{i,S}\}.$ 

In \textit{Case 3,} we have $\lambda_A =  \lambda_S$ and  $x_{i,A} = p_{i,S} \neq 0$: When $\lambda_A,\lambda_S$ is such that  $\lambda_A = \lambda_S$ and $p_{i,S} \neq 0,$ the corresponding restriction on the values of $x_{i,A}$ is given by the closed set  $\{x_{i,A}: x_{i,A} = p_{i,S}\}.$ 

Finally, in \textit{Case 4,} we have $\lambda_A,  \lambda_S$ unconstrained and  $x_{i,A} = p_{i,S} = 0$: Like in Case 3, the  restriction on the values of $x_{i,A}$ corresponding to this case equals  $\{x_{i,A}: x_{i,A} = 0\}.$ The relationship between $x_{i,A}, p_{i,S}, \lambda_A,\lambda_S$ in this case is any one of the following sub-cases: Case (4a) $\lambda_A > \lambda_S$ and $0 = x_{i,A} \leq p_{i,S} = 0,$ or Case (4b) $\lambda_A < \lambda_S$ and $0 = x_{i,A} \geq p_{i,S} = 0,$ or Case (4c) $\lambda_A = \lambda_S$ and $0 = x_{i,A} = p_{i,S} = 0.$ 

Combining the observations in the cases (1) \& (4a), (2) \& (4b), and (3) \& (4c), we obtain that the closure of $\mathcal{U}_A$ equals the collection of probability vectors $\bx_A = (x_{i,A}: i \in A)$ for which there exists a function $\lambda: \cS^\prime \rightarrow \R$ such that 
\begin{align*}
   &x_{i,A} \ \leq \  p_{i,S} \quad\text{ if } \lambda_A >  \lambda_S, \quad \forall i \in A, \ (i,S) \in \cI_\cS,\\
    &x_{i,A} \ \geq \  p_{i,S} \quad \text{ if } \lambda_A < \lambda_S, \quad \forall i \in A, \ (i,S) \in \cI_\cS,  \text{ and }  \\ 
    &x_{i,A} \ = \  p_{i,S} \quad\text{ if } \lambda_A = \lambda_S, \quad \forall i \in A, \ (i,S) \in \cI_\cS,
\end{align*}
in addition to satisfying $\lambda_S > \lambda_T$ if $p_{i,S} < p_{i,T}$ and $\lambda_S = \lambda_T$ if $p_{i,S} = p_{i,T} \neq 0,$ for all $(i,S),(i,T) \in \cI_\cS.$ The constraints in the formulation in Proposition \ref{prop:wc-rev-reform} exactly specify these conditions describing the closure of $\mathcal{U}_A.$ Observe that the objective in $\inf\{\sum_{i \in A} r_i x_{i,A}: \bx_A \in \mathcal{U}_A\}$ is continuous as a function of $\bx_A.$ Therefore, $\inf\{\sum_{i \in A} r_i x_{i,A}: \bx_A \in \mathcal{U}_A\} = \min\{\sum_{i \in A} r_i x_{i,A}: \bx_A \in \text{closure}(\mathcal{U}_A)\}.$
\end{proof}

\subsection{Proof of Proposition \ref{prop:wc-rev-milp}}
\label{pf:wc-rev-milp}
\begin{proof}
Recall the notation $\cS^\prime = \cS \cup \{A\}.$ 
Observe that the variables $(\lambda_S: S \in \cS^\prime)$ influence the value of the formulation in Proposition \ref{prop:wc-rev-reform} only via the sign of $\lambda_S - \lambda_T,$ for any pair of variables $\lambda_S,\lambda_T$ from the collection  $(\lambda_S: S \in \cS^\prime).$ Therefore the optimal value of this optimization formulation is not affected by the presence of the following additional constraints: 
 $0 \leq \lambda_S \leq 1$ for all $S \in \cS^\prime,$ and 
\begin{align*}
    \lambda_S - \lambda_T  \geq \epsilon\quad \text{ if } \quad \lambda_S > \lambda_T, \qquad \forall\, (i,S),(i,T) \in \cI_{\cS^\prime},
\end{align*} 
for some suitably small value of $\epsilon > 0.$ Indeed, this is because the signs of the differences $\{ \lambda_S - \lambda_T: S, T \in \cS^\prime\}$ are not affected by these additional constraints. Taking  $\epsilon$ to be smaller than $1/(2\vert \cS \vert),$ for example, ensures that there is a feasible assignment for  $(\lambda_S: S \in \cS^\prime)$ within the interval $[0,1]$ even if all these variables take distinct values. 

Let $F$ denote the feasible values for the variables  $(\lambda_S: S \in \cS^\prime), (x_{i,A}: i \in A)$ satisfying the constraints introduced in the above paragraph besides  those in the formulation in Proposition \ref{prop:wc-rev-reform}. 
Equipped with this feasible region $F$, we have the following deductions from \eqref{mdm:milp-a} - \eqref{mdm:milp-c} for $(\lambda_S: S \in \cS^\prime), (x_{i,A}: i \in A)$ in $F$: For every $i \in A$ and any $S \in \cS$ containing $i,$   
    \begin{itemize}[leftmargin=*]
        \item[(i)] we have $\lambda_A < \lambda_S$ if and only if  $\delta_{A,S} = 1$ and $\delta_{S,A} = 0,$ due to the constraints \eqref{mdm:milp-a} and \eqref{mdm:milp-b}; in this case, we have from \eqref{mdm:milp-c} that $p_{i,S} \leq x_{i,A} \leq 1;$
         \item[(ii)] likewise, we have $\lambda_A > \lambda_S$ if and only if  $\delta_{A,S} = 0$ and $\delta_{S,A} = 1,$ due to the constraints \eqref{mdm:milp-a} and \eqref{mdm:milp-b}; in this case, we have from \eqref{mdm:milp-c} that $0 \leq x_{i,A} \leq p_{i,S}.$
         \item[(iii)] finally, $\lambda_A = \lambda_S$ if and only if  $\delta_{A,S} = 0$ and $\delta_{S,A} = 1;$ here we have from \eqref{mdm:milp-d} that $x_{i,A} = p_{i,S}.$
    \end{itemize} 
    Thus the binary variables $\{\delta_{A,S},\delta_{S,A}: S \in \cS\}$ suitably model the  constraints collection \eqref{model:mdm_revenue_prediction-a} - \eqref{model:mdm_revenue_prediction-b} and provide an equivalent reformulation in terms of the constraints \eqref{mdm:milp-a} - \eqref{mdm:milp-d}. Therefore the optimal value of the formulations in Propositions \ref{prop:wc-rev-reform} and \ref{prop:wc-rev-milp} are identical.  
\end{proof}

\subsection{Proof of Corollary \ref{cor:prediction_structure}}
\label{pf:prediction_structure}
\begin{proof}
When $\cS^\prime$ is either nested or laminar, from the proof in Theorem \ref{thm:mdm-rum-relation}, we know that $\pmdm(\cS^\prime) = \preg(\cS^\prime)^\prime.$ So, we can solve the worst-case expected revenue in \eqref{eq:rob-revenues} with the representable conditions of the regular model which are $x_{i,A} \leq p_{i,S}$ if $S \subset A$ and $x_{i,A} \geq p_{i,S}$ if $A \subset S$ for all $i\in A$ and $(i,S)\in \cI_\cS.$ This is a linear program with $\mathcal{O}(n)$ continuous variables and $\mathcal{O}(n|\cS|)$ constraints.
\end{proof}
{\color{black}
\subsection{Proof of Proposition  \ref{prop:prediction_common_product}}
\label{pf:prediction_common_product}
\begin{proof}
Suppose that there exists a product $i^{*}$ being included in every assortment of the collection $\cS$ and we have $p_{i^{*},S_1} \geq p_{i^{*},S_2} \geq \cdots \geq p_{i^{*},S_m}.$
Recall that  $\textnormal{closure}\big(\pmdm(\cS)\big) = \text{Proj}_{\bx}(\Pi_\cS^{'})$ where $\Pi_\cS^{'}$ is
defined as 
\begin{align}
\begin{aligned}
    \Pi_\cS^{'} = \Big\{ (\bx,\boldsymbol{\lambda}) \in \mathbb{R}^{\cI_\cS} \times  \mathbb{R}^{ \cS}:  \, & x_{i,S} \geq 0,\forall (i,S)\in \cI_\cS,\, \sum_{i} x_{i,S} = 1,  \forall S \in \cS,  \; \\
    &\lambda_S \geq \lambda_T \text{ if } x_{i,S} \leq x_{i,T}, \forall (i,S), (i,T) \in \cI_\cS \Big\}. 
\end{aligned}
\end{align}
Thus, we have  $\lambda_{S_m} \geq \cdots \lambda_{S_k}\geq \lambda_{S_{k-1}} \geq \cdots \geq \lambda_{S_1}$ for any $\blambda$ such that $(\bx,\blambda) \in \Pi_{\cS}^\prime.$ For ease of notation, let $\lambda_{S_0}=-\infty$ and $\lambda_{S_{m+1}} = +\infty.$ Then for any given $A$, the corresponding $\lambda_A$ must satisfy $\lambda_{S_{k+1}} \geq \lambda_A \geq \lambda_{S_{k}}$ for some $k \in \{0,1,\ldots,m\}.$ 

From the viewpoint of $\lambda_{A} \geq \lambda_{S_k} \geq \ldots \geq \lambda_{S_0},$ we deduce the following constraints on $x_{i,A}$: For any $i\in A \cap S_k,$ we have the respective MDM 
feasibility constraints $x_{i,A} \leq p_{i,S_k} \leq \ldots \leq p_{i,S_{j}}$  for all $j \leq k$ such that $i \in S_j.$ These constraints can be equivalently summarized by $x_{i,A} \leq p_{i,S_{k}},$ and this comprises the first set of constraints for evaluating $\boldsymbol{R_k}$ in \eqref{model:prediction_common_product}. 

From the viewpoint of $\lambda_{S_m+1} \geq \cdots \lambda_{S_{k+1}}\geq \lambda_A,$ we deduce the following constraints on $x_{i,A}:$ For any $ i \in A, j \geq k+1, S_j \in \cS, i\in S_j\cap A ,$ we have the respective MDM feasibility constraint  $x_{i,A} \geq p_{i,S},$ which comprise the second set of constraints for evaluating $\boldsymbol{R_k}$ in \eqref{model:prediction_common_product}.
\end{proof} }

\subsection{Proof of Corollary \ref{cor:prediction_nest}}
\label{pf:prediction_nest}
\begin{proof}
Suppose $\cS = \{S_1,S_2,\ldots,S_m\}$ is nested as in $S_1\subset S_2\subset \cdots \subset S_m$. We have, from the regularity of MDM in Lemma \ref{Regularity}, that $p_{i,S_1} \geq p_{i,S_2} \geq \cdots p_{i,S_m},$ for all $i \in S_1.$ Recall that  $\textnormal{closure}\big(\pmdm(\cS)\big) = \text{Proj}_{\bx}(\Pi_\cS^{'})$ where $\Pi_\cS^{'}$ is
defined as 
\begin{align}
\begin{aligned}
    \Pi_\cS^{'} = \Big\{ (\bx,\boldsymbol{\lambda}) \in \mathbb{R}^{\cI_\cS} \times  \mathbb{R}^{ \cS}:  \, & x_{i,S} \geq 0,\forall (i,S)\in \cI_\cS,\, \sum_{i \in S} x_{i,S} = 1,  \forall S \in \cS,  \; \\
    &\lambda_S \geq \lambda_T \text{ if } x_{i,S} \leq x_{i,T}, \forall (i,S), (i,T) \in \cI_\cS \Big\}. \label{mdm_cl}    
\end{aligned}
\end{align}
Thus, for the given nested $\cS$, we have  $\lambda_{S_m} \geq \cdots \lambda_{S_k}\geq \lambda_{S_{k-1}} \geq \cdots \geq \lambda_{S_1}$ for any $\blambda$ such that $(\bx,\blambda) \in \Pi_{\cS}^\prime.$ For ease of notation, let $\lambda_{S_0}=-\infty$ and $\lambda_{S_{m+1}} = +\infty.$ Then for any given $A$, the corresponding $\lambda_A$ must satisfy $\lambda_{S_{k+1}} \geq \lambda_A \geq \lambda_{S_{k}}$ for some $k \in \{0,1,\ldots,m\}.$ 

From the viewpoint of $\lambda_{A} \geq \lambda_{S_k} \geq \ldots \geq \lambda_{S_0},$ we deduce the following constraints on $x_{i,A}$: For any $i\in A \cap S_k,$ we have the respective MDM 
feasibility constraints $x_{i,A} \leq p_{i,S_k} \leq \ldots \leq p_{i,S_{j}}$  for all $j \leq k$ such that $i \in S_j.$ These constraints can be equivalently summarized by $x_{i,A} \leq p_{i,S_{k}},$ and this comprises the first set of constraints for evaluating $\boldsymbol{R_k}$ in \eqref{model:prediction_nested}. 

From the viewpoint of $\lambda_{S_m+1} \geq \cdots \lambda_{S_{k+1}}\geq \lambda_A,$ we deduce the following constraints on $x_{i,A}:$ For any $i \in A, (i,S)\in \cI_\cS, S_{k+1} \subseteq S,$ we have the respective MDM feasibility constraint  $x_{i,A} \geq p_{i,S},$ which comprise the second set of constraints for evaluating $\boldsymbol{R_k}$ in \eqref{model:prediction_nested}.
\end{proof}

\section{Proofs of the Results in Section \ref{sec:lom}}
\label{sec:pf_limit}
\subsection{Proof of Proposition \ref{prop:limit-mdm}}
\label{pf:limit-mdm}
\begin{proof}
Due to Theorem \ref{thm:feascon-mdm}, we have  $\pmdm(\cS) = \text{Proj}_{\bx}(\Pi_\cS),$ following the definition in  \eqref{eq:lifted-set}. One can argue the closure of $\pmdm(\cS)$ by the similar arguments of the proof of Proposition \ref{prop:wc-rev-reform}. Consider any $(i,S) \in \cI_\cS.$  From the description of $\Pi_{\cS},$ observe that an assignment for $x_{i,S}, x_{i,T}, \lambda_S,  \lambda_T$ in any $(\bx,\blambda) \in \Pi_\cS$ necessarily satisfies one of the following four cases: 

In \textit{Case 1,} we have $\lambda_S < \lambda_T$ and  $x_{i,S} > x_{i,T}$: If $\lambda_S,\lambda_T$ is such that  $\lambda_S < \lambda_T,$ the closure of the corresponding restriction $\{x_{i,S}, x_{i,T}: x_{i,S} > x_{i,T}\}$ equals $\{x_{i,S}, x_{i,T}: x_{i,S} \geq x_{i,T}\}.$
 
In \textit{Case 2,}  we have $\lambda_S > \lambda_T$ and  $x_{i,S} < x_{i,T}$: If $\lambda_S,\lambda_T$ is such that  $\lambda_S > \lambda_T,$ the closure of the corresponding restriction $\{x_{i,S}, x_{i,T}: x_{i,S} < x_{i,T}\}$ equals $\{x_{i,S}, x_{i,T}: x_{i,S} \leq x_{i,T}\}.$
 
In \textit{Case 3,} we have $\lambda_S =  \lambda_T$ and  $x_{i,S} = x_{i,T} > 0$: When $\lambda_S,\lambda_T$ is such that  $\lambda_S = \lambda_T$ and $x_{i,S} = x_{i,T} > 0,$ the corresponding restriction on the values of $x_{i,S}, x_{i,T}$ is given by the closed set $\{x_{i,S}, x_{i,T}: x_{i,S} = x_{i,T}>0\}.$

Finally, in \textit{Case 4,} we have $\lambda_S, \lambda_T$ unconstrained and $x_{i,S} = x_{i,T} = 0$: Like in Case 3, the  restriction on the values of $x_{i,S}, x_{i,T}$ corresponding to this case equals $x_{i,S} = x_{i,T} = 0.$ The relationship between $x_{i,S}, x_{i,T}, \lambda_S,\lambda_T$ in this case is any one of the following sub-cases: Case (4a) $\lambda_S > \lambda_T$ and $0 = x_{i,S} \leq x_{i,S} = 0,$ or Case (4b) $\lambda_S < \lambda_T$ and $0 = x_{i,S} \geq x_{i,T} = 0,$ or Case (4c) $\lambda_S = \lambda_T$ and $0 = x_{i,S} = x_{i,T} = 0.$ 

Combining the observations in the cases (1) \& (4a), (2) \& (4b), and (3) \& (4c), we obtain that the closure of $\pmdm(\cS)$ equals the collection of probability vectors $\bx$ for which there exists a function $\lambda: \cS \rightarrow \R$ such that 
\begin{align*}
   &x_{i,S} \ \leq \  x_{i,T} \quad\text{ if } \quad \lambda_S >  \lambda_T, \quad \forall (i,S), \ (i,T) \in \cI_\cS,\\
    &x_{i,S} \ \geq \  x_{i,T} \quad \text{ if } \quad \lambda_S < \lambda_T, \quad \forall (i,S), \ (i,T) \in \cI_\cS, \\ 
    &x_{i,S} \ = \  x_{i,T} \quad \text{ if } \quad \lambda_S = \lambda_T, \quad \forall (i,S), \ (i,T) \in \cI_\cS.
\end{align*}
The constraints in the formulation in Proposition \ref{prop:limit-mdm} exactly specify these conditions describing the closure of $\pmdm(\cS).$ Observe that the objective in $\inf \{ \text{loss}(\bp_\cS,\bx_\cS)\,:\, \bx_{\cS} \in \pmdm(\cS) \}$ is continuous as a function of $\bx.$ Therefore, $\inf \{ \text{loss}(\bp_\cS,\bx_\cS)\,:\, \bx_{\cS} \in \pmdm(\cS) \} = \min \{ \text{loss}(\bp_\cS,\bx_\cS)\,:\, \bx_{\cS} \in \text{closure}(\pmdm(\cS))\}.$
\end{proof}

\subsection{Proof of Theorem \ref{thm:NP}}
\label{pf:np}
Before we formally prove Theorem \ref{thm:NP}, we first show the following preparatory material. Problem \eqref{model:lom} can be solved by the following optimization problem:
\begin{align}
\inf_{\blam} \quad & \text{loss}(\bx_{\cS}^*(\blam), \bp_{\cS})  \label{mdm-sequential} \\
\text{s.t.}\quad & \bx_{\cS}^*(\blam) \in \text{arg inf}_{\bx_{\cS}(\blam):(\bx_{\cS}(\blam),\blam) \in \Pi_\cS} \quad \text{loss} (\bx_{\cS}(\blam), \bp_{\cS}| \blam ),\label{mdm_sub}
\end{align}
where $\bx_{\cS}(\blam)$ can be interpreted as a collection of MDM-representable choice probabilities given $\cS$ and $\blam$. Next, we focus on the sub-problem \eqref{mdm_sub}. Let $[m] = \{1,2,\cdots,m\}$ and $[k] = \{1,2,\cdots,k\}$ for some positive integer $m$ and $k.$
\begin{assumption}
Consider $\cS=\{S_1, S_2, \cdots, S_m\}$ and $\bp_\cS$ as a $n \times m$ matrix with $n$ rows and $m$ columns, satisfying the following properties:
\begin{itemize}[leftmargin=*,nolistsep]
    \item $n = k + m$ with $k$ as a positive integer;
    \item For each $l \in [k]$, there are exactly two elements in row $l$ of $\bp_\cS$ and $\kappa = |p_{l,S_i} - p_{l,S_j}|<\frac{1}{2m}$ is a positive constant for $i,j\in [m]$ with $l\in S_i\cap S_j$;
    \item For each $i \in [m]$, there is exactly one element $p_{k+i,S_i}$ in row $k+i$ and  $p_{k+i, S_i} = 1- \sum_{j=1}^k p_{j,S_i}$ and $p_{k+i,S_i} > \frac{km\kappa}{2}$.
\end{itemize}
\label{asp:specialp}
\end{assumption}

\begin{lemma}
Given $\bp_\cS$ that satisfies Assumption \ref{asp:specialp}, the sub-problem \eqref{mdm_sub} with 1-norm objective function has a closed-form optimal objective value $ \sum_{l=1}^k \sum\limits_{i,j \in [m], \lambda_{S_i}  \leq \lambda_{S_j}: l \in S_i  \cap S_j} 2 |p_{l,S_i}- p_{l,S_j}| \mathbb{I}\{( p_{l,S_i}- p_{l,S_j}) (\lambda_{S_i} - \lambda_{S_j})  \geq 0\}$. \label{lem:closed_form}
\end{lemma}
Intuitively, $\bp_\cS$ has exactly one product-assortment pair for product $l$ with $l\in [k]$ and exactly one element for product $k+i$ with $i\in[m]$ which corresponds to the $i$th assortment in the collection.  For product $l$ with $l\in [k]$ the indicator takes the value 1 if and only if the choice probabilities of the product-assortment pair violate the MDM-representable conditions, the minimum loss to make this pair to be MDM-representable under 1-norm loss is $|p_{l,S_i} - p_{l,S_j}|$, which causes the violation of the normalization constraint of the assortments $S_i$ and $S_j$. To satisfy the normalization conditions of assortment $S_i$ and $S_j$, the least loss is also $|p_{l,S_i} - p_{l,S_j}|$. Next, we formally prove Lemma \ref{lem:closed_form} with three steps: (1) reformulate the sub-problem \eqref{mdm_sub} to a linear program; (2) construct a primal feasible solution $\bx(\blam)$ such that the desired optimal objective value is achieved, which can be served as an upper bound to \eqref{mdm_sub}; (3) derive the dual for the problem and construct a dual feasible solution such that the desired optimal objective value is achieved. 
\begin{proof}
Step (1): Formulate the sub-problem \eqref{mdm_sub} as a linear program with $\bp_\cS$ satisfying Assumption \ref{asp:specialp}. \\
Let $f_\text{sub}^*$ denote the optimal value of \eqref{mdm_sub}. Given $\bp_\cS$ that satisfies Assumption \ref{asp:specialp} , we reformulate the sub-problem \eqref{mdm_sub} as the following problem:
\begin{align}
\begin{aligned}
f_\text{sub}^*(\blam)\,=\,\min_{\bx_\cS}\quad & \sum_{l=1}^k \sum\limits_{i,j \in [m], \lambda_{S_i}  \leq \lambda_{S_j}: l \in S_i  \cap S_j} (|x_{l,S_i} - p_{l,S_i}| + |x_{l,S_j} - p_{l,S_j}|) + \sum\limits_{i=1}^m |x_{k+i,S_i} - p_{k+i,S_i} | \\
\text{s.t.} \quad & x_{l,S_i} - x_{l,S_j}  \geq 0 \text{ if } \lambda_{S_i} \leq \lambda_{S_j}   \quad \forall l\in [k],  i,j,\in[m],\, l \in S_i\cap S_j,\\
 &  \sum_{i \in S}x_{i,S} = 1,\quad \forall S \in \cS,\\
 &  x_{i,S} \geq 0, \quad \forall (i,S) \in \cI_\cS.
 \label{subproblem-abs} 
\end{aligned}
\end{align}

We reformulate Problem \eqref{subproblem-abs} as the following linear program \eqref{subproblem-lp} by introducing a new variable $\boldsymbol{z}_\cS.$
\begin{align}\label{subproblem-lp}
\min_{\bx_\cS,\boldsymbol{z}_\cS }\quad & \sum_{l=1}^k \sum\limits_{i,j \in [m], \lambda_{S_i}  \leq \lambda_{S_j}: l \in S_i  \cap S_j}
(z_{l,S_i} + z_{l,S_j}) + \sum_{i = 1}^m z_{k+i,S_i} \nonumber \\
\text{s.t.} \quad 
& z_{l,S_i} - x_{l,S_i} \geq -p_{l,S_i},\quad z_{l,S_j} - x_{l,S_j} \geq -p_{l,S_j},\quad \forall l\in[k],i,j\in [m],  \lambda_{S_i}  \leq \lambda_{S_j}:  l\in S_i\cap S_j,\nonumber \\
& z_{l,S_i} + x_{l,S_i} \geq p_{l,S_i},\quad z_{l,S_j} + x_{l,S_j} \geq p_{l,S_j},\quad \forall l\in[k],i,j\in [m],  \lambda_{S_i}  \leq \lambda_{S_j}:  l\in S_i\cap S_j,\nonumber \\
& x_{l,S_i} - x_{l,S_j}  \geq 0 \text{ if } \lambda_{S_i} \leq \lambda_{S_j}  \quad \forall l\in[k],i,j\in [m], l\in S_i\cap S_j,\\  
& z_{k+i,S_i} - x_{k+i,S_i} \geq -p_{k+i,S_i}, \quad \forall i\in [m],\nonumber \\
& z_{k+i,S_i}  + x_{k+i,S_i} \geq p_{k+i,S_i},\quad \forall i \in [m], \nonumber \\
&  \sum_{l \in S_i}x_{l,S_i} = 1,\quad \forall i \in [m], \nonumber \\
&  x_{i,S},\,z_{i,S} \geq 0, \quad \forall (i,S) \in \cI_\cS.\nonumber
\end{align}

Step (2): Construct a primal feasible solution to achieve $ f_\text{sub}^*(\blam)$. \\
Given any $\blam$, construct a solution $(\bx_\cS,\boldsymbol{z}_\cS)$ as follows.
\begin{itemize}
    \item For $l \in [k]$,$i,j\in[m]$ with $\lambda_{S_i} \leq \lambda_{S_j}$ such that $l\in S_i\cap S_j$:\\
    If $\mathbb{I}\{(p_{l,S_i} - p_{l,S_j})(\lambda_{S_i} - \lambda_{S_j}) \geq  0\} = 0$, let
    $$x_{l,S_i} = p_{l, S_i},~x_{l,S_j} = p_{l, S_j},~ z_{l,S_i} = z_{l,S_j} = 0;$$
    If $\mathbb{I}\{(p_{l,S_i} - p_{l,S_j})(\lambda_{S_i} - \lambda_{S_j}) \geq  0\} = 1$, let
    $$x_{l,S_i} =  x_{l,S_j} =  \frac{p_{l,S_i} + p_{l,S_j}}{2},~z_{l,S_i} =  z_{l,S_j} = \frac{|p_{l,S_i}-p_{l,S_j}|}{2}.$$
    \item For $i\in [m]$, let
\begin{align*}
    & x_{k+i,S_i} = p_{k+i,S_i} + \sum_{l=1}^k \sum\limits_{i,j \in [m]: l \in S_i  \cap S_j} \mathbb{I}\{(p_{l,S_i} - p_{l,S_j})(\lambda_{S_i} - \lambda_{S_j}) \geq  0\} \frac{\text{sgn}(p_{l,S_i}-p_{l,S_j})|p_{l,S_i}-p_{l,S_j}| }{2}, \\
    & z_{k+i,S_i} = \sum\limits_{l=1,i,j,\in[m],l\in S_i\cap S_j}^k \mathbb{I}\{(p_{l,S_i} - p_{l,S_j})(\lambda_{S_i} - \lambda_{S_j}) \geq  0\} \frac{|p_{l,S_i}-p_{l,S_j}| }{2},
\end{align*}
where $\text{sgn}(x):=\begin{cases}
1 \quad \text{if}\quad x>0,\\
-1 \quad \text{if}\quad x\leq 0.
\end{cases}$
\end{itemize}
{\color{black}We verify the feasibility of the constructed primal solution for \eqref{subproblem-lp} as follows:
Firstly, we have all the nonnegative constraints for $(x_{i,S},\,z_{i,S} : \forall (i,S) \in \cI_\cS)$ are satisfied.  For $l \in [k]$,$i,j\in[m]$ with $\lambda_{S_i} \leq \lambda_{S_j}$ such that $l\in S_i\cap S_j$:\\
If $\mathbb{I}\{(p_{l,S_i} - p_{l,S_j})(\lambda_{S_i} - \lambda_{S_j}) \geq  0\} = 0$, we have :
\begin{align*}
& z_{l,S_i} - x_{l,S_i}  = 0 - p_{l,S_i} = -p_{l,S_i},\  z_{l,S_j} - x_{l,S_j} = 0 - p_{l,S_j} = - p_{l,S_j}, \\ 
& z_{l,S_i} + x_{l,S_i} =  0 + p_{l,S_i} = p_{l,S_i}, \  z_{l,S_j} + x_{l,S_j} =  0 + p_{l,S_j} =  p_{l,S_j},  \\
& x_{l,S_i} - x_{l,S_j} = p_{l,S_i} - p_{l,S_j} \geq 0 \text{ if } \lambda_{S_i} \leq \lambda_{S_j}.
\end{align*}
If $\mathbb{I}\{(p_{l,S_i} - p_{l,S_j})(\lambda_{S_i} - \lambda_{S_j}) \geq  0\} = 1$, we have :
\begin{align*}
& z_{l,S_i} - x_{l,S_i}  = \frac{|p_{l,S_i}-p_{l,S_j}|}{2} - \frac{p_{l,S_i} + p_{l,S_j}}{2} \geq \frac{ p_{l,S_j} - p_{l,S_i} }{2} - \frac{p_{l,S_i} + p_{l,S_j}}{2} = - p_{l,S_i} ,\\
& z_{l,S_j} - x_{l,S_j} = \frac{|p_{l,S_i}-p_{l,S_j}|}{2} - \frac{p_{l,S_i} + p_{l,S_j}}{2} \geq \frac{ p_{l,S_i} - p_{l,S_j} }{2} - \frac{p_{l,S_i} + p_{l,S_j}}{2} = - p_{l,S_j}, \\ 
& z_{l,S_i} + x_{l,S_i} =  \frac{|p_{l,S_i}-p_{l,S_j}|}{2} + \frac{p_{l,S_i} + p_{l,S_j}}{2} \geq \frac{ p_{l,S_i} - p_{l,S_j} }{2} + \frac{p_{l,S_i} + p_{l,S_j}}{2} = p_{l,S_i}, \\  
& z_{l,S_j} + x_{l,S_j} =  \frac{|p_{l,S_i}-p_{l,S_j}|}{2} - \frac{p_{l,S_i} + p_{l,S_j}}{2} \geq \frac{ p_{l,S_j} - p_{l,S_i} }{2} + \frac{p_{l,S_i} + p_{l,S_j}}{2} =  p_{l,S_j},  \\
& x_{l,S_i} - x_{l,S_j} = \frac{p_{l,S_i} + p_{l,S_j}}{2} - \frac{p_{l,S_i} + p_{l,S_j}}{2} = 0 \text{ if } \lambda_{S_i} \leq \lambda_{S_j}.
\end{align*}
For all $i\in [m]$, we also have 
\begin{align*}
& z_{k+i,S_i} - x_{k+i,S_i} \\
= & \sum\limits_{l=1,i,j,\in[m],l\in S_i\cap S_j}^k \mathbb{I}\{(p_{l,S_i} - p_{l,S_j})(\lambda_{S_i} - \lambda_{S_j}) \geq  0\} \frac{|p_{l,S_i}-p_{l,S_j}| }{2}\\
&-( p_{k+i,S_i} + \sum_{l=1}^k \sum\limits_{i,j \in [m], \lambda_{S_i}  \leq \lambda_{S_j}: l \in S_i  \cap S_j} \mathbb{I}\{(p_{l,S_i} - p_{l,S_j})(\lambda_{S_i} - \lambda_{S_j}) \geq  0\} \frac{\text{sgn}(p_{l,S_i}-p_{l,S_j})|p_{l,S_i}-p_{l,S_j}| }{2} )\\
 & \geq -p_{k+i,S_i}, \\
& z_{k+i,S_i}  + x_{k+i,S_i} \\
= & \sum\limits_{l=1,i,j,\in[m],l\in S_i\cap S_j}^k \mathbb{I}\{(p_{l,S_i} - p_{l,S_j})(\lambda_{S_i} - \lambda_{S_j}) \geq  0\} \frac{|p_{l,S_i}-p_{l,S_j}| }{2} \\
& + p_{k+i,S_i}  + \sum_{l=1}^k \sum\limits_{i,j \in [m], \lambda_{S_i}  \leq \lambda_{S_j}: l \in S_i  \cap S_j} \mathbb{I}\{(p_{l,S_i} - p_{l,S_j})(\lambda_{S_i} - \lambda_{S_j}) \geq  0\} \frac{\text{sgn}(p_{l,S_i}-p_{l,S_j})|p_{l,S_i}-p_{l,S_j}| }{2}\\
 & \geq p_{k+i,S_i}.  
\end{align*}
Both inequalities hold due to  $$\frac{|p_{l,S_i}-p_{l,S_j}| }{2} \geq \frac{\text{sgn}(p_{l,S_i}-p_{l,S_j})|p_{l,S_i}-p_{l,S_j}| }{2}.$$
For the normalization constraints, for all $i\in [m]$, we have:
\begin{align*}
    & \sum_{l \in S_i} x_{l,S_i} = \sum_{l=1}^k x_{l,S_i} + x_{k+i,S_i}\\
    =& \sum_{l=1}^k \mathbb{I}\{(p_{l,S_i} - p_{l,S_j})(\lambda_{S_i} - \lambda_{S_j}) <  0\} p_{l,S_i}   + \sum_{l=1}^k \mathbb{I}\{(p_{l,S_i} - p_{l,S_j})(\lambda_{S_i} - \lambda_{S_j}) \geq  0\} \frac{p_{l,S_i} + p_{l,S_j}}{2}  \\
    &+ p_{k+i,S_i} + \sum_{l=1}^k \sum\limits_{i,j \in [m]: l \in S_i  \cap S_j} \mathbb{I}\{(p_{l,S_i} - p_{l,S_j})(\lambda_{S_i} - \lambda_{S_j}) \geq  0\} \frac{\text{sgn}(p_{l,S_i}-p_{l,S_j})|p_{l,S_i}-p_{l,S_j}| }{2}\\
    =& \sum_{l \in S_i} p_{l,S_i} = 1.
\end{align*}}
Then, the objective value of Problem \eqref{subproblem-lp} under the constructed solution $(\bx_\cS, \boldsymbol{z}_\cS)$ is 
$$\sum_{l=1}^k \sum\limits_{i,j \in [m], \lambda_{S_i}  \leq \lambda_{S_j}: l \in S_i  \cap S_j} 2 |p_{l,S_i}- p_{l,S_j}| \mathbb{I}\{( p_{l,S_i}- p_{l,S_j}) (\lambda_{S_i} - \lambda_{S_j})  \geq 0\}.$$ 
This implies that
$$ f_\text{sub}^*(\blam) \leq \sum_{l=1}^k \sum\limits_{i,j \in [m], \lambda_{S_i}  \leq \lambda_{S_j}: l \in S_i  \cap S_j} 2 |p_{l,S_i}- p_{l,S_j}| \mathbb{I}\{( p_{l,S_i}- p_{l,S_j}) (\lambda_{S_i} - \lambda_{S_j})  \geq 0\}.$$

Step (3): Construct a dual feasible solution to achieve $ f_\text{sub}^*(\blam)$.\\
We derive the dual of \eqref{subproblem-lp} as follows. For $l \in [k]$,$i,j\in[m]$ with  $\lambda_{S_i}  \leq \lambda_{S_j}$ such that $l\in S_i\cap S_j$, we introduce the following variables: $\alpha_{l,i}, \beta_{l,i},\alpha_{l,j}, \beta_{l,j}, u_{l,i,j} \geq 0 $. For $i\in [m]$, we introduce the following variables: $\alpha_{k+i,S_i}, \beta_{k+i,i} \geq 0,$ and $\eta_{i}$. The dual problem of \eqref{subproblem-lp} is given as:
\begin{align}\label{subproblem_dual}
\max_{\boldsymbol{\alpha}, \boldsymbol{\beta}, \boldsymbol{u}, \boldsymbol{\eta} } & \sum_{l=1}^k  \sum\limits_{i,j\in[m],  \lambda_{S_i} \leq \lambda_{S_j},  l\in S_i \cap S_j}  [ p_{l,S_i}(\beta_{l,i} - \alpha_{l,i}) + p_{l,j}(\beta_{l,j} - \alpha_{l,j})]  + \sum\limits_{i=1}^m p_{k+i,S_i}(\beta_{k+i,i}-\alpha_{k+i,i}) - \sum\limits_{i=1}^m    \eta_i  \nonumber \\
\text{s.t.}\quad 
& \alpha_{l,i} + \beta_{l,i} \leq 1,\quad \alpha_{l,j} + \beta_{l,j} \leq 1, \quad \forall l\in [k],i,j\in[m],  \lambda_{S_i} \leq \lambda_{S_j}, \  l\in S_i \cap S_j,\nonumber \\
& -\alpha_{l,i} + \beta_{l,i} + u_{l,i,j} +\eta_i   \leq 0,  \   -\alpha_{l,j} + \beta_{l,j} - u_{l,i,j}   + \eta_j  \leq 0, \  \forall l\in [k],i,j\in[m],  \lambda_{S_i} \leq \lambda_{S_j},   l\in S_i \cap S_j,\nonumber \\
& \alpha_{k+i,i} + \beta_{k+i,i}  \leq 1, \quad \forall i\in[m], \nonumber \\
& -\alpha_{k+i,i} + \beta_{k+i,i} + \eta_i \leq 0, \quad \forall i\in[m], \\
& \alpha_{l,i}, \beta_{l,i},\alpha_{l,j}, \beta_{l,j}, u_{l,i,j} \geq 0,\quad  \forall l\in [k],i,j\in[m],  \lambda_{S_i} \leq \lambda_{S_j},   l\in S_i \cap S_j,  \nonumber \\
& \alpha_{k+i,i}, \beta_{k+i,i} \geq 0, \quad \forall i\in[m]. \nonumber
\end{align}
{\color{black}Construct a dual solution for \eqref{subproblem_dual} as follows:\\
For $ l\in [k],i,j\in[m], $ with $\lambda_i \leq \lambda_j$ and $l\in S_i \cap S_j,$ 
\begin{itemize}
    \item If $\mathbb{I}\{(p_{l,S_i} - p_{l,S_j})(\lambda_{S_i} - \lambda_{S_j}) \geq  0\} = 0$, let $\alpha_{l,i} = \beta_{l,i} = \alpha_{l,j} = \beta_{l,j} = u_{l,i,j} = 0.$
    \item If $\mathbb{I}\{(p_{l,S_i} - p_{l,S_j})(\lambda_{S_i} - \lambda_{S_j}) \geq  0\} = 1$, let $\alpha_{l,i} =1,~ \beta_{l,i} = 0,~ \alpha_{l,j} = 0,~ \beta_{l,j} =1,~ u_{l,i,j} = 1.$
\end{itemize}
For $i\in[m],$ let $\alpha_{k+i,i} = \beta_{k+i,i} = \frac{1}{2}$ and $\eta_i = - \sum_{l=1}^k \sum\limits_{j\in[m]: l \in S_i \cap S_j} \frac{|p_{l,S_i}- p_{l,S_j}| \mathbb{I}\{( p_{l,S_i}- p_{l,S_j}) (\lambda_{S_i} - \lambda_{S_j})  \geq 0\}}{2}.$

Next, we verify the feasibility of the constructed dual solution for \eqref{subproblem_dual}. Firstly, we have all the nonnegative constraints for $\boldsymbol{\alpha}, \boldsymbol{\beta}, \boldsymbol{u}$ are satisfied. For $l \in [k]$,$i,j\in[m]$ with $\lambda_{S_i} \leq \lambda_{S_j}$ such that $l\in S_i\cap S_j$:\\
If $\mathbb{I}\{(p_{l,S_i} - p_{l,S_j})(\lambda_{S_i} - \lambda_{S_j}) \geq  0\} = 0$, we have : 
\begin{align*}
& \alpha_{l,i} + \beta_{l,i} =0 + 0  < 1,\quad \alpha_{l,j} + \beta_{l,j} = 0 + 0 < 1,  \\
& -\alpha_{l,i} + \beta_{l,i} + u_{l,i,j} +\eta_i = -0 +0 +0 - \sum_{l=1}^k \sum\limits_{j\in[m]: l \in S_i \cap S_j} \frac{|p_{l,S_i}- p_{l,S_j}| \mathbb{I}\{( p_{l,S_i}- p_{l,S_j}) (\lambda_{S_i} - \lambda_{S_j})  \geq 0\}}{2}  < 0,  \\
& -\alpha_{l,j} + \beta_{l,j} - u_{l,i,j}   + \eta_j = -0 +0 -0  - \sum_{l=1}^k \sum\limits_{j\in[m]: l \in S_i \cap S_j} \frac{|p_{l,S_i}- p_{l,S_j}| \mathbb{I}\{( p_{l,S_i}- p_{l,S_j}) (\lambda_{S_i} - \lambda_{S_j})  \geq 0\}}{2} < 0.
\end{align*}
If $\mathbb{I}\{(p_{l,S_i} - p_{l,S_j})(\lambda_{S_i} - \lambda_{S_j}) \geq  0\} = 1$, we have : 
\begin{align*}
& \alpha_{l,i} + \beta_{l,i} =1 + 0  = 1,\quad \alpha_{l,j} + \beta_{l,j} = 0 + 1 = 1,  \\
& -\alpha_{l,i} + \beta_{l,i} + u_{l,i,j} + \eta_i = -1 +0 +1 - \sum_{l=1}^k \sum\limits_{j\in[m]: l \in S_i \cap S_j} \frac{|p_{l,S_i}- p_{l,S_j}| \mathbb{I}\{( p_{l,S_i}- p_{l,S_j}) (\lambda_{S_i} - \lambda_{S_j})  \geq 0\}}{2}  < 0,  \\
& -\alpha_{l,j} + \beta_{l,j} - u_{l,i,j}   + \eta_j =-0+1-1 - \sum_{l=1}^k \sum\limits_{j\in[m]: l \in S_i \cap S_j} \frac{|p_{l,S_i}- p_{l,S_j}| \mathbb{I}\{( p_{l,S_i}- p_{l,S_j}) (\lambda_{S_i} - \lambda_{S_j})  \geq 0\}}{2} < 0.
\end{align*}
For all $i\in[m]$, we have 
\begin{align*}
& \alpha_{k+i,i} + \beta_{k+i,i} = \frac{1}{2} + \frac{1}{2} =  1,  \\
& -\alpha_{k+i,i} + \beta_{k+i,i} + \eta_i = - \frac{1}{2} + \frac{1}{2} - \sum_{l=1}^k \sum\limits_{j\in[m]: l \in S_i \cap S_j} \frac{|p_{l,S_i}- p_{l,S_j}| \mathbb{I}\{( p_{l,S_i}- p_{l,S_j}) (\lambda_{S_i} - \lambda_{S_j})  \geq 0\}}{2}  < 0.
\end{align*}

Then, the objective value of Problem \eqref{subproblem_dual} under the constructed solution $(\boldsymbol{\alpha}, \boldsymbol{\beta}, \boldsymbol{u}, \boldsymbol{\eta})$ is 
\begin{subequations}
\begin{align}
& \sum_{l=1}^k  \sum\limits_{i,j\in[m],  \lambda_i \leq \lambda_j,  l\in S_i \cap S_j}  [ p_{l,S_i}(\beta_{l,i} - \alpha_{l,i}) + p_{l,j}(\beta_{l,j} - \alpha_{l,j})]  + \sum\limits_{i=1}^m p_{k+i,S_i}(\beta_{k+i,i}-\alpha_{k+i,i}) - \sum\limits_{i=1}^m  \eta_i \nonumber \\
=& \sum_{l=1}^k  \sum\limits_{i,j\in[m],  \lambda_{S_i} \leq \lambda_{S_j},  l\in S_i \cap S_j} \mathbb{I}\{(p_{l,S_i} - p_{l,S_j})(\lambda_{S_i} - \lambda_{S_j}) \geq  0\}  (p_{l,j}- p_{l,S_i})   - \sum\limits_{i=1}^m  \eta_i \label{dual_obj_1} \\
=& \sum_{l=1}^k  \sum\limits_{i,j\in[m],  \lambda_{S_i} \leq \lambda_{S_j},  l\in S_i \cap S_j} \mathbb{I}\{(p_{l,S_i} - p_{l,S_j})(\lambda_{S_i} - \lambda_{S_j}) \geq  0\}  |p_{l,i}- p_{l,S_j}|   - \sum\limits_{i=1}^m  \eta_i \label{dual_obj_2} \\
=& \sum_{l=1}^k  \sum\limits_{i,j\in[m], \lambda_{S_i} \leq \lambda_{S_j},  l\in S_i \cap S_j} \mathbb{I}\{(p_{l,S_i} - p_{l,S_j})(\lambda_{S_i} - \lambda_{S_j}) \geq  0\}  |p_{l,i}- p_{l,S_j}| \nonumber \\
&+ \sum\limits_{i=1}^m  \sum_{l=1}^k \sum\limits_{j\in[m]: l \in S_i \cap S_j} \frac{|p_{l,S_i}- p_{l,S_j}| \mathbb{I}\{( p_{l,S_i}- p_{l,S_j}) (\lambda_{S_i} - \lambda_{S_j})  \geq 0\}}{2} \label{dual_obj_3} \\
=& \sum_{l=1}^k  \sum\limits_{i,j\in[m],  \lambda_{S_i} \leq \lambda_{S_j},  l\in S_i \cap S_j} \mathbb{I}\{(p_{l,S_i} - p_{l,S_j})(\lambda_{S_i} - \lambda_{S_j}) \geq  0\}  |p_{l,j}- p_{l,S_i}| \nonumber \\
& + \sum_{l=1}^k  \sum\limits_{i,j\in[m], \lambda_{S_i} \leq \lambda_{S_j},  l\in S_i \cap S_j} \mathbb{I}\{(p_{l,S_i} - p_{l,S_j})(\lambda_{S_i} - \lambda_{S_j}) \geq  0\}  |p_{l,i}- p_{l,S_j}|   \label{dual_obj_4} \\
=& \sum_{l=1}^k \sum\limits_{i,j \in [m], \lambda_{S_i}  \leq \lambda_{S_j}: l \in S_i  \cap S_j} 2 |p_{l,S_i}- p_{l,S_j}| \mathbb{I}\{( p_{l,S_i}- p_{l,S_j}) (\lambda_{S_i} - \lambda_{S_j})  \geq 0\}.\nonumber
\end{align} 
\end{subequations}
Plugging in the value of $\alpha_{l,i} = \beta_{l,i} = \alpha_{l,j} = \beta_{l,j}$, for $l \in [k]$,$i,j\in[m]$ with $\lambda_{S_i} \leq \lambda_{S_j}$ and  $l\in S_i\cap S_j$, we have \eqref{dual_obj_1} . \eqref{dual_obj_2} is equivalent to \eqref{dual_obj_1} because when $\mathbb{I}\{(p_{l,S_i} - p_{l,S_j})(\lambda_{S_i} - \lambda_{S_j}) \geq  0\} = 1$ and $ \lambda_i \leq \lambda_j,$ we have $|p_{l,i}- p_{l,S_j}| = p_{l,j}- p_{l,S_i} >0.$  Plugging in the value of $\eta_i$ for all $i\in[m]$ to \eqref{dual_obj_2}, we have \eqref{dual_obj_3}. \eqref{dual_obj_4} is equivalent to \eqref{dual_obj_3} because
\begin{align*}
    &\sum\limits_{i=1}^m  \sum_{l=1}^k \sum\limits_{j\in[m]: l \in S_i \cap S_j} \frac{|p_{l,S_i}- p_{l,S_j}| \mathbb{I}\{( p_{l,S_i}- p_{l,S_j}) (\lambda_{S_i} - \lambda_{S_j})  \geq 0\}}{2} \\
    =& \sum_{l=1}^k  \sum\limits_{i,j\in[m],  \lambda_{S_i} \leq \lambda_{S_j},  l\in S_i \cap S_j} \mathbb{I}\{(p_{l,S_i} - p_{l,S_j})(\lambda_{S_i} - \lambda_{S_j}) \geq  0\}  |p_{l,i}- p_{l,S_j}|.
\end{align*}}

By weak duality, we have $f_\text{sub}^*(\blam) \geq \sum_{l=1}^k \sum\limits_{i,j\in[m],  \lambda_{S_i} \leq \lambda_{S_j} : l \in S_i \cap S_j} 2 |p_{l,S_i}- p_{l,S_j}| \mathbb{I}\{( p_{l,S_i}- p_{l,S_j}) (\lambda_{S_i} - \lambda_{S_j}) \geq 0\}.$

Summing up, we have 
\[
f_{\text{sub}}^*(\blam)  = \sum_{l=1}^k \sum\limits_{i,j\in[m],\lambda_{S_i} \leq \lambda_{S_j} : l \in S_i \cap S_j} 2 |p_{l,S_i}- p_{l,S_j}| \mathbb{I}\{( p_{l,S_i}- p_{l,S_j}) (\lambda_{S_i} - \lambda_{S_j})  \geq 0\}.
\]
\end{proof}

Before we show the hardness of Problem \eqref{model:lom}, we provide the following three definitions (see \cite{dwork2001rank}) to describe the Kemeny optimal aggregation problem.

\begin{definition}[Full lists and partial lists] \label{dfn:full-listandpartial-list}
Let $\cM=\{1,\ldots,m\}$ be a finite set of alternatives, called universe. A ranking over $\cM$ is an ordered list.
If the ranking $\tau$ contains all the elements in $\cM$, then it is called a full list (ranking). If the ranking $\tau$ contains a subset of elements from the universe $\cM$, then it is called a partial list (ranking). 

\end{definition}

\begin{definition}[Kendall-tau distance (K-distance)]\label{dfn:k-dis}
The K-distance, denoted as $K(\sigma,\tau)$, is the number of pairs $i,j\in \cM$ such that $\sigma(i) < \sigma(j)$ but $\tau(i)>\tau(j)$ where $\sigma(i)$ stands for the position of $i$ in $\sigma$ and similar explanations are applied for $\sigma(j), \tau(i)$ and $\tau(j)$. Note that the pair $(i,j)$ has contribution to  the K-distance only if both $i,j$  appear in both lists $\sigma,\tau$.
\end{definition}

\begin{definition}[SK, Kemeny optimal]
For a collection of partial lists $\tau_1, \ldots, \tau_k$ and a full list $\pi$, we denote
$$SK(\pi, \tau_1, \ldots, \tau_k) = \sum_{i=1}^k K(\pi,\tau_i).$$
We say a permutation $\sigma$ is a Kemeny optimal aggregation of $\tau_1,\ldots,\tau_k$ if it minimizes $SK(\pi,\tau_1,\ldots,\tau_k)$ over all permutations $\pi$.
\end{definition}

\begin{lemma}[see \cite{dwork2001rank}]
Finding the Kemeny optimal solution for partial lists of length 2 is exactly the same problem as finding a minimum feedback arc set, and hence is NP-hard.
\end{lemma}

Now, we are ready to prove the hardness of Problem \eqref{model:lom}.
\begin{proof}
To show Problem \eqref{model:lom} is NP-hard, it suffices to show some instances of this problem is NP-hard.
We show that Kemeny optimal aggregation of length 2 can be reduced to Problem \eqref{model:lom}. 

The decision version of Kemeny optimal aggregation with a collection of partial lists of all length 2 is stated as follows:

\noindent 
INSTANCE: A finite set $\cM$ with $|\cM|=m$, a collection of partial lists $\tau_1,\ldots,\tau_k$ with $|\tau_i|=2$ for $i=1,\ldots,k$, an upper bound on the loss $L$.

\noindent 
QUESTION: Is there a full list $\pi$, such that $\sum_{i=1}^k K(\pi,\tau_i) \leq L$?

\noindent 
The decision version of the limit problem of MDM in Problem \eqref{model:lom} is stated as follows:\\
INSTANCE: A finite set $\cN$ with $|\cN|=n$, a collection of assortments $\cS$ with $|\cS|=m$ and $S\subseteq \cN$ for all $S\in \cS$, the observed choice probabilities $\bp_\cS = (p_{i,S}:i\in S, S\in \cS)$ with $\sum_{i\in S} p_{i,S} =1$ for all $S\in \cS$, an upper bound on the loss $L^{'}$. 

\noindent 
QUESTION: Is there a solution $(\bx_\cS, \blam)$ to Problem \eqref{model:lom} such that $\text{loss}(\bx_\cS, \bp_\cS) \leq L^{'}$?

We then will reduce the Kemeny optimal aggregation problem to Problem \eqref{model:lom}. 
Given any instance of Kemeny optimal aggregation problem with partial lists all of length 2, we can construct an instance of Problem \eqref{model:lom}  as follows. 
\begin{enumerate}[label=(\alph*),leftmargin=*,nolistsep]
    \item Let the collection of assortments $\cS=\{S_1, S_2, \ldots, S_m\}$ with $|\cS|=m$ and the set of alternatives (products),  $\cN=\{1,\ldots,k,k+1,\ldots,k+m\}$ with $|\cN| = n = k+m$. Given the observed choice data $\bp_\cS$, consider $\bp_\cS$ as a $n \times m$ matrix with $n$ rows and $m$ columns. 
    \item The values of the entries in $\bp_\cS$ are set in the following manner. For each $l \in \{1,...,k\}$, suppose $\tau_l = \{i,j\}$ with $\tau_l(i) < \tau_l(j)$, then we set $p_{l,S_i} = \frac{1}{3 \times k}$ and $p_{l,S_j} = \frac{2}{3\times k}$. 
    It's easy to see that for each $S_i$ with $1 \leq i \leq m$, $0< \sum_{j=1}^k p_{j,S_i} <1$.
    \item For each $S_i$ with $1 \leq i \leq m$, let $p_{k+i, S_i} = 1- \sum_{j=1}^k p_{j,S_i}$.
    \item Set other entries of  $\bp_\cS$ as zero. 
    \item Set the loss function in Problem \eqref{model:lom} to be 1-norm loss.
\end{enumerate}
We give Example \ref{eg:feasible} and Example \ref{eg:infeasible} to illustrate the above instance construction.
\begin{example}
\label{eg:feasible}
Given an instance of Kemeny optimal aggregation with $\cM=\{1, 2, 3\}$ and $\tau_1=( 1 \succ 2 ), \tau_2= ( 1 \succ 3 ), \tau_3 = ( 2 \succ 3 )$, we construct an instance for Problem \eqref{model:lom} with $\bp_\cS$ as shown in Table \ref{tab:feasible-instance-eg}.
\begin{table}[H]
\centering
\caption{An example of a representable instance construction}\label{tab:feasible-instance-eg}
\scalebox{0.8}{\begin{tabular}{|c|c|c|c|}
\hline
\begin{tabular}[c]{@{}c@{}}alternative\end{tabular} & $S_1=\{1,2,4\}$ & $S_2=\{1,3,5\}$ & $S_3=\{2,3,6\}$ \\ \hline
1                                                                & 1/9         & 2/9         & -           \\ \hline
2                                                                & 1/9         & -           & 2/9         \\ \hline
3                                                                &     -         & 1/9         & 2/9         \\ \hline
4                                                                & 7/9         &     -         &    -          \\ \hline
5                                                                &       -       & 6/9         &   -           \\ \hline
6                                                                &        -      &    -          & 5/9         \\ \hline
\end{tabular}}
\end{table}
\end{example}

\begin{example}
\label{eg:infeasible}
Given an instance of Kemeny optimal aggregation with $\cM=\{1, 2, 3\}$ and $\tau_1=( 1 \succ 2 ), \tau_2= ( 3 \succ 1 ), \tau_3 = ( 2 \succ 3 )$, we construct an instance for Problem \eqref{model:lom} with $\bp_\cS$ as shown in Table \ref{tab:infeasible-instance-eg}.
\begin{table}[H]
\centering
\caption{An example of an infeasible instance construction}\label{tab:infeasible-instance-eg}
\scalebox{0.8}{\begin{tabular}{|c|c|c|c|}
\hline
\begin{tabular}[c]{@{}c@{}}alternative\end{tabular} & $S_1=\{1,2,4\}$ & $S_2=\{1,3,5\}$ & $S_3=\{2,3,6\}$ \\ \hline
1                                                                & 1/9         & 2/9         & -           \\ \hline
2                                                                & 2/9         & -           & 1/9         \\ \hline
3                                                                &     -         & 1/9         & 2/9         \\ \hline
4                                                                & 6/9         &     -         &    -          \\ \hline
5                                                                &       -       & 6/9         &   -           \\ \hline
6                                                                &        -      &    -          & 6/9         \\ \hline
\end{tabular}}
\end{table}
\end{example}

In Example \ref{eg:feasible}, both the Kemeny optimal aggregation and the limit of MDM instances are feasible to their problems respectively. The optimal solution to the Kemeny optimal aggregation is $\pi = (1 \succ 2 \succ 3)$ and one of the optimal solutions to the limit of MDM is $\bx_\cS = \bp_\cS$ and $\lambda_{S_1} = 3$, $\lambda_{S_2} = 2$ and $\lambda_{S_3} = 1$. Both instances obtain 0 loss. \\
In Example \ref{eg:infeasible}, both the Kemeny optimal aggregation and the limit of MDM instance are infeasible to their problems respectively. It's trivial to see that the optimal solution to the Kemeny optimal aggregation is one of $\{(1 \succ 2 \succ 3), (2 \succ 3 \succ 1), (3 \succ 1 \succ 2)\}$. Each of such solutions obtains $SK =1$. Given Lemma \ref{lem:closed_form}, one may make a guess for one of the optimal solutions to the limit of MDM in Example \ref{eg:infeasible} to be $x^*_{2,S_1} = \frac{1}{6}$, $x^*_{2,S_3} = \frac{1}{6}$, $x^*_{4,S_1} = \frac{13}{18}$, $x^*_{6,S_3} = \frac{11}{18}$ and the optimal loss to be $\frac{2}{9}$. We will show that this guess is true.

Recall that Problem \eqref{model:lom} is equivalent to Problem \eqref{mdm-sequential}. We next show that for any fixed $\blam$ in Problem \eqref{mdm-sequential}, then the sub-problem \eqref{mdm_sub} with optimal $\bx_\cS^*(\blam)$ and $f_\text{sub}^*(\blam)$ under $\blam$, there exists $\pi$ such that $\lambda_{S_{\pi^{-1}(1)}}^* \geq \cdots \geq \lambda_{S_{\pi^{-1}(m)}}^*$ for the Kemeny optimal aggregation and $SK(\pi) = \frac{3\times k}{2} f_\text{sub}^*(\blam)$. 

By Lemma \ref{lem:closed_form}, we have 
\begin{subequations}
\begin{align}
    &f^*_\text{sub}(\blam) = \text{loss}(\bx_{\cS}^*(\blam), \bp_{\cS}) \nonumber \\
    =&  \sum_{l=1}^k \sum\limits_{i,j\in[m], \lambda_{S_i} \leq \lambda_{S_j} : l \in S_i \cap S_j} (|x^*_{l,S_i} - p_{l,S_i}| + |x^*_{l,S_j} - p_{l,S_j}|) + \sum_{i=1}^m |x^*_{k+i,S_i} - p_{k+i,S_i} | \label{1}\\
    =& \sum_{l=1}^k \sum\limits_{i,j\in[m], \lambda_{S_i} \leq \lambda_{S_j} : l \in S_i \cap S_j} 2 |p_{l,S_i}- p_{l,S_j}| \mathbb{I}\{( p_{l,S_i}- p_{l,S_j}) (\lambda_{S_i} - \lambda_{S_j})  \geq 0\}  \label{2} \\ 
    =& \frac{2}{3\times k}  \sum_{l=1}^k \sum\limits_{i,j\in[m], \lambda_{S_i} \leq \lambda_{S_j} : l \in S_i \cap S_j}
    \mathbb{I}\{(\pi(i) - \pi(j))(\tau_l(i) - \tau_l(j)) < 0\}  \label{3} \\
    =& \frac{2}{3\times k}  \sum_{l=1}^k K(\pi,\tau_l) \nonumber \\
    =& \frac{2}{3\times k} SK(\pi).  \nonumber
\end{align}
\end{subequations}
Equation \eqref{1} is due to the construction of $\bp_\cS$. Equation \eqref{2} holds because of $|p_{l,S_i} - p_{l,S_j}| = \frac{1}{3\times k}$ and the closed form objective value in  Lemma \ref{lem:closed_form}. The argument for Equation \eqref{3} is as follows: For $l\in [k]$, $i,j\in[m]: l \in S_i \cap S_j$, by instance construction, we have
$$  p_{l,S_i} < p_{l,S_j}~\text{if}~ \tau_l(i) < \tau_l(j). $$
From the relation between $\pi$ and $\blam$, we have $\lambda_{\pi^{-1}(1)} > \ldots > \lambda_{\pi^{-1}(m)}$. Then, we have 
$$\pi(i) < \pi(j)~\text{if}~ \lambda_{S_i} \geq  \lambda_{S_j}.$$
The above two inequities imply that
$$\mathbb{I}\{(\lambda_{S_i} \geq \lambda_{S_j})(p_{l,S_i} - p_{l,S_j}) >  0\} = \mathbb{I}\{(\pi(i) - \pi(j))(\tau_l(i) - \tau_l(j)) < 0\} .$$

Setting $L = \frac{3k}{2}L^{'}$. The decision problem of the limit of MDM asks is there $(\bx_\cS, \blam)$ such that $f_\text{limit}^*\leq L^{'}$ is equivalent to the decision problem of Kemeny optimal aggregation is there a full ranking $\pi$ such that  $f_\text{kemeny}^* \leq L$. 
\end{proof}

\subsection{Proof of Proposition \ref{prop:micp}}
\label{pf:micp}
\begin{proof}
Observe that the variables $(\lambda_S: S \in \cS)$ influence the value of the formulation in Proposition \ref{prop:micp} only via the sign of $\lambda_S - \lambda_T,$ for any pair of variables $\lambda_S,\lambda_T$ from the collection  $(\lambda_S: S \in \cS).$ Therefore the optimal value of this optimization formulation is not affected by the presence of the following additional constraints: $0 \leq \lambda_S \leq 1$ for all $S \in \cS.$ Indeed, this is because the signs of the differences $\{ \lambda_S - \lambda_T: S, T \in \cS \}$ are not affected by these additional constraints. Taking $\epsilon$ to be smaller than $1/(2\vert \cS \vert),$ for example, ensures that there is a feasible assignment for  $(\lambda_S: S \in \cS)$ within the interval $[0,1]$ even if all these variables take distinct values. 

Let $F$ denote the feasible values for the variables  $(\lambda_S: S \in \cS), (x_{i,S}: (i,S) \in \cI_\cS)$ satisfying the constraints introduced in the above paragraph besides those in the formulation in Proposition \ref{prop:limit-mdm}.  
Equipped with this feasible region $F$, we have the following deductions for $(\lambda_S: S \in \cS), (x_{i,S}: (i,S) \in \cI_\cS)$ in $F$: For any $S,T \in \cS$ containing $i,$   
    \begin{itemize}[leftmargin=*]
        \item[(i)] we have $\lambda_S < \lambda_T$ if and only if  $\delta_{S,T} = 1$ and $\delta_{T,S} = 0,$ due to the first set of constraints of \eqref{model:mdm_micp}; in this case, we have from the second and fourth set of constraints of \eqref{model:mdm_micp} that $0 \leq x_{i,T} \leq x_{i,S} \leq 1;$
         \item[(ii)] likewise, we have $\lambda_S > \lambda_T$ if and only if  $\delta_{S,T} = 0$ and $\delta_{T,S} = 1,$ due to the first set of constraints; in this case, we have from the second and fourth set of constraints of \eqref{model:mdm_micp} that $0 \leq x_{i,S} \leq x_{i,T} \leq 1.$
         \item[(iii)] finally, $\lambda_S = \lambda_T$ if and only if  $\delta_{S,T} = 0$ and $\delta_{T,S} = 0;$ here we have from the third set of constriants of \eqref{model:mdm_micp} that $x_{i,S} = x_{i,T}.$
    \end{itemize} 
    Thus the binary variables $\{\delta_{S,T}: S,T \in \cS\}$ suitably model the first set of constraints of \eqref{model:lom} and provide an equivalent reformulation in terms of the constraints. Therefore the optimal value of the formulations in Propositions \ref{prop:limit-mdm} and \ref{prop:micp} are identical.  
\end{proof}

\section{Illustrative Examples for Section \ref{sec:lom} }
\label{sec:eg}
\subsection{An example to show the non-convexity of MDM feasible region }
\label{sec:mdm_nonconvexity}
\begin{example}\label{eg:mdm_nonconvexity}
$\bx_\cS$ is MDM-representable because $x_{1,A}<x_{1,B}$,$x_{2,A} < x_{2,C}$ and $x_{3,B} < x_{3,C}$ implies $\lambda_A > \lambda_B$, $\lambda_A > \lambda_C$ and $\lambda_B>\lambda_C$. The values $\lambda_A=12,\lambda_B = 10,\lambda_C=8$ satify this. $\by_\cS$ is MDM-representable because $y_{1,A} > y_{1,B}$, $y_{2,A} > y_{2,C}$ and $y_{3,B} > y_{3,C}$ implies $\lambda_A < \lambda_B$, $\lambda_A < \lambda_C$ and $\lambda_B < \lambda_C$. The values $\lambda_A=8,\lambda_B = 10,\lambda_C=12$ satify this.  $\boldsymbol{w}=0.4\bx_\cS+0.6\by_\cS$ is a convex combination of $\bx_\cS$ and $\by_\cS$ but it can not be represented by MDM because $w_{1,A} > w_{1,B}$, $w_{1,A} < w_{2,C}$ and $w_{3,B} > w_{3,C}$ which implies $\lambda_A < \lambda_B$, $\lambda_A > \lambda_C$ and $\lambda_B < \lambda_C$, i.e., $\lambda_B < \lambda_C < \lambda_A < \lambda_B$. This means  $\boldsymbol{w}$ can not be represented by MDM.
\end{example}

\begin{table}[H]
\centering
\begin{minipage}[t]{0.32\textwidth} 
\scalebox{0.64}{\begin{tabular}{cccc}
\multicolumn{4}{c}{$\bx_\cS$}                                                                                                             \\ \hline
\multicolumn{1}{|c|}{Alternative} & \multicolumn{1}{c|}{A=\{1,2\}} & \multicolumn{1}{c|}{B=\{1,3\}} & \multicolumn{1}{c|}{C=\{2,3\}} \\ \hline
\multicolumn{1}{|c|}{1}           & \multicolumn{1}{c|}{0.3}       & \multicolumn{1}{c|}{0.9}       & \multicolumn{1}{c|}{}          \\ \hline
\multicolumn{1}{|c|}{2}           & \multicolumn{1}{c|}{0.7}       & \multicolumn{1}{c|}{}          & \multicolumn{1}{c|}{0.8}       \\ \hline
\multicolumn{1}{|c|}{3}           & \multicolumn{1}{c|}{}          & \multicolumn{1}{c|}{0.1}       & \multicolumn{1}{c|}{0.2}       \\ \hline
\end{tabular}}
\end{minipage}
\begin{minipage}[t]{0.32\textwidth} 
\scalebox{0.64}{\begin{tabular}{cccc}
\multicolumn{4}{c}{$\by_\cS$}                                                                                                             \\ \hline
\multicolumn{1}{|c|}{Alternative} & \multicolumn{1}{c|}{A=\{1,2\}} & \multicolumn{1}{c|}{B=\{1,3\}} & \multicolumn{1}{c|}{C=\{2,3\}} \\ \hline
\multicolumn{1}{|c|}{1}           & \multicolumn{1}{c|}{0.75}      & \multicolumn{1}{c|}{0.1}       & \multicolumn{1}{c|}{}          \\ \hline
\multicolumn{1}{|c|}{2}           & \multicolumn{1}{c|}{0.25}      & \multicolumn{1}{c|}{}          & \multicolumn{1}{c|}{0.2}       \\ \hline
\multicolumn{1}{|c|}{3}           & \multicolumn{1}{c|}{}          & \multicolumn{1}{c|}{0.9}       & \multicolumn{1}{c|}{0.8}       \\ \hline
\end{tabular}}   
\end{minipage}
\begin{minipage}[t]{0.32\textwidth} 
\scalebox{0.64}{\begin{tabular}{cccc}
\multicolumn{4}{c}{$\boldsymbol{w} = 0.4\bx_\cS+0.6\by_\cS$}                                                                         \\ \hline
\multicolumn{1}{|c|}{Alternative} & \multicolumn{1}{c|}{A=\{1,2\}} & \multicolumn{1}{c|}{B=\{1,3\}} & \multicolumn{1}{c|}{C=\{2,3\}} \\ \hline
\multicolumn{1}{|c|}{1}           & \multicolumn{1}{c|}{0.57}      & \multicolumn{1}{c|}{0.42}      & \multicolumn{1}{c|}{}          \\ \hline
\multicolumn{1}{|c|}{2}           & \multicolumn{1}{c|}{0.43}      & \multicolumn{1}{c|}{}          & \multicolumn{1}{c|}{0.44}      \\ \hline
\multicolumn{1}{|c|}{3}           & \multicolumn{1}{c|}{}          & \multicolumn{1}{c|}{0.58}      & \multicolumn{1}{c|}{0.56}      \\ \hline
\end{tabular}} 
\end{minipage}
\end{table}

{\color{black}\subsection{An example to show using \eqref{mdm:perturb} to get MDM-representable probabilities  }
\label{sec:mdm_feasible_prob}
\begin{example}
    \label{eg:mdm_feasible_prob}
    Consider the observed choice probabilities $\bp_\cS$ is given in Table (a) below with $n=3$ and $\cS=\{A = \{1,2\},B = \{1,3\},C = \{2,3\}\}.$ $\bp_\cS$ is not representable by MDM since $p_{1,A} > p_{1,B}$ implies $\lambda_{A} < \lambda_{B} $, $p_{2,A} < p_{2,C}$ implies $\lambda_{A} > \lambda_{C}$, and $p_{3,B} > p_{3,C}$ implies $\lambda_{B} < \lambda_{C},$ that is $\lambda_{A} < \lambda_{B} < \lambda_{C} < \lambda_{A}.$ 
    
    Suppose the assortments $A, B, C$ are shown once in the historical data. We choose the limit loss function to be 1-norm $\sum_{(i,S)\in \cI_\cS} |x_{i,S} - p_{i,S}|$, and set $\epsilon = 0.001$ in \eqref{model:mdm_micp}. We then obtain MDM limit probabilities $\bx^{*}_\cS$ in Table (b) below and $\lambda_A =0$, $\lambda_B =0.001$, and $\lambda_B =0.002$, that is $\lambda_A < \lambda_B < \lambda_C $. $\bx^{*}_\cS$ satisfies all the constraints in \eqref{model:lom}, meaning that $\bx^{*}_\cS \in \text{closure}(\pmdm(\cS))$. However, $\bx^{*}_\cS \not\in \pmdm(\cS)$ since $x_{2,A}^{*} = x_{2,C}^{*}$ when $\lambda_A < \lambda_C.$

    To obtain feasible MDM probabilities, we apply \eqref{mdm:perturb} by setting $\delta = 10^{-5}.$ We then obtain MDM representable choice probabilities in Table (c) below.
    
\begin{table}[H]
\centering
\begin{minipage}[t]{0.32\textwidth} 
\scalebox{0.64}{\begin{tabular}{cccc}
\multicolumn{4}{c}{(a) $\bp_\cS$ that is not representable by MDM}                       \\ \hline
\multicolumn{1}{|c|}{\begin{tabular}[c]{@{}c@{}}Alternative\end{tabular}} & \multicolumn{1}{c|}{A=\{1,2\}} & \multicolumn{1}{c|}{B=\{1,3\}} & \multicolumn{1}{c|}{C=\{2,3\}} \\ \hline
\multicolumn{1}{|c|}{1}                                                                & \multicolumn{1}{c|}{0.57}         & \multicolumn{1}{c|}{0.42}       & \multicolumn{1}{c|}{-}       \\ \hline
\multicolumn{1}{|c|}{2}                                                                & \multicolumn{1}{c|}{0.43}        & \multicolumn{1}{c|}{-}       & \multicolumn{1}{c|}{0.44}          \\ \hline
\multicolumn{1}{|c|}{3}                                                                & \multicolumn{1}{c|}{-}        & \multicolumn{1}{c|}{0.58}          & \multicolumn{1}{c|}{0.56}        \\ \hline
\end{tabular}}
\end{minipage}
\begin{minipage}[t]{0.32\textwidth} 
\scalebox{0.64}{\begin{tabular}{cccc}
\multicolumn{4}{c}{(b) $\bx^{*}_\cS$ obtained by \eqref{model:mdm_micp}}                       \\ \hline
\multicolumn{1}{|c|}{\begin{tabular}[c]{@{}c@{}}Alternative\end{tabular}} & \multicolumn{1}{c|}{A=\{1,2\}} & \multicolumn{1}{c|}{B=\{1,3\}} & \multicolumn{1}{c|}{C=\{2,3\}} \\ \hline
\multicolumn{1}{|c|}{1}                                                                & \multicolumn{1}{c|}{0.56}         & \multicolumn{1}{c|}{0.42}       & \multicolumn{1}{c|}{-}       \\ \hline
\multicolumn{1}{|c|}{2}                                                                & \multicolumn{1}{c|}{0.44}        & \multicolumn{1}{c|}{-}       & \multicolumn{1}{c|}{0.44}          \\ \hline
\multicolumn{1}{|c|}{3}                                                                & \multicolumn{1}{c|}{-}        & \multicolumn{1}{c|}{0.58}          & \multicolumn{1}{c|}{0.56}        \\ \hline
\end{tabular}}
\end{minipage}
\begin{minipage}[t]{0.32\textwidth} 
\scalebox{0.64}{\begin{tabular}{cccc}
\multicolumn{4}{c}{(c) feasible MDM probabilities obtained by \eqref{mdm:perturb}}                       \\ \hline
\multicolumn{1}{|c|}{\begin{tabular}[c]{@{}c@{}}Alternative\end{tabular}} & \multicolumn{1}{c|}{A=\{1,2\}} & \multicolumn{1}{c|}{B=\{1,3\}} & \multicolumn{1}{c|}{C=\{2,3\}} \\ \hline
\multicolumn{1}{|c|}{1}                                                                & \multicolumn{1}{c|}{0.564823}         & \multicolumn{1}{c|}{0.42}       & \multicolumn{1}{c|}{-}       \\ \hline
\multicolumn{1}{|c|}{2}                                                                & \multicolumn{1}{c|}{0.435177}        & \multicolumn{1}{c|}{-}       & \multicolumn{1}{c|}{0.435176}          \\ \hline
\multicolumn{1}{|c|}{3}                                                                & \multicolumn{1}{c|}{-}        & \multicolumn{1}{c|}{0.58}          & \multicolumn{1}{c|}{0.564824}        \\ \hline
\end{tabular}}
\end{minipage}
\end{table}
\end{example}}

\section{An Algorithm for Evaluating the Limit of MDM when $\vert \cS \vert$ is Small}
\label{alg:limit}
\SetKwComment{Comment}{/* }{ */}
\begin{algorithm}[htb]
\caption{An algorithm solves the limit of MDM polynomial in $n$}\label{alg:poly_n}
\KwIn{Observed choice probabilities $\bp_\cS$, collection $\cS$, product universe $\cN$.}
\KwOut{MDM choice probabilities $\bx_\cS^*$, optimal loss $f^*$, optimal ranking of assortments $\tau^*$.}
$\boldsymbol{T} \gets $\{all rankings of $(S: S\in \cS$)\}\;
$f^* \gets +\infty$ keeps tracking of the optimal value of Problem \eqref{model:lom}\;
$\bx_\cS^* \gets \boldsymbol{0}$ keeps tracking of the optimal solution\;
\For{$\tau \in \boldsymbol{T} $ }{
    Solve 
    \vspace{-30pt}
        \begin{align*}
             \min_{\bx_\cS} \quad &\text{loss}(\bx_\cS, \bp_\cS)\\
            \text{s.t.}\quad & x_{i,S} - x_{i,T} \geq 0,\quad \text{ if } \tau(T)<\tau(S) \quad \forall (i,S),(i,T) \in \cI_\cS,\\
            & \sum_{i\in S}x_{i,S} = 1,\quad \forall S\in \cS, \\
            & x_{i,S} \geq 0, \quad \forall(i,S)\in \cI_\cS. \tag{Limit-LP}\label{P2}
        \end{align*}
    $f \gets$ the output optimal objective value of \eqref{P2}\;
    $\bx_\cS \gets$ the output optimal solution of \eqref{P2}\;
  \If{$f<f^*$ }{
    $\bx_\cS^* \gets \bx_\cS $\;
    $f^*  \gets  f$ \;
    $\tau^* \gets \tau$\;
  }
}
\end{algorithm}
In Algorithm \ref{alg:poly_n}, for a fixed $\blam$, we just need to solve a convex optimization problem with $\mathcal{O}(n \vert \cS \vert)$ continuous variables and $\mathcal{O}(n \vert \cS \vert^2)$ linear constraints to compute the limit loss. There are $m!$ possible $\blam$. Thus, Algorithm \ref{alg:poly_n} is polynomial in the alternative size $n$.

\section{Additional Useful Details on the Experiments in the Paper}
\label{sec:implementation}
In this section, we give details on the implementation of the experiments. We used a MacBook Pro Laptop with a 2 GHz 4 core Intel Core i5 processor for all experiments.

\subsection{Representation test experiment implementation details}
\label{sec:implementation_representation}
\subsubsection{Checking the representability of MDM} \label{sec:representation_mdm}
For a collection of observed choice probabilities $\bp_\cS$, by Theorem \ref{thm:feascon-mdm},  we check the representability of MDM with the following linear program:
\begin{align}
\begin{aligned}
    \max_{\epsilon} \quad &\epsilon\\
    \text{s.t.}\quad 
    &\lambda_S- \lambda_T -  \epsilon \geq 0,\text{ if }p_{i,S}<p_{i,T}\quad\forall (i,S),(i,T)\in \cI_\cS,\\
    & \lambda_S- \lambda_T = 0,\text{ if }p_{i,S} = p_{i,T} >0\quad\forall (i,S),(i,T)\in \cI_\cS. \label{lp:mdm_representability}
\end{aligned}
\end{align}
If the optimal value of \eqref{lp:mdm_representability} is strictly positive, then $\bp_\cS$ can be represented by MDM. Otherwise, $\bp_\cS$ cannot be represented by MDM.

\subsubsection{Checking the representability of RUM} \label{sec:representation_rum}
For a collection of observed choice probabilities $\bp_\cS$, we check the representability of RUM with the following linear program:
\begin{align}
\begin{aligned}
    \max_{\blam} \quad\!\! &  0\\
    \text{s.t.}\quad & \sum_{\sigma \in \Sigma_n}  \lambda(\sigma) \mathbb{I}[\sigma, i, S] - p_{i,S} =0, \quad \forall (i,S) \in \cI_\cS,\\
    & \sum_{\sigma \in \Sigma_n} \lambda(\sigma)=1,\quad \lambda(\sigma) \geq 0, \quad \forall \sigma \in \Sigma_n. \label{lp:rum_representability}
\end{aligned}
\end{align}
If \eqref{lp:rum_representability} is feasible, $\bp_\cS$ can be represented by RUM. Otherwise, $\bp_\cS$ cannot be represented by RUM.

\subsubsection{Checking the representability of MNL}
\label{sec:representation_mnl}
For a collection of observed choice probabilities $\bp_\cS$, we check the representability of MNL with the following linear program:
\begin{align}
\begin{aligned}
    \max_{\boldsymbol{\nu}} \quad\!\! & \sum_{i\in \cN} \nu_i\\
    \text{s.t.} \quad & p_{i,S}\, \sum_{j\in \nu_j} \nu_j - \nu_i = 0,\quad \forall (i,S)\in \cI_\cS,\\
    & \nu_i \geq 0, \quad \forall i\in \cN.   \label{lp:mnl_representability}
\end{aligned}
\end{align}
If the optimal value of \eqref{lp:mnl_representability} is strictly positive, $\bp_\cS$ can be represented by MNL. Otherwise, $\bp_\cS$ cannot be represented by MNL.


{\color{black}\subsection{Limit experiment implementation details}
\label{sec:limit_exp}
\subsubsection{Limit of MDM} \label{sec:limit_exp_mdm}
Given an instance $\bp_\cS$ with collection $\cS$, we solve the limit of MDM with \eqref{model:mdm_micp}. When the loss function is chosen to be the average absolute deviation loss, we set $\sum_{S\in \cS}  \textnormal{loss}(\bp_{S},\bx_{S}) = \sum_{S\in \cS} n_S \sum_{i\in S} p_{i,S} |x_{i,S} - p_{i,S}|/\sum_{S\in \cS} n_S $  in \eqref{model:mdm_micp}. 

When the loss function is chosen to be the average KL loss function, the limit of MDM can not be easily solved by existing solvers including Gurobi, Mosek, and  CVXPY. We take the following 2-stage approach to solve the limit of MDM with the average KL loss function: (1) We first solve \eqref{model:mdm_micp} with the average absolute deviation loss and get $\blam^*$. (2) We then solve \eqref{model:mdm_micp} with the average KL loss function and $\blam^*$ as follows:
\begin{align*} 
&\!\!\!\!\!\!\!\!\!\!\!\!\!\!\!\!\!\!\!\min_{\bx} \,  - \sum_{S\in \cS} n_S \sum_{i\in S} p_{i,S} \text{log} (x_{i,S}/p_{i,S})/\sum_{S\in \cS} n_S \\
 \text{s.t.}\quad 
& x_{i,S}  \geq x_{i,T}  \text{ if } \lambda_S^\ast \leq \lambda_T^\ast, \quad\qquad  \forall (i,S),(i,T) \in \cI_\cS,\\
&\sum_{i \in S}x_{i,S} = 1,\ \forall S \in \cS,
\quad  x_{i,S} \geq 0,  \ \forall (i,S) \in \cI_\cS.   
\end{align*}

\subsubsection{Limit of RUM}  \label{sec:limit_exp_rum}
Given an instance $\bp_\cS$ with collection $\cS$, we solve the limit of RUM with the following convex program:
\begin{align}
\begin{aligned}
    \min_{\bx,\blambda}\quad & \sum_{S\in \cS}  \textnormal{loss}(\bp_{S},\bx_{S})\\
    \text{s.t.}\quad & 
    x_{i,S} -  \sum_{\sigma \in \Sigma_n}  \lambda(\sigma) \mathbb{I}[\sigma, i, S]  =0, \quad \forall (i,S) \in \cI_\cS,\\
    & \sum_{\sigma \in \Sigma_n} \lambda(\sigma)=1,\quad \lambda(\sigma) \geq 0, \quad \forall \sigma \in \Sigma_n. \label{model:rum_limit}
\end{aligned}
\end{align}
We set the loss function of \eqref{model:rum_limit} based on the chosen loss function. \textcolor{black}{In our experiments, the limit of RUM is computed when the loss function is chosen as the average absolute deviation loss $\sum_{S\in \cS}  \textnormal{loss}(\bp_{S},\bx_{S}) = \sum_{S\in \cS} n_S \sum_{i\in S} p_{i,S} |x_{i,S} - p_{i,S}|/\sum_{S\in \cS} n_S $, which can be reformulated as a linear program, and we can solve it in the Gurobi solver. The limit losses of RUM of all instances reported in the paper are solved to optimality.}

\subsubsection{Limit of MNL} \label{sec:limit_exp_mnl}
For the parametric model MNL, given an instance $\bp_\cS$ with collection $\cS$, we first estimate the parameters of MNL by solving the maximum loglikelihood with the given data and then compute the best-fitting choice probabilities of MNL to $\bp_\cS$ with the estimated parameters. The parameters of MNL are estimated as follows: 
\begin{align}
\begin{aligned}
     \bv^{*} =\text{argmax}_{\bv} \, ll(\bv | \bp_\cS) &= \sum_{S\in \cS} n_S \sum_{i\in S} p_{i,S}\, \log (\frac{\text{exp}\,v_i}{\sum_{j\in S} \text{exp}\, v_j})\\
    & = \sum_{S\in \cS} n_S (\sum_{i\in S} p_{i,S}\,v_i -  \log \sum_{j\in S} \text{exp}\, v_j) .\label{model:mnl_mle} 
\end{aligned}
\end{align}
When the limit loss function is chosen to be the average absolute deviation loss, we compute the loss between the choice probability collection and the probability collection with the estimated MLE of MNL as $\sum_{S\in \cS}n_S \sum_{i\in S} p_{i,S} |p_{i,S} - \frac{\text{exp}\,v^{*}_i}{\sum_{j\in S} \text{exp}\, v^{*}_j} |/\sum_{S\in \cS} n_S $. When the limit loss function is chosen to be the average KL loss, we first solve \eqref{model:mnl_mle} and record the optimal values as $ll^*$, and then compute the average KL loss of MNL as $(-ll^* + \sum_{S\in \cS} n_S \sum_{i\in S} p_{i,S} \log p_{i,S})/\sum_{S \in \cS} n_S$.

\subsubsection{Limit of MCCM} \label{sec:limit_exp_mccm}
By the definition of MCCM, the choice probabilities satisfy the constraints in \eqref{eq:limit_mccm_outsideopt} below, so the limit of MCCM can be solved by the following continuous optimization problem:
\begin{align}  \label{eq:limit_mccm_outsideopt}
\min_{\boldsymbol{\lambda},\boldsymbol{\rho},\boldsymbol{x}_\cS,\boldsymbol{y}_\cS }& \quad  \textnormal{loss}(\bp_{S},\bx_{S})\nonumber\\
\text{s.t.}\quad 
    & x_{i,S} = \lambda_i + \sum_{j\in \cN_0 \setminus S} y_{j,S} \rho_{j,i}, \quad \forall i\in S\cup \{0\}, S\in \cS,\nonumber\\ 
    & y_{i,S} = \lambda_i + \sum_{j\in \cN_0 \setminus S} y_{j,S} \rho_{j,i}, \quad \forall i\in \cN \setminus S, S\in \cS,\nonumber\\
    & x_{i,S} = 0,\quad \forall  i\in \cN \setminus S, S\in \cS, \\
    & y_{i,S} =0 \quad \forall i\in S \cup \{0\}, S\in \cS, \nonumber \\
    & \sum_{i=0}^n \lambda_i = 1, \quad \lambda_i \geq 0 , \ \forall i\in \cN_0, \nonumber \\
    & \sum_{j=0}^n \rho_{ij} =1,\  \forall i \in \cN_0, \quad \rho_{0i}=0,\, \rho_{ii} =0, \  \forall i \in \cN, \quad \rho_{ij} \geq 0, \ \forall i,j \in \cN_0, \nonumber \\
    & \sum_{i\in S\cup \{0\}} x_{i,S} =1,\  S\in \cS,\quad  x_{i,S} \geq 0,\  \forall i\in S\cup \{0\}, S\in \cS,\nonumber \\
    & y_{i,S} \geq 0,\quad \forall i\in \cN \setminus S, S\in \cS,\nonumber
\end{align}
where $0$ denotes for the outside option and  $\cN_0 = \cN \cup \{0\}.$
We solve \eqref{eq:limit_mccm_outsideopt} with ‘ipopt’ in the continuous optimization solver pyomo. We set the loss function of \eqref{eq:limit_mccm_outsideopt} based on the chosen loss function.  

When the outside option is not included in the assortments, the limit of MCCM is computed as follows:
\begin{align}  \label{eq:limit_mccm_no_outsideopt}
\min_{\boldsymbol{\lambda},\boldsymbol{\rho},\boldsymbol{x}_\cS,\boldsymbol{y}_\cS }& \quad  \textnormal{loss}(\bp_{S},\bx_{S})\nonumber\\
\text{s.t.}\quad 
    & x_{i,S} = \lambda_i + \sum_{j\in \cN \setminus S} y_{j,S} \rho_{j,i}, \quad \forall i\in S, S\in \cS,\nonumber\\ 
    & y_{i,S} = \lambda_i + \sum_{j\in \cN \setminus S} y_{j,S} \rho_{j,i}, \quad \forall i\in \cN \setminus S, S\in \cS,\nonumber\\
    & x_{i,S} = 0,\quad \forall  i\in \cN \setminus S, S\in \cS, \\
    & y_{i,S} =0 \quad \forall i\in S , S\in \cS, \nonumber \\
    & \sum_{i\in \cN} \lambda_i = 1, \quad \lambda_i \geq 0 , \ \forall i\in \cN, \nonumber \\
    & \sum_{j\in \cN} \rho_{ij} =1,\ \rho_{ii} = 0, \  \forall i \in \cN,  \quad \rho_{ij} \geq 0, \ \forall i,j \in \cN, \nonumber \\
    & \sum_{i\in S } x_{i,S} =1,\  S\in \cS,\quad  x_{i,S} \geq 0,\  \forall i\in S , S\in \cS,\nonumber \\
    & y_{i,S} \geq 0,\quad \forall i\in \cN \setminus S, S\in \cS.\nonumber
\end{align}


\subsubsection{Limit of LC-MNL}  \label{sec:limit_exp_lcmnl}
The limit computation of the limit of LC-MNL is similar to the limit of MNL. Given an instance $\bp_\cS$ with collection $\cS$, we first estimate the parameters of LC-MNL and then compute the best-fitting choice probabilities of LC-MNL to $\bp_\cS$ with the estimated parameters. Lastly, we compute the limit of LC-MNL with the best-fitting choice probabilities and the chosen loss function. The parameters are estimated by the following loglikelihood of LC-MNL:
\begin{align}
\begin{aligned}
\boldsymbol{w}^{*},\boldsymbol{v}^{*} = \argmax_{\boldsymbol{w},\boldsymbol{v}} & LL(\boldsymbol{w},\boldsymbol{v}| \bp_\cS)= \sum_{S\in \cS} \sum_{i\in S} n_S p_{i,S} \log \sum_{l=1}^L w_l \frac{\exp{v_{li}}}{\sum_{j\in S} \exp{v_{lj}}}\\
        \text{s.t.}\quad 
        & \sum_{l=1}^L w_l =1,\\
        & w_l \geq 0,\quad \forall l\in [L], \label{eq:ll_lcmnl}
\end{aligned}
\end{align}
where we assume that the population is described by a mixture of MNL models consisting of $L$ classes, with $w_l$ denoting the fraction of customers in class $l$, and  $\boldsymbol{v}_l = (v_{l0},v_{l1},\cdots, v_{ln} )$ denoting the parameters of the corresponding MNL model of customers in class $l$. \eqref{eq:ll_lcmnl} is non-convex in the model parameters $\boldsymbol{w}$ and $\boldsymbol{v}$. We obtain an approximate solution by using the expectation-maximization (EM) algorithm, see \cite{train2008algorithms} for details. The number of iterations has been set to 50 for the estimation problem of LC-MNL with 10 classes and  40 times for the estimation problem of LC-MNL with 2 classes. For some rare randomly generated instances, we allow the iteration up to 100 to ensure convergence.

\subsection{Prediction experiment implementation details}  \label{sec:prediction_exp}
\subsubsection{Revenue prediction with MDM using estimate-then-predict} \label{sec:prediction_exp_mdm}
For an unseen assortment $A$, we follow the revenue prediction procedure in Figure \ref{fig:flowchat} to obtain the revenue estimate of the assortment $A$.

\subsubsection{Revenue prediction with MDM under data uncertainty} \label{sec:prediction_exp_mdm_ci}
Considering the data uncertainty, for any given $\bp_\cS$, the worst case revenue prediction for assortment $A$ is computed as follows:
\begin{align}
\begin{aligned}
\min_{\bx_A,\blambda,\bdelta\!\!,\,}\quad&  \quad  \sum_{i \in A}  r_i x_{i,A}\\
 \text{s.t.}\quad \ -&\delta_{A,S} \ \leq\    \lambda_A  -  \lambda_S \ \leq \  1 -  (1+\epsilon)\delta_{A,S}, \qquad\quad\  \ \forall \,  i \in A, \, (i,S) \in \cI_\cS, \\
 -&\delta_{S,A}  \ \leq\    \lambda_S  -  \lambda_A \ \leq \  1 -  (1+\epsilon)\delta_{S,A}, \qquad\quad\  \ \forall \,  i \in A, \, (i,S) \in \cI_\cS, \\
&  \delta_{A,S} - 1 \ \leq \  x_{i,A} - p_{i,S} \ \leq \ 1 - \delta_{S,A}, \quad\qquad\quad\,   \forall \,  i \in A, \, (i,S) \in \cI_\cS,\\
 -&(\delta_{A,S}+  \delta_{S,A}) \ \leq \ x_{i,A} - p_{i,S} \ \leq \ \delta_{A,S} + \delta_{S,A},\ \quad     \forall \,  i \in A, \, (i,S) \in \cI_\cS,\\ 
& \lambda_S - \lambda_T \ \geq \ \epsilon, \qquad\qquad\qquad\! \forall \, (i,S), (i,T) \in \cI_\cS \ \textnormal{ s.t. }\  p_{i,S} < p_{i,T},\\
&\lambda_S - \lambda_T \ = \ 0, \qquad\qquad\quad\ \ \, \forall \, (i,S), (i,T) \in \cI_\cS \ \textnormal{ s.t. }\  p_{i,S} = p_{i,T} \neq 0 ,\\
& \sum_{i \in A} x_{i,A} \ = \ 1,   \\
& 0 \leq \lambda_A \leq 1,\quad\  x_{i,A} \ \geq\  0,  \ \forall \,  i \in A ,\quad\ 0 \, \leq\,  \lambda_S \leq\,  1, \quad  \delta_{A,S}, \delta_{S,A} \, \in \, \{0,1\}, \  \forall \, S \in \cS,\\
& p_{i,S} (1-z \alpha_{i,S}) \  \leq \ x_{i,S} \ \leq \ p_{i,S} (1+z \alpha_{i,S}),  \quad\qquad\quad\,   \forall \,  (i,S) \in \cI_\cS,
\end{aligned} \label{model:mdm_ci}
\end{align}
where $\alpha_{i,S} \sqrt{\frac{1-p_{i,S}}{n_{i,S}}}$ with $n_{i,S} = n_S * p_{i,S}$ representing the number of customers chose product $i$ when assortment $S$ is offered. Here, \( p_{i,S} \alpha_{i,S} \) is the standard error, and \( z \) is a constant multiplier that determines the width of the confidence interval. By replacing the "min" operator with the "max" operator, we obtain the best-case revenue estimate for assortment $A.$

\subsubsection{Revenue prediction with RUM}  \label{sec:prediction_exp_rum}
Given the training data $\bp_\cS$, for an unseen assortment $A$, we estimate the revenue of $A$ under the RUM by following the robust prediction method in \citep{farias2013nonparametric}. To obtain the revenue estimate of $A$ by solving the following linear optimization problem:
\begin{align}\label{model:rum_prediction}
    \min_{\bx_A,\bx_\cS,\blam}&\quad \boldsymbol{r}^T \bx_A\nonumber\\
    \text{s.t.}\quad 
    & \sum_{\sigma \in \Sigma_n}  \lambda(\sigma) \mathbb{I}[\sigma, i, S] - x_{i,S} =0, \quad \forall (i,S) \in \cI_\cS,\nonumber \\
    & p_{i,S} (1-z\epsilon_{i,S}) \leq x_{i,S} \leq p_{i,S} (1+z\epsilon_{i,S}),\quad \forall (i,S) \in \cI_\cS, \nonumber\\
    & \sum_{\sigma \in \Sigma_n}  \lambda(\sigma) \mathbb{I}[\sigma, i, S] - x_{i,A} =0, \quad \forall i\in A, \\
    & \sum_{\sigma \in \Sigma_n} \lambda(\sigma)=1,\quad \lambda(\sigma) \geq 0, \quad \forall \sigma \in \Sigma_n, \nonumber
\end{align}
where $\epsilon_{i,S} = \sqrt{\frac{1-p_{i,S}}{n_{i,S}}}$ with $n_{i,S} = n_S * p_{i,S}$ representing the number of customers chose product $i$ when assortment $S$ is offered and $z$ is a constant multiplier that determines the width of the confidence interval. For our real data experiments, we have set $z$ to be 2.5758, which corresponded to the smallest value of $z$ of which \eqref{model:rum_prediction} was feasible over all tested instances; incidentally, this value of $z$ corresponds to approximately 99\% confidence interval for $ x_{i,S}.$ For the synthetic data experiment, we have set z to be 2.807, which corresponds to approximately 99.5\% confidence interval for $ x_{i,S},$ which corresponds $z$ to be 2.807. We report the infeasible proportion of \eqref{model:rum_prediction} in the tested instances. 

\subsubsection{Revenue prediction with MNL, LC-MNL} \label{sec:prediction_exp_mnl_lc_mnl}
For parametric models MNL and LC-MNL, we first estimate the parameters of the models with maximum likelihood estimation. With the estimators of the model parameters, we then compute the choice probabilities of the products in the given unseen assortment $A$ and evaluate the revenue of assortment  $A$.

\subsubsection{Revenue prediction with MCCM} \label{sec:prediction_exp_mccm}
For MCCM, we first estimate the parameters of MCCM, the arrival rates $\boldsymbol{\lambda}$ and transition matrix $\boldsymbol{\rho}$ with maximum likelihood estimation. We then obtain the estimated revenue for an assortment $A$ by solving the following linear optimization problem \citep{feldman2017revenue}:
\begin{align}\label{model:mccm_revenue}
    \max_{\bx, \boldsymbol{y} } \quad & \sum_{i\in \cN_0} \ r_i x_{i}\nonumber\\
    \text{s.t.} \quad & x_i = \lambda_i + \sum_{j\in \cN_0 \setminus A} y_j \rho_{j,i},\quad \forall i\in A,\\
    & y_i = \lambda_i + \sum_{j\in \cN_0 \setminus A} y_j \rho_{j,i},\quad \forall i\in \cN_0 \setminus A,\nonumber \\
    & x_i = 0,\quad \forall i \in \cN_0 \setminus A,\nonumber \\
    & y_i = 0,\quad \forall i\in A,\nonumber \\
    & \sum_{i\in A} x_i = 1,\nonumber \\
    & x_i, y_i \geq 0,\ \forall i \in \cN_0\nonumber.
\end{align}

\subsubsection{Average Ranking of Models Based on Prediction Performance} \label{sec:prediction_exp_ranking}

For the prediction experiments, we evaluate and rank five different models, including MDM, RUM, MNL, MCCM, and LC-MNL, based on their prediction performance across 50 testing instances. Specifically, we assess each model using a designated performance metric (Kendall Tau distance or relative revenue regret of the top-ranked assortment) and rank them from 1 to 5 for each testing instance. Rank 1 indicates the best performance, while rank 5 indicates the worst. For each of the 50 testing instances, the models are ranked individually. In cases where two or more models achieve identical performance scores within a testing instance, they are considered tied. Tied models are assigned an average rank that reflects their shared position. For example, if two models are tied for the second-best performance, they both receive a rank of 2.5 (the average of ranks 2 and 3). The remaining models are ranked accordingly, with subsequent positions adjusted to account for any ties. After ranking the models across all 50 testing instances, we calculate the average ranking for each model. This average rank provides an overall performance summary, indicating which models consistently perform better or worse across the multiple testing instances.


\section{Details on Synthetic Data Generation}
\label{sec:synthetic_data_generation}

\subsection{A general procedure of generating instances for synthetic data experiments in Section \ref{sec:synthetic_exp} and Section \ref{sec:exp_sampling_uncertainty}}
\label{sec:synthetic_data_multinational_draws}

Generate 50 instances with underlying models, either HEV or Probit with negative correlations in utilities. For each instance $\bp_\cS$, we generate with the following Step 1 - 3. 
\begin{description}
    \item \textbf{Step 1:} Generate true underlying choice probabilities of the chosen model (details to be provided in the next two subsections).
    \item \textbf{Step 2:} Generate multinomial samples for purchased records based on step 1.
    \begin{description}
        \item \textbf{Input:} The collection of choice probabilities from Step 1, the historical assortments, and their offer times.
        \item \textbf{Procedure:} For each instance: generate multinomial samples based on the choice probabilities, the assortments, and the offer times.
        \item  \textbf{Output:} A collection of multinomial samples representing product choices within each assortment.
    \end{description}
    \item \textbf{Step 3:} Obtain choice probabilities $\bp_\cS$ by computing average purchase frequencies.
    \begin{description}
        \item \textbf{Input:} The multinomial samples from Step 2, the collection of historical assortments. 
        \item \textbf{Procedure:} For each instance, compute the average purchase frequency of each product in each assortment based on the samples generated in Step 2. These frequencies form the observed choice data $\bp_\cS$ for the experiments.
        \item \textbf{Output:} A collection of average purchase frequencies for each product across the assortments.
    \end{description}
\end{description}

\subsubsection{HEV instance in Section \ref{sec:synthetic_exp}}
\label{sec:data_gen_hev}
Recall that the HEV model is a parametric subclass of the random utility model and assumes independent extreme value error distributions with nonidentical scales. There is no closed-form expression for computing choice probabilities. For each underlying HEV instance, we generate with following procedure: 

\begin{description}

    \item \textbf{Step 1:} Generate deterministic utilities $\{\nu_i: i\in \cN \cup \{0\} \} $ and scales parameters $\{ \beta_i : i\in \cN \cup \{0\} \}$
    \begin{description}
        \item \textbf{Input:} A collection of products and the corresponding price of each product $\{r_i: i\in \cN  \} $. 
        \item \textbf{Procedure for generating deterministic utilities  $\{\nu_i: i\in \cN \cup \{0\}  \} $ } 
        \begin{itemize}[leftmargin=0pt,labelindent=0pt]
            \item Compute the mean of prices $\bar{r} = \sum_{i\in \cN} r_i.$
            \item Compute the standard deviation of $\{r_i: i\in \cN  \} $ and denote it as $\delta.$
            \item Let $z_i = (r_i - \bar{r}) / \delta,$ for all $i \in \cN.$
            \item Generate $\nu_i = \rho \, z_i + \sqrt{1-\rho^2} \, w_i,$ for all $i \in \cN,$ where $\rho = -0.5$ and $w_i$ is a random variable that follows the standard normal distribution, that is $w_i \sim N(0,1).$
            \item Let the deterministic utility of the outside option be greater than the deterministic utility of each product by setting $\nu_0 = \max_{i\in \cN} \nu_i +1. $
        \end{itemize}
        \item \textbf{Procedure for generating scales parameters $\{ \beta_i : i\in \cN \cup \{0\}\}$: } Let $\beta_0 = 10, $ and $\beta_i$ be generated from a uniform distribution with range $[0.04,1],$ for all $i \in \cN.$ This indicates that the utility of the outside option has the largest variance. 
        \item \textbf{Output:} Deterministic utilities $(\nu_i: i\in \cN \cup \{0\} )$ and scales parameters $( \beta_i: i\in \cN \cup \{0\}).$ 
    \end{description}
    
    \item \textbf{Step 2: } Generate 10000 utility samples.
    \begin{description}
        \item \textbf{Input:} 10000 repetitions of Step 1 to generate deterministic utilities $(\nu_i^k: i\in \cN \cup \{0\} )$ and scales parameters $( \beta_i^k: i\in \cN \cup \{0\}),$ for $k = 1,\ldots, 10000.$ 
        \item \textbf{Procedure:}  Generate 10000 utility samples. For each utility sample $\{\tilde{U}_i^k: i\in \cN \cup \{0\} \}$,  it is draw from a Gumbel distribution with mean $\{\nu_i^k: i\in \cN \cup \{0\}  \} $ and scales $\{ \beta_i^k : i\in \cN \cup \{0\}\}$.
        \item \textbf{Output:} Utility samples $\{\tilde{U}_i^k: i\in \cN \cup \{0\} \}$ for $k = 1,\ldots, 10000.$ 
    \end{description}
    \item \textbf{Step 3:} Compute the number of times each alternative in each assortment is purchased. 
    \begin{description}
        \item \textbf{Input: } Utility samples $\{\tilde{U}_i^k: i\in \cN \cup \{0\} \}$ for $k = 1,\ldots, 10000,$ and the collection of historical assortments. 
        \item \textbf{Procedure:}
        \begin{itemize}[leftmargin=0pt,labelindent=0pt]
            \item For each utility sample $\{\tilde{U}_i^k: i\in \cN \cup \{0\} \}$, determine the product purchased by the customer based on the utility maximization principle. For every assortment $S$ in the collection, the product $i\in S$ is purchased if $i=\argmax_{j\in S \cup \{0\}} \tilde{U}_j^k.$
            \item Obtain the number of purchasing times of each product in each assortment based on the 10000 utility samples.
        \end{itemize} 
        \item \textbf{Output:} The number of purchasing times of each product in each assortment given 10000 purchasing records. 
    \end{description}
    \item \textbf{Step 3:} Obtain choice probabilities $\bp_\cS$ by dividing the number of purchasing times of each product in each assortment by 10000.
\end{description}

\subsubsection{Probit instance with negative correlations in utilities in Section \ref{sec:synthetic_exp}} \label{sec:data_gen_Probit}
Recall that the Probit model is a parametric subclass of the random utility model and assumes that the stochastic components of utilities follow a Gaussian distribution. There is no closed-form expression for computing choice probabilities.  For each underlying Probit instance with negative correlations in utilities, the instance generation procedure is similar to the above for the HEV instances. The key difference lies in Step 1 to generate model-dependent parameters and a slight difference in Step 2 to generate utility samples while Steps 3 and 4 are the same. Hence, we only give details for Steps 1 and 2 in the following. 
\begin{description}
    \item \textbf{Step 1:} Generate deterministic utilities $\{\nu_i: i\in \cN \cup \{0\} \} $ and covariance matrix $\{ \Sigma_{ij} : i,j \in \cN \cup \{0\} \}$.
    \begin{description}
        \item \textbf{Input:} A collection of products. 
        \item \textbf{Procedure for generating deterministic utilities  $\{\nu_i: i\in \cN \cup \{0\}  \} $} 
        \begin{itemize}[leftmargin=0pt,labelindent=0pt]
            \item Generate $\nu_i$ for all $i \in \cN,$ such that $\nu_i$ is a random draw from the uniform distribution with range $[-2,2].$
            \item Let the deterministic utility of the outside option be greater than the deterministic utility of each product, where  $\nu_0$ is a random draw from the uniform distribution with range $[6,10].$
        \end{itemize}
        \item \textbf{Procedure for generating the covariance matrix $\{ \Sigma_{ij} : i,j \in \cN \cup \{0\} \}$}
        \begin{itemize}[leftmargin=0pt,labelindent=0pt]
            \item Generate  $\{ \Sigma_{ij} : i,j \in \cN \cup \{0\} \}$ as a positive semidefinite matrix with negative correlation. 
            \item For the diagonal entries which represent variances of the utilities of products, let $\Sigma_{00}$ be a random draw from the uniform distribution with range $[30,40],$ and for each $i\in \cN$, let $\Sigma_{ii}$ be a random draw from the uniform distribution with range $[0.5,2].$ 
        \end{itemize}
        \item \textbf{Output:} The deterministic utilities $\{\nu_i: i\in \cN \cup \{0\} \} $ and the covariance matrix $\{ \Sigma_{ij} : i,j \in \cN \cup \{0\} \}$.
    \end{description}
    
    \item \textbf{Step 2: } Generate 10000 utility samples.
    \begin{description} [leftmargin=0pt]
        \item \textbf{Input:} 10000 repetitions of Step 1 to generate deterministic utilities $(\nu_i^k: i\in \cN \cup \{0\} )$ and  
 the covariance matrix $\{ \Sigma_{ij}^k : i,j \in \cN \cup \{0\} \}$, for $k = 1,\ldots, 10000.$ 
        \item \textbf{Procedure:}  Generate 10000 utility samples. For each utility sample $\{\tilde{U}_i^k: i\in \cN \cup \{0\} \}$,  it is draw from a multivariate Gaussian distribution with mean $\{\nu_i^k: i\in \cN \cup \{0\}  \} $ and covariance matrix $\{ \Sigma_{ij}^k : i,j \in \cN \cup \{0\} \}$ .
        \item \textbf{Output:} Utility samples $\{\tilde{U}_i^k: i\in \cN \cup \{0\} \}$ for $k = 1,\ldots, 10000.$ 
    \end{description} 
\end{description}

\subsection{Data Generation for experiments in Section \ref{sec:discuss_representability}} \label{sec:data_gen_true_hev_probit}

\subsubsection*{HEV instances}
We generate HEV instances where the deterministic utilities of products are randomly drawn from a uniform distribution on the interval $[-2,2]$. The scale parameters of the Gumbel distributions are randomly drawn from a uniform distribution on the interval $[0.04,2]$.

\subsubsection*{Probit instances}
We generate Probit instances where the deterministic utilities of products are randomly drawn from a uniform distribution on the interval $[-2,2]$. The off diagonals of the covariance matrix, which capture the correlation of utilities, are randomly drawn from a uniform distribution on the interval $[-0.9,-0.1].$ The diagonals of the covariance matrix, which capture the variances of utilities, are randomly drawn from a uniform distribution on the interval $[0.5,2].$ 

\subsection{Data generation for experiments in Section \ref{sec:exp_sampling_uncertainty}} \label{sec:data_gen_mem}
We use the collection of historical assortments $\cS$ from the JD.com dataset with $n_\cS \geq 20$, where we have 29 assortments and the price of the dataset. To illustrate the effect of sampling uncertainty where the prediction performance will be affected, we generate underlying parametric MDM with marginals being non-identical exponential distributions, which is known as the marginal exponential model (MEM). 

The data used in the prediction experiment of Section \ref{sec:exp_sampling_uncertainty} is generated with the following procedure:

\textbf{Step 1:} For the collection $\cS$, we generate an underlying MEM instance with the deterministic utilities of products being randomly generated from uniform [-5,5] and the scale parameters being randomly generated from [1,10]. 

\textbf{Step 2:} We randomly select an assortment from the historical assortments and call it $A.$ We compute the true revenue of $A$ as $r_A^{*}.$

\textbf{Step 3:} We keep $\cS^\prime = \cS \setminus \{A\}$ as the training collection of assortments. 

\textbf{Step 4:} We generate the training instances under two testing scenarios 50 times. For each time, 
we generate two groups of choice data, with the first group of choice data generated with the assumption that each assortment has been offered 20 times, while the other group of choice data generated with the assumption that each assortment has been offered 1000 times. With this setup, we generate the training instances by following the procedure in Section \ref{sec:synthetic_data_multinational_draws} and get the training choice data $\bp_\cS^\prime$ with assortment offered 20 times and $\bp_\cS^\prime$ with assortment offered 1000 times.

\textbf{Step 5:} For any given $\bp_\cS^\prime$, the worst-case revenue prediction and the best-case revenue prediction under data uncertainty for assortment $A$ are computed with \eqref{model:mdm_ci}. Here, we choose a 95\% confidence interval.

\textbf{Step 6:} For the 50 instances of each group, we do the worst-case and best-case revenue estimate for $A,$ and report the corresponding revenue estimates and the proportion of being feasible for both groups. Note that the prediction problem may not be always feasible.

\subsection{Data generation for experiments in Section \ref{sec:exp_correlation}} \label{sec:coupla_data}
 
We generate parametric RUM instances where the deterministic utilities of products are randomly drawn from a uniform distribution on the interval $[-2,2]$ and the scale parameters of the Gumbel distributions are randomly drawn from a uniform distribution on the interval [0.04,2].

\textbf{Instances with Independent Copula:} 50 instances are generated using independent copulas, following the procedure outlined in Section \ref{sec:data_gen_hev}. The deterministic utilities and scale parameters are set according to the aforementioned uniform distributions.

\textbf{Instances with Comonotonic Copula:} 50 instances are generated with Comonotonic copulas, following a modified version of the procedure in Section\ref{sec:data_gen_hev}. The key difference lies in Step 2: instead of generating utility samples independently, a single uniform random variable $U \in (0,1)$ is drawn for each utility sample. The utilities for all products are then generated by applying the inverse transform of the marginals for each product.}

{\color{black}\section{Robustness of numerical results in Section \ref{sec:synthetic_exp}} \label{sec:addtional_exp_robustness}
In Section \ref{sec:synthetic_exp}, we use the product-assortment structure from the JD.com dataset to generate random instances for our experiments. A natural question is whether the performance of the compared models remains robust under different assortment structures. To investigate this, we generate random HEV-based instances with varied assortment structures and evaluate the models’ estimation and prediction performance. The results reported under this randomly generated assortment setting are consistent with the findings  in Section \ref{sec:synthetic_exp}.

\subsubsection*{Data generation.} 
Similar to the JD.com dataset, we consider $n = 8$ products, with the outside option always included in each assortment. The price of each product is set to the same as in the JD.com dataset. 
We construct eight distinct test scenarios that primarily differ in the assortment structures of the collection $\cS,$ specifically in how likely assortments of different sizes are to be included in $\cS.$ We consider two generation strategies for the assortment structures of the collections: one that favors the inclusion of smaller sized assortments and another that favors larger sized assortments.
Let $k \in \{1, 2, 3, 4\}$ denote the number of products in an assortment. We define a probability distribution mass function $p(k)$ parameterized by $\alpha \in \{0.6, 0.7, 0.8, 0.9\}$. Under the strategy favoring the inclusion of assortments with fewer products in $\cS$, we set $p(k) = \alpha^k / \sum_{j=1}^4 \alpha^j.$  Under the strategy favoring the inclusion of assortments with more products in $\cS$, we set $p(k) = {\alpha^{5-k}}/{\sum_{j=1}^4 \alpha^j}.$ 
Specifically, given $p(k)$ for $ k \in \{1, 2, 3, 4\}$, for each $\alpha \in \{0.6, 0.7, 0.8, 0.9\}$ and each fixed collection size $\vert \cS \vert \in \{30,25,20,15,10\},$ we generate a random assortment collection $\cS$ as follows. For $i = 1,\ldots, \vert \cS \vert,$ 
\begin{itemize}
    \item generate an independent sample $X_i$ with probability mass function $P(X_i = k) = p(k),$ for $k = 1,2,3,4.$
    \item given $X_i$ generated as above, take assortment $S_i$ to be the subset selected uniformly at random from all subsets of $\{1,2,\ldots,8\}$ whose
    size is $X_i.$
    \item add the outside option to $S_i$ and include the resulting $S_i$ in the collection $\cS.$
\end{itemize}
For every collection $\cS$ generated above, we generate 20 random HEV-based instances under $\cS$. Every such instance $\bp_\cS$ is generated by following the procedure described in Section~\ref{sec:synthetic_data_multinational_draws} with $\mu_0 = \max_{i \in \cN} \mu_i + 10$ to specify the deterministic part of the utility of the outside option. The number of offer times $n_S$ of each assortment $S \in \cS$ is sampled based on its size $|S|$ as follows: if $|S| = 2$, then $n_S \sim \text{Uniform}(20, 2000)$; if $|S| = 3$, then $n_S \sim \text{Uniform}(20, 250)$; if $|S| = 4$, then $n_S \sim \text{Uniform}(20, 60)$; if $|S| = 5$, then $n_S \sim \text{Uniform}(20, 30)$. Overall, we evaluate the performance of models with a total of 800 HEV-based instances, consisting of 20 random instances for each combination of collection size in $\{30,25,20,15,10\}$ and each of the 8 test scenarios with different assortment structures of collections.

\subsubsection*{Comparison of explanatory, predictive, and prescriptive abilities.} 
The performance of the compared models in estimation and prediction is almost consistent across different test scenarios, so we present results from a representative scenario. Specifically, we select the case with $\alpha = 0.8$ and $p(k) = \alpha^k / \sum_{j=1}^4 \alpha^j,$ for $k \in \{1,2,3,4\} $, where assortments with smaller sizes are more likely to be included in the assortment collection.

We report the average absolute deviation loss and standard error over the 20 randomly generated HEV-based instances under each considered collection size for the MDM, RUM, MNL, MCCM, and LC-MNL in Table \ref{tab:hev-estimation-alpha}. The results show that the models perform similarly to their performance on the JD.com dataset. Notably, the nonparametric MDM and RUM models achieve significantly lower average absolute deviation loss when approximating HEV-based instances, indicating superior explanatory power compared to the parametric models. Specifically, MDM exhibits a loss reduction of approximately 85\% compared to the best-fitting MNL, 47\% compared to MCCM, and 37\% compared to LC-MNL for the largest collection size considered.  Additionally, similar to experiments in Section \ref{sec:discuss_representability}, due to the nonconvex feasible region of the limit of MCCM, some rare randomly generated instances encounter convergence issues when solved using a continuous solver (IPOPT). The absolute deviation loss that we report for MCCM in Table \ref{tab:hev-estimation-alpha} is only for the instances that successfully converged to an acceptable level.  
To further validate MDM's estimation performance, we also compare the average KL loss across MDM, MNL, and LC-MNL.
Although we report sub-optimal KL losses for MDM due to computational challenges, the results in Table \ref{tab:hev-kl-alpha} demonstrate that both MDM and LC-MNL significantly outperform MNL across all tested scenarios, further highlighting MDM's robustness in accurately estimating choice probabilities.
We do not report the computational time of the models, as the estimations were performed on different machines, making direct comparisons unreliable. However, we note that solving the limit of MDM can be computationally intensive, particularly when $|\cS| = 30$ or when collections are with larger size. We leave it to future research to explore more efficient formulations and algorithms that can improve the scalability of solving the limit of MDM.

Next, we report the quality of predicted rankings of the tested collection of assortments captured by Kendall Tau distances in Table \ref{tab:hev-kendall-alpha} and the average relative error in the revenue of the top-ranked assortment in Table \ref{tab:hev-regret-alpha}. For executing the robust approach to RUM developed in \cite{farias2013nonparametric} in these tables, we have taken the instances to be lying in the RUM representable region with 99.99\% confidence. Robust revenue prediction using RUM has not always been feasible with this choice despite 99.99\% being a very high confidence level to use: we find  20\% of instances in the dataset obtained by restricting to $|\cS|= 30$ turn out to be infeasible under RUM; 10\% of instances are not feasible when $|\cS|= 20$; and 5\% instances are infeasible for  $|\cS|= 15$ and  $|\cS|= 10$. Additionally, similar to experiments in Section \ref{sec:discuss_representability}, due to the nonconvex feasible region of the MLE formulation of MCCM, some rare randomly generated instances encounter convergence issues when solved using a continuous solver (IPOPT). Therefore, in Tables \ref{tab:hev-kendall-alpha} - \ref{tab:hev-regret-alpha}, the numbers reported for RUM are obtained by restricting only to feasible instances, and the numbers reported for MCCM are obtained by restricting only to instances that successfully converged to an acceptable level. 
Overall, from Tables \ref{tab:hev-kendall-alpha} and \ref{tab:hev-regret-alpha}, we infer that MCCM and LC-MNL perform comparably well, followed by MDM, then MNL, with RUM ranking near the last among the models considered.

\vspace{-0.3cm}
\begin{table}[htb]
\centering
\caption{Average absolute deviation loss ($10^{-3}$) comparison with HEV-based instances when $\alpha=0.8$}
\label{tab:hev-estimation-alpha}
\scalebox{0.85}{\begin{tabular}{c|c|c|c|c|c}
\hline
$|\cS|$ & MDM & RUM & MNL & MCCM & LC-MNL \\
\hline
30 & 4.90 (0.44) & 4.76 (0.39) & 33.61 (1.52) & 9.20 (3.52) & 7.77 (0.71) \\
25 & 4.97 (0.45) & 4.59 (0.43) & 32.61 (1.76) & 10.59 (3.74) & 7.39 (0.58) \\
20 & 4.74 (0.73) & 4.19 (0.71) & 28.33 (1.39) & 5.73 (1.05) & 7.12 (0.85) \\
15 & 5.09 (0.77) & 4.17 (0.64) & 26.31 (2.34) & 4.78 (0.78) & 7.81 (1.04) \\
10 & 3.42 (0.55) & 2.60 (0.48) & 18.43 (1.80) & 3.45 (0.68) & 5.93 (1.05) \\
\hline
\end{tabular}}
\vspace{0.15cm}
\begin{flushleft}\footnotesize Notes. Standard errors are reported in parentheses. \end{flushleft}
\end{table}

\vspace{-0.5cm}
\begin{table}[htb]
\centering
\caption{Comparison on average KL loss ($10^{-3}$) among MDM, MNL, and LC-MNL with HEV-based instances when $\alpha=0.8$}
\label{tab:hev-kl-alpha}
\scalebox{0.85}{\begin{tabular}{c|c|c|c}
\hline
$|\cS|$ &MDM & MNL & LC-MNL \\
\hline
30 & 4.39 (0.43)  & 513.30 (4.12) & 2.54 (0.33) \\
25 & 7.52 (1.19)  & 514.76 (5.08) & 2.59 (0.35) \\
20 & 9.59 (4.82)  & 515.60 (6.30) & 2.22 (0.31) \\
15 & 6.44 (1.07)  & 511.15 (8.17) & 2.44 (0.48) \\
10 & 4.48 (1.19)  & 520.00 (8.53) & 1.34 (0.25) \\
\hline
\end{tabular}}
\vspace{0.15cm}
\begin{flushleft}\footnotesize Notes. Standard errors are reported in parentheses. \end{flushleft}
\end{table}

\vspace{-0.5cm}
\begin{table}[htb!]
\centering
\caption{Kendall Tau distance comparison with HEV-based instances when $\alpha = 0.8$}
\label{tab:hev-kendall-alpha}
\scalebox{0.85}{\begin{tabular}{cc|c|c|c|c|c}
\hline
$|\cS|$ & $|\mathcal{A}|$ & MDM & RUM & MNL & MCCM & LC-MNL \\
\hline
24 & 6 & 4.85 (0.49) & 6.06 (0.68) & 5.70 (0.49) & 4.06 (0.41) & \textbf{4.05} (0.41) \\
20 & 5 & 4.05 (0.46) & 5.30 (0.49) & 4.60 (0.45) & \textbf{3.21} (0.44) & 3.40 (0.29) \\
16 & 4 & 1.65 (0.30) & 2.33 (0.36) & 1.65 (0.27) & 1.41 (0.35) & \textbf{1.05} (0.22) \\
12 & 3 & 1.50 (0.18) & 1.63 (0.20) & 1.50 (0.21) & \textbf{1.05} (0.17) & \textbf{1.05} (0.15) \\
8  & 2 & \textbf{0.15} (0.08) & 0.42 (0.11) & 0.30 (0.10) & 0.29 (0.11) & 0.30 (0.10) \\
\hline
\end{tabular}}
\vspace{0.15cm}
\begin{flushleft}\footnotesize Notes. Standard errors are reported in parentheses. Bold values indicate the best-performing model in each row. \end{flushleft}
\end{table}

\vspace{-0.5cm}
\begin{table}[htb!]
\centering
\caption{Comparison of relative error (\%) in revenue of the top-ranked assortment with HEV-based instances when $\alpha = 0.8$}
\label{tab:hev-regret-alpha}
\scalebox{0.85}{\begin{tabular}{cc|c|c|c|c|c}
\hline
$|\cS|$ & $|\mathcal{A}|$ & MDM & RUM & MNL & MCCM & LC-MNL \\
\hline 
24 & 6 & 25.59 (6.40) & 30.64 (5.84) & 28.47 (4.99) & 12.89 (3.86) & \textbf{10.17} (3.41) \\
20 & 5 & 23.79 (5.04) & 30.90 (6.32) & \textbf{22.30} (5.11) & 16.76 (4.93) & 23.53 (5.17) \\
16 & 4 & 14.56 (4.84) & 27.89 (6.46) & 21.04 (5.21) & 14.48 (5.40) & \textbf{6.06} (2.81) \\
12 & 3 & 25.99 (5.61) & 26.43 (4.97) & 26.91 (5.88) & \textbf{13.37} (4.83) & 17.01 (5.67) \\
8  & 2 & \textbf{2.52} (1.61) & 14.78 (5.11) & 6.57 (2.92) & 8.84 (4.15) & 5.18 (2.67) \\
\hline
\end{tabular}}
\vspace{0.15cm}
\begin{flushleft}\footnotesize Notes. Standard errors are reported in parentheses. Bold values indicate the best-performing model in each row. \end{flushleft}
\end{table}
\vspace{-0.4cm}

\subsubsection*{Comparison of overall predictive and prescriptive abilities.}

To assess the overall predictive and prescriptive performance of the models across all instances of all testing scenarios, we summarize the distributions of their rankings based on the average Kendall Tau distance of testing assortments and the average relative error in revenue of the top-ranked assortment in Tables~\ref{tab:overall_kendall_alpha} and~\ref{tab:overall_rev_reg_alpha}, respectively. In total, 800 instances were tested. After excluding instances where nonparametric RUM prediction was not feasible or where the MLE estimation of MCCM failed to converge, 634 instances remained for comparison.
When ties occur in model rankings, we assign the same rank to the tied models. 
Based on the results, all models demonstrate a reasonable ability to predict both the Kendall Tau distances of testing assortments and the best assortments in the test sets, meaning that the models are often tied at the Rank 1 position.  However, their effectiveness varies in terms of how frequently they achieve top rankings, with some models consistently outperforming others in securing the first rank.
Overall, both tables indicate that MCCM and LC-MNL perform comparably well, followed by MDM, while RUM and MNL exhibit the weakest performance among the models considered.

\begin{figure}[htb!]
    \centering
    \begin{minipage}{0.48\textwidth}
        \centering 
\vspace{-11pt}
\begin{table}[H]
\caption{Distributions (\%) of Models' Ranking based on Kendall Tau Distance}
\begin{tabular}{|c|c|c|c|c|c|}
\hline
Rank   & 1     & 2     & 3     & 4     & 5     \\ \hline
MDM    & 45.74 & 18.45 & 17.67 & 11.99 & 6.15  \\ \hline
RUM    & 38.33 & 10.57 & 12.30 & 14.35 & 24.45 \\ \hline
MNL    & 39.27 & 13.72 & 15.14 & 20.98 & 10.88 \\ \hline
MCCM   & 58.04 & 18.77 & 10.88 & 9.15  & 3.15  \\ \hline
LC-MNL & 57.89 & 22.87 & 10.41 & 6.62  & 2.21  \\ \hline
\end{tabular}\label{tab:overall_kendall_alpha}
\vspace{0.15cm}
\begin{flushleft}\footnotesize Notes. Each row represents the distribution of rankings of each model. The sum of values along each column does not equal 100\% due to ties between models. These ties occur most frequently at the top rank.  \end{flushleft}
\end{table}
    \end{minipage}
    \hfill
    \begin{minipage}{0.48\textwidth}
    \vspace{-1cm}
\begin{table}[H]
\caption{Distributions (\%) of Models' Ranking based on Relative Error in Revenue of the Top-ranked Assortment}
\begin{tabular}{|c|c|c|c|c|c|}
\hline
Rank   & 1     & 2     & 3     & 4     & 5     \\ \hline
MDM    & 52.84 & 8.68  & 9.15  & 19.24 & 10.09 \\ \hline
RUM    & 46.37 & 5.36  & 8.83  & 18.30 & 21.14 \\ \hline
MNL    & 50.79 & 10.88 & 12.46 & 11.67 & 14.20 \\ \hline
MCCM   & 65.30 & 11.99 & 11.04 & 7.57  & 4.10  \\ \hline
LC-MNL & 66.72 & 12.62 & 11.36 & 7.10  & 2.21  \\ \hline
\end{tabular}\label{tab:overall_rev_reg_alpha}
\vspace{0.15cm}
\begin{flushleft}\footnotesize Notes. Each row represents the distribution of rankings of each model. The sum of values along each column does not equal 100\% due to ties between models. These ties occur most frequently at the top rank.  \end{flushleft}
\end{table}
\end{minipage}
\end{figure}}


\section{Additional numerical results with synthetic data}
\label{sec:addtional_exp_results}
In Experiments 1 - 2 below, we compare the representational ability of MDM with RUM and MNL model. Experiment 3 compares the prediction performance offered by the nonparametric approach proposed in this paper with that offered by models involving parametric assumptions. Experiment 4 compares the limit of approximating choice probabilities with  MDM, RUM, and MNL models. Additional useful details on the precise setup of all the experiments are furnished in  \ref{sec:implementation}.

\subsection{The representation power and tractability of MDM compared to RUM and MNL.}\label{sec:addtional_exp_results_rep}
In Experiment 1, we investigate the representational power of MDM for a large number of alternatives ($n = 1000$) by randomly perturbing choice probabilities obtained from an underlying MNL model. We test for the fraction of instances that can be represented by MDM where the parameter $\alpha$ controls the fraction of choice probabilities that are perturbed from the MNL model (a larger value indicates more entries are modified from the underlying MNL model). While checking the representability of these models can be done by solving linear programs, RUM quickly becomes intractable as $n$ increases. In Figure \ref{fig:large_multiply_noise}, we see that even with small perturbations to the choice probabilities of the underlying MNL model, none of the MNL models can represent the perturbed choice data. However, MDM which subsumes the MNL model can capture many of these instances. This shows that MDM is a much more robust model than MNL model. The runtimes for these large instances were less than 1 second. The computational requirements for RUM make it impossible to run at this scale. 

In Experiment 2, we compare the representational power and computational time for MDM and RUM for a small number of alternatives. We find that both MDM and RUM show good representational power: In particular, with the collection size $\vert \cS \vert = 20$, round 80 percent of the instances can still be represented by MDM when 25 percent of the choice probability entries are perturbed; this drops to 60 percent when 100 percent of entries are perturbed,  RUM has better representation power in these examples (see Figure \ref{fig:small_multiply}). However, this comes at a significant run time cost even at this scale as seen in Figure \ref{fig:small_multiply} as compared to MDM. 
\begin{figure}[htb!]
    \centering   
    \includegraphics[scale=0.5]{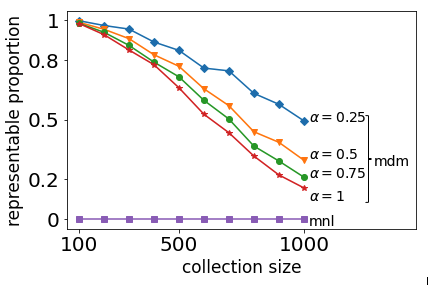}
    \caption{The representational power of MDM }
\label{fig:large_multiply_noise}
\end{figure}

\begin{figure}[htb!]
\centering
\begin{minipage}[t]{0.45\textwidth} 
\centerline{ Representation Power}
\centering
\includegraphics[scale=0.45]{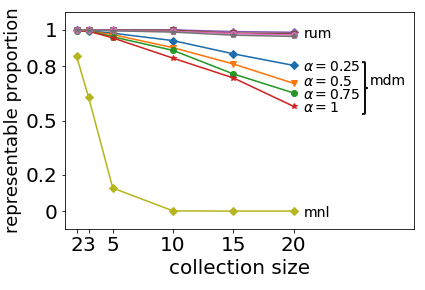}
\end{minipage}
\;
\begin{minipage}[t]{0.45\textwidth} 
\centerline{  Run Time }
\centering
\includegraphics[scale=0.45]{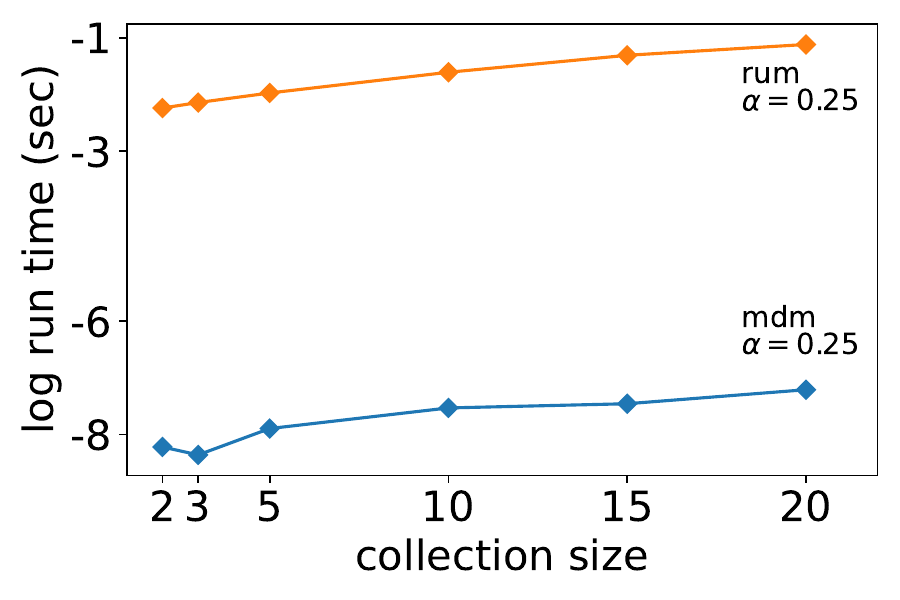}
\end{minipage}
\caption{Comparison of the performance of MDM and RUM } \label{fig:small_multiply}
\end{figure}
\subsection{Revenue and choice probability prediction with nonparametric MDM.}\label{sec:addtional_exp_results_prediction}
In Experiment 3, we generate 20 random instances with a product size of 7 and a collection size $\vert \cS \vert$ ranging among \{20, 40, 80\}, using nonidentical exponential distributions for the marginal distributions to generate the underlying choice probabilities. In Figure \ref{fig:prediction_withMEM},  we compare the predictions offered by the following two methods: (1) computing the nonparametric MDM lower and upper bounds of revenue and choice probabilities for each instance by solving $\underline{r}(A)$ and $\bar{r}(A)$; and (2) restricting the marginal distributions for MDM to be identical exponential distributions (which leads to the underlying choice model being MNL), and estimating the preference parameters using maximum likelihood estimation (MLE); we then using the estimated MNL model to predict revenue for unseen assortments in each instance. While Figure \ref{fig:prediction_withMEM} reveals the  proposed nonparametric approach to be correctly predicting the true revenue or choice probabilities, the MLE of the parametric approach with mis-specified marginals is often found to lead to inaccurate predictions which are far from the truth and also far out of the nonparametric MDM prediction intervals. When more assortments are offered, the prediction under nonparametric MDM becomes more accurate while the prediction results made under the incorrect parametric model become worse. Thus, besides revealing the benefits of the  proposed nonparametric data-driven approach for prediction based on MDM, Experiment 3 brings out the risks in stipulating apriori distributional assumptions on the model.

\begin{figure}[htb]
\centering
\begin{minipage}[t]{0.3\textwidth}
\centerline{collection size = 20}
\centerline{\includegraphics[scale=0.345]{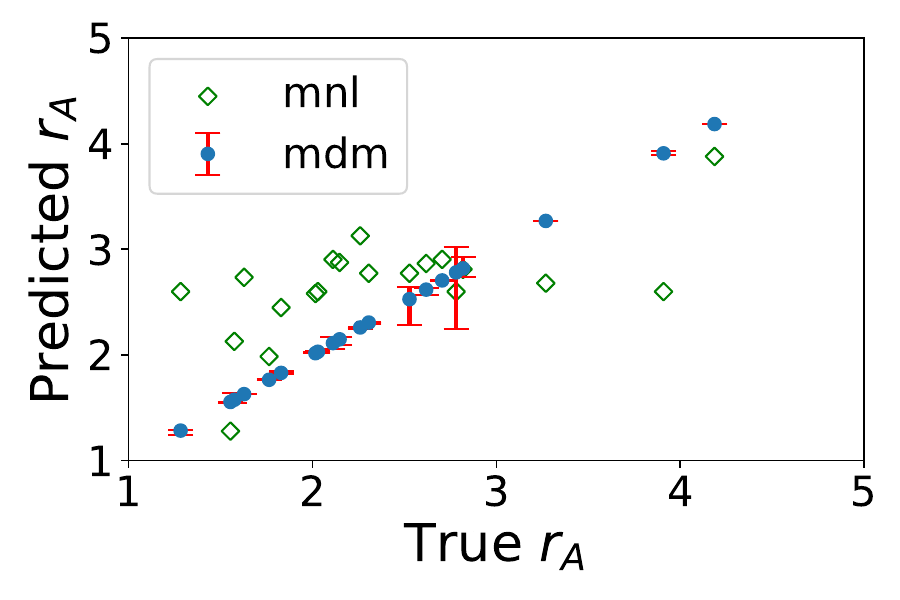}}
\end{minipage}
\;
\begin{minipage}[t]{0.3\textwidth}
\centerline{collection size = 40}
\centerline{\includegraphics[scale=0.345]{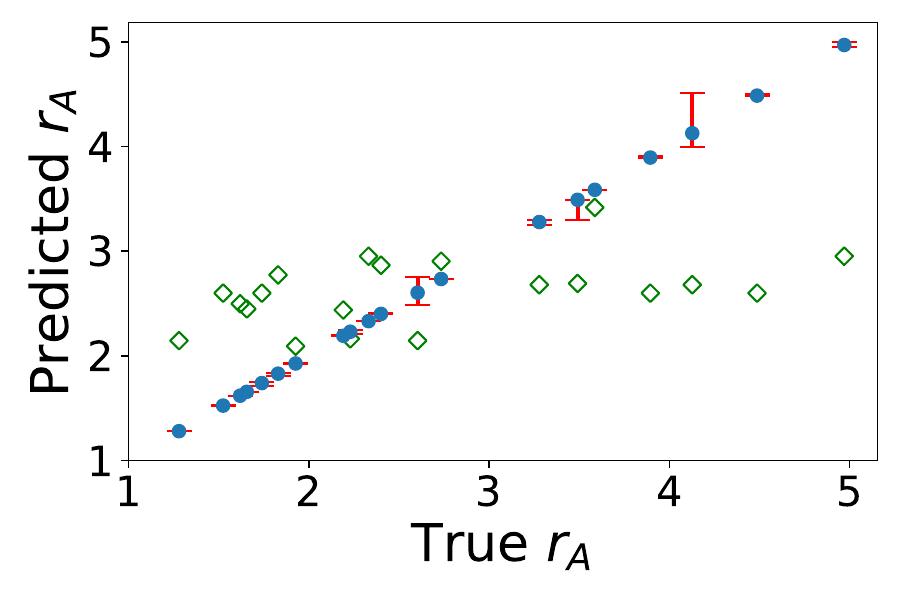}}
\end{minipage}
\;
\begin{minipage}[t]{0.3\textwidth}
\centerline{collection size = 80}
\centerline{\includegraphics[scale=0.345]{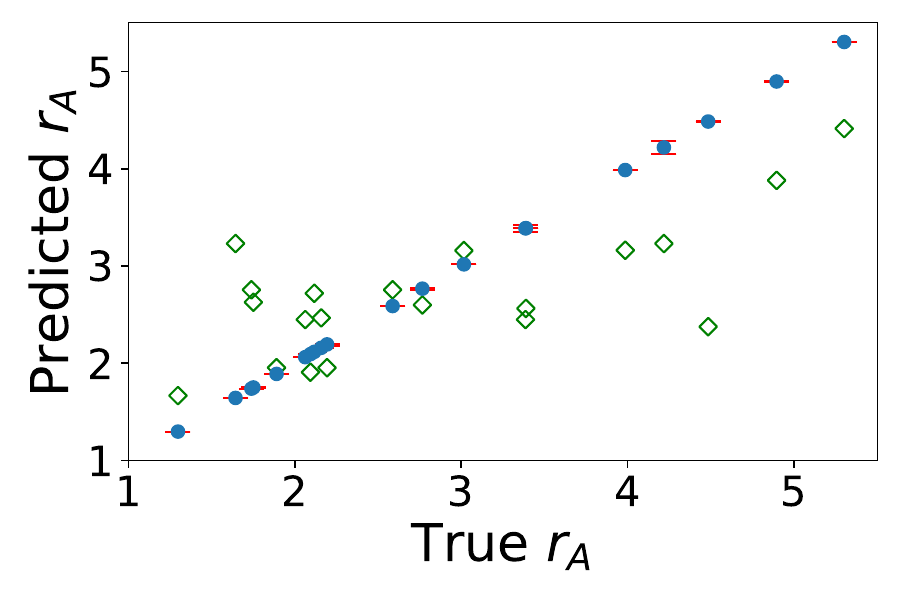}}
\end{minipage}
\;
\begin{minipage}[t]{0.3\textwidth}
\centerline{\includegraphics[scale=0.355]{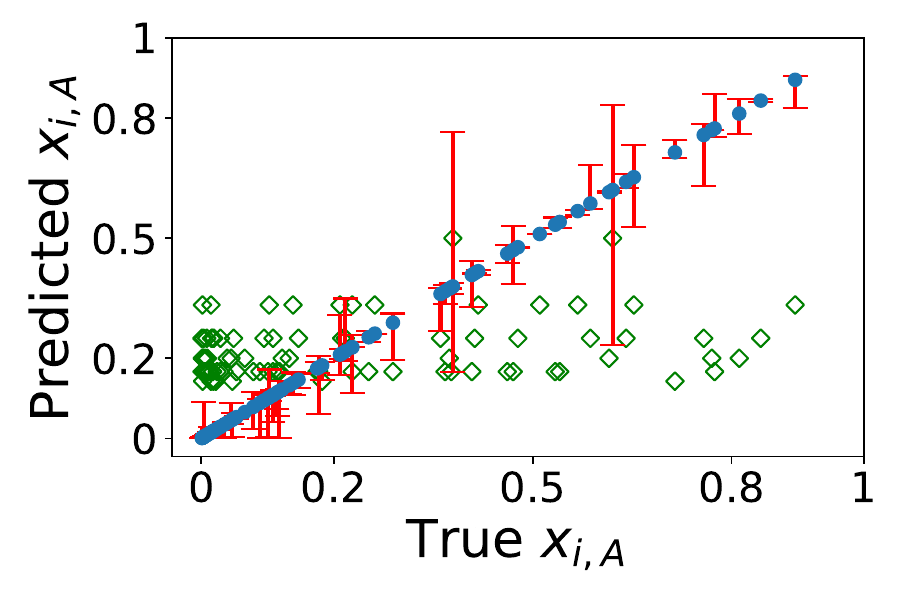}}
\end{minipage}
\;
\begin{minipage}[t]{0.3\textwidth} 
\centerline{\includegraphics[scale=0.345]{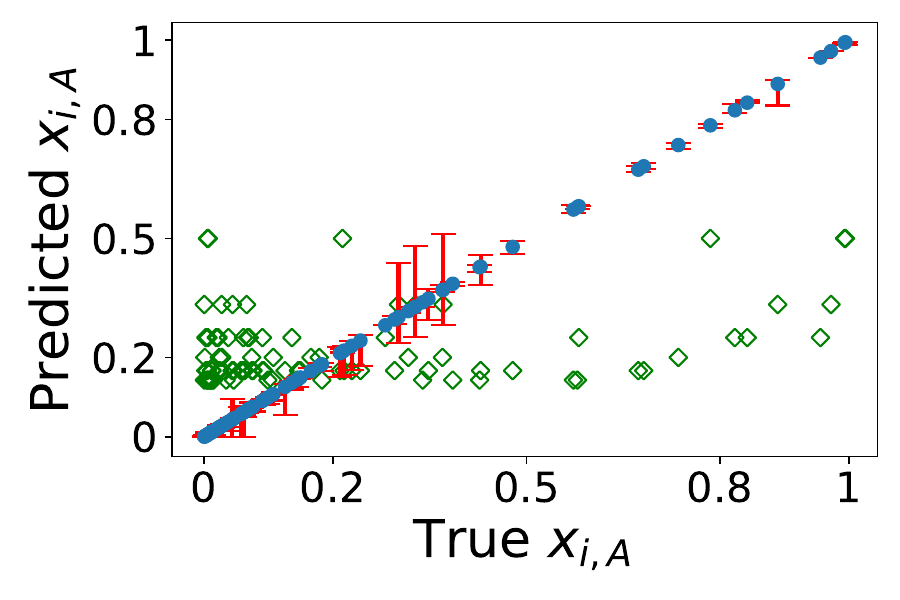}}
\end{minipage}
\;
\begin{minipage}[t]{0.3\textwidth} 
\centerline{\includegraphics[scale=0.345]{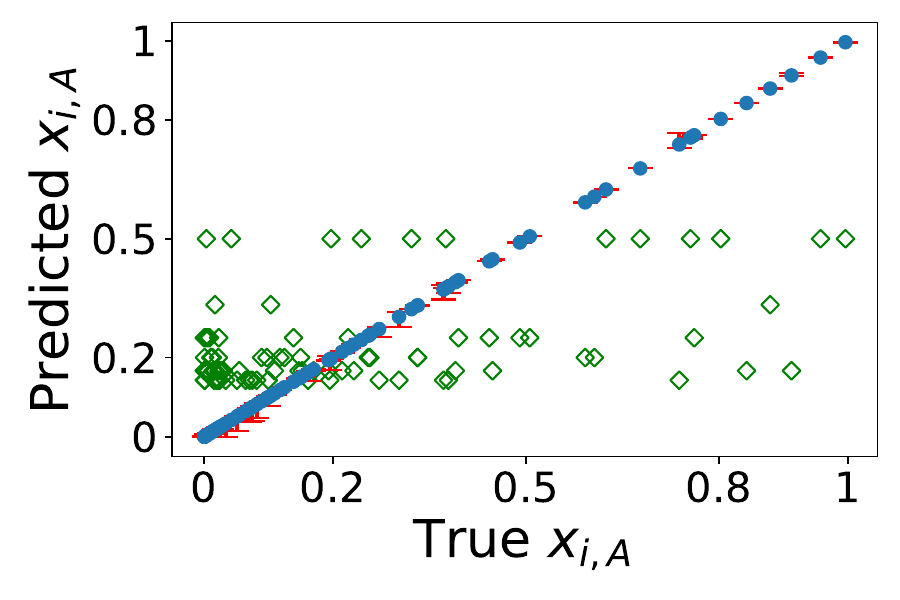}}
\end{minipage}
\begin{flushleft}\footnotesize Notes. In each figure, the blue dots represent the true revenues or choice probabilities, while the red ranges represent the predicted revenue intervals or choice probability intervals with the nonparametric MDM, and the green squares represent the predicted revenues or choice probabilities using the MLE of the MNL model. \end{flushleft}
\caption{Comparison of prediction accuracy between MDM and MNL with randomly generated instances}
\label{fig:prediction_withMEM}
\end{figure}
\subsection{Estimation performance of MDM compared to RUM and MNL.} \label{sec:addtional_exp_results_estimation}

In Experiment 4, we compare the explanatory ability of MDM, RUM and MNL models by examining the cumulative absolute deviation loss suffered in fitting them to uniformly generated choice data instances. Figure \ref{fig:limit} reveals that nonparametric MDM and RUM models are competitive and have much higher explanatory ability than MNL with increasing collection sizes. In particular,  MDM incurs about 44\% lesser loss, on average, than the best-fitting MNL model for the largest collection size considered.


\begin{figure}[htb!]
\centering
\includegraphics[scale=0.4]{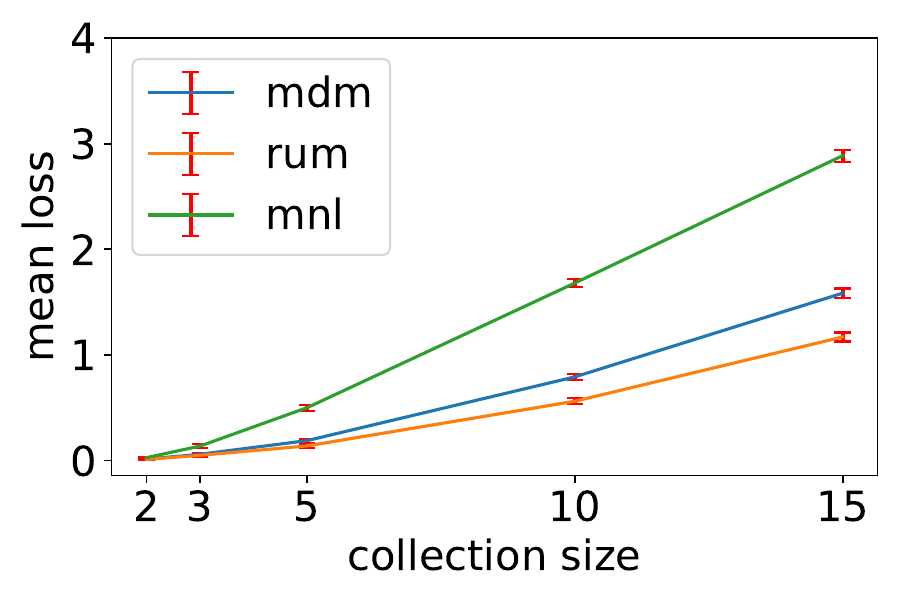}
\caption{The limit loss comparison among MDM, RUM, and MNL}
\label{fig:limit}
\end{figure}

{\color{black}\subsection{Impact on explanatory abilities of models with common alternatives in the assortments.} \label{sec:representability_outside}
We examine the explanatory power of the compared models by excluding the outside option in the assortments. We utilize the assortment data from the JD.com dataset, excluding both the outside option and any assortments containing only one product. The performance of models is evaluated by focusing on assortments offered at least $n_\cS$ times, with $n_\cS = 15$ resulting in  $|\cS| = 30$ and $n_\cS = 20$ resulting in $|\cS| = 21$. For each scenario, we randomly generate 50 underlying HEV and Probit instances. Since the exact characterization of MCCM and LC-MNL is not fully explored in the literature, we compare the representational power of each model by reporting the fraction of instances with an average KL loss below $10^{-6}$. Since the MLE formulation of MCCM has a nonconvex feasible region, there are convergence issues for some rare randomly generated instances when the formulation is directly solved with a continuous solver. Thus, we report the average KL loss for all models only for those instances that converged to an acceptable level.

Table \ref{tab:rep_kl_no_outsideopt_m=30_new} and Table \ref{tab:rep_kl_no_outsideopt_m=21_new} report the fraction of instances with average KL loss being less than $10^{-6}$ for each model, approximating the representable proportions of the testing instances for each model, and the average KL losses. The results show that MDM maintains good representability and results in a low average KL loss across both HEV and Probit, while LC-MNL with 10 classes achieves slightly better representability and a lower average KL loss in the HEV setting and MCCM performs better in the Probit setting. In contrast, MNL and LC-MNL with 2 classes exhibit much higher average KL losses than the other models. }

\begin{table}[htb!]
\centering
\caption{Representation and average KL loss comparison with $|\cS| = 30$} 
\label{tab:rep_kl_no_outsideopt_m=30_new}
\scalebox{0.9}{\begin{tabular}{c|cc|cc}
\hline
\multirow{2}{*}{Model}                                        & \multicolumn{2}{c|}{HEV}                                                                                                      & \multicolumn{2}{c}{Probit}                                                                                                   \\ \cline{2-5} 
                                                              & \multicolumn{1}{c|}{\begin{tabular}[c]{@{}c@{}}fraction of instances \\ with average KL \\ loss < $10^{-6}$\end{tabular}} & \multicolumn{1}{c|}{average KL loss}                                               & \multicolumn{1}{c|}{\begin{tabular}[c]{@{}c@{}}fraction of instances\\  with average KL \\ loss < $10^{-6}$\end{tabular}} & \multicolumn{1}{c}{average KL loss}                                            \\ \hline
MDM                                                           & \multicolumn{1}{c|}{0.66}                                                                                                     & \multicolumn{1}{c|}{$1.05 \times 10^{-5}$ ($3.3 \times 10^{-6}$)}                   & \multicolumn{1}{c|}{0.71}                                                                                                     & \multicolumn{1}{c}{$1.47 \times 10^{-6}$ ($3.6 \times 10^{-7}$)}                   \\ \hline
MNL                                                           & \multicolumn{1}{c|}{0}                                                                                                        & \multicolumn{1}{c|}{$1.07 \times 10^{-2}$ ($1.0 \times 10^{-3}$)}                   & \multicolumn{1}{c|}{0}                                                                                                        & \multicolumn{1}{c}{$6.84 \times 10^{-4}$ ($5.8 \times 10^{-5}$)}                   \\ \hline
MCCM                                                          & \multicolumn{1}{c|}{0.14}                                                                                                     & \multicolumn{1}{c|}{$3.60 \times 10^{-4}$ ($1.62 \times 10^{-4}$)}                   & \multicolumn{1}{c|}{0.90}                                                                                                     & \multicolumn{1}{c}{$3.65 \times 10^{-7}$ ($2.36 \times 10^{-7}$)}                   \\ \hline
\begin{tabular}[c]{@{}c@{}}LC-MNL\\ (2 classes)\end{tabular}  & \multicolumn{1}{c|}{0}                                                                                                        & \multicolumn{1}{c|}{$2.16 \times 10^{-3}$ ($2.6 \times 10^{-4}$)}                   & \multicolumn{1}{c|}{0}                                                                                                        & \multicolumn{1}{c}{$2.82 \times 10^{-4}$ ($2.0 \times 10^{-5}$)}                   \\ \hline
\begin{tabular}[c]{@{}c@{}}LC-MNL\\ (10 classes)\end{tabular} & \multicolumn{1}{c|}{0.86}                                                                                                     & \multicolumn{1}{c|}{$2.64 \times 10^{-6}$ ($1.92 \times 10^{-6}$)}                   & \multicolumn{1}{c|}{0.73}                                                                                                     & \multicolumn{1}{c}{$1.05 \times 10^{-4}$ ($2.9 \times 10^{-5}$)}                   \\ \hline
\end{tabular}}
\vspace{0.15cm}
\begin{flushleft}\footnotesize Notes. Standard errors are reported in parentheses. \end{flushleft}
\end{table}

\begin{table}[htb!]
\centering
\caption{Representation and average KL loss comparison with $|\cS| = 21$} 
\label{tab:rep_kl_no_outsideopt_m=21_new}
\scalebox{0.9}{\begin{tabular}{c|cc|cc}
\hline
\multirow{2}{*}{Model}                                        & \multicolumn{2}{c|}{HEV}                                                                                                      & \multicolumn{2}{c}{Probit}                                                                                                   \\ \cline{2-5} 
                                                              & \multicolumn{1}{c|}{\begin{tabular}[c]{@{}c@{}}fraction of instances \\ with average KL \\ loss < $10^{-6}$\end{tabular}} & \multicolumn{1}{c|}{average KL loss}                                               & \multicolumn{1}{c|}{\begin{tabular}[c]{@{}c@{}}fraction of instances\\  with average KL \\ loss < $10^{-6}$\end{tabular}} & \multicolumn{1}{c}{average KL loss}                                            \\ \hline
MDM                                                           & \multicolumn{1}{c|}{0.98}                                                                                                     & \multicolumn{1}{c|}{$1.62 \times 10^{-7}$ ($1.61 \times 10^{-7}$)}                   & \multicolumn{1}{c|}{0.98}                                                                                                     & \multicolumn{1}{c}{$2.88 \times 10^{-7}$ ($2.82 \times 10^{-7}$)}                   \\ \hline
MNL                                                           & \multicolumn{1}{c|}{0}                                                                                                     & \multicolumn{1}{c|}{$7.18 \times 10^{-3}$ ($8.3 \times 10^{-4}$)}                   & \multicolumn{1}{c|}{0}                                                                                                     & \multicolumn{1}{c}{$5.03 \times 10^{-5}$ ($2.24 \times 10^{-5}$)}                   \\ \hline
MCCM                                                          & \multicolumn{1}{c|}{0.56}                                                                                                     & \multicolumn{1}{c|}{$1.05 \times 10^{-4}$ ($3.6 \times 10^{-5}$)}                   & \multicolumn{1}{c|}{0.88}                                                                                                     & \multicolumn{1}{c}{$3.24 \times 10^{-7}$ ($3.20 \times 10^{-7}$)}                   \\ \hline
\begin{tabular}[c]{@{}c@{}}LC-MNL\\ (2 classes)\end{tabular}  & \multicolumn{1}{c|}{0}                                                                                                     & \multicolumn{1}{c|}{$1.22 \times 10^{-3}$ ($1.9 \times 10^{-4}$)}                   & \multicolumn{1}{c|}{0}                                                                                                     & \multicolumn{1}{c}{$1.55 \times 10^{-4}$ ($1.41 \times 10^{-4}$)}                   \\ \hline
\begin{tabular}[c]{@{}c@{}}LC-MNL\\ (10 classes)\end{tabular} & \multicolumn{1}{c|}{0.98}                                                                                                     & \multicolumn{1}{c|}{$5.01 \times 10^{-6}$ ($4.95 \times 10^{-6}$)}                   & \multicolumn{1}{c|}{0.42}                                                                                                     & \multicolumn{1}{c}{$1.08 \times 10^{-4}$ ($1.79 \times 10^{-4}$)}                   \\ \hline
\end{tabular}}
\vspace{0.15cm}
\begin{flushleft}\footnotesize Notes. Standard errors are reported in parentheses. \end{flushleft}
\end{table}

\section{Additional Experiment Results with Real-World Data}
\label{sec:addtional_exp_results_real}
In Experiments 5-7 below, we provide additional experiment results by using the dataset from JD.com (introduced in \citealt{shen2020jd}) to evaluate the feasibility of representing it with an MDM, the efficacy of predictions obtained by the proposed nonparametric approach, and the explanatory ability captured by the limit formulations. In Experiments 5-7, we ignore the constraints of the models on the outside option. 

\subsubsection*{Representational power comparison among several models.}
Experiment 5 compares the representation power of MDM with the MNL model and the class of regular choice models. The tested instances feature assortments which are offered at least $n_S$ times, with $n_S$ values ranging from 60 to 100. If we include data on the outside option, none of the models considered are found to exactly represent the data even when $n_S=100$. By focusing on the sales data of the products offered by the firm, Table \ref{tab:jd_representability} shows that the nonparametric MDM and regular choice models are able to represent the choice data obtained from $n_S = 75$ and $100,$ whereas MNL models fail to represent any of the instances. We could not report the results for RUM here because of its intractability. 

\begin{table}[htb!]
\centering
\caption{The representability of MNL, MDM, and the class of regular choice models}
\label{tab:jd_representability}
\begin{tabular}{ccccc}
\hline
 $n_S \geq $ &  $|\cS|$  & MNL & MDM  & Regular Model \\ \hline
60   & 13   & 0   & 0     & 0          \\
75   & 12   & 0   & 1     & 1          \\
100  & 11   & 0   & 1      & 1          \\ \hline
\end{tabular}
\vspace{0.1cm}
\begin{flushleft}\footnotesize Notes. $1$ denotes an instance that can be represented by the tested model, while $0$ denotes the opposite.  \end{flushleft}
\end{table}
\vspace{-0.2cm}
\subsubsection*{Estimation performance comparison between MDM and MNL.} 
In Experiment 6, we compare the explanatory ability of the MDM and MNL model by computing the limit loss over choice data obtained by considering assortments that are offered at least $n_S$ times, where $n_S$ is set to vary from $1$ to $100$.  Using 1-norm as the loss function, the results in Table \ref{tab:jd-estimation_add} show that nonparametric MDM suffers much lesser cost in approximating the choice data, and hence greater explanatory ability. We also observe that the run time of solving the limit of MDM grows when the size of the assortment collection becomes larger.
\begin{table}[htb!]
\centering
\caption{Comparison of the estimation performance of MDM and MNL model}\label{tab:jd-estimation_add}
\begin{tabular}{cc|cc|cc}
\hline
\multirow{2}{*}{\begin{tabular}[c]{@{}c@{}} $n_S$\\ $\geq$  \end{tabular}} & \multirow{2}{*}{$\vert \cS \vert$ }  & \multicolumn{2}{c|}{MDM} & \multicolumn{2}{c}{MNL} \\ \cline{3-6} 
                      &                       &  loss    & time (sec)      &  loss     & time (sec)     \\ \hline
1                     & 134                   &   0.223          & 3600      & 0.19         & 0.839    \\
10                    & 42                    & 0.027      & 3600      & 0.16         & 0.253    \\
20                    & 29                    & 0.019      & 13.654    & 0.15         & 0.193    \\
30                    & 24                    & 0.017      & 6.298     & 0.15         & 0.318    \\
40                    & 19                    & 0.016      & 2.379     & 0.14         & 0.149    \\
50                    & 15                    & 0.012      & 0.237     & 0.14         & 0.146    \\
60                    & 13                    & 0.011      & 0.094     & 0.13         & 0.114    \\
75                    & 12                    & 0.0097     & 0.060     & 0.13         & 0.120    \\
100                   & 11                    & 0.0098     & 0.046     & 0.13         & 0.112    \\ \hline
\end{tabular}
\end{table}
We further assess the accuracy of the nonparametric MDM and MNL models by comparing the observed (true) and estimated choice probabilities via scatter plots. 
Figure \ref{fig:scatters} shows these scatter plots, where each point represents an observed and estimated choice probability pair. The horizontal axis shows the observed probability and the vertical axis shows the estimated probability. The closer the points are to the 45-degree line segment (green segment in Figure \ref{fig:scatters}), the better the estimation accuracy. The scatter plots reveal the following findings:
\begin{enumerate}[leftmargin=*]
\item[(i)] When $n_S = \{50, 60, 75, 100\}$, MDM is seen to correctly estimate most data points  due to its proximity to the $45^o$ line while most points from MNL estimation are still away from the $45^o$ line. 
\item[(ii)] When $n_S = \{10, 20, 30, 40\}$, although both MDM and MNL models are limited in their abilities to exactly represent the choice data, MDM shows much better estimation accuracy than MNL model with most points by MDM being much closer to the $45^o$ line than the MNL model. 
\item[(iii)] In the noisy environment where many assortments are just shown once (corresponding to $n_S = 1$), both MDM and MNL  fail understandably with most points falling away from the $45^o$ line.
\end{enumerate}


\begin{figure}[htb]
\centering
\begin{minipage}[t]{0.3\textwidth} 
\centerline{$n_S = 100$}
\centerline{\includegraphics[scale=0.33]{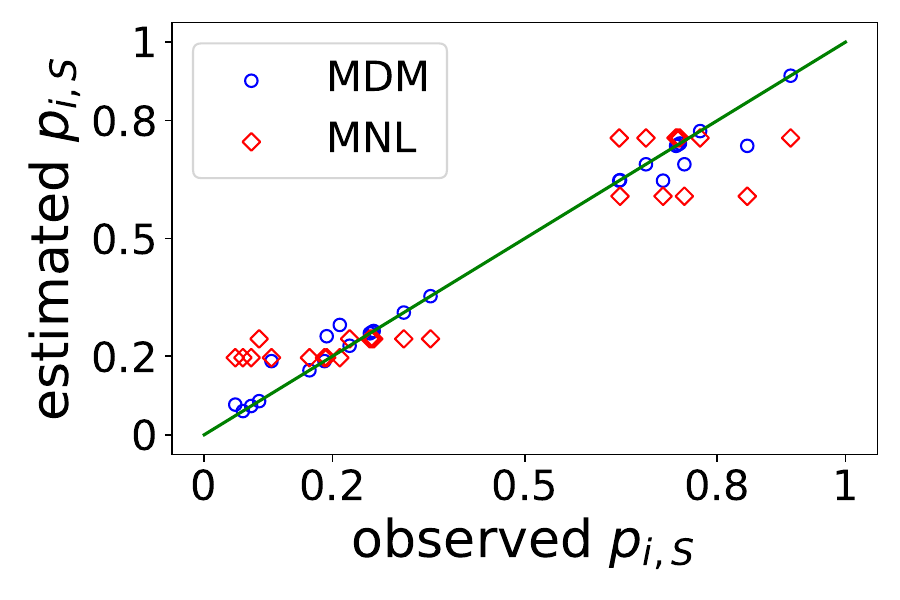}}
\end{minipage}
\;
\begin{minipage}[t]{0.3\textwidth} 
\centerline{$n_S = 75$}
\centerline{\includegraphics[scale=0.33]{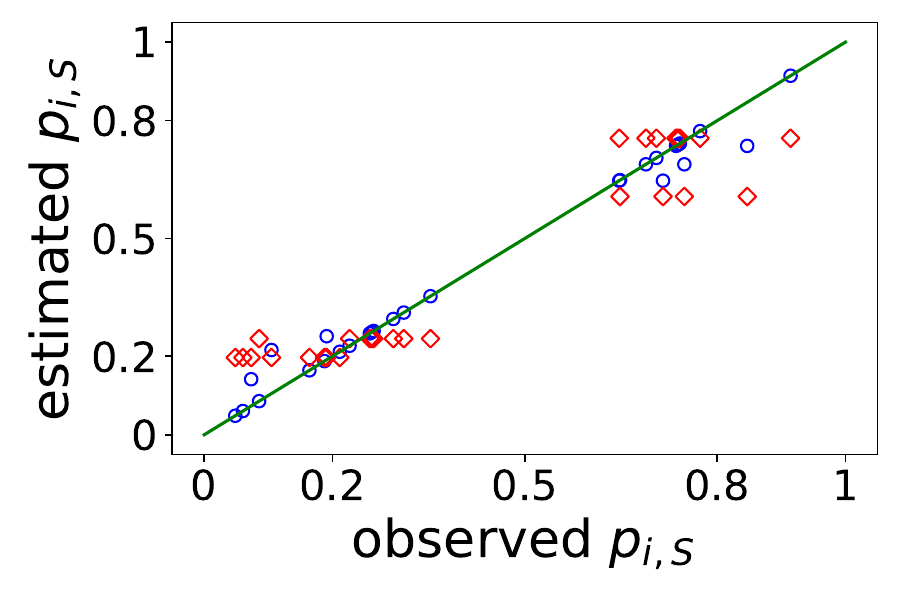}}
\end{minipage}
\;
\begin{minipage}[t]{0.3\textwidth} 
\centerline{$n_S = 60$}
\centerline{\includegraphics[scale=0.33]{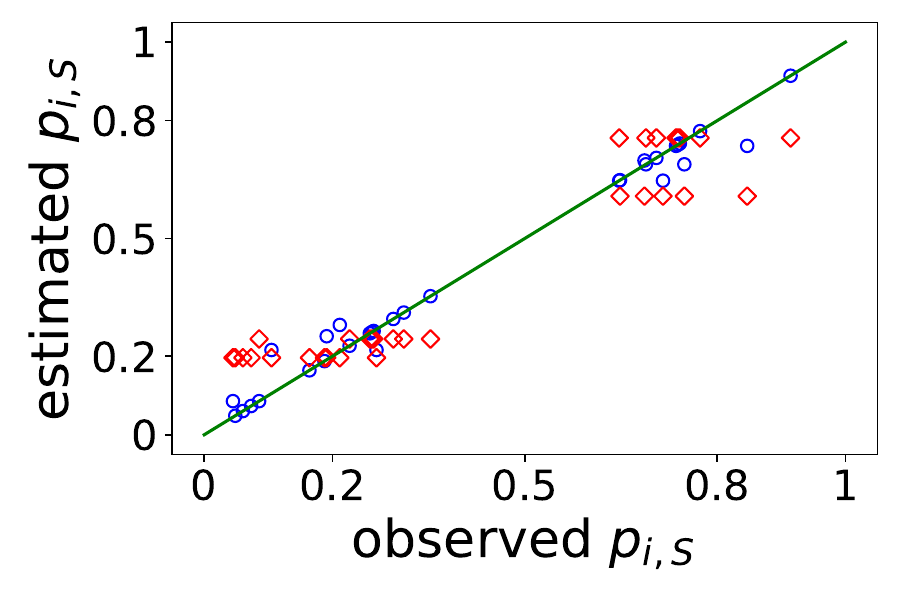}}
\end{minipage}
\;
\begin{minipage}[t]{0.3\textwidth} 
\centerline{$n_S = 50$}
\centerline{\includegraphics[scale=0.33]{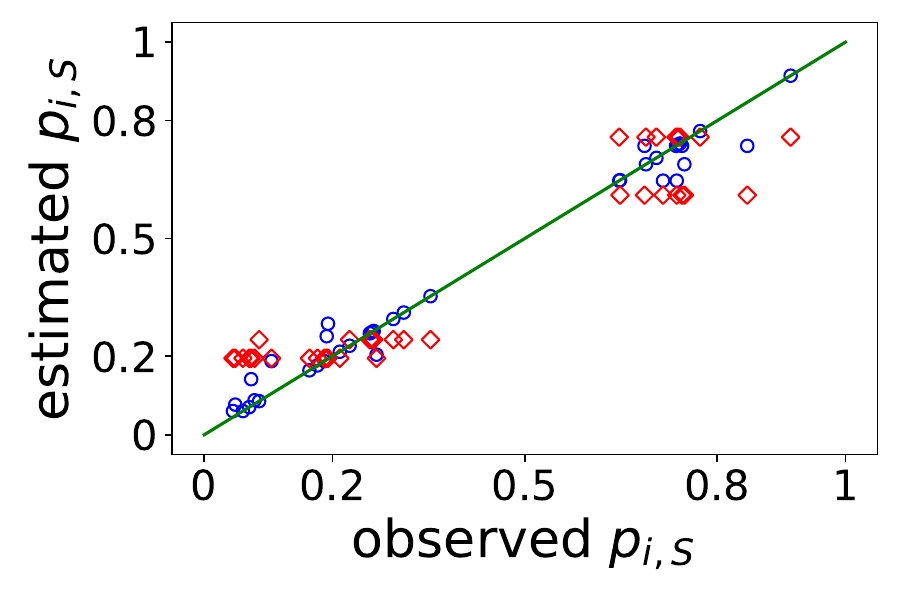}}
\end{minipage}
\;
\begin{minipage}[t]{0.3\textwidth} 
\centerline{$n_S = 40$}
\centerline{\includegraphics[scale=0.33]{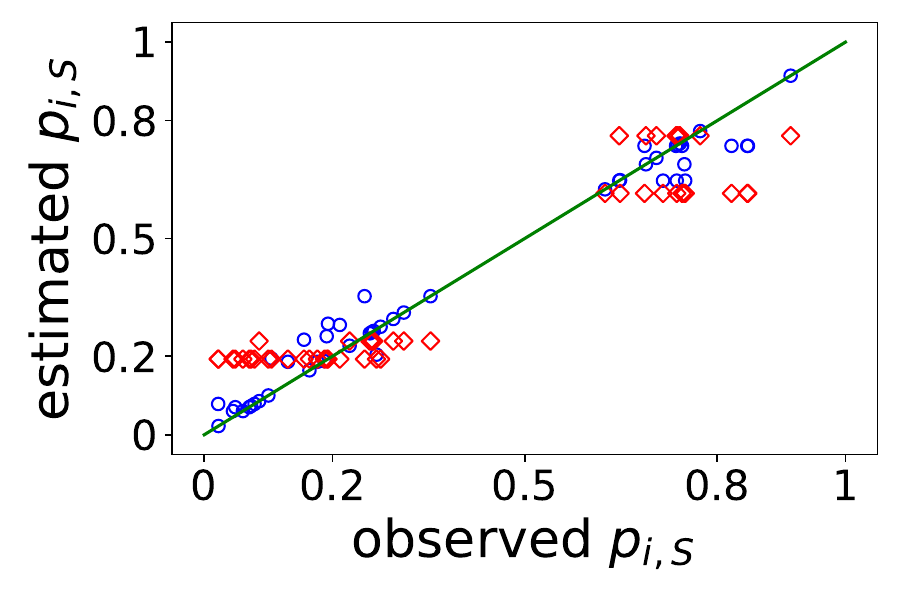}}
\end{minipage}
\;
\begin{minipage}[t]{0.3\textwidth} 
\centerline{$n_S = 30$}
\centerline{\includegraphics[scale=0.33]{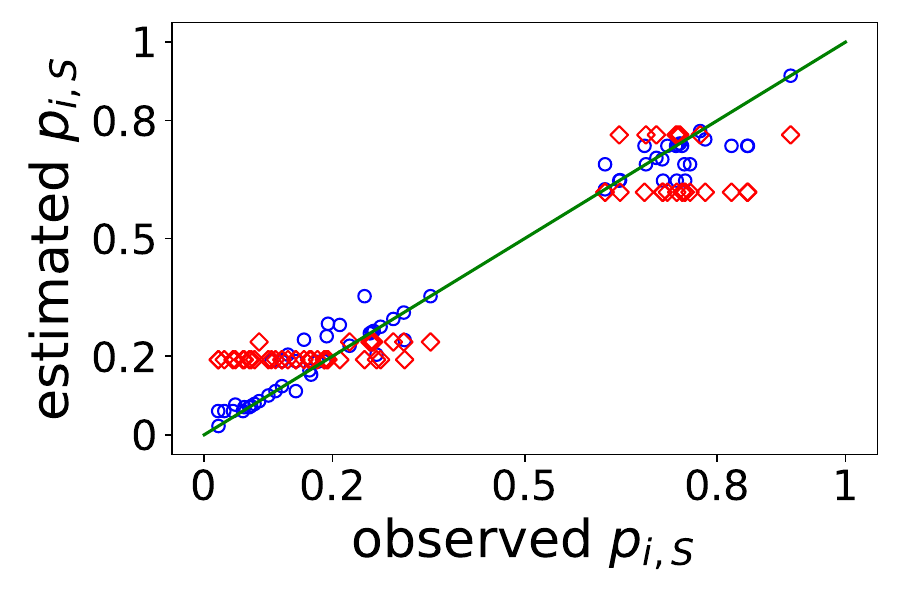}}
\end{minipage}
\;
\begin{minipage}[t]{0.3\textwidth} 
\centerline{$n_S = 20$}
\centerline{\includegraphics[scale=0.335]{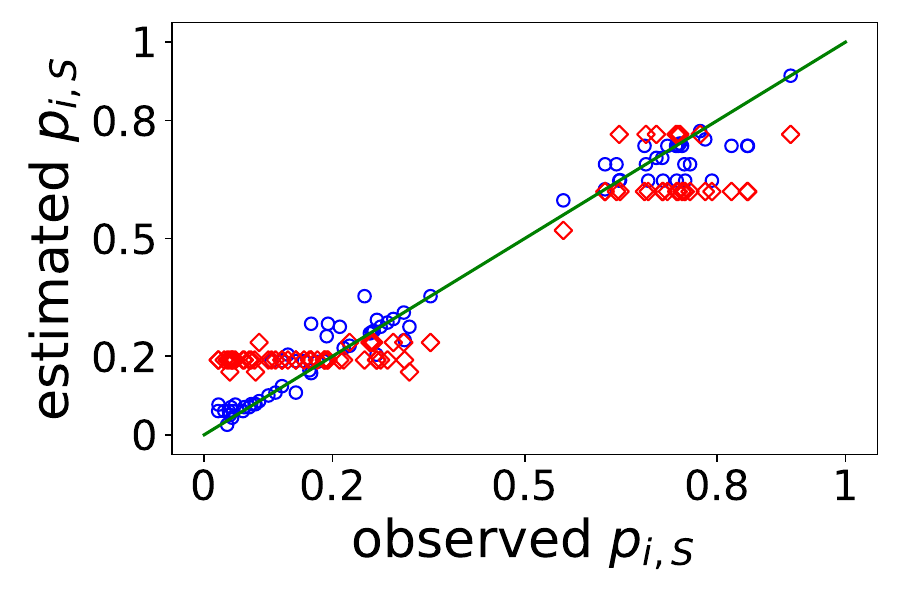}}
\end{minipage}
\;
\begin{minipage}[t]{0.3\textwidth} 
\centerline{$n_S = 10$}
\centerline{\includegraphics[scale=0.335]{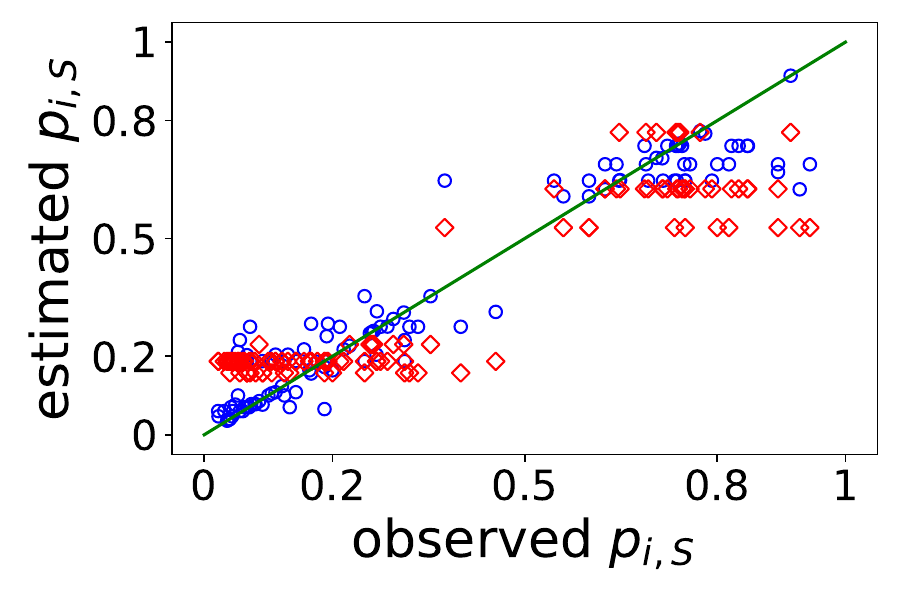}}
\end{minipage}
\;
\begin{minipage}[t]{0.3\textwidth} 
\centerline{$n_S =1$}
\centerline{\includegraphics[scale=0.335]{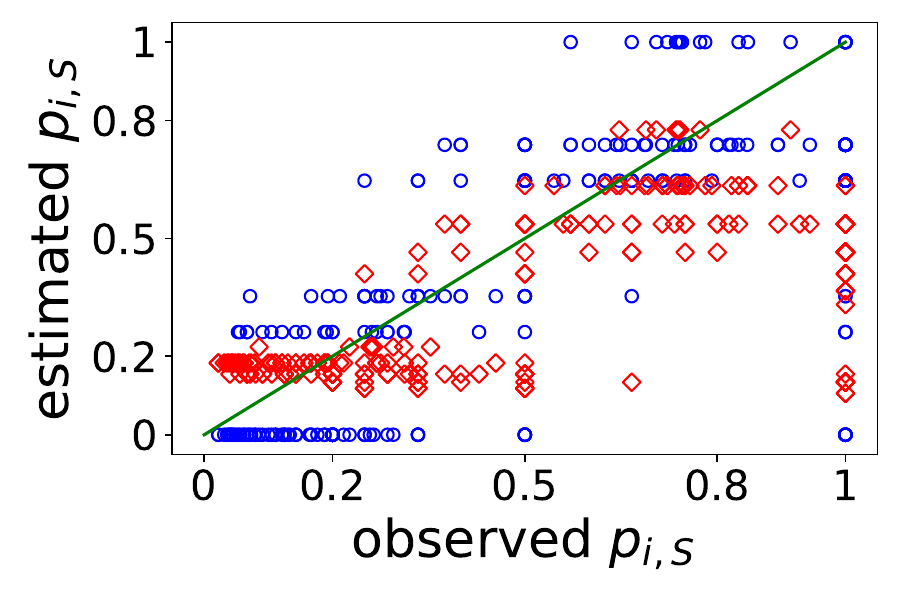}}
\end{minipage}
\vspace{0.1cm}
\begin{flushleft}\footnotesize Notes. In each plot, each point corresponds to the coordinate (true choice probability, estimated choice probability) of each observation of the processed data with $n_S$. The green line is the $45^o$ line. The blue dots represent the estimation results with the nonparametric MDM and the red squares represent the estimation results with the MNL model. \end{flushleft}
\caption{Scatter plots to compare the estimation accuracy of MDM and MNL}
\label{fig:scatters}
\end{figure}

\subsubsection*{Prediction performance comparison between MDM and MNL.}
Experiment 7 evaluates the predictive-cum-prescriptive abilities of the nonparametric MDM and the MNL model by comparing their accuracies in identifying (i) a ranking over unseen test assortments based on their expected revenues, and (ii) the average revenue of the assortment identified to offer the largest revenue in the test set. Considering assortments that are shown at least $n_S$ times (with $n_S$ taken to vary from 20 to 50), we report the average out-of-sample performance over instances generated by randomly picking $\max\{ 2, \lfloor 0.2|\cS| \rfloor \}$ fraction of the assortments to be the test set and the remaining to be the training set. For MNL, we use the Maximum Likelihood Estimator (MLE) obtained from training data to estimate choice probabilities for the test assortments and use them subsequently to rank the test assortments in a decreasing order of expected revenues. For MDM, we compute the robust revenue $\underline{r}(A)$ and the optimistic revenue $\bar{r}(A)$ and record the corresponding choice probabilities for each tested assortment $A$ in the test set if the training data can be represented by MDM. If we find the training data to be not exactly representable by MDM, we solve the limit of MDM (Problem \eqref{model:mdm_micp}) with the training data and use the choice probabilities yielded by solving  \eqref{mdm:perturb} to proceed as before with ranking the assortments in a decreasing order of expected revenues.

For comparing the quality of rankings offered by the MDM and the MNL model, we take the well-known Kendall Tau distance (see Definition \ref{dfn:k-dis}) as a natural metric for evaluating the closeness of the predicted ranking with the ground truth hidden from training. For both models, we also compare the true revenues of the assortments which are predicted to rank at the top. The average of these out-of-sample performance metrics across randomly generated train-test splits are reported in Table \ref{tab:out-of-sample-prediction}.
The results in Table \ref{tab:out-of-sample-prediction} show that the optimistic prediction results of nonparametric MDM outperform the MNL model, yielding uniformly lower average Kendall tau distances and higher average revenues for the predicted best assortments across all scenarios. Similarly, the robust prediction results of nonparametric MDM outperform the MNL models in most scenarios, except for instances with $n_S=30$ in terms of average Kendall tau distances and instances with $n_S=50$ in terms of average revenue predictions. Figure \ref{fig:predition_performance_realdata} illustrates that the nonparametric MDM approach predicts the true revenue more accurately than MNL, as the predicted intervals include the true revenue or are closer to it. 
\begin{table}[H]

\centering
\caption{The prediction performance of MDM and MNL}
\label{tab:out-of-sample-prediction}
\begin{tabular}{ccccccccc}
\hline
\multirow{2}{*}{\begin{tabular}[c]{@{}c@{}} $n_S$\\ $\geq$  \end{tabular}} & \multirow{2}{*}{$\vert \cS \vert$ }  & \multirow{2}{*}{\begin{tabular}[c]{@{}c@{}}\#test \\ assortments\end{tabular}} & \multicolumn{3}{c}{Average Kendall Tau Distance} & \multicolumn{3}{c}{\begin{tabular}[c]{@{}c@{}}Average Revenue of the\\  Predicted Best Assortments\end{tabular}} \\ \cline{4-9} 
                      &                       &                                                                                  & MNL           & MDM\_LB         & MDM\_UB        & MNL                                & MDM\_LB                              & MDM\_UB                              \\ \hline
20                    & 23                    & 6                                                                               & 4.9         & 3.1           & 3.9            & 0.385                              & 0.416                                & 0.422                                \\ \hline
30                    & 20                    & 4                                                                               & 3.0           & 3.1             & 2.5            & 0.389                              & 0.420                                & 0.403                                \\ \hline
40                    & 15                    & 4                                                                               & 1.8           & 1.2             & 1.1            & 0.225                              & 0.271                                & 0.275                                \\ \hline
50                    & 12                    & 3                                                                               & 1.2           & 0.9             & 0.6            & 0.257                              & 0.249                                & 0.271                                \\ \hline
\end{tabular}
\vspace{5pt}
\begin{flushleft}
\footnotesize Notes. MDM\_LB represents the results by solving $\underline{r}(A)$ while MDM\_UB represents the results by solving $\bar{r}(A)$.
\end{flushleft}
\end{table}

\begin{figure}[H]
\centering
\begin{minipage}[t]{0.32\textwidth}
\centering
\centerline{$n_S=20$}
\includegraphics[scale=0.35]{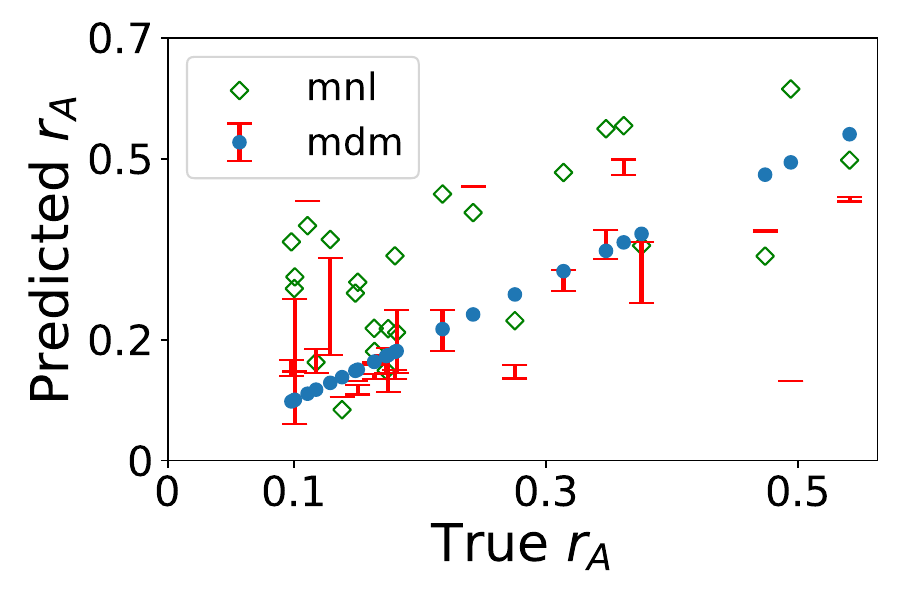}
\end{minipage}
\begin{minipage}[t]{0.32\textwidth} 
\centering
\centerline{$n_S=40$}
\includegraphics[scale=0.35]{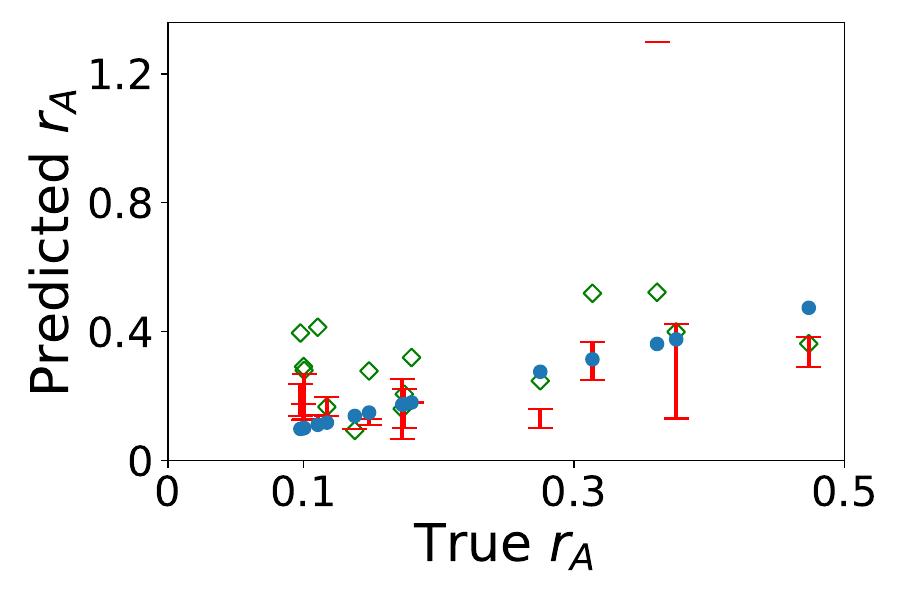}
\end{minipage}
\begin{minipage}[t]{0.32\textwidth} 
\centering
\centerline{$n_S=50$}
\includegraphics[scale=0.35]{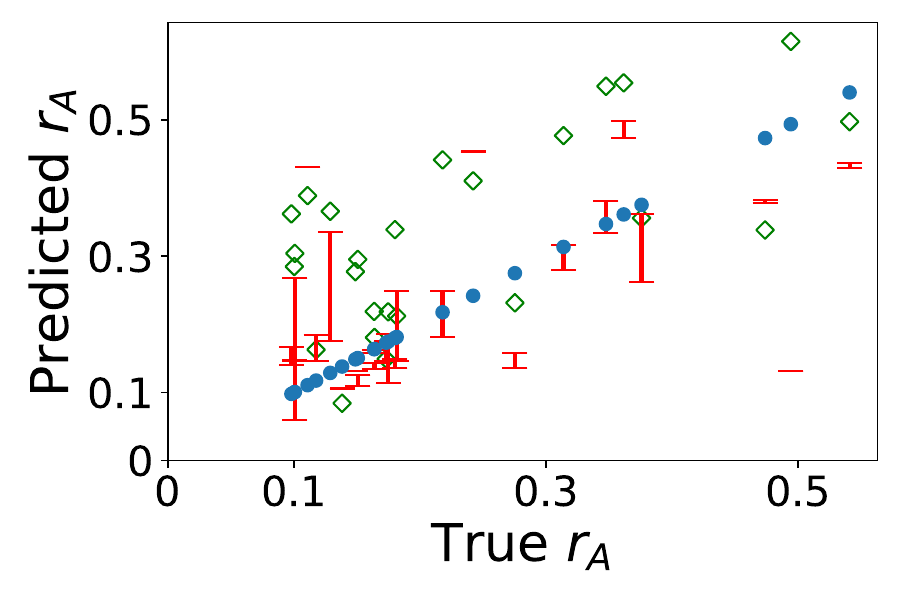}
\end{minipage}
\begin{flushleft}\footnotesize Notes. In each figure, the blue dots represent the true revenues, 
while the red ranges represent the predicted revenue intervals 
with the nonparametric MDM, \& the green squares represent the predicted revenues 
using the MNL model. 
\end{flushleft}
\caption{Revenue 
predictions vs. true revenue and 
for the nonparametric MDM and the MNL model}
\label{fig:predition_performance_realdata}
\end{figure}

\end{document}